\documentclass[preprint]{elsarticle}
\usepackage{times}
\usepackage{url}
\usepackage{hyperref}
\journal{\ldots}

\usepackage{latexsym}
\usepackage{amsmath}
\usepackage{amsthm}
\usepackage{amssymb}

\usepackage{cancel}
\usepackage{float}
\usepackage{extarrows} 

\usepackage{pgf}
\usepackage{tikz}
\usetikzlibrary{automata}
\usetikzlibrary{trees}

\usepackage{booktabs} 
\usepackage{hhline} 
\usepackage{graphicx} 
\usepackage{subfigure}
\usepackage{pgfplots}

\usepackage{longtable}
\usepackage{threeparttable}  
\usepackage{multirow}
\usepackage{lineno}

\usepackage{defs}    

\newtheorem{convention}{Convention}
\newtheorem{notation}{Notation}
\newtheorem{remark}{Remark}
\newtheorem{definition}{Definition}
\newtheorem{example}{Example}
\newtheorem{observation}{Observation}
\newtheorem{lemma}{Lemma}
\newtheorem{proposition}{Proposition}
\newtheorem{theorem}{Theorem}
\newtheorem{corollary}{Corollary}
\begin{document}
\begin{frontmatter}
\title{
	More on extension-based semantics of argumentation
}
\author[nuaa]{Lixing Tan}
\author[nuaa]{Zhaohui Zhu\corref{cor}}
\author[nau]{Jinjin Zhang}
\cortext[cor]{Corresponding author. E-mail addresses: zhaohui@nuaa.edu.cn (Zhaohui Zhu)}
\address[nuaa]{College of Computer Science and Technology, Nanjing University of Aeronautics and Astronautics}
\address[nau]{College of Information Science, Nanjing Audit University}

\begin{abstract}
	After a few decades of development, computational argumentation has become one of the active realms in AI. This paper considers extension-based concrete and abstract semantics of argumentation. For concrete ones, based on Grossi and Modgil's recent work, this paper considers some issues on graded extension-based semantics of abstract argumentation framework (AAF, for short). First, an alternative fundamental lemma is given, which generalizes the corresponding result due to Grossi and Modgil by relaxing the constraint on parameters. This lemma provides a new sufficient condition for preserving conflict-freeness and brings a Galois adjunction between admissible sets and complete extensions, which is of vital importance in constructing some special extensions in terms of iterations of the defense function. Applying such a lemma, some flaws in Grossi and Modgil's work are corrected, and the structural property and universal definability of various extension-based semantics are given. Second, an operator so-called reduced meet modulo an ultrafilter is presented, which is a simple but powerful tool in exploring infinite AAFs. The neutrality function and the defense function, which play central roles in Dung's abstract argumentation theory, are shown to be distributive over reduced meets modulo any ultrafilter. A variety of fundamental semantics of AAFs, including \textit{conflict-free, admissible, complete and stable semantics}, etc, are shown to be closed under this operator. Based on this fact, a number of applications of such operators are considered. In particular, we provide a simple and uniform method to prove the universal definability of a family of range related semantics. Since all graded concrete semantics considered in this paper are generalizations of corresponding non-graded ones, all results about them obtained in this paper also hold in the traditional situation. Apart from concrete semantics, abstract ones, defined by a first order language $\textit{FS}(\Omega)$, are also considered in this paper. We establish a connection between ultraproducts of models in model theory and reduced meets of extensions modulo ultrafilters, and characterize the extension-based semantics that is closed under reduced meets modulo any ultrafilter in terms of $\textit{FS}(\Omega)$-definability, which provides an example of utilizing model theory to explore the theoretical problems of AAFs. It brings metatheorems concerning the universal definability of abstract range related semantics, the Lindenbaum property, $\varepsilon$-extensible and $\varepsilon$-inference associated with extension-based abstract semantics $\varepsilon$, etc. 
\end{abstract}

\end{frontmatter}

\section{Introduction}\label{Sec:Introduction}
Since Dung proposed his seminal theory based on the notion of an AAF \cite{Dung95Acceptability}, after two decades of development, computational argumentation has become one of the most active realms in artificial intelligence and multi-agent systems \cite{Capon07Introduction, RahSim09Book, Besnard14Introduction, Atkinson17Introduction, Gaggl20Introduction}.
To formalize various reasoning phenomenons in different application contexts, Dung's framework has been developed in many ways. First, the structure of AAFs has been enriched by introducing additional structures on arguments, e.g., preference relations \cite{Amgoud02aPreAF, Amgoud02bPreAF, Modgil09PreAF}, probability \cite{Hunter2017ProAF} and support relations \cite{Cayrol13BipolarAF}  (in \cite{Amgoud16BipolarAF}, support relations are adopted as the starting point instead of attack relations). Second, in order to capture ``\textit{graduality}", an inherent feature that exists in several applications, gradual semantics of AAFs have been explored, which may be traced back to the work of Pollock in \cite{Pollock01Graduality} for indicating the degrees of justification of beliefs. Traditional semantics of AAFs provide simple ways to evaluate the acceptability of arguments, while gradual ones provide a more fine grained assignment of status to arguments based on numerical scales or rankings, which has gained increasing attention in recent years \cite{Bonzon2016Graduality,Baroni19Graduality} and may be roughly divided into two types: ones based on extensions and ones based on rankings.
\par 
The former focuses on extension-based semantics and intends to provide a finer-grained evolution of the acceptability of arguments. Various mechanisms have been proposed to realize this. In \cite{YiningWu10Graduality}, all possible values of arguments, assigned by so-called complete labellings, are used to distinguish their levels of acceptability. In weighted argumentation systems \cite{Dunne11weightedAF}, numeric weights are assigned to every attack in the argument graph, and then ``inconsistency budgets'' are introduced to define the weighted variants of Dung's standard semantics. In \cite{Gabbay16Graduality}, the notion of a U-approach solution is adopted to measure the extent to which an argument is accepted, rejected or undecided, moreover, the relationship between these solutions and traditional extensions is explored in \cite{Gabbay15Graduality}. The graded argumentation system provided in \cite{Grossi19Graded,Grossi15Graded} is a faithful generalization of Dung's system on the aspect of semantics. In this framework, graded defense and neutrality functions are defined in terms of the numbers of attackers and defenders, and two fundamental semantic principles proposed by Dung (i.e., self-defense and conflict-freeness) are preserved.
\par 
The latter deviates from traditional Dung's system and intends to provide a unique ranking on the set of arguments. To rank classical logic arguments, an h-categorizer function for acyclic argument diagrams is presented in \cite{Besnard01Graduality}.
In \cite{Cayrol05Graduality}, through generalizing h-categorizer functions in the context of AAFs, a local approach for evaluating arguments is given, moreover, a global one satisfying independence properties is also considered. Compensation-based \cite{Amgoud2013Graduality, Amgoud2016Graduality} and similarity-based \cite{Amgoud18Graduality} semantics assign each argument a score that represents its overall strength. Through interpreting ``undecided'' as being the most controversial status of an argument, a ranking method based on stratified labellings is given in \cite{Thimm2014Graduality}. In \cite{Matt08Graduality}, two-person zero-sum strategic games with imperfect information are used in computing the strength of arguments.
\paragraph{Paper contribution and related work}
Based on the framework proposed by Grossi and Modgil in \cite{Grossi19Graded,Grossi15Graded}, the paper considers some issues on extension-based concrete semantics of AAFs. Moreover, we also explore extension-based abstract semantics defined by a first order language and lift a number of results on concrete semantics to a metatheoretical level. The main contributions of this paper include:
\begin{itemize}
	\item[(I)] An alternative graded fundamental lemma is provided. Grossi and Modgil generalize the well-known fundamental lemma, obtained by Dung \cite{Dung95Acceptability}, to the graded situation in \cite{Grossi19Graded,Grossi15Graded}, which asserts that $\ell$-conflict-freeness is preserved during iterations of the defense function $D_n^m$ starting at admissible sets whenever the parameters are subject to certain conditions. Based on this result, Grossi and Modgil provide a number of explicit constructions for some special extensions in terms of iterations of the defense function, see, e.g., Theorem 2, Facts 10 and 11, and Corollary 1 in \cite{Grossi19Graded}. However, by giving counterexamples, this paper will point out that the graded fundamental lemma obtained in \cite{Grossi19Graded,Grossi15Graded} is too weak to support their constructions in Theorem 2 and Fact 11 mentioned above. In fact, none of these two results always holds. An alternative fundamental lemma will be given in this paper (Lemma \ref{lemma:fundamental lemma}), which generalizes the one obtained in \cite{Grossi19Graded,Grossi15Graded} by relaxing the constraints on the parameters. Moreover, the flaws in \cite{Grossi19Graded} will be corrected based on this lemma. Our work reveals that, \textit{in general, to preserve $\ell$-conflict-freeness during the iteration of the defense function $D_n^m$ starting at an $\ell mn$-admissible set $E$, different constraints on the parameters $\ell$, $m$ and $n$ are required in different situations, which depends on whether $E$ can be exceeded by the iteration of $D_n^m$ starting at $\emptyset$ (see Corollaries \ref{corollary:properties_defense}($a$) and \ref{corollary:construction_co}), however, such difference disappears for well-founded AAFs (see Corollary \ref{corollary:well-founded_co_empty}($c$)).}
	\item[(II)] As applications of the fundamental lemma given in this paper, a number of results concerning various graded extension-based semantics are provided, which refer to universal definability (see Table \ref{table:conclusion2}), structural properties (see Table \ref{table:conclusion3}), explicit constructions of some special extensions in terms of iterations of the defense function (see Table \ref{table:conclusion4}) and the equivalence of some semantics (Propositions \ref{proposition:structure_gr}, \ref{proposition:equ_pr} and \ref{proposition:equ_ss and rra} and Corollary \ref{corollary:equ_ad_def}), etc. A marked difference compared to related results in the traditional situation is that these results depend on not only the class of AAFs but also the parameters. We will illustrate this and discuss related work where relevant results are given.
	\item[(III)] A useful operator $\bigcap_D$, called \textit{reduced meet modulo an ultrafilter}, is proposed, which provides a simple but powerful construction method. For finitary AAFs, the neutrality function $N_\ell$ and the defense function $D_n^m$ are shown to be distributive over this operator, and a variety of fundamental semantics (including conflict-free, admissible, complete and stable semantics, etc) are shown to be closed under this operator (Theorem \ref{theorem:close under meet_fundamental semantics}). Based on this fact, a simple and uniform proof method is provided to confirm the universal definability of a family of range related semantics (Theorem \ref{theorem:existence_stg_ss_rra_rrs}). For infinite AAFs, the problem of establishing the universal definability of range related semantics is often nontrivial \cite{Caminada10Conjecture}. In the non-graded situation, Weydert firstly solved the open problem presented in \cite{Caminada10Conjecture} and declared the universal definability of semi-stable semantics (i.e., the range related semantics induced by complete extensions) using first-order logic (FoL, for short) \cite{Weydert11SsForInf}. Later on, Baumann and Spanring provided an alternative proof for it based on transfinite induction \cite{Baumann15Infinite}. Our proof takes advantage of the operator $\bigcap_D$ instead of using FoL or transfinite induction, moreover, it is suitable to not only semi-stable semantics but also a family of range related semantics. 
	\item[(IV)] In addition to the fundamental semantics mentioned in (III), we also consider the related properties of the operator $\bigcap_D$ w.r.t. various derived semantics (Propositions \ref{proposition:RMMU_na}, \ref{proposition:RMMU_pr_Dung} and \ref{proposition:RMMU_pr} and Theorem \ref{theorem:RMMU_interval}). As more applications of the operator $\bigcap_D$, for a family of extension-based semantics, a number of theoretical results are provided by taking benefit of this operator, including the (co-)compactness of $\varepsilon$-extensibility and $\varepsilon$-inference (Theorems \ref{theorem:epsilon_extensible} and \ref{theorem:epsilon_inference} and Corollaries \ref{corollary:epsilon_extensible} and \ref{corollary:RM of interval semantics}), the Lindenbaum property (Theorem \ref{theorem:Lindenbaum_RM} and Corollary \ref{corollary:RM of interval semantics}), characterizing equivalent AAFs modulo	$\varepsilon_{\textit{max}, \sigma}$ (or, $\varepsilon_{\textit{max}, \sigma}$-inference) in terms of anti-$\varepsilon_{\sigma}$ sets (Theorem \ref{theorem:anti-epsilon set} and Corollary \ref{corollary:anti-epsilon set}), providing a safe operator related to the conflict-free semantics (Theorem \ref{theorem:safe operators}) and establishing representation theorems of $\varepsilon_{cf}^\ell$ (Theorems \ref{theorem:Representation theorem I of l-cf}, \ref{theorem:Representation theorem II of l-cf} and \ref{theorem:Representation theorem III of l-cf}), etc. All these results indicate that the operator $\bigcap_D$ is useful for studying theoretically on AAFs.
	\item[(V)] Taking advantage of model theory, a family of extension-based abstract semantics, defined by a first order language, are considered. Unlike (II) and (III) above, this part focuses on extension-based abstract semantics instead of concrete ones. We define the language $\textit{FS}(\Omega)$ (Definition \ref{definition:FS_formulas}) and explore extension-based abstract semantics defined by it. To show that $\textit{FS}(\Omega)$ is nontrivial, for the class $\Omega$ of all finitary AAFs, it is proved that $\textit{FS}(\Omega)$ has enough expressive power to define a variety of fundamental extension-based semantics w.r.t. $\Omega$ (Proposition \ref{proposition:FoL_fundamental semantics}). Further, for any abstract semantics $\varepsilon$ defined by $\textit{FS}(\Omega)$, through establishing the connection between reduced meets modulo ultrafilters of extensions and ultraproducts of models (Lemma \ref{lemma:RM_ultraproduct models}), the semantics $\varepsilon$ is shown to be closed under the operator $\bigcap_D$ w.r.t. $\Omega$ (Theorem \ref{theorem:RM_FoL}), which brings us metatheorems concerning the universal definability of abstract range related semantics (Theorem \ref{theorem:FoL and universal definability}), the Lindenbaum property (Corollary \ref{corollary:epsilon_extensible2}) and the compactness of extensibility (Corollary \ref{corollary:epsilon_extensible2}) associated with extension-based abstract semantics, etc.
\end{itemize}
\paragraph{Paper overview}
 The remainder of this paper is organized as follows. Section \ref{Sec:Preliminaries} recalls some concepts and elementary properties. In Section \ref{Sec:Fundamental Lemma and its application}, an alternative graded fundamental lemma is given. Based on it, the universal definability and structural property of various extension-based semantics are considered. Moreover, some comments on relevant results obtained in \cite{Grossi19Graded} are given. Section \ref{Sec: Graded range related semantics} focuses on a family of range related semantics. In particular, we provide a uniform method to establish their universal definability based on the operator reduced meet modulo an ultrafilter. Section \ref{Sec: More on RM} considers more properties of reduced meet modulo an ultrafilter. Section \ref{Sec: Graded parametrized semantics} introduces interval semantics and preliminarily studies its elementary properties. Applying model-theoretical tools, some results in Sections \ref{Sec: Graded range related semantics} and \ref{Sec: More on RM} are lifted to a metatheoretical level in Section \ref{Sec: Extension-based abstract semantics}. Finally, we summarize the paper in Section \ref{Sec:Conclusion}.
\section{Preliminaries} \label{Sec:Preliminaries}
In order to make the paper self-contained, this section recalls some concepts and elementary properties concerning AAFs. 
\begin{definition}[AAF\cite{Dung95Acceptability}]
	An AAF $F$ is a pair $\tuple{\A, \rightarrow}$ where $\A \neq \emptyset$ is a set of arguments, and $\rightarrow \subseteq \A\times\A$ is a binary attack relation over $\A$. $F$ is called finitary if the set $\set{b \in \A \mid b \rightarrow a}$ is finite for each $a \in \A$, whereas $F$ is said to be finite whenever $\A$ is finite, else infinite.
\end{definition}
A finite AAF is often represented graphically. For example, the graph in Figure \ref{figure:graded acceptability} represents the AAF $F$ with the set of arguments $\set{a, b, \cdots, g}$ whose attack relation is captured by arrows.
\par 
A variety of semantics of AAFs have been proposed in the literature. In the following, we will recall graded versions of them introduced in \cite{Grossi19Graded}. Thus the parameters $\ell$, $m$ and $n$ are involved. For the sake of convenience, we adopt the convention below, where the parameter $\eta$ will be used in Section \ref{Sec: Graded range related semantics}.
\begin{convention}
	We assume that the parameters $\ell$, $m$, $n$ and $\eta$ involved in this paper are positive integers.
\end{convention}
\begin{notation}
	($a$) Given a set $X$, the notation $\wp(X)$ (or, $\wp_f(X)$) is used to denote the set of all subsets (finite subsets, resp.) of $X$, moreover $|X|$ is used to represent the cardinal of the cardinality of $X$.
	\par ($b$) As usual, this paper uses $\omega$ to denote the least infinite limit ordinal, equivalently, the set of all natural numbers. The notation $n < \omega$ (equivalently, $n \in \omega$) represents that $n$ is a natural number.
\end{notation}
\begin{definition}[Graded neutrality function\cite{Grossi19Graded}] \label{definition:graded-neutrality}
	Given an AAF $F = \tuple{\A, \rightarrow}$, the graded neutrality function with the parameter $\ell$, denoted by $N_{\ell} \colon \pw{\A} \to \pw{\A}$, is defined as, for any $E \subseteq \A$,
	\begin{eqnarray*}
	N_{\ell}(E)  \triangleq  \set{a \in \A \mid \nexistsn{b}{\ell}(b \rightarrow a \text{~and~} b \in E)}.
	\end{eqnarray*}	
\end{definition}
Here, for simplicity, we ignore the dependence of $N_\ell$ on $F$ in the notation, and adopt $N_\ell$ instead of $N_\ell^F$.
Similar simplification also appears in Definition \ref{definition:grade_defense}. Intuitively, $a \in N_{\ell}(E)$ means that the argument $a$ is attacked by at most $\ell - 1$ arguments in $E$.
\begin{figure}[t]
	\setlength{\abovecaptionskip}{-0.2cm}
	\begin{center}
		\includegraphics[scale=1]{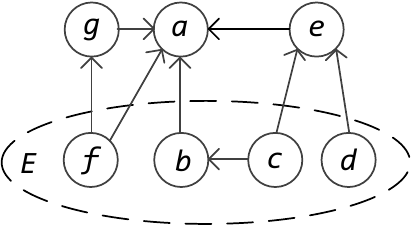}
	\end{center}
	\caption{The AAF for Example \ref{Ex:graded neutrality}.}
	\label{figure:graded acceptability}
\end{figure}
\begin{example} \label{Ex:graded neutrality}
	Consider the AAF $F = \tuple{\A, \rightarrow}$ in Figure \ref{figure:graded acceptability}.
	In the set $E \triangleq \set{b, c, d, f}$, there are altogether two arguments $b$ and $f$ attacking the argument $a$, thus $a \notin N_2(E)$ but $a \in N_3(E)$.
\end{example}
\begin{definition}[Graded defense\cite{Grossi19Graded}] \label{definition:grade_defense}
	Let $F = \tuple{\A, \rightarrow}$ be an AAF. The graded defense function with parameters $m$ and $n$, denoted by $D_n^m \colon \pw{\A} \to \pw{\A}$, is defined as, for any $E \subseteq \A$,
	\begin{eqnarray*}
	D_n^m(E) \triangleq \set{a \in \A \mid \nexistsn{b}{m}(~b \rightarrow a \text{ and } \nexistsn{c}{n}(c \rightarrow b \text{ and } c \in E))}.
	\end{eqnarray*}
\end{definition}
Roughly speaking, $a \in D_n^m(E) $ means that the argument $a$ is defended by $E$, that is, $E$ contains sufficient arguments to attack enough $a$'s attackers. Here the parameters $m$ and $n$ are used to provide a quantitative criterion to evaluate it.
\begin{example}
	Consider the AAF $F$ and $E$ in Example \ref{Ex:graded neutrality} again. There are altogether four arguments $b$, $e$, $f$ and $g$ attacking the argument $a$ but $a$ still is in $D_2^4(E)$ because $e$ is attacked by two arguments in $E$. Similarly, since each of $b$, $e$ and $g$ is attacked by at least one argument in $E$, we also have $a \in D_1^2(E)$.
\end{example}
\begin{definition}[$mn$-self-defended\cite{Grossi19Graded}] \label{definition:def}
	Let $F = \tuple{\A, \rightarrow}$ be an AAF. A set $E \subseteq \A$ is said to be $mn$-self-defended, in symbols $E \in \textit{Def}^{mn}(F)$, if $E \subseteq D_n^m(E)$, that is, $E$ is a post fixpoint of $D_n^m$.
\end{definition}
Similar to the standard (or, non-graded) situation (i.e., $\ell = m = n = \eta = 1$) \cite{Dung95Acceptability}, the functions $N_\ell$ and $D_n^m$ are fundamental for Grossi and Modgil's system, which are the starting point to define semantics. As usual, a graded extension-based semantics is a function $\varepsilon$ which, for any AAF $F = \tuple{\A, \rightarrow}$, assigns a set $\varepsilon(F) \subseteq \pw{\A}$ to $F$. Elements in $\varepsilon(F)$ are so-called extensions of $F$. A number of notions for extensions, which induce corresponding extension-based semantics respectively, have been presented in the literature. Some of them are listed below. Here the notation $\varepsilon_\sigma^\varrho$ is used to denote the graded extension-based semantics induced by $\sigma$-extensions with the parameter $\varrho$.
\begin{definition}[Graded extensions]\label{definition:graded_extensions}
	Let $F = \tuple{\A, \rightarrow}$ be an AAF and $E \subseteq \A$. 
	\begin{itemize}
		\item $E$ is an $\ell$-conflict-free set of $F$, i.e., $E \in \varepsilon_{cf}^\ell(F)$, iff $E \subseteq N_{\ell}(E)$.
		\item $E$ is an $\ell mn$-admissible set of $F$, i.e., $E \in \varepsilon_{ad}^{\ell mn}(F)$, iff $E \subseteq N_{\ell}(E)$ and $E \subseteq D_n^m(E)$.
		\item $E$ is an $\ell mn$-complete extension of $F$, i.e., $E \in \varepsilon_{co}^{\ell mn}(F)$, iff $E \subseteq N_{\ell}(E)$ and $E = D_n^m(E)$.
		\item $E$ is an $\ell mn$-stable extension of $F$, i.e., $E \in \varepsilon_{stb}^{\ell mn}(F)$, iff  $E = N_n(E) = N_m(E) \subseteq N_{\ell}(E)$.
		\item $E$ is the $\ell mn$-grounded extension of $F$, i.e., $E \in \varepsilon_{gr}^{\ell mn}(F)$, iff $E \in \varepsilon_{co}^{\ell mn}(F)$ and $\forall E' \in \varepsilon_{co}^{\ell mn}(F) (E \subseteq E')$.
		\item $E$ is an $\ell mn$-preferred extension of $F$, i.e., $E \in \varepsilon_{pr}^{\ell mn}(F)$, iff $E \in \varepsilon_{co}^{\ell mn}(F)$ and $\nexists E' \in \varepsilon_{co}^{\ell mn}(F) (E \subset E')$.
	\end{itemize}
	\par In addition to the ones above considered in \cite{Grossi19Graded}, the following one will be involved in this paper:
	\begin{itemize}
		\item $E$ is an $\ell$-naive extension of $F$, i.e., $E \in \varepsilon_{na}^{\ell}(F)$, iff $E \in \varepsilon_{cf}^\ell(F)$ and $\nexists E' \in \varepsilon_{cf}^\ell(F) (E \subset E')$.
	\end{itemize}
\end{definition}
These notions agree with the standard ones introduced in \cite{Dung95Acceptability, Bodarenko97Naive} whenever $\ell = m = n =1$.
\begin{example} \label{Ex: graded extensions}
	\begin{figure}[t]
		\centering
		\includegraphics[width=0.4\linewidth]{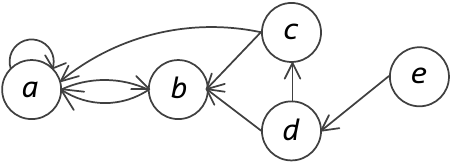}
		\caption{The AAF for Example \ref{Ex: graded extensions}.}
		\label{fig:AF}
	\end{figure}
	Consider the AAF $F = \tuple{\A, \rightarrow}$ in Figure \ref{fig:AF}. For $\ell = 2, m = 1$ and $n = 3$, we have\\
	\[\begin{gathered}
	{\varepsilon_{cf}^\ell(F)} = \left\{ {\begin{array}{*{20}{c}}
		{\emptyset , \set{a}, \set{b}, \set{c}, \set{d}, \set{e}, \set{a, d}, \set{a, e}, \set{b, c}, \set{b, d}, \set{b, e},} \\ 
		{\set{c, d}, \set{c, e}, \set{d, e}, \set{a, d, e}, \set{b, c, e}, \set{b, d, e}, \set{c, d, e}}
		\end{array}} \right\}, \hfill \\
	{\varepsilon_{ad}^{\ell mn}(F)} = \set{\emptyset, \set{e}}, \varepsilon_{co}^{\ell mn}(F) = {\varepsilon_{pr}^{\ell mn}(F)} = {\varepsilon_{gr}^{\ell mn}(F)} = \set{\set{e}}, \hfill \\
	{\varepsilon_{na}^\ell(F)} = \set{\set{a, d, e}, \set{b, c, e}, \set{b,d,e}, \set{c,d,e}} \text{~and} \hfill \\
	{\varepsilon_{stb}^{\ell mn}(F)} = \emptyset. \hfill \\
	\end{gathered} \]
\end{example}
\begin{definition}[\cite{Baumann18Existence}]
	Given a semantics $\varepsilon$ and a class $\Omega$ of AAFs, $\varepsilon$ is said to be universally (or, uniquely) defined w.r.t. $\Omega$ if $|\varepsilon(F)| \geq 1$ ($|\varepsilon(F)| = 1$, resp.) for any $F \in \Omega$.
\end{definition}
Similar to the standard semantics, universal definability is one of the core issues in exploring graded ones. However, for the graded semantics, this issue is related to not only the class of AAFs but also the parameters, which is a marked difference compared to ones of the standard semantics. Moreover, the reader will find that some semantics, which are equivalent in the standard situation, will no longer be equivalent unconditionally in the graded case.
\par 
A number of results obtained in \cite{Grossi19Graded} are recalled below. Here, given a partially ordered set (poset, for short) $\tuple{L, \leq}$, a nonempty subset $X$ of $L$ is said to be (upward) \textit{directed} if, for any $x_1$, $x_2 \in X$, $x_1 \leq x_3$ and $x_2 \leq x_3$ for some $x_3 \in X$.
\begin{lemma}[\cite{Grossi19Graded}] \label{lemma:Grossi_simple}
	Let  $F = \tuple{\A, \rightarrow}$ be an AAF and $X, Y \subseteq \A$.
	\begin{itemize}
		\item[a.] If $X \subseteq Y$ then $N_{\ell}(Y) \subseteq N_{\ell}(X)$.
		\item[b.] If $X \subseteq Y$ then $D_n^m(X) \subseteq  D_n^m(Y)$.
		\item[c.] $N_m(N_n(X)) = D_n^m(X)$.
		\item[d.] If $F$ is finitary then the function $D_n^m$ is continuous, that is, for any directed set $\D \subseteq \wp_f(\A)$, $D_n^m(\bigcup_{X \in \D} X) = \bigcup_{X \in \D} D_n^m(X)$.
	\end{itemize}
\end{lemma}
\begin{remark} \label{remark:continuous}
	It is well known that the equation involved in the clause ($d$) holds for any directed subset of $\pw{\A}$ whenever $D_n^m$ is continuous. For any set $X$, the set $\wp_f(X)$ is a directed set such that $X = \bigcup \wp_f(X)$. Thus, for any directed set $\D \subseteq \pw{\A}$, by the equation in ($d$), we have
	\begin{eqnarray*}
		D_n^m(\bigcup_{X \in \D} X) &=& D_n^m(\bigcup_{Y \in \wp_f(\bigcup_{X \in \D} X)} Y) = \bigcup_{Y \in \wp_f(\bigcup_{X \in \D} X)} D_n^m(Y) \text{, and}\\
		D_n^m(X) &=& D_n^m(\bigcup_{Y \in \wp_f(X)} Y) = \bigcup_{Y \in \wp_f(X)} D_n^m(Y) \text{~for each~} X \in \D.
	\end{eqnarray*}
	Moreover, since $\D$ is directed, it holds that $\wp_f(\bigcup_{X \in \D} X) = \bigcup_{X \in \D} \wp_f(X)$.
	Thus, by the two equations above, we have
	\begin{eqnarray*}
		D_n^m(\bigcup_{X \in \D} X) = \bigcup_{X \in \D} D_n^m(X).
	\end{eqnarray*}
	Hence, in the situation that $F$ is finitary, this paper will use the equation above for any directed set $\D \subseteq \pw{\A}$ without illustrating explicitly.
\end{remark}
The notions of conflict-freeness and self-defense play central roles in the argument systems in the style of Dung. From Definition \ref{definition:graded_extensions}, it is obvious that the semantics of such systems are developed based on these two notions. Some known elementary properties of them are summarized below, which have been applied explicitly or implicitly in the literature.
\begin{lemma} \label{lemma:properties_cf_def}
	Let $F = \tuple{\A, \rightarrow}$ be an AAF.
	\begin{itemize}
		\item[a.] (union closeness) $\bigcup \D \in \textit{Def}^{mn}(F)$ for any set $\D \subseteq \textit{Def}^{mn}(F)$.
		\item[b.] (directed De Morgan law) $N_\ell(\bigcup_{X \in \D} X) = \bigcap_{X \in \D} N_\ell(X)$ for any directed set $\D \subseteq \pw{\A}$.
		\item[c.] (directed union closeness) $\bigcup \D \in \varepsilon_{cf}^{\ell}(F)$ for any directed set $\D \subseteq \varepsilon_{cf}^{\ell}(F)$.
		\item[d.] (down closeness) $X \subseteq Y \in \varepsilon_{cf}^{\ell}(F)$ implies $X \in \varepsilon_{cf}^{\ell}(F)$ for any $X$, $Y \subseteq \A$.
		\item[e.] For each $X \in \textit{Def}^{mn}(F)$, $X \in \varepsilon_{ad}^{\ell mn}(F)$ iff $X \subseteq E$ for some $E \in \varepsilon_{cf}^{\ell}(F)$. Hence $\varepsilon_{cf}^{\ell}(F) \cap \set{X \subseteq \A \mid Y \subseteq X \text{~for some~} Y \in \textit{Def}^{mn}(F) - \varepsilon_{ad}^{\ell mn}(F)} = \emptyset$.
	\end{itemize}
\end{lemma}
\begin{proof}
	The clause ($e$) immediately follows from ($d$), while the latter is implied by Lemma \ref{lemma:Grossi_simple}($a$). Next we prove ($a$)-($c$) in turn.
	\par ($a$) Let $\D \subseteq \textit{Def}^{mn}(F)$. 
	Then, by Definition \ref{definition:def} and Lemma \ref{lemma:Grossi_simple}($b$), we get
	\begin{eqnarray*}
		\bigcup \D \subseteq \bigcup_{X \in \D} D_n^m(X) \subseteq D_n^m(\bigcup \D),
	\end{eqnarray*}
	and hence $\bigcup \D \in \textit{Def}^{mn}(F)$.
	\par ($b$) ``$\subseteq$'' Immediately follows from Lemma \ref{lemma:Grossi_simple}($a$).
	``$\supseteq$'' Assume $a \notin N_\ell(\bigcup_{X \in \D} X)$.
	Then, there exist $b_i \in \bigcup\D$ with $1 \leq i \leq \ell$ such that $b_i \rightarrow a$ in $F$ and $b_i \neq b_j$ whenever $i \neq j$.
	Since $\D$ is directed, $\set{b_1, \cdots, b_\ell} \subseteq X_0$ for some $X_0 \in \D$.
	Thus, $a \notin N_\ell(X_0)$, and hence $a \notin \bigcap_{X \in \D} N_\ell(X)$.
	\par ($c$) Let $\D \subseteq \varepsilon_{cf}^{\ell}(F)$ be directed.
	Assume $a \in \bigcup \D$.
	It suffices to show that $a \in N_\ell(\bigcup \D)$.
	Suppose $X \in \D$.
	Since $\D$ is directed, $a \in X_1$ and $X \subseteq X_1$ for some $X_1 \in \D$.
	Then, by $X_1 \in \varepsilon_{cf}^{\ell}(F)$ and Lemma \ref{lemma:Grossi_simple}($a$),
	\begin{eqnarray*}
		a \in X_1 \subseteq N_\ell(X_1) \subseteq N_\ell(X).
	\end{eqnarray*}
	Further, since $X$ is an arbitrary element in $\D$, we get $a \in \bigcap_{Y \in \D} N_\ell(Y)$, and hence $a \in N_\ell(\bigcup \D)$ by ($b$). 
\end{proof}
By the clauses ($a$) and ($c$) in Lemma \ref{lemma:properties_cf_def}, the properties of conflict-freeness and self-defense are both compact, that is, if all finite subsets of $X$ are conflict-free (or, self-defended) then so is $X$ itself.
\par 
Although the assertions in Lemma \ref{lemma:properties_cf_def} are almost trivial, they imply some useful structural properties concerning the sets $\textit{Def}^{mn}(F)$, $\varepsilon_{cf}^\ell(F)$ and $\varepsilon_{ad}^{\ell mn}(F)$. Before giving them, some concepts in domain theory \cite{Gierz03Book} are recalled firstly. Let $\tuple{L, \leq}$ be a poset. For any $x, y \in L$, we say that $x$ is \textit{way below} $y$, in symbols $x \ll y$, iff for all directed sets $\D \subseteq L$ such that $\textit{sup}\D$ exists, $y \leq \textit{sup}\D$ implies $x \leq d$ for some $d \in \D$. An element $x$ satisfying $x \ll x$ is said to be \textit{compact}. A poset $L$ is said to be \textit{continuous} if it satisfies the axiom of approximation below,
\begin{eqnarray*}
	x = \textit{sup}^\uparrow \appr x \text{~for any~} x \in L.
\end{eqnarray*}
Here, the notation $\textit{sup}^\uparrow \D$ represents the supremum of the directed subset $\D$. Thus, the axiom of approximation says that, for each $x \in L$, the set $\appr x \triangleq \set{u \in L \mid u \ll x}$ is directed and $x = \textit{sup} \set{u \in L \mid u \ll x}$. A poset is said to be a \textit{dcpo} (directed complete partial order) if its every directed subset has a supremum. A dcpo is said to be a \textit{cpo} (complete partial order) if it has the bottom. A cpo $\tuple{L, \leq}$ is said to be \textit{algebraic} if
\begin{eqnarray*}
	x = \textit{sup}^\uparrow \set{y \in \textit{K(L)} \mid y \leq x} \text{~for each~} x \in L.
\end{eqnarray*}
Here $\textit{K(L)}$ is the set of all compact elements in $L$. A \textit{(algebraic) complete semilattice} is a (algebraic) cpo such that every nonempty subset has an infimum. Finally, a poset is a complete lattice if every subset has an infimum (equivalently, every subset has a supremum). Thus, a complete semilattice is a complete lattice if it has the top (i.e., the infimum of the empty set).
\par 
Based on Lemma \ref{lemma:properties_cf_def}, the structural properties of $\textit{Def}^{mn}(F)$, $\varepsilon_{cf}^{\ell}(F)$ and $\varepsilon_{ad}^{\ell mn}(F)$ are summarized below. In the non-graded situation, Dung has shown that $\tuple{\varepsilon_{ad}^{111}(F), \subseteq}$ is a cpo \cite[Theorem 11]{Dung95Acceptability}.
\begin{corollary} \label{corollary:structure_cf_ad}
	Let $F = \tuple{\A, \rightarrow}$ be an AAF.
	\begin{itemize}
		\item[a.] $\tuple{\textit{Def}^{mn}(F), \subseteq}$ is a complete lattice with the bottom $\emptyset$, the top $\nu D_n^m$ (the greatest fixed point of $D_n^m$ in $\tuple{\pw{\A}, \subseteq}$), $\textit{sup}S = \bigcup S$ and $\textit{inf}S = \bigcup (\bigcap S)^\downarrow$ for any set $S \subseteq \textit{Def}^{mn}(F)$, where $(\bigcap S)^\downarrow \triangleq \set{X \in \textit{Def}^{mn}(F) \mid X \subseteq \bigcap S}$.
		\item[b.] $\tuple{\varepsilon_{cf}^{\ell}(F), \subseteq}$ is a continuous poset. In fact, it is an algebraic complete semilattice with the bottom $\emptyset$, $\textit{sup}\D = \bigcup \D$ for any directed set $\D \subseteq \varepsilon_{cf}^{\ell}(F)$ and $\textit{inf}S = \bigcap S$ for any nonempty set $S \subseteq \varepsilon_{cf}^{\ell}(F)$.
		\item[c.] $\tuple{\varepsilon_{ad}^{\ell mn}(F), \subseteq}$ is a complete semilattice with the bottom $\emptyset$, $\textit{sup}\D = \bigcup \D$ for any directed set $\D \subseteq \varepsilon_{ad}^{\ell mn}(F)$ and $\textit{inf}S = \bigcup (\bigcap S)^\downarrow$ for any nonempty set $S \subseteq \varepsilon_{ad}^{\ell mn}(F)$, where $(\bigcap S)^\downarrow \triangleq \set{X \in \varepsilon_{ad}^{\ell mn}(F) \mid X \subseteq \bigcap S}$.
		\item[d.] For any $X \in \varepsilon_{ad}^{\ell mn}(F)$ (or, $\varepsilon_{cf}^{\ell}(F)$), $\tuple{X^\downarrow, \subseteq}$ is a complete lattice (an algebraic complete lattice, resp.) with $\textit{sup}S = \bigcup S$ for any $S \subseteq X^\downarrow$, where $X^\downarrow \triangleq \set{Y \in \varepsilon_{ad}^{\ell mn}(F) (\varepsilon_{cf}^{\ell}(F), \text{resp.}) \mid Y \subseteq X}$.
	\end{itemize}
\end{corollary}
\begin{proof}
	($a$) It follows straightforwardly from Lemma \ref{lemma:properties_cf_def}($a$). The assertion that $\nu D_n^m$ is the top comes from the facts that, by Definition \ref{definition:def}, the set $\textit{Def}^{mn}(F)$ consists exactly of all post fixpoints of $D_n^m$, and $\nu D_n^m$ is the largest among them by the Knaster-Tarski fixed point theorem.
	\par 
	($b$) First, we show that $\tuple{\varepsilon_{cf}^{\ell}(F), \subseteq}$ is continuous. 
	Let $X \in \varepsilon_{cf}^{\ell}(F)$.
	By Lemma \ref{lemma:properties_cf_def}($d$), $\wp_f(X) \subseteq \varepsilon_{cf}^{\ell}(F)$, moreover $\wp_f(X)$ is directed and $\textit{sup}(\wp_f(X)) = \bigcup \wp_f(X) = X$ due to Lemma \ref{lemma:properties_cf_def}($c$).
	To complete the proof, it suffices to show 
	\begin{eqnarray*}
		\wp_f(X) = \rotatebox[]{90}{$\twoheadleftarrow$}X \triangleq \set{Y \in \varepsilon_{cf}^{\ell}(F) \mid Y \ll X}.
	\end{eqnarray*}
	Let $Y \in \rotatebox[]{90}{$\twoheadleftarrow$}X$. Thus, $Y \ll X$.
	Since $X = \textit{sup}(\wp_f(X))$ and $\wp_f(X)$ is a directed subset of $\varepsilon_{cf}^{\ell}(F)$, $Y \subseteq X_0$ for some $X_0 \in \wp_f(X)$.
	Hence, $Y \in \wp_f(X)$. So, $\rotatebox[]{90}{$\twoheadleftarrow$}X \subseteq \wp_f(X)$.
	On the other hand, assume $Z \in \wp_f(X)$.
	Then, $Z \in \varepsilon_{cf}^{\ell}(F)$ due to Lemma \ref{lemma:properties_cf_def}($d$).
	Moreover, for any directed subset $\D$ of $\varepsilon_{cf}^{\ell}(F)$ with $X \subseteq \textit{sup}\D$, by Lemma \ref{lemma:properties_cf_def}($c$), we get $\textit{sup}\D = \bigcup \D$.
	Thus, it follows from $Z \in \wp_f(X)$ and $X \subseteq \bigcup \D$ that $Z \subseteq W$ for some $W \in \D$.
	Hence, $Z \in \rotatebox[]{90}{$\twoheadleftarrow$}X$.
	Thus, $\wp_f(X) \subseteq \rotatebox[]{90}{$\twoheadleftarrow$}X$, as desired.
	So, the poset $\tuple{\varepsilon_{cf}^{\ell}(F), \subseteq}$ is continuous.
	\par 
	Next we show $\tuple{\varepsilon_{cf}^{\ell}(F), \subseteq}$ is an algebraic cpo.
	Clearly, $\emptyset \in \varepsilon_{cf}^{\ell}(F)$ is the bottom.
	Moreover, by Lemma \ref{lemma:properties_cf_def}($c$), $\textit{sup}\D$ exists for any directed subset $\D$ of $\varepsilon_{cf}^{\ell}(F)$.
	Hence, $\tuple{\varepsilon_{cf}^{\ell}(F), \subseteq}$ is a cpo.
	Further, similar to the above, it is easy to verify that the compact elements in $\varepsilon_{cf}^{\ell}(F)$ exactly are the finite ones in $\varepsilon_{cf}^{\ell}(F)$.
	Then, by Lemma \ref{lemma:properties_cf_def}($c$), we have, for each $X \in \varepsilon_{cf}^{\ell}(F)$,
	\begin{eqnarray*}
		X = \bigcup \wp_f(X) = \textit{sup}^\uparrow \set{Y \in \varepsilon_{cf}^{\ell}(F) \mid Y \text{~is compact and~} Y \subseteq X}.
	\end{eqnarray*}
	Hence, $\tuple{\varepsilon_{cf}^{\ell}(F), \subseteq}$ is an algebraic cpo.
	\par 
	Finally, for any nonempty set $S \subseteq \varepsilon_{cf}^{\ell}(F)$, by Lemma \ref{lemma:properties_cf_def}($d$), $\bigcap S \in \varepsilon_{cf}^{\ell}(F)$.
	Thus, $\textit{inf}S$ exists in $\tuple{\varepsilon_{cf}^{\ell}(F), \subseteq}$, in fact, $\textit{inf}S = \bigcap S$.
	Consequently, $\tuple{\varepsilon_{cf}^{\ell}(F), \subseteq}$ is an algebraic complete semilattice.
	\par 
	($c$) $\textit{sup}\D = \bigcup \D$ for any directed set $\D \subseteq \varepsilon_{ad}^{\ell mn}(F)$ immediately follows from ($a$) and ($b$) based on the fact that $\varepsilon_{ad}^{\ell mn}(F) = \varepsilon_{cf}^{\ell}(F) \cap \textit{Def}^{mn}(F)$.
	Moreover, $\emptyset$ is the bottom in $\varepsilon_{ad}^{\ell mn}(F)$.
	Thus, $\tuple{\varepsilon_{ad}^{\ell mn}(F), \subseteq}$ is a cpo.
	Next we consider the infimum.
	Let $\emptyset \neq S \subseteq \varepsilon_{ad}^{\ell mn}(F)$.
	Since $S \neq \emptyset$, there exists an element, say $X_0$, in $S$.
	Then $\bigcup (\bigcap S)^\downarrow \subseteq X_0 \subseteq N_\ell(X_0) \subseteq N_\ell(\bigcup (\bigcap S)^\downarrow)$, that is, $\bigcup (\bigcap S)^\downarrow \in \varepsilon_{cf}^{\ell}(F)$.
	Further, since $(\bigcap S)^\downarrow \subseteq \textit{Def}^{mn}(F)$, by Lemma \ref{lemma:properties_cf_def}($a$), $\bigcup (\bigcap S)^\downarrow \in \textit{Def}^{mn}(F)$.
	Hence, $\bigcup (\bigcap S)^\downarrow \in \varepsilon_{ad}^{\ell mn}(F)$.
	Finally, it is straightforward to verify that $\bigcup (\bigcap S)^\downarrow$ is the largest lower bound of $S$ in $\tuple{\varepsilon_{ad}^{\ell mn}(F), \subseteq}$.
	\par 
	($d$) Immediately follows from the clauses ($b$) and ($c$) in this corollary based on the fact that $X$ itself is the top in $\tuple{X^\downarrow, \subseteq}$.
\end{proof}
The compactness of $\ell$-conflict-freeness and $mn$-self-defense implies the Lindenbaum property given below. For the reader being familiar with FoL, it is a simple matter to see that the next corollary is analogue to the well-known Lindenbaum Theorem in FoL (see, e.g., \cite{Chang1990Book}). It is the reason that it is so named. More general results on this property will be given in Sections \ref{Sec: Graded range related semantics} and \ref{Sec: Extension-based abstract semantics}, see Theorem \ref{theorem:Lindenbaum_RM} and Corollary \ref{corollary:epsilon_extensible2}.
\begin{corollary}[Lindenbaum property of $\varepsilon_{cf}^{\ell}$ and $\varepsilon_{ad}^{\ell mn}$]\label{corollary:Lindenbaum}
	Let $F$ be an AAF. Then each $X \in P$ can be extended to a maximal one in $P$, and hence there exist maximal elements in $P$, where $P$ is either $\varepsilon_{cf}^{\ell}(F)$ or $\varepsilon_{ad}^{\ell mn}(F)$.
\end{corollary}
\begin{proof}
	For any $X \in P$, by the clauses ($b$) and ($c$) in Corollary \ref{corollary:structure_cf_ad}, it isn't difficult to see that $\tuple{\set{Y \in P \mid X \subseteq Y}, \subseteq}$ is a cpo.
	By applying Zorn's Lemma on this cpo based on the fact that each chain is directed, it follows easily that there exists a maximal element in $\set{Y \in P \mid X \subseteq Y}$, as desired.
	Further, there exist maximal elements in $P$ due to $\varepsilon_{cf}^{\ell}(F) \neq \emptyset$ and $\varepsilon_{ad}^{\ell mn}(F) \neq \emptyset$.
\end{proof}
\section{Fundamental lemma and its applications} \label{Sec:Fundamental Lemma and its application}
In addition to the compactness of $\ell$-conflict-freeness and $mn$-self-defense, it is also a desired property that $mn$-self-defense and $\ell$-conflict-freeness are preserved during iterations of the defense function $D_n^m$. The former is always preserved, see Observation \ref{observation:D_ordinal} below. However, the latter isn't preserved in general. Fortunately, in the standard situation, the well-known fundamental lemma, due to Dung \cite{Dung95Acceptability}, asserts that $1$-conflict-freeness is preserved when the function $D_1^1$ is applied on $111$-admissible sets. Recently, a graded version of Dung's result is provided in \cite{Grossi19Graded,Grossi15Graded}, which plays a central role in constructing some special extensions based on iterations of the function $D_n^m$. This section will provide an alternative graded fundamental lemma for all AAFs. 
\subsection{Fundamental lemma}
We begin with giving a trivial but useful result below, which is an instantiation of a standard result in partial order theory.
\begin{observation} \label{observation:D_ordinal}
	Let  $F = \tuple{\A, \rightarrow}$ be an AAF and $E \subseteq \A$. If $E \subseteq D_n^m(E)$ (i.e., $E \in \textit{Def}^{mn}(F)$) then, for any ordinal $\xi_1$ and $\xi_2$ with $\xi_1 < \xi_2$, $D_{\substack{m \\ n}}^{\xi_1}(E) \subseteq D_{\substack{m \\ n}}^{\xi_2}(E)$.
	Here, for each ordinal $\xi$ and $X \subseteq \A$,
	\begin{eqnarray*}
		D_{\substack{m \\ n}}^\xi(X) \triangleq
		\begin{cases}
			X &\text{\textit{if}~~$\xi = 0$}  \\ 
			D_n^m(D_{\substack{m \\ n}}^k(X)) &\text{\textit{if}~~$\xi = k + 1$}\\ 
			\bigcup_{k < \xi} D_{\substack{m \\ n}}^k(X) &\text{\textit{if}~~$\xi (\neq 0)$ is a limit ordinal}
		\end{cases}
	\end{eqnarray*}
\end{observation}
\begin{proof}
	We first prove that $D_{\substack{m \\ n}}^{\xi}(E) \subseteq D_{\substack{m \\ n}}^{\xi + 1}(E)$ for any ordinal $\xi$.
	For $\xi = 0$, it holds trivially due to $D_{\substack{m \\ n}}^0(E) = E \subseteq D_n^m(E)$. Now we consider the inductive step. 
	\par If $\xi$ is a successor ordinal, that is, $\xi = \xi_0 + 1$ for some $\xi_0$. By IH, $D_{\substack{m \\ n}}^{\xi_0}(E) \subseteq D_{\substack{m \\ n}}^{\xi_0 + 1}(E) = D_{\substack{m \\ n}}^{\xi}(E)$. Then, by Lemma \ref{lemma:Grossi_simple}(b), $D_{\substack{m \\ n}}^{\xi}(E) = D_{\substack{m \\ n}}^{\xi_0 + 1}(E) \subseteq D_{\substack{m \\ n}}^{\xi + 1}(E)$. 
	\par If $\xi$ is a limit ordinal, then $k + 1 <  \xi$ for each $k < \xi$. Thus, for each $k < \xi$, by IH and Lemma \ref{lemma:Grossi_simple}(b),
	\begin{eqnarray*}
	D_{\substack{m \\ n}}^{k}(E) \subseteq D_{\substack{m \\ n}}^{k + 1}(E)
	\subseteq D_n^m(\bigcup_{\xi' < \xi} D_{\substack{m \\ n}}^{\xi'}(E))
	= D_n^m(D_{\substack{m \\ n}}^{\xi}(E))
	= D_{\substack{m \\ n}}^{\xi + 1}(E).
	\end{eqnarray*}
	Hence, $D_{\substack{m \\ n}}^{\xi}(E) = \bigcup_{k < \xi} D_{\substack{m \\ n}}^k(E) \subseteq D_{\substack{m \\ n}}^{\xi + 1}(E)$.
	\par Next we prove this observation itself by induction on $\xi_2$. 
	Assume $\xi_1 < \xi_2$.
	Clearly, $\xi_2 \neq 0$. If $\xi_2$ is a limit ordinal, it holds trivially because
	$D_{\substack{m \\ n}}^{\xi_1}(E) \subseteq \bigcup_{k < \xi_2} D_{\substack{m \\ n}}^{k}(E) = D_{\substack{m \\ n}}^{\xi_2}(E)$ due to $\xi_1 < \xi_2$. Now we assume $\xi_2 = \xi_3 + 1$. Then $\xi_1 \leq \xi_3$. Hence, by IH, we get $D_{\substack{m \\ n}}^{\xi_1}(E) \subseteq D_{\substack{m \\ n}}^{\xi_3}(E)$. 
	Further,  by the assertion given at the beginning of this proof, we have $D_{\substack{m \\ n}}^{\xi_1}(E) \subseteq D_{\substack{m \\ n}}^{\xi_3}(E) \subseteq D_{\substack{m \\ n}}^{\xi_2}(E)$.
\end{proof}
An alternative fundamental lemma is given below. Restricting this lemma to the special situation that $F$ is finitary, $n \geq \ell = m$ and $0 \leq \xi < \omega$, we get the one obtained in \cite[Lemma 3]{Grossi19Graded}. It should be emphasized that it is of vital importance that weakening the assumption $n \geq \ell = m$ to $n \geq \ell \geq m$ for asserting that $\ell$-conflict-freeness is preserved during iterations of $D_n^m$ starting at any $\ell mn$-admissible set whenever $n \geq \ell \geq m$.
\begin{lemma}[Fundamental lemma]\label{lemma:fundamental lemma}
	Let  $F = \tuple{\A, \rightarrow}$ be an AAF and $n \geq \ell \geq m$. Then for any $E \in \varepsilon_{ad}^{\ell mn}(F)$ and ordinal $\xi$,
    \begin{eqnarray*}
    	E \subseteq D_{\substack{m \\ n}}^\xi(E) \subseteq N_\ell(D_{\substack{m \\ n}}^\xi(E)) \subseteq N_\ell(E).
    \end{eqnarray*}
	In particular, by setting $\ell \triangleq m$, we have, for any $E \in \varepsilon_{ad}^{mmn}(F)$,
	\begin{eqnarray*}
		E \subseteq D_{\substack{m \\ n}}^\xi(E) \subseteq N_m(D_{\substack{m \\ n}}^\xi(E)) \subseteq N_m(E) \text{~whenever $n \geq m$.}
	\end{eqnarray*}
\end{lemma}
\begin{proof}
	We proceed by induction on $\xi$.
	Due to $E \in \varepsilon_{ad}^{\ell mn}(F)$, we obtain $E \subseteq N_{\ell}(E)$ and $E \subseteq D_n^m(E)$. Since $D_{\substack{m \\ n}}^0$  is the identity function, it immediately follows that
	\begin{eqnarray*}
		E = D_{\substack{m \\ n}}^0(E) \subseteq N_\ell(E) =  N_\ell(D_{\substack{m \\ n}}^0(E)).
	\end{eqnarray*}
	Assume that $E \subseteq D_{\substack{m \\ n}}^{\xi'}(E) \subseteq N_\ell(D_{\substack{m \\ n}}^{\xi'}(E)) \subseteq N_\ell(E)$ for all $\xi' < \xi$ (IH).	Next we deal with the inductive step by distinguishing two cases based on $\xi$.\\
	\\
	Case 1 : $\xi$ is a successor ordinal. \\
	Then $\xi = k + 1$ for some ordinal $k$. 
	Due to $E \subseteq D_n^m(E)$, by Observation \ref{observation:D_ordinal}, we get 
	\begin{eqnarray*}
		E \subseteq D_n^m(E) \subseteq D_{\substack{m \\ n}}^{k + 1}(E) = D_{\substack{m \\ n}}^\xi(E).
	\end{eqnarray*}
	Hence, by Lemma \ref{lemma:Grossi_simple}(a), we get $N_\ell(D_{\substack{m \\ n}}^\xi(E)) \subseteq N_{\ell}(E)$. It remains to show that $D_{\substack{m \\ n}}^{\xi}(E) \subseteq N_\ell(D_{\substack{m \\ n}}^{\xi}(E))$. Conversely, suppose that $a \notin N_\ell(D_{\substack{m \\ n}}^{\xi}(E))$ for some $a \in D_{\substack{m \\ n}}^{\xi}(E) = D_n^m(D_{\substack{m \\ n}}^{k}(E))$. Thus $\nexistsn{b}{m}(b \rightarrow a \text{~and~} \nexistsn{c}{n}(c \rightarrow b \text{~and~} c \in D_{\substack{m \\ n}}^{k}(E)))$ and $\existsn{b}{\ell}(b \rightarrow a \text{~and~} b \in D_{\substack{m \\ n}}^{\xi}(E))$.
	Hence, due to $\ell \geq m$, there exists an argument, say $b_0$, such that $b_0 \in D_{\substack{m \\ n}}^{\xi}(E)$ and
	\begin{eqnarray}
	\existsn{c}{n}(c \rightarrow b_0 \text{~and~} c \in D_{\substack{m \\ n}}^{k}(E)).  \label{formula:b_0}
	\end{eqnarray}
	Clearly, it follows from $b_0 \in D_{\substack{m \\ n}}^{\xi}(E) = D_n^m(D_{\substack{m \\ n}}^{k}(E))$ that
	\begin{eqnarray*}
		\nexistsn{c}{m}(c \rightarrow b_0 \text{~and~} \nexistsn{d}{n}(d \rightarrow c \text{~and~} d \in D_{\substack{m \\ n}}^{k}(E))).
	\end{eqnarray*}
	Thus, by (\ref{formula:b_0}) and $n \geq m$, there exists an argument, say $c_0$, such that $c_0 \in D_{\substack{m \\ n}}^{k}(E)$ and
	\begin{eqnarray*}
		\existsn{d}{n}(d \rightarrow c_0 \text{~and~} d \in D_{\substack{m \\ n}}^{k}(E)).
	\end{eqnarray*}
	Then, $\existsn{d}{\ell}(d \rightarrow c_0 \text{~and~} d \in D_{\substack{m \\ n}}^{k}(E))$ due to $n \geq \ell$. Hence, $c_0 \notin N_\ell(D_{\substack{m \\ n}}^{k}(E))$, which contradicts that $c_0 \in D_{\substack{m \\ n}}^{k}(E)$ and IH.\\
	\\
	Case 2 : $\xi$ is a limit ordinal.\\
	By IH, we have $E \subseteq D_{\substack{m \\ n}}^k(E) \subseteq N_\ell(D_{\substack{m \\ n}}^k(E)) \subseteq N_\ell(E) \text{~for all~} k < \xi.$
	Hence, $E \subseteq \bigcup_{k < \xi} D_{\substack{m \\ n}}^k(E) = D_{\substack{m \\ n}}^\xi(E)$. 
	Moreover, by (IH) and Observation \ref{observation:D_ordinal}, $\set{D_{\substack{m \\ n}}^k(E) \mid k < \xi}$ is a chain in $\varepsilon_{cf}^\ell(F)$, thus $D_{\substack{m \\ n}}^\xi(E) \subseteq N_\ell(D_{\substack{m \\ n}}^\xi(E))$ by Lemma \ref{lemma:properties_cf_def}($c$).
	Further, due to Lemma \ref{lemma:Grossi_simple}($a$), we also have $N_\ell(D_{\substack{m \\ n}}^\xi(E)) \subseteq N_\ell(E)$, as desired.
\end{proof}
The following example reveals that the assumption $n \geq \ell \geq m$ in Lemma \ref{lemma:fundamental lemma} is necessary even for finite AAFs. Thus, in general case, this requirement can't be weaken any more.
\begin{example}\label{Ex:counterexample}
	\begin{figure}[t]
		\centering
		\includegraphics[width=0.5\linewidth]{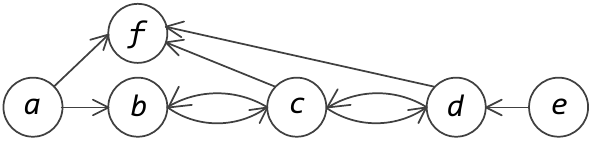}
		\caption{The AAF for Examples \ref{Ex:counterexample} and \ref{Ex:ad_co_pr}.}
		\label{figure:counterexample}
	\end{figure}
	Consider the AAF in Figure \ref{figure:counterexample}. 
	\par Let $\ell = 3$, $m = n = 2$ and $E_1 = \set{a, b, c, d, e}$. Since each argument in $E_1$ is attacked by at most two arguments in $E_1$, we get $E_1 \subseteq N_\ell(E_1)$. Moreover, it is easy to check that $D_n^m(E_1) = \set{a, b, c, d, e, f}$, and hence $E_1 \subseteq D_n^m(E_1)$. Thus, $E_1 \in \varepsilon_{ad}^{\ell mn}(F)$. However, $f \notin N_{\ell}(D_n^m(E_1))$ because $f$ is attacked by three arguments $a$, $c$ and $d$ in $D_n^m(E_1)$, which, together with $f \in D_n^m(E_1)$, implies $D_n^m(E_1) \nsubseteq N_{\ell}(D_n^m(E_1))$.
	\par Let $m = 3$, $\ell = n = 2$ and $E_2 = \set{a, b, d, e}$. Since each argument in $E_2$ is attacked by at most one argument in $E_2$, we get $E_2 \subseteq N_\ell(E_2)$. Furthermore, it is easy to verify that $D_n^m(E_2) = \set{a, b, c, d, e, f}$, and thus $E_2 \subseteq D_n^m(E_2)$. Hence, $E_2 \in \varepsilon_{ad}^{\ell mn}(F)$. However, $c \notin N_{\ell}(D_n^m(E_2))$ because $c$ is attacked by two arguments $b$ and $d$ in $D_n^m(E_2)$, which, together with $c \in D_n^m(E_2)$, implies $D_n^m(E_2) \nsubseteq N_{\ell}(D_n^m(E_2))$. 
	\par Consequently, in the situation that $\ell > n$ or $m > \ell$, $\ell$-conflict-freeness is not always preserved during iterations of the graded defense function $D_n^m$ starting at $\ell mn$-admissible sets.
\end{example}
Before applying the fundamental lemma to obtain further results, we recall the fixed point theorem on ordered sets simply.
One of statements asserted by the well-known Knaster-Tarski fixed point theorem is that,  for any monotone function $f$ on a complete lattice (or, cpo) $\tuple{L, \leq}$ and $d \in L$ with $d \leq f(d)$, there is a least fixed point over $d$ (i.e., it is the least one among fixed points being greater than $d$), which may be obtained based on the iteration of $f$ starting at $d$. In particular, if $d$ is the bottom of $L$, such iteration generates the least fixed point of $f$. The lengths of these iterations are in general transfinite, but are bounded at worst by the cardinal of the cardinality of the lattice (cpo, resp.) under consideration. In particular, if $f$ is continuous, the lengths of these iterations are bounded by $\omega$. This result allows us to adopt the convention below.
\begin{convention}[$\lambda_F$]\label{convention:lambda_F}
	Given an AAF $F = \tuple{\A, \rightarrow}$, we choose arbitrarily and fix an ordinal $\lambda_F$ so that, for any $X \in \textit{Def}^{mn}(F)$, $D_{\substack{m \\ n}}^{\lambda_F}(X)$ is the least fixed point over $X$ of the function $D_n^m$ in the complete lattice $\tuple{\pw{A}, \subseteq}$. For example, we may put $\lambda_F \triangleq | \pw{\A} |$ in any case, and put $\lambda_F \triangleq \omega$ whenever $F$ is finitary.
\end{convention}
In this paper, the notation $\lambda$ sometimes is also used in lambda-expressions to define functions. However, no confusion can arise because that one is an ordinal and the other is a standard notation in lambda-expressions. A principal significance of the fundamental lemma is that it brings an interesting connection between the semantics $\varepsilon_{ad}^{\ell mn}$ and $\varepsilon_{co}^{\ell mn}$ whenever $n \geq \ell \geq m$. Such a connection is captured by the following well-known notion.
\begin{definition}[Galois adjunction \cite{Davey02Book}]
	Let $\tuple{C, \leq_C}$ and $\tuple{D, \leq_D}$ be posets and let $f \colon C \to D$ and $g \colon D \to C$ be two monotone functions. These two functions form a Galois adjunction (or, Galois connection), denoted by $f \dashv g$, if, for each $x \in C$ and $y \in D$, $f(x) \leq_D y$ iff $x \leq_C g(y)$.
\end{definition}
In the situation that $n \geq \ell \geq m$, some useful properties are listed below. 
\begin{corollary}\label{corollary:properties_defense}
	Let $F = \tuple{\A, \rightarrow}$ be an AAF and $n \geq \ell \geq m$. 
	\begin{itemize}
		\item[a.] For any ordinal $\xi$, $E \in \varepsilon_{ad}^{\ell mn}(F)$ implies $D_{\substack{m \\ n}}^{\xi}(E) \in \varepsilon_{ad}^{\ell mn}(F)$.
		\item[b.] For any $E \in \varepsilon_{ad}^{\ell mn}(F)$, $D_{\substack{m \\ n}}^{\lambda_F}(E)$ is the least $\ell mn$-complete extension containing $E$.	
		\item[c.] $D_{\substack{m \\ n}}^{\lambda_F} \dashv \textit{incl}$ is a Galois adjunction between the posets $\tuple{\varepsilon_{ad}^{\ell mn}(F), \subseteq}$ and \\$\tuple{\varepsilon_{co}^{\ell mn}(F), \subseteq}$, where $\textit{incl}$ is the embedding function from $\varepsilon_{co}^{\ell mn}(F)$ to $\varepsilon_{ad}^{\ell mn}(F)$ defined as $\textit{incl} \triangleq \lambda_X.X$, i.e., $\textit{incl}(E) \triangleq E$ for each $E \in \varepsilon_{co}^{\ell mn}(F)$.
		\item[d.] $D_{\substack{m \\ n}}^{\lambda_F} \colon \varepsilon_{ad}^{\ell mn}(F) \to \varepsilon_{co}^{\ell mn}(F)$ preserves $\textit{sup}s$ and $\textit{incl} \colon \varepsilon_{co}^{\ell mn}(F) \to \varepsilon_{ad}^{\ell mn}(F)$ preserves $\textit{inf}s$.
	\end{itemize}	
\end{corollary}
\begin{proof}
	($a$) Immediately follows from Observation \ref{observation:D_ordinal} and Lemma \ref{lemma:fundamental lemma}.
	\par ($b$) 
	Due to $E \in \varepsilon_{ad}^{\ell mn}(F) \subseteq \textit{Def}^{mn}(F)$, by Convention \ref{convention:lambda_F}, $D_{\substack{m \\ n}}^{\lambda_F}(E)$ is the least fixed point over $E$.	
	Moreover, by ($a$), $D_{\substack{m \\ n}}^{\lambda_F}(E) \in \varepsilon_{co}^{\ell mn}(F)$.
	Then, since each $\ell mn$-complete extension is a fixed point of $D_n^m$, $D_{\substack{m \\ n}}^{\lambda_F}(E)$ is the least $\ell mn$-complete extension containing $E$.
	\par ($c$) By ($b$), $\lambda_{X \in \varepsilon_{ad}^{\ell mn}(F)}.D_{\substack{m \\ n}}^{\lambda_F}(X)$ is a monotone function from $\tuple{\varepsilon_{ad}^{\ell mn}(F), \subseteq}$ to $\tuple{\varepsilon_{co}^{\ell mn}(F), \subseteq}$.
	Moreover, it is easy to check that, for any $E \in \varepsilon_{ad}^{\ell mn}(F)$ and $E' \in \varepsilon_{co}^{\ell mn}(F)$, we have $D_{\substack{m \\ n}}^{\lambda_F}(E) \subseteq E'$ iff $E \subseteq E' = \textit{incl}(E')$.
	\par ($d$) It is an instantiation of the well-known result about Galois adjunctions which asserts that any upper adjoint (e.g., $\textit{incl}$) preserves $\textit{inf}$s, any lower adjoint (e.g., $D_{\substack{m \\ n}}^{\lambda_F}$) preserves $\textit{sup}$s (see, e.g., Theorem O-3.3 in \cite{Gierz03Book}).
\end{proof}
In fact, in the situation that $n \geq \ell \geq m$, by Corollary \ref{corollary:properties_defense}($b$), it is easy to check that $\tuple{\varepsilon_{co}^{\ell mn}(F), \subseteq}$ is isomorphic to $\tuple{\varepsilon_{ad}^{\ell mn}(F) / \!\!\! \sim, \leq}$. Here, $\varepsilon_{ad}^{\ell mn}(F) / \!\!\! \sim$ is the quotient set of $\varepsilon_{ad}^{\ell mn}(F)$ induced by the kernel $\sim$ of the function $D_{\substack{m \\ n}}^{\lambda_F}$, where $\sim$ is the equivalent relation on $\varepsilon_{ad}^{\ell mn}(F)$ defined as $E_1 \sim E_2$ iff $D_{\substack{m \\ n}}^{\lambda_F}(E_1) = D_{\substack{m \\ n}}^{\lambda_F}(E_2)$, and the order $\leq$ is defined as $[E_1] \leq [E_2]$ iff $D_{\substack{m \\ n}}^{\lambda_F}(E_1) \subseteq D_{\substack{m \\ n}}^{\lambda_F}(E_2)$, where $[E]$ is used to denote the equivalence class $\set{X \in \varepsilon_{ad}^{\ell mn}(F) \mid E \sim X}$. Clearly, the relation $\leq$ is well-defined, that is, it doesn't depend on the representatives.
\subsection{Some applications of fundamental lemma}
This subsection will apply the fundamental lemma to deal with the equivalence and universal definability of some extension-based semantics.
\par For the notion of a grounded extension, in addition to the one in Definition \ref{definition:graded_extensions}, there exist two variants worthy of consideration, which are graded versions of the ones defined in the literature:
\begin{enumerate}
	\item[(a)] a graded version of Dung's definition \cite{Dung95Acceptability}: $E \in \varepsilon_{gr, \textit{Dung}}^{\ell mn}(F)$ iff $E$ is the least fixed point of $D_n^m$. 
	\item[(b)] a graded version of Dunne's definition \cite{Dunne15Grounded}: $E \in \varepsilon_{gr, \textit{Dunne}}^{\ell mn}(F)$ iff $E = \bigcap \varepsilon_{co}^{\ell mn}(F)$. 
\end{enumerate}
\par In the standard situation, the semantics $\varepsilon_{gr, \textit{Dung}}^{\ell mn}$,  $\varepsilon_{gr, \textit{Dunne}}^{\ell mn}$ and $\varepsilon_{gr}^{\ell mn}$ are always equivalent to each other. However, they are conditionally equivalent in the graded situation.
\begin{proposition}\label{proposition:structure_gr}
	 If $\ell \geq m$ and $n \geq m$ then $\varepsilon_{gr}^{\ell mn} = \varepsilon_{gr, \textit{Dung}}^{\ell mn} = \varepsilon_{gr, \textit{Dunne}}^{\ell mn}$ w.r.t. all AAFs.
\end{proposition}
\begin{proof}
	Let $F = \tuple{\A, \rightarrow}$ be an AAF. Since $P \triangleq \tuple{\pw{\A}, \subseteq}$ is a complete lattice and the function $D_n^m$ is a monotone function on $P$, by the Knaster-Tarski fixed point theorem, $D_n^m$ has the least fixed point $\lfp(D_n^m) = D_{\substack{m \\ n}}^{\lambda_F}(\emptyset)$.
	Since $\emptyset \in \varepsilon_{ad}^{mmn}(F)$ and $n \geq m$, by Lemma \ref{lemma:fundamental lemma}, we obtain 
	\begin{eqnarray*}
		\lfp(D_n^m) = D_{\substack{m \\ n}}^{\lambda_F}(\emptyset) \subseteq N_m(D_{\substack{m \\ n}}^{\lambda_F}(\emptyset)) = N_m(\lfp(D_n^m)).
	\end{eqnarray*}
	Consequently, $\lfp(D_n^m) \subseteq N_m(\lfp(D_n^m)) \subseteq N_\ell(\lfp(D_n^m))$ due to $\ell \geq m$, which, together with $\lfp(D_n^m) = D_n^m(\lfp(D_n^m))$, implies $\lfp(D_n^m) \in \varepsilon_{co}^{\ell mn}(F)$. Moreover, since $E = D_n^m(E)$ for each $E \in \varepsilon_{co}^{\ell mn}(F)$, we get $\lfp(D_n^m) = \bigcap \varepsilon_{co}^{\ell mn}(F)$. Finally, by the definition of $\ell mn$-grounded extension (see, Definition \ref{definition:graded_extensions}), we have $\varepsilon_{gr}^{\ell mn}(F) = \set{\lfp(D_n^m)} = \set{\bigcap \varepsilon_{co}^{\ell mn}(F)}$.
\end{proof}
\begin{example}\label{Ex:counterexampleforgr}
	Without the assumption that $\ell \geq m$ and $n \geq m$, the semantics $\varepsilon_{gr, \textit{Dung}}^{\ell mn}$,  $\varepsilon_{gr, \textit{Dunne}}^{\ell mn}$ and $\varepsilon_{gr}^{\ell mn}$ aren't always equivalent.
	\begin{figure}[t]
		\setlength{\abovecaptionskip}{-0.2cm}
		\begin{center}
			\includegraphics[scale=1]{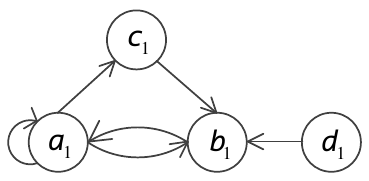}
			\hspace{0.5cm}
			\includegraphics[scale=1]{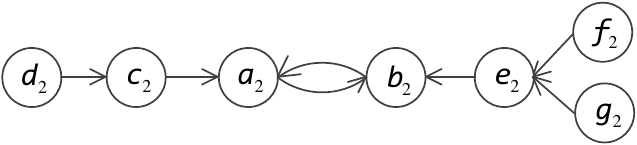}
		\end{center}
		\caption{The AAFs for Example \ref{Ex:counterexampleforgr}. }
		\label{figure:counterexampleforgr}
	\end{figure}
	\par (a) Consider the left AAF $F$ in Figure \ref{figure:counterexampleforgr}.
	For $\ell = 1$ and $m = n = 2$, it is easy to check that 
	\begin{eqnarray*}
		&&\varepsilon_{cf}^{\ell}(F) = \set{\emptyset, \set{b_1}, \set{c_1}, \set{d_1}, \set{c_1, d_1}}, D_n^m(\set{c_1, d_1}) = \set{a_1, c_1, d_1} \text{~and~}\\
		&& D_n^m(\emptyset) = D_n^m(\set{b_1}) = D_n^m(\set{c_1}) = D_n^m(\set{d_1})  = \set{c_1, d_1}.
	\end{eqnarray*}
	Clearly, none in $\varepsilon_{cf}^{\ell}(F)$ is a fixed point of $D_n^m$.
	Hence $\varepsilon_{co}^{\ell mn}(F) = \varepsilon_{gr}^{\ell mn}(F) = \emptyset$. 
	Thus, $\varepsilon_{gr, \textit{Dunne}}^{\ell mn}(F) = \set{\bigcap \varepsilon_{co}^{\ell mn}(F)} = \set{\set{a_1, b_1, c_1, d_1}}$.
	Moreover, since $D_{\substack{m \\ n}}^3(\emptyset) = D_{\substack{m \\ n}}^2(\emptyset) = \set{a_1, c_1, d_1}$, $\varepsilon_{gr, \textit{Dung}}^{\ell mn}(F) = \set{\lfp(D_n^m)} = \set{\set{a_1, c_1, d_1}}$.
	\par (b) Consider the right AAF $F$ in Figure \ref{figure:counterexampleforgr}, which comes from \cite[Example 10]{Grossi19Graded}. For $\ell = m = 2$ and $n = 1$, as shown in \cite[Example 10]{Grossi19Graded}, $\varepsilon_{gr, \textit{Dung}}^{\ell mn}(F) = \set{\lfp(D_n^m)} = \set{\set{a_2, b_2, c_2, d_2, f_2, g_2}}$. Moreover, it is easy to verify that $\varepsilon_{gr}^{\ell mn}(F) = \varepsilon_{co}^{\ell mn}(F) = \emptyset$. Thus, $\varepsilon_{gr, \textit{Dunne}}^{\ell mn}(F) = \set{\bigcap \varepsilon_{co}^{\ell mn}(F)} = \set{\set{a_2, b_2, c_2, d_2, e_2, f_2, g_2}}$.
	\par Consequently, in both situations above, none of these three semantics equals to each other.
\end{example}
The careful reader may find the common phenomenon in the examples ($a$) and ($b$) above, that is, $\varepsilon_{co}^{\ell mn}(F) = \emptyset$. In fact, it is not accidental. Formally, we have the fact below.
\begin{observation}\label{observation:properties_gr}
	Let $F$ be an AAF. Then the following are equivalent:
	\begin{itemize}
		\item[a.] $\varepsilon_{gr}^{\ell mn}(F) = \varepsilon_{gr, \textit{Dung}}^{\ell mn}(F) = \varepsilon_{gr, \textit{Dunne}}^{\ell mn}(F)$.
		\item[b.] $\varepsilon_{co}^{\ell mn}(F) \neq \emptyset$.
		\item[c.] $\varepsilon_{gr}^{\ell mn}(F) = \set{D_{\substack{m \\ n}}^{\lambda_F}(\emptyset)}$.
		\item[d.] $D_{\substack{m \\ n}}^{\lambda_F}(\emptyset) \in \varepsilon_{cf}^{\ell}(F)$.
	\end{itemize}
\end{observation}
\begin{proof}
	It is trivial that ($a \Rightarrow b$) and ($c \Rightarrow d \Rightarrow a$).
	For ($b \Rightarrow c$), assume that $E \in \varepsilon_{co}^{\ell mn}(F)$.
	Thus, $D_{\substack{m \\ n}}^{\lambda_F}(\emptyset) \subseteq D_{\substack{m \\ n}}^{\lambda_F}(E) = E$, and hence, by Lemma \ref{lemma:properties_cf_def}($d$), $D_{\substack{m \\ n}}^{\lambda_F}(\emptyset) \in \varepsilon_{cf}^{\ell}(F)$ due to $E \in \varepsilon_{cf}^{\ell}(F)$.
	Then, it is straightforward to see that $\varepsilon_{gr}^{\ell mn}(F) = \set{D_{\substack{m \\ n}}^{\lambda_F}(\emptyset)}$.
\end{proof}
Let us mention some important consequences of Proposition \ref{proposition:structure_gr}. First, for the clauses ($a$) and ($b$) in Corollary \ref{corollary:properties_defense}, the assumption $n \geq \ell \geq m$ may be weakened to $\ell \geq m$ and $n \geq m$ whenever $E \subseteq D_{\substack{m \\ n}}^{\lambda_F}(\emptyset)$, that is
\begin{corollary}\label{corollary:construction_co}
	Let $F$ be any AAF, $\ell \geq m$ and $n \geq m$. For any $E \in \varepsilon_{ad}^{\ell mn}(F)$, if $E \subseteq D_{\substack{m \\ n}}^{\lambda_F}(\emptyset)$ then $D_{\substack{m \\ n}}^{\xi}(E) \in \varepsilon_{ad}^{\ell mn}(F)$ for each ordinal $\xi$, and hence $D_{\substack{m \\ n}}^{\lambda_F}(E)$ is the least $\ell mn$-complete extension containing $E$, in fact, $D_{\substack{m \\ n}}^{\lambda_F}(E) = D_{\substack{m \\ n}}^{\lambda_F}(\emptyset)$ is the least $\ell mn$-complete extension of $F$.
\end{corollary}
\begin{proof}
	By Observation \ref{observation:D_ordinal}, $D_{\substack{m \\ n}}^{\xi}(E) \in \textit{Def}^{mn}(F)$.
	Moreover, $D_{\substack{m \\ n}}^{\xi}(E) \subseteq D_{\substack{m \\ n}}^{\lambda_F}(\emptyset)$ due to $E \subseteq D_{\substack{m \\ n}}^{\lambda_F}(\emptyset)$.
	Thus, by Proposition \ref{proposition:structure_gr} and Lemma \ref{lemma:properties_cf_def}($d$), $D_{\substack{m \\ n}}^{\xi}(E) \in \varepsilon_{cf}^{\ell}(F)$, and hence $D_{\substack{m \\ n}}^{\xi}(E) \in \varepsilon_{ad}^{\ell mn}(F)$, as desired.
\end{proof}
Second, since the least fixed point of $D_n^m$ always exists, Proposition \ref{proposition:structure_gr} also asserts the universal definability of $\varepsilon_{gr}^{\ell mn}$ and $\varepsilon_{co}^{\ell mn}$, that is
\begin{corollary} \label{corollary:existence_gr_co}
	For any AAF $F$, if $\ell \geq m$ and $n \geq m$ then $\varepsilon_{gr}^{\ell mn}(F) = \set{D_{\substack{m \\ n}}^{\lambda_F}(\emptyset)}$ and $|\varepsilon_{co}^{\ell mn}(F)| \geq 1$.
\end{corollary}
Notice that the requirement $\ell \geq m$ and $n \geq m$ in Corollary \ref{corollary:existence_gr_co} is necessary, otherwise, $\varepsilon_{co}^{\ell mn}$ isn't universally defined even for finite AAFs, see Example \ref{Ex:counterexampleforgr}.
\begin{corollary} \label{corollary:structure_co}
	For any finitary AAF $F$, let $P \triangleq \tuple{\varepsilon_{co}^{\ell mn}(F), \subseteq}$. If $\varepsilon_{co}^{\ell mn}(F) \neq \emptyset$ then $P$ is a cpo with $\textit{sup}\D = \bigcup \D$ for any directed set $\D \subseteq \varepsilon_{co}^{\ell mn}(F)$. In particular, $P$ is a cpo whenever $\ell \geq m$ and $n \geq m$.
\end{corollary}
\begin{proof}
	Since $F$ is finitary, by Lemma \ref{lemma:Grossi_simple}($d$), the function $D_n^m$ is continuous.
	Thus, we may set $\lambda_F = \omega$.
	Then, by Observation \ref{observation:properties_gr}, $D_{\substack{m \\ n}}^{\omega}(\emptyset) \in \varepsilon_{co}^{\ell mn}(F)$ is the bottom of $\varepsilon_{co}^{\ell mn}(F)$.
	For any directed set $\D \subseteq \varepsilon_{co}^{\ell mn}(F)$, by Lemma \ref{lemma:Grossi_simple}($d$), we get
	\begin{eqnarray*}
		\bigcup \D = \bigcup_{X \in \D} D_n^m(X) = D_n^m(\bigcup \D).
	\end{eqnarray*}
	Moreover, by Lemma \ref{lemma:properties_cf_def}($c$), it holds that $\bigcup \D \in \varepsilon_{cf}^{\ell}(F)$.
	Thus, $\bigcup \D \in \varepsilon_{co}^{\ell mn}(F)$, and hence $\bigcup \D$ is the supremum of $\D$ in $P$.
	Consequently, $P$ is a cpo, as desired.
\end{proof}
Clearly, the proof above depends on the continuity of the function $D_n^m$. At the moment, we don't know whether Corollary \ref{corollary:structure_co} holds for all AAFs. But, under the assumption that $n \geq \ell \geq m$, for any AAF $F$, since the functions $D_{\substack{m \\ n}}^{\lambda_F}$ and $\textit{incl}$ form a Galois adjunction between the posets $\tuple{\varepsilon_{ad}^{\ell mn}(F), \subseteq}$ and $\tuple{\varepsilon_{co}^{\ell mn}(F), \subseteq}$ (see, Corollary \ref{corollary:properties_defense}), similar to the former, the latter is also a complete semilattice. In the situation that $\ell = m = n = 1$, Dung has obtained a similar result in \cite[Theorem 25]{Dung95Acceptability}.
\begin{proposition} \label{proposition:structure_co}
	Let $F$ be an AAF, $P \triangleq \tuple{\varepsilon_{co}^{\ell mn}(F), \subseteq}$ and $n \geq \ell \geq m$. Then $P$ is a complete semilattice with
	\begin{itemize}
		\item[a.] $\textit{inf}S = \bigcup (\bigcap S)^\downarrow$ for any nonempty set $S \subseteq \varepsilon_{co}^{\ell mn}(F)$ in $P$, where \\
		$(\bigcap S)^\downarrow \triangleq \set{X \in \varepsilon_{ad}^{\ell mn}(F) \mid X \subseteq \bigcap S}$.
		\item[b.] For any directed set $\D \subseteq \varepsilon_{co}^{\ell mn}(F)$, $\textit{sup}\D = D_{\substack{m \\ n}}^{\lambda_F}(\bigcup \D)$, in particular, $\textit{sup}\D = \bigcup \D$ whenever $F$ is finitary.
	\end{itemize}
	Hence, $\tuple{E^\downarrow, \subseteq}$ is a complete lattice for each $E \in \varepsilon_{co}^{\ell mn}(F)$, where $E^\downarrow$ is the set of all $\ell mn$-complete extensions of $F$ that are contained in $E$.
\end{proposition}
\begin{proof}
	By Proposition \ref{proposition:structure_gr}, $D_{\substack{m \\ n}}^{\lambda_F}(\emptyset)$ is the bottom in $P$.
	To complete the proof, we show ($a$) and ($b$) in turn.
	\par ($a$) Since $\emptyset \neq S \subseteq \varepsilon_{co}^{\ell mn}(F) \subseteq \varepsilon_{ad}^{\ell mn}(F)$, by Corollary \ref{corollary:structure_cf_ad}($c$), $\bigcup (\bigcap S)^\downarrow \in \varepsilon_{ad}^{\ell mn}(F)$.
	Then, by Corollary \ref{corollary:structure_cf_ad}($d$), the poset $\tuple{\textit{Q}, \subseteq}$ is a complete lattice, where
	\begin{eqnarray*}
		\textit{Q} \triangleq \set{Y \in \varepsilon_{ad}^{\ell mn}(F) \mid Y \subseteq \bigcup (\bigcap S)^\downarrow}.
	\end{eqnarray*}
	Clearly, $\bigcup (\bigcap S)^\downarrow$ is the top in $\textit{Q}$. Next we give a simple fact about the set $\textit{Q}$.
	\\
	\\ \textbf{Claim~~} The set $\textit{Q}$ is closed under the function $D_n^m$, that is, $D_n^m(Y) \in \textit{Q}$ for each $Y \in \textit{Q}$.\\
	Let $Y \in \textit{Q}$.
	For each $X \in S$, we have $Y \subseteq X$, and hence $D_n^m(Y) \subseteq D_n^m(X) = X$ due to $X \in \varepsilon_{co}^{\ell mn}(F)$.
	Thus, $D_n^m(Y) \subseteq \bigcap S$.
	Moreover, by Corollary \ref{corollary:properties_defense}($a$), $D_n^m(Y) \in \varepsilon_{ad}^{\ell mn}(F)$, and hence $D_n^m(Y) \in \textit{Q}$.\\
	\\
	By the claim above, the function $D_n^m$ may be regarded as an endofunction on the complete lattice $\tuple{\textit{Q}, \subseteq}$.
	Further, since the function $D_n^m$ is monotone, by the Knaster-Tarski fixed point theorem, $D_n^m$ has the largest fixed point $\nu^{\scalebox{0.5}[0.5]{\textit{Q}}} D_n^m$ in $\tuple{\textit{Q}, \subseteq}$, which is obtained as
	\begin{eqnarray*}
		\nu^{\scalebox{0.5}[0.5]{\textit{Q}}} D_n^m = \textit{sup}_{\tuple{\textit{Q}, \subseteq}} \set{X \in \textit{Q} \mid X \subseteq D_n^m(X)}.
	\end{eqnarray*}
	Thus, $\nu^{\scalebox{0.5}[0.5]{\textit{Q}}} D_n^m = \textit{sup}_{\tuple{\textit{Q}, \subseteq}} \textit{Q}$ due to $\textit{Q} \subseteq \varepsilon_{ad}^{\ell mn}(F)$.
	Further, since $\bigcup (\bigcap S)^\downarrow$ is the top in $\textit{Q}$, we get $\nu^{\scalebox{0.5}[0.5]{\textit{Q}}} D_n^m = \bigcup (\bigcap S)^\downarrow$.
	Thus, $D_n^m(\bigcup (\bigcap S)^\downarrow) = \bigcup (\bigcap S)^\downarrow$, and hence $\bigcup (\bigcap S)^\downarrow \in \varepsilon_{co}^{\ell mn}(F)$.
	Moreover, it is easy to see that $\bigcup (\bigcap S)^\downarrow \subseteq X$ for each $X \in S$, and hence $\bigcup (\bigcap S)^\downarrow$ is a lower bound of $S$ in $P$.
	Finally, since $\textit{Q}$ contains all lower bounds of $S$ in the poset $P$, we have $\textit{inf}S = \bigcup (\bigcap S)^\downarrow$.
	\par ($b$) Since $\D$ is a directed subset of $\varepsilon_{co}^{\ell mn}(F)$, it is also a directed subset of $\varepsilon_{ad}^{\ell mn}(F)$.
	Then, by Corollary \ref{corollary:structure_cf_ad}($c$), $\bigcup \D$ is the supremum of $\D$ in $\tuple{\varepsilon_{ad}^{\ell mn}(F), \subseteq}$.
	By Corollary \ref{corollary:properties_defense}($d$), $D_{\substack{m \\ n}}^{\lambda_F}(\bigcup \D)$ is the supremum of the set $\set{D_{\substack{m \\ n}}^{\lambda_F}(X) \mid X \in \D}$ in $P$.
	Then, due to $D_{\substack{m \\ n}}^{\lambda_F}(X) = X$ for each $X \in \D$, we get $\textit{sup}\D = D_{\substack{m \\ n}}^{\lambda_F}(\bigcup \D)$ in $P$, as desired.
	Finally, if $F$ is finitary, by Corollary \ref{corollary:structure_co}, $\textit{sup}\D = \bigcup \D$.
\end{proof} 
\begin{remark}
	Since the poset $\tuple{\varepsilon_{co}^{\ell mn}(F), \subseteq}$ doesn't always have the top, in other words, $\textit{inf}(\emptyset)$ doesn't always exist, the proposition above doesn't imply that the poset is a complete lattice.
	In the proof ($a$) above, the assertion that $\bigcup (\bigcap S)^\downarrow$ is a fixed point of $D_n^m$ depends on the assumption that $S \subseteq \varepsilon_{co}^{\ell mn}(F)$.
	For the poset $\tuple{\varepsilon_{ad}^{\ell mn}(F), \subseteq}$, Corollary \ref{corollary:structure_cf_ad}($c$) also asserts that $\textit{inf}S = \bigcup (\bigcap S)^\downarrow$ for any nonempty set $S \subseteq \varepsilon_{ad}^{\ell mn}(F)$.
	In this situation, although a similar complete lattice 
	\begin{eqnarray*}
				\tuple{\set{Y \in \varepsilon_{ad}^{\ell mn}(F) \mid Y \subseteq \bigcup (\bigcap S)^\downarrow}, \subseteq}
	\end{eqnarray*}
	is at hand due to Corollary \ref{corollary:structure_cf_ad}($d$), $\bigcup (\bigcap S)^\downarrow$ is not always a fixed point of $D_n^m$ because the set $\set{Y \in \varepsilon_{ad}^{\ell mn}(F) \mid Y \subseteq \bigcup (\bigcap S)^\downarrow}$ is not always closed under the function $D_n^m$, and hence $D_n^m$ can't be regarded as the endofunction on this complete lattice, which obstructs the application of the Knaster-Tarski fixed point theorem in this situation.	
\end{remark}
\begin{corollary}[Lindenbaum property of $\varepsilon_{co}^{\ell mn}$] \label{corollary:existence_pr}
	Let $F = \tuple{\A, \rightarrow}$ be an AAF. Under the assumptions in Corollary \ref{corollary:structure_co} or Proposition \ref{proposition:structure_co}, we have
	\begin{itemize}
		\item[a.] There exist maximal elements in $\varepsilon_{co}^{\ell mn}(F)$, that is, $|\varepsilon_{pr}^{\ell mn}(F)| \geq 1$.
		\item[b.] Each $\ell mn$-complete extension of $F$ can be extended to a maximal one (i.e., an extension in $\varepsilon_{pr}^{\ell mn}(F)$).
	\end{itemize}	
\end{corollary}
\begin{proof}
	Under the assumptions in Proposition \ref{proposition:structure_co} or Corollary \ref{corollary:structure_co}, the poset $\tuple{\varepsilon_{co}^{\ell mn}(F), \subseteq}$ is a cpo, thus ($a$) immediately follows by applying Zorn's lemma. The clause ($b$) may be shown similarly to Corollary \ref{corollary:Lindenbaum}.
\end{proof}
\begin{remark}\label{remark:pr}
	Corollary \ref{corollary:existence_pr} provides two assertions about the universal definability of $\varepsilon_{pr}^{\ell mn}$: one is that $|\varepsilon_{pr}^{\ell mn}(F)| \geq 1$ for any AAF $F$ whenever $n \geq \ell \geq m$, another is that $|\varepsilon_{pr}^{\ell mn}(F)| \geq 1$ for any finitary AAF $F$ whenever $\ell \geq m$ and $n \geq m$.
	In the situation that $\ell \geq m$ and $n \geq m$, although the semantics $\varepsilon_{co}^{\ell mn}$ is universally defined w.r.t. the class of all AAFs (see Corollary \ref{corollary:existence_gr_co}), at the moment, we don't know whether so is $\varepsilon_{pr}^{\ell mn}$. Equivalently, in this situation, it is an open problem that whether $\varepsilon_{co}^{\ell mn}(F)$ has the Lindenbaum property for any AAF $F$. Moreover, a related problem is that whether $\tuple{\varepsilon_{co}^{\ell mn}(F), \subseteq}$ is a cpo for any AAF $F$ whenever $\ell \geq m$ and $n \geq m$ (refer to Corollary \ref{corollary:structure_co}). We are inclined to answer these problems negatively.
\end{remark}
For the notion of a preferred extension, in addition to the one in Definition \ref{definition:graded_extensions}, there is another variant, which is a graded version of the one introduced by Dung in \cite{Dung95Acceptability}.
\begin{definition} \label{definition:pr_Dung}
	Let $F = \tuple{\A, \rightarrow}$ be an AAF and $E \subseteq \A$. Then $E \in \varepsilon_{pr, \textit{Dung}}^{\ell mn}(F)$ iff $E \in \varepsilon_{ad}^{\ell mn}(F)$ and $\nexists E' \in \varepsilon_{ad}^{\ell mn}(F)(E \subset E')$.	
\end{definition}
By Corollary \ref{corollary:Lindenbaum}, it is easy to see that $\varepsilon_{pr, \textit{Dung}}^{\ell mn}$ is universally defined for all AAFs. The next result reveals that the semantics $\varepsilon_{pr, \textit{Dung}}^{\ell mn}$ and $\varepsilon_{pr}^{\ell mn}$ are equivalent whenever $n \geq \ell \geq m$. 
However, they aren't equivalent in general even for finite AAFs with $\ell \geq m$ and $n \geq m$, see Example \ref{Ex:ad_co_pr}.
\begin{proposition}\label{proposition:equ_pr}
	If $n \geq \ell \geq m$ then $\varepsilon_{pr}^{\ell mn} = \varepsilon_{pr, \textit{Dung}}^{\ell mn}$ w.r.t. all AAFs.
\end{proposition}
\begin{proof}
	Let $F$ be any AAF. We intend to show $\varepsilon_{pr}^{\ell mn}(F) = \varepsilon_{pr, \textit{Dung}}^{\ell mn}(F)$.
	\par ``$\supseteq$'' Assume that $E \in \varepsilon_{ad}^{\ell mn}(F)$ and $\nexists E' \in \varepsilon_{ad}^{\ell mn}(F)(E \subset E')$. Then $E \subseteq N_{\ell}(E)$ and $E \subseteq D_n^m(E)$. 
	Due to $E \in \varepsilon_{ad}^{\ell mn}(F)$ and $n \geq \ell \geq m$, by Corollary \ref{corollary:properties_defense}($a$), we get $D_n^m(E) \in \varepsilon_{ad}^{\ell mn}(F)$. 
	Then, since $E$ is a maximal $\ell mn$-admissible set, it follows from $E \subseteq D_n^m(E)$ that $E = D_n^m(E) $. 
	Thus, $E \in \varepsilon_{co}^{\ell mn}(F)$. 
	Further, since $\varepsilon_{co}^{\ell mn}(F) \subseteq \varepsilon_{ad}^{\ell mn}(F)$, $E$ is also a maximal element in $\varepsilon_{co}^{\ell mn}(F)$. Consequently, $E \in \varepsilon_{pr}^{\ell mn}(F)$, as desired.
	\par ``$\subseteq$'' Let $E \in \varepsilon_{pr}^{\ell mn}(F)$. 
	Clearly, $E \in \varepsilon_{ad}^{\ell mn}(F)$. 
	To complete the proof, it suffices to show that $E$ is a maximal $\ell mn$-admissible set of $F$. 
	Suppose that $E \subseteq E'$ with $E' \in \varepsilon_{ad}^{\ell mn}(F)$. 
	By Corollary \ref{corollary:properties_defense}($b$), there exists an $E^* \in \varepsilon_{co}^{\ell mn}(F)$ such that $E \subseteq E' \subseteq E^*$.
	Then, due to $E \in \varepsilon_{pr}^{\ell mn}(F)$, we have $E = E^*$, and hence $E = E'$, as desired.
\end{proof}
Since $\varepsilon_{pr}^{\ell mn}$ is universally defined for finitary AAFs whenever $\ell \geq m$ and $n \geq m$ (see, Corollary \ref{corollary:existence_pr}) and $\varepsilon_{pr,\textit{Dung}}^{\ell mn}$ is universally defined for any $\ell$, $m$ and $n$ (see, Corollary \ref{corollary:Lindenbaum}), a rational conjecture arises at this point: for finitary AAFs, can we weaken the assumption $n \geq \ell \geq m$ in the proposition above to $\ell \geq m$ and $n \geq m$?
The next example answers this conjecture negatively.
However, it indeed holds for all well-founded AAFs, see, Corollary \ref{corollary:well-founded_AAF_stb}.
\begin{example} \label{Ex:ad_co_pr}
	Consider the AAF $F$ in Figure \ref{figure:counterexample}. For $\ell = 3$ and $m = n = 2$, it is easy to verify that
	\[\begin{gathered}
	{\varepsilon_{ad}^{\ell mn}(F)} = \left\{ {\begin{array}{*{15}{c}}
		{\emptyset , \set{a}, \set{e}, \set{a, c}, \set{a, e}, \set{b, d}, \set{c, e}, \set{a, b, d}, } \\ 
		{\set{a, c, e}, \set{b, d, e}, \set{a, b, c, d}, \set{a, b, d, e},} \\
		{\set{b, c, d, e}, \set{a, b, c, d, e}, \set{b, c, d, e, f}} 
		\end{array}} \right\}, \hfill \\
	{\varepsilon_{co}^{\ell mn}(F)} = \set{\set{a, e}, \set{a, c, e}, \set{a, b, d, e}} \text{~and~} \hfill \\
	{\varepsilon_{pr}^{\ell mn}(F)} = \set{\set{a, c, e}, \set{a, b, d, e}}. \hfill \\
	\end{gathered} \]
	Obviously, $\varepsilon_{pr, \textit{Dung}}^{\ell mn}(F) = \set{\set{a, b, c, d, e}, \set{b, c, d, e, f}} \neq \varepsilon_{pr}^{\ell mn}(F)$.
\end{example}
In the situation that $n \geq \ell \geq m$, by Proposition \ref{proposition:equ_pr}, $\ell mn$-preferred extensions form the top bounder of $\ell mn$-admissible sets. The next result provides a characterization of the distributions of $\ell$-conflict-free sets, $mn$-defended sets and $\ell mn$-admissible sets. In Section \ref{Sec:Conclusion}, it is displayed graphically in Figure \ref{figure:distributions}($a$).
\begin{corollary}[Non-interpolant]\label{corollary:non-interpolant}
	For any AAF $F$, if $n \geq \ell \geq m$ then, for each $X_1 \in \varepsilon_{cf}^{\ell}(F)$ and $X_2 \in \varepsilon_{pr}^{\ell mn}(F)$, there is no $Y \in \textit{Def}^{mn}(F)$ such that $X_2 \subset Y \subseteq X_1$.
\end{corollary}
\begin{proof}
	Immediately follows from Lemma \ref{lemma:properties_cf_def}($e$) and Proposition \ref{proposition:equ_pr}.
\end{proof}
In the above result, the assumption $n \geq \ell \geq m$ is necessary, otherwise, for instance, consider $X_1 = \set{a, b, c, d, e}$ and $X_2 = \set{a, b, d, e}$ in Example \ref{Ex:ad_co_pr}. 
\par Let $E$ be any $\ell mn$-admissible set. To ensure that $D_{\substack{m \\ n}}^{\xi}(E)$ is $\ell$-conflict-free for any ordinal $\xi$, the parameters $\ell$, $m$ and $n$ often are required to satisfy some conditions.
Compare the clause ($a$) in Corollary \ref{corollary:properties_defense} with Corollary \ref{corollary:construction_co}, the reader may find that such conditions are different for $\emptyset$ and arbitrary other admissible set: for any AAF $F$, $D_{\substack{m \\ n}}^{\xi}(\emptyset)$ is always $\ell$-conflict-free whenever $\ell \geq m$ and $n \geq m$, however, in general, we need the stronger assumption $n \geq \ell \geq m$ to ensure that, for an arbitrary $\ell mn$-admissible set $E$, $D_{\substack{m \\ n}}^{\xi}(E)$ is $\ell$-conflict-free, which can't always be relaxed to the assumption $\ell \geq m$ and $n \geq m$, see, Example \ref{Ex:counterexampleforGrossiTh2}.
From a viewpoint of proof-technique, such difference comes from the fact that $\emptyset$ is always an $\ell mn$-admissible set for any $\ell$, $m$ and $n$.
In fact, due to $\emptyset \in \varepsilon_{ad}^{mmn}(F)$ and $n \geq m$, by Lemma \ref{lemma:fundamental lemma}, we can get $D_{\substack{m \\ n}}^{\xi}(\emptyset) \subseteq N_m(D_{\substack{m \\ n}}^{\xi}(\emptyset))$, and hence $D_{\substack{m \\ n}}^{\xi}(\emptyset) \in \varepsilon_{cf}^\ell(F)$ whenever $\ell \geq m$.
Unfortunately, such proof trick doesn't work for arbitrary $\ell mn$-admissible sets $E$ because $E \in \varepsilon_{ad}^{\ell mn}(F)$ doesn't always implies $E \in \varepsilon_{ad}^{mmn}(F)$ under the assumption $\ell \geq m$. In a word, to preserve $\ell$-conflict-freeness during iterations of $D_n^m$ starting at $\ell mn$-admissible sets, different assumptions on the parameters are required in different situations. It should be pointed out that the assumption $\ell \geq m$ and $n \geq m$ is sufficient for iterations of $D_n^m$ starting at any $\ell mn$-admissible set $E$ with $E \subseteq D_{\substack{m \\ n}}^{\lambda_F}(\emptyset)$, see Corollary \ref{corollary:construction_co}. In fact, for any AAF $F$ with a well-founded attack relation, the reader will find that this assumption works well for the iteration of $D_n^m$ starting at any $\ell mn$-admissible set. In the following, we will consider related issues for AAFs having certain well-foundedness.
\subsection{Local well-foundedness}
\begin{definition} \label{definition:well-founded}
	Let $F = \tuple{\A, \rightarrow}$ be an AAF and $X \subseteq \A$. The attack relation $\rightarrow$ is said to be (local) well-founded on the set $X$ if there is no infinite sequence, like $a_0 \leftarrow a_1 \leftarrow a_2 \leftarrow \cdots$, in $X$ (equivalently, for any nonempty subset $Y$ of $X$, there exists an argument $a \in Y$ such that $\neg \exists b \in Y(b \rightarrow a)$).
	Similarly, the transitive closure $\rightarrow^+$ (of $\rightarrow$) is said to be (local) well-founded on the set $X$ if there is no infinite sequence, like $a_0~^+\!\!\leftarrow a_1~^+\!\!\leftarrow a_2~^+\!\!\leftarrow \cdots$ with $a_i \in X$ for each $i \leq \omega$.
	Moreover, following \cite{Dung95Acceptability}, $F$ itself is said to be (global) well-founded whenever $\rightarrow$ is well-founded on $\A$.
\end{definition}
In the standard situation, the properties of well-founded AAFs have been considered in \cite{Dung95Acceptability}. Unlike \cite{Dung95Acceptability}, this subsection will focus on local well-foundedness. For the global case, $\rightarrow$ is well-founded on $\A$ iff so is $\rightarrow^+$. However, for the local case, it doesn't hold in general.
Clearly, due to $\rightarrow \subseteq \rightarrow^+$, if $\rightarrow^+$ is well-founded on $X$ then so is the attack relation $\rightarrow$ itself. However, the converse is not true. For example, consider the AAF below, 
\begin{figure}[H]
	\centering
	\vspace{-0.2cm}
	\includegraphics[width=0.2\linewidth]{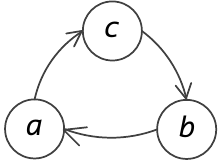}
	\label{fig:Well-founded}
\end{figure}
\noindent the attack relation $\rightarrow$ is well-founded on $\set{a, b}$, but $\rightarrow^+$ is not well-founded on $\set{a, b}$ due to $b \rightarrow^+ a$ and $a \rightarrow^+ b$. 
\par This subsection aims to show that properties of well-founded AAFs obtained in \cite{Dung95Acceptability} still hold in the situation that $\rightarrow^+$ is locally well-founded on a related set, which are consequences of the next result. It is well-known that, similar to induction on well-ordered relations, we may use the inductive method to prove propositions based on well-founded relations (see, e.g., \cite{Kunen1992Book}). The clause ($a$) in the next lemma asserts that a property holds for all arguments in the well-founded structure $\tuple{\A - X, \rightarrow^+}$, we will adopt the well-founded induction to prove it, which is similar to the well-known inductive proof on natural numbers.
\begin{lemma} \label{lemma:transitive closure}
	Let $F = \tuple{\A, \rightarrow}$ be an AAF and $X \subseteq \A$ such that $D_n^m(X) \subseteq X$. If the transitive closure $\rightarrow^+$ is well-founded on $\A - X$ then
	\begin{itemize}
		\item[a.] For each $a \in \A - X$, there is no $Y \in \textit{Def}^{mn}(F)$ such that $X \cup \set{a} \subseteq Y$. 
		\item[b.] For each $Y \in \textit{Def}^{mn}(F)$, if $X \subseteq Y$ then $X = Y$.
		\item[c.] $X$ is the largest fixed point of the function $D_n^m$ in the complete lattice $\tuple{\pw{\A}, \subseteq}$ whenever $X = D_n^m(X)$.
	\end{itemize}
\end{lemma}
\begin{proof}
	($a$) We will prove it by induction on the well-founded relation $\rightarrow^+$ on $\A - X$.
	Let $a \in \A - X$.
	Conversely, suppose $X \cup \set{a} \subseteq Y$ for some $Y \in \textit{Def}^{mn}(F)$.
	Since $D_n^m(X) \subseteq X$ and $a \in \A - X$, $a \notin D_n^m(X) = N_m N_n (X)$.
	So, there exist distinct arguments $b_0, \cdots, b_{m - 1} \in \A$ such that $b_i \rightarrow a$ and $b_i \in N_n(X)$ for each $i \leq m - 1$.
	Then, it follows from $a \in Y \subseteq D_n^m(Y)$ that $b_{i_0} \notin N_n(Y)$ for some $i_0 \leq m - 1$.
	Hence, there exist distinct arguments $c_0, \cdots, c_{n - 1} \in \A$ such that $c_i \rightarrow b_{i_0}$ and $c_i \in Y$ for each $i \leq n - 1$. To complete the proof, a simple assertion about these $c_i$'s is given below.\\
	\\ \textbf{Claim~~} $c_i \in X$ for each $i \leq n - 1$.\\
	Conversely, assume that $c_i \notin X$ for some $i \leq n - 1$. Then, $c_i \in \A - X$. By IH, it follows from $c_i \in \A - X$ and $c_i \rightarrow^+ a$ (due to $c_i \rightarrow b_{i_0} \rightarrow a$) that there is no $Z \in \textit{Def}^{mn}(F)$ such that $X \cup \set{c_i} \subseteq Z$, which contradicts that $X \cup \set{c_i} \subseteq Y$ and $Y \in \textit{Def}^{mn}(F)$.\\
	\\
	Now, $b_{i_0} \notin N_n(X)$ comes from that $c_i \rightarrow b_{i_0}$ and $c_i \in X$ for each $i \leq n - 1$, hence a contradiction arises, as desired.
	\par ($b$) Clearly, $X \cup Y \subseteq Y$ due to $X \subseteq Y$.
	Then, by the clause ($a$), we get $Y \subseteq X$, and hence $X = Y$.
	\par ($c$) By the Knaster-Tarski fixed point theorem, $D_n^m$ has the largest fixed point $\nu D_n^m$ in the complete lattice $\tuple{\pw{\A}, \subseteq}$.
	Thus, $X \subseteq \nu D_n^m$ due to $X = D_n^m(X)$.
	Further, by the clause ($b$), it follows from $\nu D_n^m \in \textit{Def}^{mn}(F)$ that $X = \nu D_n^m$.	
\end{proof}
\begin{example} \label{Ex:counterexampleforatt}
	\begin{figure}[t]
		\centering
		\includegraphics[width=0.4\linewidth]{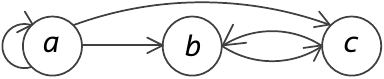}
		\caption{The AAF for Examples \ref{Ex:counterexampleforatt}, \ref{Ex:counterexampleforpr2} and \ref{Ex: Counterexample for V}.}
		\label{figure:counterexampleforatt}
	\end{figure}
	The assumption that $\rightarrow^+$ is well-founded on $\A - X$ in Lemma \ref{lemma:transitive closure} can't be weaken to that $\rightarrow$ is well-founded on $\A - X$. Consider the AAF in Figure \ref{figure:counterexampleforatt} with $n = m = 2$. For $X \triangleq \set{a, b}$, it is obvious that $\rightarrow$ instead of $\rightarrow^+$ (notice that $c \rightarrow^+ c$) is well-founded on $\A - X = \set{c}$, moreover, $Y \in \textit{Def}^{mn}(F)$ due to $Y = D_n^m(Y)$ for each $Y \in \set{\set{a, b}, \set{a, b, c}}$. Then, it is easy to check that none of the clauses ($a$)-($c$) in Lemma \ref{lemma:transitive closure} holds.
\end{example}
The lemma above is crucial in providing constructions of some special extensions in terms of iterations of $D_n^m$ under the assumption of local well-foundedness. One of its useful consequences is that well-foundedness may make some iterations of $D_n^m$ saturated, which induces the formation of the top in $\varepsilon_{co}^{\ell mn}(F)$ (i.e., the unique $\ell mn$-preferred extension) as stated below.
\begin{corollary} \label{corollary:well-founded_co_nonempty}
	Let $F = \tuple{\A, \rightarrow}$ be an AAF, $E \in \varepsilon_{ad}^{\ell mn}(F)$ and $n \geq \ell \geq m$. If $\rightarrow^+$ is well-founded on $\A - D_{\substack{m \\ n}}^{\lambda_F}(E)$ then 
	\begin{itemize}
		\item[a.] $D_{\substack{m \\ n}}^{\lambda_F}(E)$ is the largest $\ell mn$-complete extension of $F$ and hence an $\ell mn$-preferred extension containing $E$, in fact, $D_{\substack{m \\ n}}^{\lambda_F}(E)$ is the unique $\ell mn$-complete extension containing $E$.
		\item[b.] $P \triangleq \tuple{\varepsilon_{ad}^{\ell mn}(F), \subseteq}$ is a complete lattice with the top $D_{\substack{m \\ n}}^{\lambda_F}(E)$.
		\item[c.] $Q \triangleq \tuple{\varepsilon_{co}^{\ell mn}(F), \subseteq}$ is a complete lattice with the top $D_{\substack{m \\ n}}^{\lambda_F}(E)$.		
	\end{itemize}
\end{corollary}
\begin{proof}
	($a$) By Corollary \ref{corollary:properties_defense}($b$), $D_{\substack{m \\ n}}^{\lambda_F}(E)$ is the least $\ell mn$-complete extension containing $E$.
	Further, by Lemma \ref{lemma:transitive closure}($c$), it is the largest $\ell mn$-complete extension.
	Thus, $D_{\substack{m \\ n}}^{\lambda_F}(E)$ is the unique $\ell mn$-complete extension containing $E$.
	\par ($b$) By Corollary \ref{corollary:structure_cf_ad}($c$), $\textit{inf}S$ exists in the poset $P$ for any nonempty $S \subseteq \varepsilon_{ad}^{\ell mn}(F)$.
	To complete the proof, it suffices to show that $D_{\substack{m \\ n}}^{\lambda_F}(E)$ is the largest one in $P$.
	Let $E' \in \varepsilon_{ad}^{\ell mn}(F)$.
	By Corollary \ref{corollary:properties_defense}($b$), $D_{\substack{m \\ n}}^{\lambda_F}(E') \in \varepsilon_{co}^{\ell mn}(F)$ due to $n \geq \ell \geq m$.	
	Then, by the clause ($a$), we have $E' \subseteq D_{\substack{m \\ n}}^{\lambda_F}(E') \subseteq D_{\substack{m \\ n}}^{\lambda_F}(E)$. Thus, $\textit{inf}(\emptyset) = D_{\substack{m \\ n}}^{\lambda_F}(E)$ in $P$.
	\par ($c$) Since $n \geq \ell \geq m$, by Proposition \ref{proposition:structure_co}, $\textit{inf}S$ exists in the poset $Q$ for any nonempty $S \subseteq \varepsilon_{co}^{\ell mn}(F)$.
	Moreover, by the clause ($a$), $D_{\substack{m \\ n}}^{\lambda_F}(E)$ is the largest one in $Q$, hence $\textit{inf}(\emptyset) = D_{\substack{m \\ n}}^{\lambda_F}(E)$.
	Thus, $Q$ is a complete lattice.
\end{proof}
By Corollary \ref{corollary:properties_defense}($a$), the monotone function $D_n^m$ may be regarded as an endofunction on $\varepsilon_{ad}^{\ell mn}(F)$ whenever $n \geq \ell \geq m$.
Thus, the clause ($c$) in Corollary \ref{corollary:well-founded_co_nonempty} is nothing but a consequence of the Knaster-Tarski fixed point theorem which asserts that all fixed points of monotone endofunction (self-map) on a complete lattice forms a complete lattice in itself (see, e.g., Theorem O-2.3 in \cite{Gierz03Book}). By the way, although $\varepsilon_{co}^{\ell mn}(F) \subseteq \varepsilon_{ad}^{\ell mn}(F)$, $Q$ isn't a sublattice of $P$ in general.
\par In the extreme situation, the structure $\tuple{\varepsilon_{co}^{\ell mn}(F), \subseteq}$ collapses to a singleton set, moreover, the assumption $n \geq \ell \geq m$ in Corollary \ref{corollary:well-founded_co_nonempty} may be relaxed to $\ell \geq m$ and $n \geq m$.
Formally, we have
\begin{corollary} \label{corollary:well-founded_co_empty}
	Let $F = \tuple{\A, \rightarrow}$ be an AAF. If $\rightarrow^+$ is well-founded on $\A - D_{\substack{m \\ n}}^{\lambda_F}(\emptyset)$ then
	\begin{itemize}
		\item[a.] $D_{\substack{m \\ n}}^{\lambda_F}(\emptyset) = D_{\substack{m \\ n}}^{\lambda_F}(X)$ for any $X \in \textit{Def}^{mn}(F)$, and hence $D_{\substack{m \\ n}}^{\lambda_F}(\emptyset)$ is the unique fixed point of $D_n^m$.
	\end{itemize}
	\par Moreover, in the situation that $\ell \geq m$ and $n \geq m$, we also have
	\begin{itemize}
		\item[b.] $\varepsilon_{co}^{\ell mn}(F) = \set{D_{\substack{m \\ n}}^{\lambda_F}(\emptyset)}$, and hence $D_{\substack{m \\ n}}^{\lambda_F}(\emptyset)$ ($= D_{\substack{m \\ n}}^{\lambda_F}(E)$) is the unique $\ell mn$-preferred extension containing $E$ whenever $E \in \varepsilon_{ad}^{\ell mn}(F)$.
		\item[c.] $D_{\substack{m \\ n}}^{\xi}(E) \in \varepsilon_{ad}^{\ell mn}(F)$ for any ordinal $\xi$ and $E \in \varepsilon_{ad}^{\ell mn}(F)$.
		\item[d.] $P \triangleq \tuple{\varepsilon_{ad}^{\ell mn}(F), \subseteq}$ is a complete lattice with the top $D_{\substack{m \\ n}}^{\lambda_F}(\emptyset)$.
	\end{itemize}
\end{corollary}
\begin{proof}
	($a$) Let $X \in \textit{Def}^{mn}(F)$. Clearly, by Lemma \ref{lemma:Grossi_simple}($b$), $D_{\substack{m \\ n}}^{\lambda_F}(\emptyset) \subseteq D_{\substack{m \\ n}}^{\lambda_F}(X)$. 
	Further, since $D_{\substack{m \\ n}}^{\lambda_F}(X)$ is a fixed point of the function $D_n^m$, by Lemma \ref{lemma:transitive closure}($c$), we get $D_{\substack{m \\ n}}^{\lambda_F}(\emptyset) = D_{\substack{m \\ n}}^{\lambda_F}(X)$.
	\par ($b$) Immediately follows from Proposition \ref{proposition:structure_gr} and the clause ($a$).
	\par ($c$) Let $\xi$ be any ordinal and $E \in \varepsilon_{ad}^{\ell mn}(F)$. Since $D_{\substack{m \\ n}}^{\lambda_F}(E)$ is a fixed point containing $E$, by the clause ($a$), $E \subseteq D_{\substack{m \\ n}}^{\lambda_F}(\emptyset)$. Then, by Corollary \ref{corollary:construction_co}, $D_{\substack{m \\ n}}^{\xi}(E) \in \varepsilon_{ad}^{\ell mn}(F)$.
	\par ($d$) For each $E \in \varepsilon_{ad}^{\ell mn}(F)$, by the clauses ($a$) and ($c$), we have $E \subseteq D_{\substack{m \\ n}}^{\lambda_F}(E) = D_{\substack{m \\ n}}^{\lambda_F}(\emptyset) \in \varepsilon_{ad}^{\ell mn}(F)$, and hence $D_{\substack{m \\ n}}^{\lambda_F}(\emptyset)$ is the top in $P$.
	Further, by Corollary \ref{corollary:structure_cf_ad}($d$), it is easy to see that $P$ is a complete lattice.
\end{proof}
By Observation \ref{observation:properties_gr}, it is easy to see that the above clauses ($b$), ($c$) and ($d$) still hold if replacing the assumption $\ell \geq m$ and $n \geq m$ by $\varepsilon_{co}^{\ell mn}(F) \neq \emptyset$. In fact, all these three clauses and $\varepsilon_{co}^{\ell mn}(F) \neq \emptyset$ are equivalent to each other in the situation that $\rightarrow^+$ is well-founded on $\A - D_{\substack{m \\ n}}^{\lambda_F}(\emptyset)$. We leave its verification to the reader. By the way, Corollary \ref{corollary:properties_defense}($a$) asserts the same conclusion as the clause $(c)$ in Corollary \ref{corollary:well-founded_co_empty} under the assumption that $n \geq \ell \geq m$. In the situation that $\rightarrow^+$ is well-founded on $\A - D_{\substack{m \\ n}}^{\lambda_F}(\emptyset)$, three facts below ensure that the assumption $n \geq \ell \geq m$ may be relaxed to $\ell \geq m$ and $n \geq m$: 
\begin{itemize}
	\item $D_{\substack{m \\ n}}^{\lambda_F}(\emptyset)$ is $\ell$-conflict-free under the assumption that $\ell \geq m$ and $n \geq m$ (Corollary \ref{corollary:existence_gr_co}).
	\item The well-foundedness of $\rightarrow^+$ on $\A - D_{\substack{m \\ n}}^{\lambda_F}(\emptyset)$ leads to that any iteration of $D_n^m$ starting at any $X \in \textit{Def}^{mn}(F)$ can't exceed the set $D_{\substack{m \\ n}}^{\lambda_F}(\emptyset)$ (Lemma \ref{lemma:transitive closure}($c$)).
	\item $\varepsilon_{cf}^{\ell}(F)$ is down-closed (Lemma \ref{lemma:properties_cf_def}($d$)).
\end{itemize}
\par For any AAF having certain well-foundedness, its admissible sets agree with self-defended sets whenever the parameters $\ell$, $m$ and $n$ are subject to some conditions. Formally, we have
\begin{corollary} \label{corollary:equ_ad_def}
	For any AAF $F = \tuple{\A, \rightarrow}$, under the assumption of Corollary \ref{corollary:well-founded_co_nonempty} or Corollary \ref{corollary:well-founded_co_empty}, it holds that $\varepsilon_{ad}^{\ell mn}(F) = \textit{Def}^{mn}(F)$.
\end{corollary}
\begin{proof}
	It suffices to show that $\textit{Def}^{mn}(F) \subseteq \varepsilon_{ad}^{\ell mn}(F)$.
	Let $X \in \textit{Def}^{mn}(F)$. By Lemma \ref{lemma:transitive closure}($c$) and Corollary \ref{corollary:well-founded_co_nonempty}($a$) (Corollary \ref{corollary:well-founded_co_empty}($b$), resp.), the largest fixed point $\nu D_n^m$ of $D_n^m$ in $\tuple{\pw{\A}, \subseteq}$ is $\ell$-conflict-free.
	Since $X$ is a post fixed point of $D_n^m$, $X \subseteq \nu D_n^m$. Then, $X \in \varepsilon_{cf}^{\ell}(F)$ due to Lemma \ref{lemma:properties_cf_def}($d$), and hence $X \in \varepsilon_{ad}^{\ell mn}(F)$.
\end{proof} 
Clearly, for all well-founded AAF $F = \tuple{\A, \rightarrow}$, $\rightarrow^+$ is well-founded on any subset of $\A$. Thus, all results obtained in this subsection always hold for $F$. These results together with related ones in the next subsection will be summarized at the end of this section.
\subsection{Some comments on Grossi and Modgil's work}
This subsection will comment on relevant results obtained in \cite{Grossi19Graded}, which are listed below. Notice that $\bigcup_{0 \leq k < \omega} D_{\substack{m \\ n}}^{k}(E)$ is exactly $D_{\substack{m \\ n}}^{\omega}(E)$.
\par \emph{(I) (Theorem 2 in \cite{Grossi19Graded}) Let $F = \tuple{\A, \rightarrow}$ be a finitary AAF, $E \in \varepsilon_{ad}^{\ell mn}(F)$, $\ell \geq m$ and $n \geq m$. Then $\bigcup_{0 \leq k < \omega} D_{\substack{m \\ n}}^{k}(E)$ is the smallest $\ell mn$-complete extension containing $E$.}
\par \emph{(II) (Fact 10 in \cite{Grossi19Graded}) Let $F = \tuple{\A, \rightarrow}$ be a finitary AAF, $E \in \varepsilon_{ad}^{\ell mn}(F)$, $\ell \geq m$ and $n \geq m$. If $\A \subseteq \Sigma_E \triangleq \set{a \mid b \rightarrow^+ a \text{ for some } b \in E}$ then $\bigcup_{0 \leq k < \omega} D_{\substack{m \\ n}}^{k}(E)$ is the $\ell mn$-preferred extension containing $E$.}
\par \emph{(III) (Corollary 1 in \cite{Grossi19Graded}) Let $F = \tuple{\A, \rightarrow}$ be a finitary AAF, $\ell \geq m$ and $n \geq m$. Then $\bigcup_{0 \leq k < \omega} D_{\substack{m \\ n}}^{k}(\emptyset)$ is the $\ell mn$-grounded extension of $F$.}
\par \emph{(IV) (Fact 11 in \cite{Grossi19Graded}) Let $F = \tuple{\A, \rightarrow}$ be a finitary AAF, $E \in \varepsilon_{ad}^{\ell mn}(F)$, $\ell \geq m$ and $n \geq m$. Then $\bigcup_{0 \leq k < \omega} D_{\substack{m \\ n}}^{k}(E)$ is the smallest $\ell mn$-stable extension containing $E$ iff $\bigcup_{0 \leq k < \omega} D_{\substack{m \\ n}}^{k}(E) = \bigcap_{0 \leq k < \omega} N_{n}(D_{\substack{m \\ n}}^{k}(E))$.}
\par \emph{(V) (Fact 12 in \cite{Grossi19Graded}) $\varepsilon_{stb}^{\ell mn}(F) \subseteq \varepsilon_{pr}^{\ell mn}(F)$ for each AAF $F$.}\\
\par All these interesting results except (V) aim to provide fixpoint constructions for some special extensions based on iterations of the function $D_n^m$, which are the main contributions of \cite{Grossi19Graded}. Among them, (I) and (III) are generalized versions of results obtained in \cite{Grossi15Graded} by relaxing the assumption $n \geq \ell = m$ to $\ell \geq m$ and $n \geq m$. All these results except (V) depend on the following fundamental lemma (Lemma 3 in \cite{Grossi19Graded})\footnote{ For the sake of convenience in discussing, we adopt the assumption $\ell = m$ instead of, as done in \cite{Grossi19Graded}, omitting the parameter $\ell$ and writing $m$ where $\ell$ occurs in the inequality (\ref{equation:GrossiFunLemma}) and the notation $\varepsilon_{ad}^{\ell mn}$.}:\\
\par \emph{
	Let $F = \tuple{\A, \rightarrow}$ be a finitary AAF, $E \in \varepsilon_{ad}^{\ell mn}(F)$ and $n \geq \ell = m$. Then for any $k$ such that $0 \leq k < \omega$,
	\begin{equation}
	E \subseteq D_{\substack{m \\ n}}^{k}(E) \subseteq N_{\ell}(D_{\substack{m \\ n}}^{k}(E)) \subseteq N_{\ell}(E). \label{equation:GrossiFunLemma}
	\end{equation}
}\\
Unfortunately, in the journal version \cite{Grossi19Graded}, such a fundamental lemma with the constraint $n \geq \ell = m$ is applied incorrectly to prove (I) and (II) with the assumptions $\ell \geq m$ and $n \geq m$. 
The following two examples reveal that neither (I) nor (II) always holds. 
By the way, the careful reader may find a common nature of these two examples:  attack circulars occur in both AAFs involved in Examples \ref{Ex:counterexampleforGrossiTh2} and \ref{Ex:counterexampleforpr1}.
It is not accidental, in fact, both (I) and (II) indeed hold for all well-founded AAFs, see Corollary \ref{corollary:well-founded_co_empty}($b$).
\begin{example}[Counterexample for I] \label{Ex:counterexampleforGrossiTh2}
	Consider the AAF in Figure \ref{figure:counterexample} again. For $\ell = 3$, $m = n = 2$ and $E_1 = \set{a, b, c, d, e}$, we have showed that $E_1 \in \varepsilon_{ad}^{\ell mn}(F)$ and $D_n^m(E_1) = \set{a, b, c, d, e, f}$ in Example \ref{Ex:counterexample}.
	Moreover, it is easy to check that $D_n^m(E_1)$ is a fixed point of $D_n^m$. Hence,
	\begin{eqnarray*}
		D_{\substack{m \\ n}}^{\omega}(E_1) = D_n^m(E_1) = \set{a, b, c, d, e, f}.
	\end{eqnarray*}
	However, since $f \notin N_\ell(D_n^m(E_1))$, we have $D_n^m(E_1) \nsubseteq N_\ell(D_n^m(E_1))$, and hence $D_n^m(E_1) \notin \varepsilon_{cf}^{\ell}(F)$.
	Thus, $D_{\substack{m \\ n}}^{\omega}(E_1)$ is not an $\ell mn$-complete extension at all.
\end{example}
Clearly, under the assumption that $F$ is finitary, since the function $D_n^m$ is continuous, all results obtained so far still hold if we replace $\lambda_F$ by $\omega$.
In particular, by Proposition \ref{proposition:structure_gr} and Corollaries \ref{corollary:properties_defense}($b$) and \ref{corollary:construction_co}, (I) may be corrected as below, which distinguishes two cases based on whether a given $\ell mn$-admissible set $E$ can be exceeded by the iteration of $D_n^m$ starting at $\emptyset$.
\begin{corollary} \label{corollary:lfp_omega}
	Let $F$ be a finitary AAF.
	\begin{itemize}
		\item[a.] $D_{\substack{m \\ n}}^{\omega}(E)$ is the least $\ell mn$-complete extension containing $E$ for any $\ell mn$-admissible set $E$ whenever $n \geq \ell \geq m$.
		\item[b.] $D_{\substack{m \\ n}}^{\omega}(E)$ is the least $\ell mn$-complete extension containing $E$ for any $\ell mn$-admissible set $E$ such that $E \subseteq D_{\substack{m \\ n}}^{\omega}(\emptyset)$ whenever $\ell \geq m$ and $n \geq m$. In fact, in this situation, $D_{\substack{m \\ n}}^{\omega}(E) = D_{\substack{m \\ n}}^{\omega}(\emptyset)$ is the least $\ell mn$-complete extension of $F$.
	\end{itemize}	
\end{corollary}
\begin{example}[Counterexample for II]\label{Ex:counterexampleforpr1}
	\begin{figure}[t]
		\centering
		\includegraphics[width=0.4\linewidth]{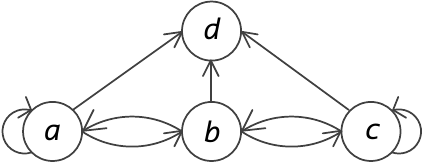}
		\caption{The AAF for Example \ref{Ex:counterexampleforpr1}.}
		\label{figure:counterexampleforpr}
	\end{figure}
	Let $\ell = 3$ and $m = n = 2$.
	Consider the AAF $F$ in Figure \ref{figure:counterexampleforpr} and $E = \set{a, b, c}$.
	Clearly, $\A \subseteq \Sigma_E$.
	Moreover, it is easy to check that $D_n^m(E) = D_n^m(D_n^m(E)) = \set{a, b, c, d}$, and hence $D_{\substack{m \\ n}}^{\omega}(E) = \set{a, b, c, d}$.	
	Since each argument in $E$ is attacked by only two arguments in $E$, we obtain $E \subseteq N_\ell(E)$, which, together with $E = \set{a, b, c} \subseteq \set{a, b, c, d} = D_n^m(E)$, implies $E \in \varepsilon_{ad}^{\ell mn}(F)$. Moreover, $d \notin N_\ell(D_n^m(E))$ because that there are three arguments $a$, $b$ and $c$ attacking $d$ in $D_n^m(E)$. 
	Thus, we have $D_n^m(E) \nsubseteq N_\ell(D_n^m(E))$, and hence $D_{\substack{m \\ n}}^{\omega}(E) = D_n^m(E) \notin \varepsilon_{cf}^{\ell}(F)$. 
	So, $D_{\substack{m \\ n}}^{\omega}(E)$ is not an $\ell mn$-complete extension of $F$ at all, and hence $D_{\substack{m \\ n}}^{\omega}(E) \notin \varepsilon_{pr}^{\ell mn}(F)$.
\end{example}
Corollaries \ref{corollary:well-founded_co_nonempty}($a$) and \ref{corollary:well-founded_co_empty}($b$) deal with the same issue as (II) under the assumption of well-foundedness instead of reachability (i.e., $\A \subseteq \Sigma_E$). It is not difficult to see that the reachability in (II) and well-foundedness are incomparable, that is, none implies the other. Example \ref{Ex:counterexampleforpr1} illustrates that the assumption $\ell \geq m$ and $n \geq m$ is too weak to preserve $\ell$-conflict-freeness. At this point, someone may intend to adopt the reachability instead of well-foundedness to show (II) under the assumption $n \geq \ell \geq m$ by induction on the length $d(E,a)$ of the shortest attack path from $E$ to an argument $a$. Although $d(E, a)$ always exists for each argument $a$ because of the reachability, this proof method is not feasible because, for $a \rightarrow b$ in $F$, it is that $d(E,b) \leq d(E,a) + 1$ always holds instead of $d(E,a) < d(E,b)$, which obstructs the inductive step because of lacking an available inductive hypothesis. By contrast, the well-foundedness of $\rightarrow^+$ on $\A - D_{\substack{m \\ n}}^{\lambda_F}(E)$ ensures that there exists a rank function $\rho$ which assigns ordinals to arguments such that $\rho(a) < \rho(b)$ whenever $a \rightarrow^+ b$ with $a, b \in \A - D_{\substack{m \\ n}}^{\lambda_F}(E)$, which is an instantiation of a well-known result about well-founded relations, see, eg., Theorem 2.27 in \cite{Jech03Book}. In fact, the assertion (II) still doesn't hold even if adopting the assumption $n \geq \ell \geq m$ instead of $\ell \geq m$ and $n \geq m$. Inspired by the analysis above, a counterexample for it is given below. The reader may find that $b \rightarrow c$ but $d(\set{a}, b) = d(\set{a}, c) = 1$ in this example.
\begin{example} \label{Ex:counterexampleforpr2}
	Consider the AAF $F$ in Figure \ref{figure:counterexampleforatt} with $\ell = m = n = 2$. It is easy to see that $\set{a} \in \varepsilon_{ad}^{\ell mn}(F)$ and $\Sigma_{\set{a}} = \set{a, b, c}$. Moreover, $\set{a} \in \varepsilon_{co}^{\ell mn}(F)$ due to $D_n^m(\set{a}) = \set{a}$. Hence, $D_{\substack{m \\ n}}^{\omega}(\set{a}) = \set{a}$. However, since $\set{a} \subset \set{a, b} \in \varepsilon_{co}^{\ell mn}(F)$, $D_{\substack{m \\ n}}^{\omega}(\set{a}) \notin \varepsilon_{pr}^{\ell mn}(F)$.
\end{example}
Based on the discussion and example above, it is natural to conjecture that the assertion (II) holds if adding the assumption that $d(E,a) < d(E,b)$ whenever $a \rightarrow b$ and adopting the assumption $n \geq \ell \geq m$. It indeed holds, moreover, the assumption $\ell \geq m$ and $n \geq m$ is sufficient in this situation. Thus, (II) may be corrected as follows, which is implied by Corollary \ref{corollary:well-founded_co_empty}.
\begin{corollary} \label{corollary:reachability}
	Let $F = \tuple{\A, \rightarrow}$ be an AAF, $E \in \varepsilon_{ad}^{\ell mn}(F)$, $\ell \geq m$ and $n \geq m$. If $\A \subseteq \Sigma_E$ and for each $a, b \in \A$, $d(E,a) < d(E,b)$ whenever $a \rightarrow b$ then $D_{\substack{m \\ n}}^{\lambda_F}(E)$ is the unique $\ell mn$-preferred extension containing $E$.
\end{corollary}
\begin{proof}
	Since $\A \subseteq \Sigma_E$ and $d(E,a) < d(E,b)$ whenever $a \rightarrow b$ for all $a, b \in \A$, it isn't difficult to see that $F$ itself is well-founded, otherwise, there exists an infinite decreasing sequence of natural numbers $n_0 > n_1 > n_2 > \cdots > n_i > \cdots (i < \omega)$, which is impossible. Thus, by Corollary \ref{corollary:well-founded_co_empty}, $D_{\substack{m \\ n}}^{\lambda_F}(E)$ is the unique $\ell mn$-preferred extension containing $E$.
\end{proof}
More generally, by Corollaries \ref{corollary:well-founded_co_nonempty} and \ref{corollary:well-founded_co_empty}, Corollary \ref{corollary:reachability} may be strengthened as follows.
\begin{corollary}
	Let $F = \tuple{\A, \rightarrow}$ be an AAF, $E \in \varepsilon_{ad}^{\ell mn}(F)$, $\A - D_{\substack{m \\ n}}^{\lambda_F}(E) \subseteq \Sigma_E$ and for each $a, b \notin D_{\substack{m \\ n}}^{\lambda_F}(E)$, $d(E,a) < d(E,b)$ whenever $a \rightarrow^+ b$. If $n \geq \ell \geq m$ (or, $E \subseteq D_{\substack{m \\ n}}^{\lambda_F}(\emptyset)$, $\ell \geq m$ and $n \geq m$) then $D_{\substack{m \\ n}}^{\lambda_F}(E)$ is the unique $\ell mn$-preferred extension containing $E$.
\end{corollary}
Next we continue to comment on (III), (IV) and (V). In \cite{Grossi19Graded}, (III) is obtained as a corollary of (I). Although (I) doesn't hold, (III) is true in itself. It isn't difficult to see that (III) is implied by Corollary \ref{corollary:lfp_omega}($b$) in this paper. 
\par The result (IV) indeed holds but the assumption $n \geq m$ is superfluous. A proof without using this assumption is given below.
\begin{proposition} \label{proposition:structure_stb}
	Let $F$ be an AAF, $E \in \varepsilon_{ad}^{\ell mn}(F)$ and $\ell \geq m$. Then $D_{\substack{m \\ n}}^{\lambda_F}(E)$ is the least $\ell mn$-stable extension containing $E$ iff $D_{\substack{m \\ n}}^{\lambda_F}(E) = N_n(D_{\substack{m \\ n}}^{\lambda_F}(E))$.
\end{proposition}
\begin{proof}
	We only show that the right implies the left, the other one is trivial.
	Since $E \in \varepsilon_{ad}^{\ell mn}(F)$, $E \subseteq D_{\substack{m \\ n}}^{\lambda_F}(E) = D_n^m(D_{\substack{m \\ n}}^{\lambda_F}(E))$.
	Moreover, due to the assumption $D_{\substack{m \\ n}}^{\lambda_F}(E) = N_n(D_{\substack{m \\ n}}^{\lambda_F}(E))$, by Lemma \ref{lemma:Grossi_simple}(c), we have 
	\begin{eqnarray*}
		D_{\substack{m \\ n}}^{\lambda_F}(E) = D_n^m(D_{\substack{m \\ n}}^{\lambda_F}(E)) = N_m(N_n(D_{\substack{m \\ n}}^{\lambda_F}(E))) = N_m(D_{\substack{m \\ n}}^{\lambda_F}(E)).
	\end{eqnarray*}
	Further, it follows from $\ell \geq m$ that $N_m(D_{\substack{m \\ n}}^{\lambda_F}(E)) \subseteq N_\ell(D_{\substack{m \\ n}}^{\lambda_F}(E))$. 
	Summarizing, we have 
	\begin{eqnarray*}
		N_n(D_{\substack{m \\ n}}^{\lambda_F}(E)) = D_{\substack{m \\ n}}^{\lambda_F}(E) = N_m(D_{\substack{m \\ n}}^{\lambda_F}(E)) \subseteq N_\ell(D_{\substack{m \\ n}}^{\lambda_F}(E)).
	\end{eqnarray*}
	Hence $D_{\substack{m \\ n}}^{\lambda_F}(E) \in \varepsilon_{stb}^{\ell mn}(F)$.
	Further, since $D_{\substack{m \\ n}}^{\lambda_F}(E)$ is the least among fixpoints of $D_n^m$ containing $E$ and each $\ell mn$-stable extension is a fixed point of $D_n^m$, $D_{\substack{m \\ n}}^{\lambda_F}(E)$ is the least $\ell mn$-stable extension containing $E$.
\end{proof}
For finitary AAFs, by Lemma \ref{lemma:Grossi_simple}($d$), the function $D_n^m$ is continuous, and hence the proposition above still holds if replacing $\lambda_F$ by $\omega$. Moreover, since $\set{D_{\substack{m \\ n}}^{k}(E) \mid k < \omega}$ is a chain, $N_n(D_{\substack{m \\ n}}^{\omega}(E)) = \bigcap_{k < \omega} N_n(D_{\substack{m \\ n}}^{k}(E))$ by \textit{directed De Morgan law} in Lemma \ref{lemma:properties_cf_def}.
Thus, the proposition above restricted to finitary AAFs coincides with (IV) except dispensing with the assumption $n \geq m$.
\begin{corollary} \label{corollary:well-founded_stb}
	Let $F = \tuple{\A, \rightarrow}$ be an AAF, $E \in \varepsilon_{ad}^{\ell mn}(F)$ and $\ell = m = n$.
	If $\rightarrow^+$ is well-founded on $\A - D_{\substack{m \\ n}}^{\lambda_F}(E)$ then $\varepsilon_{stb}^{\ell mn}(F) = \set{D_{\substack{m \\ n}}^{\lambda_F}(E)}$.
\end{corollary}
\begin{proof}
	Since $n = m$ and $D_{\substack{m \\ n}}^{\lambda_F}(E)$ is a fixed point of $D_n^m$, by Lemma \ref{lemma:Grossi_simple}($c$), we have 
	\begin{eqnarray*}
		N_n(D_{\substack{m \\ n}}^{\lambda_F}(E)) 
		\!=\! N_n(D_n^m(D_{\substack{m \\ n}}^{\lambda_F}\!(E)\!))
		\!=\! N_n(N_m(N_n(D_{\substack{m \\ n}}^{\lambda_F}\!(E)\!)))
		\!=\! D_n^m(N_n(D_{\substack{m \\ n}}^{\lambda_F}(E))).
	\end{eqnarray*}
	Hence 
	\begin{eqnarray}
		N_n(D_{\substack{m \\ n}}^{\lambda_F}(E)) \in \textit{Def}^{mn}(F).  \label{formula:cor_stb1}
	\end{eqnarray}
	Moreover, since $\ell = m = n$, by Corollary \ref{corollary:properties_defense}($b$), $D_{\substack{m \\ n}}^{\lambda_F}(E) \in \varepsilon_{co}^{\ell mn}(F)$, and hence $D_{\substack{m \\ n}}^{\lambda_F}(E) \subseteq N_\ell(D_{\substack{m \\ n}}^{\lambda_F}(E)) = N_n(D_{\substack{m \\ n}}^{\lambda_F}(E))$.
	Further, by (\ref{formula:cor_stb1}) and Lemma \ref{lemma:transitive closure}($b$), we get $D_{\substack{m \\ n}}^{\lambda_F}(E) = N_n(D_{\substack{m \\ n}}^{\lambda_F}(E))$.
	Thus, by Proposition \ref{proposition:structure_stb} and Corollary \ref{corollary:well-founded_co_nonempty}($a$), $D_{\substack{m \\ n}}^{\lambda_F}(E)$ is the largest $\ell mn$-stable extension of $F$.
	Let $E' \in \varepsilon_{stb}^{\ell mn}(F)$.
	Then, $E' \subseteq D_{\substack{m \\ n}}^{\lambda_F}(E)$, and hence, by Definition \ref{definition:graded_extensions} and Lemma \ref{lemma:Grossi_simple}($a$), $D_{\substack{m \\ n}}^{\lambda_F}(E) = N_n(D_{\substack{m \\ n}}^{\lambda_F}(E)) \subseteq N_n(E') = E'$.
	Thus, $D_{\substack{m \\ n}}^{\lambda_F}(E)$ is the unique element in $\varepsilon_{stb}^{\ell mn}(F)$.
\end{proof}
Finally, we consider the assertion (V). It is well known that $\varepsilon_{stb}^{111}(F) \subseteq \varepsilon_{pr}^{111}(F)$ for each AAF $F$ \cite{Dung95Acceptability}. In \cite{Grossi19Graded}, the same assertion is also given in the graded situation (see, Fact 12 in \cite{Grossi19Graded}). Unfortunately, it doesn't hold in the graded case.
\begin{example}[Counterexample for V] \label{Ex: Counterexample for V}
	Consider the AAF $F$ in Figure \ref{figure:counterexampleforatt} again with $\ell = 3$ and $m = n = 2$. It is easy to verify that $\set{a, b} \in \varepsilon_{stb}^{\ell mn}(F)$. However, $\set{a, b} \notin \varepsilon_{pr}^{\ell mn}(F)$ due to $\set{a, b} \subset \set{a, b, c} \in \varepsilon_{co}^{\ell mn}(F)$.
\end{example}
The assertion (V) may be corrected as below.
\begin{observation}
	For any AAF $F$, $\varepsilon_{stb}^{\ell mn}(F) \subseteq \varepsilon_{pr}^{\ell mn}(F)$ whenever $\ell \leq m$ or $\ell \leq n$.
\end{observation}
\begin{proof}
	Let $E \in \varepsilon_{stb}^{\ell mn}(F)$. Then $E \in \varepsilon_{co}^{\ell mn}(F)$. Assume $E \subseteq E'$ with $E' \in \varepsilon_{co}^{\ell mn}(F)$. Then, by the clause ($a$) in Lemma \ref{lemma:Grossi_simple}, $N_\ell(E') \subseteq N_\ell(E)$. Moreover, since $E = N_n(E) = N_m(E)$, $N_\ell(E) \subseteq E$ due to $\ell \leq m$ or $\ell \leq n$. Thus $E' \subseteq N_\ell(E') \subseteq N_\ell(E) \subseteq E$ due to $E' \in \varepsilon_{co}^{\ell mn}(F)$. Consequently, $E = E'$, as desired.
\end{proof}
Moreover, by Lemma \ref{lemma:transitive closure}, it is easy to see that for any $E \in \varepsilon_{stb}^{\ell mn}(F)$, $E \in \varepsilon_{pr}^{\ell mn}(F)$ whenever $\rightarrow^+$ is well-founded on $\A - E$. Thus, we always have $\varepsilon_{stb}^{\ell mn}(F) \subseteq \varepsilon_{pr}^{\ell mn}(F)$ for all well-founded AAF $F$. Since $\rightarrow^+$ is locally well-founded on any subset $X$ of $\A$ whenever a given AAF $\tuple{\A, \rightarrow}$ is well-founded, a number of properties of well-founded AAFs immediately follow from ones of local well-founded AAFs, which are summarized below. Among these, the clauses ($a$) and ($b$) together provide a graded variant of the one due to Dung \cite[Theorem 30]{Dung95Acceptability}.
\begin{corollary} \label{corollary:well-founded_AAF_stb}
	Let $F = \tuple{\A, \rightarrow}$ be a well-founded AAF. 
	\begin{itemize}
		\item[a.] $\varepsilon_{co}^{\ell mn}(F) = \varepsilon_{pr, \textit{Dung}}^{\ell mn}(F) = \varepsilon_{gr, \textit{Dung}}^{\ell mn}(F) = \set{D_{\substack{m \\ n}}^{\lambda_F}(\emptyset)}$ whenever $\ell \geq m$ and $n \geq m$.
		\item[b.] $\varepsilon_{stb}^{\ell mn}(F) = \set{D_{\substack{m \\ n}}^{\lambda_F}(\emptyset)}$ whenever $\ell = m = n$.
		\item[c.] $\varepsilon_{ad}^{\ell mn}(F) = \textit{Def}^{mn}(F)$ whenever $\ell \geq m$ and $n \geq m$.
		\item[d.] $\varepsilon_{co}^{\ell mn}(F) \neq \emptyset$ iff $D_{\substack{m \\ n}}^{\xi}(E) \in \varepsilon_{ad}^{\ell mn}(F)$ for any $E \in \varepsilon_{ad}^{\ell mn}(F)$ and ordinal $\xi$.
		\item[e.] $\varepsilon_{stb}^{\ell mn}(F) \subseteq \varepsilon_{pr}^{\ell mn}(F)$.
	\end{itemize}
\end{corollary}
\begin{proof}
	The clauses ($a$), ($b$) and ($c$) follow from Corollaries \ref{corollary:well-founded_co_empty}, \ref{corollary:equ_ad_def} and \ref{corollary:well-founded_stb}, and the clause ($e$) comes from Lemma \ref{lemma:transitive closure}. For ($d$), it is obvious that the right implies the left, another direction comes from Observations \ref{observation:D_ordinal} and \ref{observation:properties_gr}, Lemma \ref{lemma:transitive closure} and the clause ($d$) in Lemma \ref{lemma:properties_cf_def}.
\end{proof}
\begin{example} \label{Ex:counterexampleforwell-founded}
	Without the assumptions about parameters $\ell$, $m$ and $n$, none of the clauses ($a$), ($b$) and ($c$) in Corollary \ref{corollary:well-founded_AAF_stb} always holds.
	\begin{figure}[t]
		\setlength{\abovecaptionskip}{-0.2cm}
		\begin{center}
			\includegraphics[scale=1]{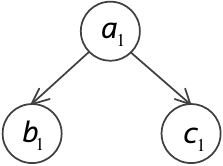}
			\hspace{0.5cm}
			\includegraphics[scale=1]{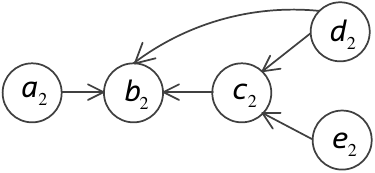}
		\end{center}
		\caption{The AAFs for Example \ref{Ex:counterexampleforwell-founded}. }
		\label{counterexampleforwell-founded}
	\end{figure}
	\par First, consider the left AAF $F$ in Figure \ref{counterexampleforwell-founded}.
	For $\ell = 1$ and $m = n = 4$, it is easy to check that 
	\begin{eqnarray*}
		&& D_n^m(X) = \set{a_1, b_1, c_1} \text{~for each~} X \subseteq \set{a_1, b_1, c_1},\\
		&& \varepsilon_{cf}^{\ell}(F) = \varepsilon_{ad}^{\ell mn}(F) = \set{\emptyset, \set{a_1}, \set{b_1}, \set{c_1}, \set{b_1, c_1}} \text{~and~}\\
		&& \varepsilon_{pr, \textit{Dung}}^{\ell mn}(F) = \set{\set{a_1}, \set{b_1, c_1}}.
	\end{eqnarray*}
	Clearly, none in $\varepsilon_{cf}^{\ell}(F)$ is a fixed point of $D_n^m$. Hence $\varepsilon_{co}^{\ell mn}(F) = \varepsilon_{stb}^{\ell mn}(F) = \emptyset$. Moreover, since $D_n^m(\set{a_1, b_1, c_1}) = \set{a_1, b_1, c_1}$, $\set{a_1, b_1, c_1} \in \textit{Def}^{mn}(F)$, and hence $\varepsilon_{ad}^{\ell mn}(F) \neq \textit{Def}^{mn}(F)$ and $\varepsilon_{gr, \textit{Dung}}^{\ell mn}(F) = \set{\set{a_1, b_1, c_1}} \neq \varepsilon_{co}^{\ell mn}(F)$. 
	\par Second, consider the right AAF $F$ in Figure \ref{counterexampleforwell-founded} with $\ell = m = 3$ and $n = 2$. Since $D_{\substack{m \\ n}}^{3}(\emptyset) = D_{\substack{m \\ n}}^{2}(\emptyset) = \set{a_2, b_2, c_2, d_2, e_2}$, $D_{\substack{m \\ n}}^{\lambda_F}(\emptyset) = \set{a_2, b_2, c_2, d_2, e_2} \in \textit{Def}^{mn}(F)$. However, $b_2 \notin N_\ell(D_{\substack{m \\ n}}^{\lambda_F}(\emptyset))$ because that $b_2$ is attacked by three arguments $a_2$, $c_2$ and $d_2$ in $D_{\substack{m \\ n}}^{\lambda_F}(\emptyset)$, which, together with $b_2 \in D_{\substack{m \\ n}}^{\lambda_F}(\emptyset)$, implies that $D_{\substack{m \\ n}}^{\lambda_F}(\emptyset) \notin \varepsilon_{ad}^{\ell mn}(F)$, and hence $\varepsilon_{ad}^{\ell mn}(F) \neq \textit{Def}^{mn}(F)$. Moreover, it is easy to verify that $\varepsilon_{pr, \textit{Dung}}^{\ell mn}(F) = \set{\set{a_2, b_2, d_2, e_2}, \set{a_2, c_2, d_2, e_2}, \set{b_2, c_2, d_2, e_2}}$ and $\varepsilon_{co}^{\ell mn}(F) = \varepsilon_{stb}^{\ell mn}(F) = \emptyset$.
	Obviously, none of $\varepsilon_{co}^{\ell mn}(F)$, $\varepsilon_{stb}^{\ell mn}(F)$ and $\varepsilon_{pr, \textit{Dung}}^{\ell mn}(F)$ is equal to $\varepsilon_{gr, \textit{Dung}}^{\ell mn}(F) (= \set{D_{\substack{m \\ n}}^{\lambda_F}(\emptyset)})$.
\end{example}
\section{Graded range related semantics and the operator $\bigcap_D$} \label{Sec: Graded range related semantics}
Extension-based semantics (e.g., the ones in Definition \ref{definition:graded_extensions}) often pay close attention to acceptable arguments but lose sight of arguments being in other status (rejected or undecided), while labelling-based semantics, another kind of semantics for AAFs (see, e.g. \cite{Caminada06Labelling}), intend to describe the status of all arguments. Range related semantics is proposed firstly in \cite{Verheij1996Stage} to model the stages of the argumentation process, where a range is a set of arguments that are accepted or rejected. The motivation behind this semantics lies in capturing extensions that allow us to minimize arguments being in the undecided status that is the most controversial status of an argument \cite{Thimm2014Graduality}. To this end, a parameter $\eta$ is involved, which is used to provide a quantitative criterion for assigning the rejected-status to arguments. This section will devote itself to considering a family of graded range related semantics and the operator called reduced meet modulo an ultrafilter.
\subsection{Basic concepts and properties}
\begin{notation} \label{notation:attack}
	Let $F = \tuple{\A, \rightarrow}$ be an AAF, $a \in \A$ and $E \subseteq \A$. 
	\begin{itemize}
		\item $E_\eta^+ \triangleq \set{a \in \A \mid \existsn{b \in E}\eta(b \rightarrow a)}$.
		\item $a^- \triangleq \set{b \in \A \mid b \rightarrow a}$.
		\item $[a^-]^\eta \triangleq \set{X \mid X \subseteq a^- \text{~with~} |X| = \eta}$.
	\end{itemize}                             
\end{notation}
Intuitively, $E_\eta^+$ is the set of all rejected arguments if arguments in $E$ are accepted. A trivial but useful property about the operator $(\cdot)_\eta^+$ is given below.
\begin{observation}\label{observation:eta}
	If $\eta \geq \ell$ and $E \in \varepsilon_{cf}^\ell(F)$ then $E \cap E_\eta'^+ = \emptyset$ for any $E'$ with $E' \subseteq E$.
\end{observation}
\begin{proof}
	Trivially.
\end{proof}
Given an extension-based semantics, a corresponding range related semantics may be derived in terms of the maximality and the operator $(\cdot)_\eta^+$. A family of range related semantics are defined below.
\begin{definition} \label{definition:range related extensions}
	Let $F = \tuple{\A, \rightarrow}$ be an AAF and $E \subseteq \A$. 
	\begin{itemize}
		\item $E$ is an $\ell$-$\eta$-stage extension of $F$, i.e., $E \in \varepsilon_{stg}^{\ell \eta}(F)$, iff $E \in \varepsilon_{cf}^{\ell}(F)$ and $\nexists E' \in \varepsilon_{cf}^{\ell}(F)(E \cup E_\eta^+ \subset E' \cup E_\eta'^+)$.
		\item $E$ is an $\ell mn$-$\eta$-semi-stable extension of $F$, i.e., $E \in \varepsilon_{ss}^{\ell mn \eta}(F)$, iff $E \in \varepsilon_{co}^{\ell mn}(F)$ and $\nexists E' \in \varepsilon_{co}^{\ell mn}(F)(E \cup E_\eta^+ \subset E' \cup E_\eta'^+)$. 
		\item $E$ is an $\ell mn$-$\eta$-range related admissible set of $F$, i.e., $E \in \varepsilon_{rra}^{\ell mn \eta}(F)$, iff $E \in \varepsilon_{ad}^{\ell mn}(F)$ and $\nexists E' \in \varepsilon_{ad}^{\ell mn}(F)(E \cup E_\eta^+ \subset E' \cup E_\eta'^+)$.
		\item $E$ is an $\ell mn$-$\eta$-range related stable extension of $F$, i.e., $E \in \varepsilon_{rrs}^{\ell mn \eta}(F)$, iff $E \in \varepsilon_{stb}^{\ell mn}(F)$ and $\nexists E' \in \varepsilon_{stb}^{\ell mn}(F)(E \cup E_\eta^+ \subset E' \cup E_\eta'^+)$. 
	\end{itemize}
\end{definition}
In addition to these concrete range related semantics, this paper also adopt the notation $\varepsilon_{rr \sigma}^\eta$ to denote the range related semantics associated with an extension-based abstract semantics $\varepsilon_{\sigma}$. The semantics $\varepsilon_{stg}^{\ell \eta}$ and $\varepsilon_{ss}^{\ell mn \eta}$ are graded variants of stage and semi-stable semantics introduced in \cite{Verheij1996Stage,Caminada2012Semistable} respectively. In the non-graded situation (i.e., $\ell = m = n = \eta = 1$), $\varepsilon_{rra}^{\ell mn \eta}$ coincides with $\varepsilon_{ss}^{\ell mn \eta}$ \cite{Caminada2012Semistable}, and $\varepsilon_{rrs}^{\ell mn \eta}$ degenerates into $\varepsilon_{stb}^{\ell mn}$. These observations may be generalized as the next two propositions. However, unlike the non-graded case, $\varepsilon_{rra}^{\ell mn \eta}$ (or, $\varepsilon_{rrs}^{\ell mn \eta}$) doesn't always agree with $\varepsilon_{ss}^{\ell mn \eta}$ ($\varepsilon_{stb}^{\ell mn}$, resp.), see Examples \ref{Ex:counterexampleforrra} and \ref{Ex:counterexampleforrrs}. Thus, in the graded situation, both $\varepsilon_{rra}^{\ell mn \eta}$ and $\varepsilon_{rrs}^{\ell mn \eta}$ deserve to be considered in themselves.
\begin{proposition} \label{proposition:equ_ss and rra}
	Let $F$ be any AAF and $n \geq \ell \geq m$.
	\begin{itemize}
			\item[a.] $\varepsilon_{ss}^{\ell mn \eta}(F) \supseteq \varepsilon_{rra}^{\ell mn \eta}(F)$ whenever $\eta \geq \ell$.
			\item[b.] $\varepsilon_{ss}^{\ell mn \eta}(F) \subseteq \varepsilon_{rra}^{\ell mn \eta}(F)$.
		\end{itemize}
\end{proposition}
\begin{proof}
	($a$) 
	Let $E \in \varepsilon_{rra}^{\ell mn \eta}(F)$. 
	Since $\varepsilon_{co}^{\ell mn}(F) \subseteq \varepsilon_{ad}^{\ell mn}(F)$, it is easy to see that $E \in \varepsilon_{ss}^{\ell mn \eta}(F)$ whenever $E \in \varepsilon_{co}^{\ell mn}(F)$.
	Thus, it suffices to show $E \in \varepsilon_{co}^{\ell mn}(F)$.
	By Definition \ref{definition:graded_extensions}, it remains to prove that $E = D_n^m(E)$.
	Since $E \in \varepsilon_{ad}^{\ell mn}(F)$ and $n \geq \ell \geq m$, by Corollary \ref{corollary:properties_defense}($a$), we have $D_n^m(E) \in \varepsilon_{ad}^{\ell mn}(F)$.
	Moreover, it follows from $E \subseteq D_n^m(E)$ that $E_\eta^+ \subseteq D_n^m(E)_\eta^+$.
	Then, by Definition \ref{definition:range related extensions} and $E \in \varepsilon_{rra}^{\ell mn \eta}(F)$, we have
	\begin{eqnarray}
	E \cup E_\eta^+ = D_n^m(E) \cup D_n^m(E)_\eta^+ \label{formula:range1}
	\end{eqnarray}
	On the other hand, since $\eta \geq \ell$, by Observation \ref{observation:eta}, $D_n^m(E) \cap D_n^m(E)_\eta^+ = \emptyset$.
	Further, it follows immediately from (\ref{formula:range1}) that $E = D_n^m(E)$ due to $E \subseteq D_n^m(E)$ and $E_\eta^+ \subseteq D_n^m(E)_\eta^+$ , as desired.
	\par ($b$) 
	Let $E \in \varepsilon_{ss}^{\ell mn \eta}(F)$.
	Then, $E \in \varepsilon_{ad}^{\ell mn}(F)$.
	Suppose $E \cup E_\eta^+ \subseteq E' \cup E_\eta'^+$ with $E' \in \varepsilon_{ad}^{\ell mn}(F)$.
	Since $n \geq \ell \geq m$, by Corollary \ref{corollary:properties_defense}($b$), $E' \subseteq E''$ for some $E'' \in \varepsilon_{co}^{\ell mn}(F)$.
	Thus,
	\begin{eqnarray*}
		E \cup E_\eta^+ \subseteq E' \cup E_\eta'^+ \subseteq E'' \cup E_\eta''^{+}.
	\end{eqnarray*}
	Then, by $E \in \varepsilon_{ss}^{\ell mn \eta}(F)$ and Definition \ref{definition:range related extensions}, we have $E \cup E_\eta^+ = E'' \cup E_\eta''^{+}$, and hence $E \cup E_\eta^+ = E' \cup E_\eta'^{+}$.
	Thus, $E \in \varepsilon_{rra}^{\ell mn \eta}(F)$. 
\end{proof}
For all well-founded AAFs, the assumption $n \geq \ell \geq m$ in the proposition above can be relaxed to $n \geq m$ and $\ell \geq m$, see Corollary \ref{corollary:well-founded_ss_rrs_empty}.
However, even if $\eta \geq \ell \geq m$ and $n \geq m$, $\varepsilon_{rra}^{\ell mn \eta}$ doesn't always coincide with $\varepsilon_{ss}^{\ell mn \eta}$ as shown by the example below.
\begin{example} \label{Ex:counterexampleforrra}
	Consider the AAF $F$ in Figure \ref{figure:counterexample} again. For $\ell = \eta = 3$ and $m = n = 2$, it is easy to check that
	\begin{eqnarray*}
		\varepsilon_{rra}^{\ell mn \eta}(F) = \set{\set{a, b, c, d, e}} \text{~and~} \varepsilon_{ss}^{\ell mn \eta}(F) = \set{\set{a, c, e}, \set{a, b, d, e}}.
	\end{eqnarray*}
	Thus, in this situation, $\varepsilon_{rra}^{\ell mn \eta}(F) \neq \varepsilon_{ss}^{\ell mn \eta}(F)$.
\end{example}
\begin{proposition} \label{proposition:relations_stb_rrs_stg_rra_ss}
	Let $F = \tuple{\A, \rightarrow}$ be an AAF. If $\eta \leq n$ or $\eta \leq m$ then
	\begin{eqnarray*}
		\varepsilon_{stb}^{\ell mn}(F) = \varepsilon_{rrs}^{\ell mn \eta}(F) \subseteq \varepsilon_{stg}^{\ell \eta}(F) \cap \varepsilon_{rra}^{\ell mn \eta}(F) \cap \varepsilon_{ss}^{\ell mn \eta}(F).
	\end{eqnarray*}
\end{proposition}
\begin{proof}
	Let $E \in \varepsilon_{stb}^{\ell mn}(F)$. Then $E = N_n(E) = N_m(E)$, and hence $E \supseteq N_\eta(E)$ due to $\eta \leq n$ or $\eta \leq m$.
	Thus, $\A - E \subseteq A - N_\eta(E) = E_\eta^+$.
	Hence, $E \cup E_\eta^+ \supseteq E \cup (\A - E) = \A$.
	Thus, $E \cup E_\eta^+ = \A$ due to $E \cup E_\eta^+ \subseteq \A$, which implies $E \in \varepsilon_{rrs}^{\ell mn \eta}(F)$.
	So, $\varepsilon_{stb}^{\ell mn}(F) = \varepsilon_{rrs}^{\ell mn \eta}(F)$, and 
	\begin{eqnarray*}
		\varepsilon_{rrs}^{\ell mn \eta}(F) \subseteq \varepsilon_{stg}^{\ell \eta}(F) \cap \varepsilon_{rra}^{\ell mn \eta}(F) \cap \varepsilon_{ss}^{\ell mn \eta}(F)
	\end{eqnarray*}
	follows immediately from $\varepsilon_{stb}^{\ell mn}(F) \subseteq \varepsilon_{co}^{\ell mn}(F) \subseteq \varepsilon_{ad}^{\ell mn}(F) \subseteq \varepsilon_{cf}^{\ell}(F)$ and the fact shown above, that is, $E \cup E_\eta^+ = \A$ for each $E \in \varepsilon_{stb}^{\ell mn}(F)$.
\end{proof}
The following two examples reveal that the assumption that $\eta \leq n$ or $\eta \leq m$ in Proposition \ref{proposition:relations_stb_rrs_stg_rra_ss} is necessary. 
\begin{example}\label{Ex:counterexampleforstb}
	\begin{figure}[t]
		\centering
		\includegraphics[width=0.5\linewidth]{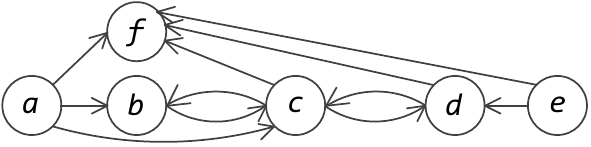}
		\caption{The AAF for Example \ref{Ex:counterexampleforrrs}.}
		\label{figure:counterexampleforrrs}
	\end{figure}
	Consider the AAF $F$ in Figure \ref{figure:counterexampleforpr}. For $\ell = \eta = 3$ and $m = n = 2$, it is easy to check that
	\begin{eqnarray*}
		\begin{split}
			&\varepsilon_{stg}^{\ell \eta}(F) &=& \quad \varepsilon_{rra}^{\ell mn \eta}(F) = \set{\set{a, b, c}},\\
			&\varepsilon_{ss}^{\ell mn \eta}(F) &=& \quad \set{\set{a, b}, \set{a, c}, \set{b, c}} \text{~and}\\
			&\varepsilon_{stb}^{\ell mn}(F) &=& \quad \set{\set{a, c}}.
		\end{split}
	\end{eqnarray*}
	Obviously, $\varepsilon_{stb}^{\ell mn}(F) \nsubseteq \varepsilon_{stg}^{\ell \eta}(F) \cap \varepsilon_{rra}^{\ell mn \eta}(F) \cap \varepsilon_{ss}^{\ell mn \eta}(F)$.
\end{example}
\begin{example}\label{Ex:counterexampleforrrs}
	Consider the AAF $F$ in Figure \ref{figure:counterexampleforrrs}. For $\ell = \eta = 3$ and $m = n = 2$, it is easy to verify that 
	\begin{eqnarray*}
		\varepsilon_{co}^{\ell mn}(F) &=& \set{\set{a, e}, \set{a, c, e}, \set{a, b, d, e}},\\
		N_2(\set{a, e}) &=& \set{a, b, c, d, e},\\
		N_2(\set{a, c, e}) &=& \set{a, c, e} \text{~and~}\\
		N_2(\set{a, b, d, e}) &=& \set{a, b, d, e}.
	\end{eqnarray*}
	Then, by Definition\ref{definition:graded_extensions}, $\varepsilon_{stb}^{\ell mn}(F) = \set{\set{a, c, e}, \set{a, b, d, e}}$.
	Moreover, since
	\begin{eqnarray*}
		\set{a, c, e} \cup \set{a, c, e}_\eta^+ = \set{a, c, e, f} \subset \set{a, b, c, d, e, f} = \set{a, b, d, e} \cup \set{a, b, d, e}_\eta^+,
	\end{eqnarray*}
	$\set{a, c, e} \notin \varepsilon_{rrs}^{\ell mn \eta}(F)$, and thus $\varepsilon_{stb}^{\ell mn}(F) \neq \varepsilon_{rrs}^{\ell mn \eta}(F)$.
\end{example}
As mentioned at the beginning of this section, given an extension $E \in \varepsilon_{cf}^\ell(F)$, $E_\eta^+$ may be regarded as the set of rejected arguments. Intuitively, it is a rational constraint that $E \cap E_\eta^+ = \emptyset$. By Observation \ref{observation:eta}, this holds whenever $\eta \geq \ell$. Moreover, in the view of labelling-based semantics \cite{Baroni11Introduction}, to provide a labelling function $L_E \colon \A \to \set{1, 0, ?}$ that agrees with a given extension $E$, i.e., $\set{a \in \A \mid L_E(a) = 1} = E$ and $\set{a \in \A \mid L_E(a) = 0} = E_\eta^+$, the parameters $\eta$ and $\ell$ are also subject to the condition $\eta \geq \ell$ in general for the sake of the functionality of $L_E$. Although this paper does not adopt ``$E \cap E_\eta^+ = \emptyset$'' as a requirement in defining the notions of range related semantics, we will care about the results caused by the assumption $\eta \geq \ell$.
\begin{proposition}\label{proposition:range_relative}
	For any AAF $F$ and extension-based semantics $\varepsilon_{\sigma}$ with $\varepsilon_{\sigma}(F) \subseteq \varepsilon_{cf}^\ell(F)$, if $\eta \geq \ell$ then the extensions in $\varepsilon_{rr \sigma}^{\eta}(F)$ are maximal in the poset $\tuple{\varepsilon_{\sigma}(F), \subseteq}$,
	in particular, $\varepsilon_{stg}^{\ell \eta}(F) \subseteq \varepsilon_{na}^{\ell}(F)$, $\varepsilon_{rra}^{\ell mn \eta}(F) \subseteq \varepsilon_{pr, \textit{Dung}}^{\ell mn}(F)$ (see, Definition \ref{definition:pr_Dung}) and $\varepsilon_{ss}^{\ell mn \eta}(F) \subseteq \varepsilon_{pr}^{\ell mn}(F)$.
\end{proposition}
\begin{proof}
	Let $E \in \varepsilon_{rr \sigma}^{\eta}(F)$ and $E \subseteq E'$ with $E' \in \varepsilon_{\sigma}(F)$.
	Then $E \cup E_\eta^+ = E' \cup E_\eta'^{+}$.
	Further, by Observation \ref{observation:eta}, $E' \cap E_\eta^+ = \emptyset$.
	Thus, $E' \subseteq E$, and hence $E = E'$.
	Consequently, $E$ is maximal in $\tuple{\varepsilon_{\sigma}(F), \subseteq}$, as desired.
\end{proof}
\begin{corollary} \label{corollary:well-founded_ss_rrs_nonempty}
	Let $F = \tuple{\A, \rightarrow}$ be an AAF and $E \in \varepsilon_{ad}^{\ell mn}(F)$. If $\rightarrow^+$ is well-founded on $\A - D_{\substack{m \\ n}}^{\lambda_F}(E)$ then
	\begin{itemize}
		\item[a.] $D_{\substack{m \\ n}}^{\lambda_F}(E)$ is the unique $\ell mn$-$\eta$-semi-stable extension of $F$ containing $E$ whenever $n \geq \ell \geq m$.
		\item[b.] $\varepsilon_{ss}^{\ell mn \eta}(F) = \varepsilon_{rra}^{\ell mn \eta}(F) = \set{D_{\substack{m \\ n}}^{\lambda_F}(E)}$ whenever $\eta \geq \ell$ and $n \geq \ell \geq m$.
		\item[c.] $\varepsilon_{rrs}^{\ell mn \eta}(F) = \set{D_{\substack{m \\ n}}^{\lambda_F}(E)}$ whenever $\ell = m = n$.
	\end{itemize}
\end{corollary}
\begin{proof}
	Follows from Propositions \ref{proposition:equ_ss and rra} and \ref{proposition:range_relative}, and Corollaries \ref{corollary:well-founded_co_nonempty}($a$) and \ref{corollary:well-founded_stb}.
\end{proof}
\begin{corollary} \label{corollary:well-founded_ss_rrs_empty}
	Let $F = \tuple{\A, \rightarrow}$ be an AAF. If $\rightarrow^+$ is well-founded on $\A - D_{\substack{m \\ n}}^{\lambda_F}(\emptyset)$ then
	\begin{itemize}
		\item[a.] $\varepsilon_{ss}^{\ell mn \eta}(F) = \set{D_{\substack{m \\ n}}^{\lambda_F}(\emptyset)}$ whenever $\ell \geq m$ and $n \geq m$.
		\item[b.] $\varepsilon_{rra}^{\ell mn \eta}(F) = \set{D_{\substack{m \\ n}}^{\lambda_F}(\emptyset)}$ whenever $\eta \geq \ell \geq m$ and $n \geq m$.
	\end{itemize}
\end{corollary}
\begin{proof}
	Immediately follows from Proposition \ref{proposition:range_relative} and Corollary \ref{corollary:well-founded_co_empty}($b$)($d$).
\end{proof}
\subsection{Universal definability and reduced meet modulo an ultrafilter}
So far, we have introduced a family of graded range related semantics and provided a number of elementary properties of them. However, we still have not done one vital thing: we have not actually shown that these semantics are universally defined. This issue for range related semantics is often nontrivial. For example, in the standard situation, although the universal definability of semi-stable semantics (i.e., $\varepsilon_{ss}^{1111}$) for finite AAFs is obvious, this problem is difficult for infinite ones.  For finitary AAFs, it was conjectured in \cite{Caminada10Conjecture} that there is at least one semi-stable extension. This conjecture was firstly proven positively by Weydert in \cite{Weydert11SsForInf} using FoL based on labelling-based semantics. Later on, Baumann and Spanring provided an alternative proof based on extension-based semantics using transfinite induction \cite{Baumann15Infinite}. All these work deal with $\varepsilon_{ss}^{\ell mn \eta}$ in the case that $\ell = m = n = \eta = 1$. 
\par In the following, we will consider this problem for arbitrary positive numbers $\ell$, $m$, $n$ and $\eta$, and provide a uniform method to establish the universal definability of all graded range related semantics defined in Definition \ref{definition:range related extensions} for finitary AAFs. 
Different from \cite{Weydert11SsForInf, Baumann15Infinite} mentioned above, our proof uses neither FoL nor transfinite induction.
The operator so-called \textit{reduced meet modulo an ultrafilter} will play a central role in our work.
Such an operator is applied widely in mathematics, however, to our knowledge, it seems that this operator has not been applied in argumentation until now. In fact, the reader will find that this operator is helpful for us to explore theoretical issues of AAFs and it deserves attention in itself. Some auxiliary notions are introduced below firstly.
\begin{definition}\label{definition:propto}
	Let $F = \tuple{\A, \rightarrow}$ be an AAF and $X, Y \subseteq \A$. Then $X \propto Y$ iff $X \cup X_\eta^+ \subseteq Y \cup Y_\eta^+$.
\end{definition}
Given an extension-based semantics $\varepsilon_{\sigma}$, to establish the universal definability of the range related semantics $\varepsilon_{rr \sigma}^\eta$ associated with $\varepsilon_{\sigma}$ by applying Zorn's lemma, we need to provide a method to construct an upper bound for any $\propto$-chain in the pre-ordered structure $\tuple{\varepsilon_{\sigma}(F), \propto}$. Such construction is often nontrivial. The operator union in set theory is inadequate for this task because $\varepsilon_{\sigma}(F)$ is not closed under the union of a $\propto$-chain in general. This paper will adopt the operator \textit{reduced meet modulo an ultrafilter} to provide the desired construction. The notion of an ultrafilter is recalled below (see for instance \cite{Chang1990Book}).
\begin{definition}[Ultrafilter] \label{definition:ultrafilter}
	Let $I$ be a nonempty set. An ultrafilter $D$ over $I$ is defined to be a set $D \subseteq \pw{I}$ such that
	\begin{itemize}
		\item[$U_1$.] $I \in D$ and $\emptyset \notin D$.
		\item[$U_2$.] If $X, Y \in D$ then $X \cap Y \in D$.
		\item[$U_3$.] If $X \in D$ and $X \subseteq Y \subseteq I$ then $Y \in D$.
		\item[$U_4$.] For any $X, Y \subseteq I$, either $X \in D$ or $Y \in D$ whenever $X \cup Y \in D$.
	\end{itemize}
\end{definition}
Some trivial but useful properties of ultrafilters are listed below, which will be applied without illustrating. For any $X, X_i \subseteq I (1 \leq i \leq n < \omega)$,
\begin{itemize}
	\item $\bigcap_{1 \leq i \leq n} X_i \in D$ iff $X_i \in D$ for each $1 \leq i \leq n$,
	\item $\bigcup_{1 \leq i \leq n} X_i \in D$ iff $X_i \in D$ for some $1 \leq i \leq n$ and
	\item $X \in D$ iff $I - X \notin D$.
\end{itemize}
\par Let $I$ be a nonempty set. A set $\Omega \subseteq \pw{I}$ is said to have the \textit{finite intersection property} iff the intersection of any finite number of elements of $\Omega$ is nonempty. A well-known result about ultrafilters is recalled below (see, e.g., \cite{Chang1990Book} again).
\begin{theorem}[Ultrafilter Theorem] \label{theorem:ultrafilter theorem}
	Let $I$ be a nonempty set. If $\Omega \subseteq \pw{I}$ and $\Omega$ has the finite intersection property, then there exists an ultrafilter $D$ over $I$ such that $\Omega \subseteq D$.
\end{theorem}
The following notion plays a central role in our work, which provides a simple but powerful construction method.
\begin{definition}\label{definition:reduced meet}
	Let $\A$ and $I$ be nonempty sets and $D$ an ultrafilter over $I$. For any family of sets $X_i \subseteq \A$ with $i \in I$, the reduced meet of these sets modulo $D$ is defined below, where $\hat{x} \triangleq \set{i \in I \mid x \in X_i}$ for each $x \in \A$.
	\begin{eqnarray*}
		\bigcap_D X_i \triangleq \set{x \in \A \mid \hat{x} \in D}.
	\end{eqnarray*} 
\end{definition}
Roughly speaking, unlike the operator meet in set theory, $\bigcap_D X_i$ consists of all objects that belong to the majority of $X_i$'s.
\begin{example}
	Let $\A = I \triangleq \omega$. For each $n < \omega$, set $X_n \triangleq \set{i \in \omega \mid i \leq n}$, and thus $\hat{n} = \set{i \in \omega \mid i \geq n}$. Clearly, the set $\set{\hat{n} \mid n < \omega}$ has the finite intersection property. Hence, by Theorem \ref{theorem:ultrafilter theorem}, there exists an ultrafilter $D$ over $I$ such that $\set{\hat{n} \mid n < \omega} \subseteq D$. Moreover, $\bigcap_D X_n = \omega$ due to $\hat{n} \in D$ for each $n < \omega$. By contrast, $\bigcap_{n < \omega} X_n = \set{0}$.
\end{example}
In establishing the universal definability of range related semantics by applying Zorn's lemma, as mentioned above, it is a crucial step that constructing an upper bound for any $\propto$-chain. The next two lemmas intend to show that the operator $\bigcap_D$ provides a potential method to realize this step.
\begin{lemma} \label{lemma:out_RM}
	Let $F = \tuple{\A, \rightarrow}$ be a finitary AAF and $X_i \subseteq \A$ for each $i \in I \neq \emptyset$. For any ultrafilter $D$ over $I$ and $a \in \A$, if $\set{i \mid a \in (X_i)_\eta^+} \in D$ then $a \in (\bigcap_D X_i)_\eta^+$.
\end{lemma}
\begin{proof}
	Let $i \in I$ such that $a \in (X_i)_\eta^+$.
	Then there exists $S_0 \in [a^-]^\eta$ such that $S_0 \subseteq X_i$.
	Thus, $i \in \hat{b}$ for each $b \in S_0$, that is, $i \in \bigcap_{b \in S_0} \hat{b}$.
	Hence, we have
	\begin{eqnarray*}
		\set{i \mid a \in (X_i)_\eta^+} \subseteq \bigcup_{S \in [a^-]^\eta} (\bigcap_{b \in S} \hat{b}).
	\end{eqnarray*}
	Further, by the clause ($U_3$) in Definition \ref{definition:ultrafilter}, $\bigcup_{S \in [a^-]^\eta} (\bigcap_{b \in S} \hat{b}) \in D$ due to the assumption $\set{i \mid a \in (X_i)_\eta^+} \in D$.
	Since $F$ is finitary, $[a^-]^\eta$ is finite. 
	Thus, by the clause ($U_4$) in Definition \ref{definition:ultrafilter}, there exists $S_1 \in [a^-]^\eta$ such that $\bigcap_{b \in S_1} \hat{b} \in D$.
	Hence, $\hat{b} \in D$ for each $b \in S_1$ and thus, by Definition \ref{definition:reduced meet}, $S_1 \subseteq \bigcap_D X_i$.
	So, $a \in (\bigcap_D X_i)_\eta^+$ due to $S_1 \in [a^-]^\eta$.
\end{proof}
\begin{lemma} \label{lemma:RM_ub}
	Let $F = \tuple{\A, \rightarrow}$ be a finitary AAF and $X_i \subseteq \A$ for each $i \in I \neq \emptyset$. If $\set{X_i \mid i \in I}$ is a $\propto$-chain, then there exists an ultrafilter $D$ over $I$ such that $\bigcap_D X_i$ is an upper bound of $\set{X_i \mid i \in I}$ w.r.t. $\propto$.
\end{lemma}
\begin{proof}
	For each $a \in \A$, put
	\begin{eqnarray*}
		J_a \triangleq \set{i \in I \mid a \in X_i \cup (X_i)_\eta^+} = \hat{a} \cup \set{i \in I \mid a \in (X_i)_\eta^+}
	\end{eqnarray*}
	and
	\begin{eqnarray*}
		M \triangleq \set{J_a \mid a \in \bigcup_{i \in I} (X_i \cup (X_i)_\eta^+)}.
	\end{eqnarray*}
	We first prove the following claim.\\
	\\ \textbf{Claim~~} $M$ has the finite intersection property.\\
	\\
	Let $J_{a_1}, \dots, J_{a_n} \in M (n \geq 1)$.
	Then, for each $i$ with $1 \leq i \leq n$, $a_i \in \bigcup_{i \in I} (X_i \cup (X_i)_\eta^+)$, and hence $a_i \in X_{j_i} \cup (X_{j_i})_\eta^+$ for some $j_i \in I$.
	Since $\set{X_i \mid i \in I}$ is a $\propto$-chain, we may put
	\begin{eqnarray*}
		i_0 \triangleq \text{~the index of the largest one among~} X_{j_1}, \dots, X_{j_n} \text{~w.r.t.~} \propto.
	\end{eqnarray*}
	Then, for each $a_i$ with $1 \leq i \leq n$, $a_i \in X_{i_0} \cup (X_{i_0})_\eta^+$ due to $a_i \in X_{j_i} \cup (X_{j_i})_\eta^+$ and $X_{j_i} \propto X_{i_0}$, and hence $i_0 \in J_{a_i}$.
	Thus, $J_{a_1} \cap \dots \cap J_{a_n} \neq \emptyset$ due to $i_0 \in J_{a_1} \cap \dots \cap J_{a_n}$.\\
	\\
	Now we turn to the proof of the lemma. By the claim above and Theorem \ref{theorem:ultrafilter theorem}, there exists an ultrafilter $D$ over $I$ such that $M \subseteq D$. Let $a \in \bigcup_{i \in I} (X_i \cup (X_i)_\eta^+)$. Since $J_a \in M \subseteq D$ and $J_a = \hat{a} \cup \set{i \in I \mid a \in (X_i)_\eta^+}$, by Definition \ref{definition:ultrafilter}, we get either $\hat{a} \in D$ or $\set{i \in I \mid a \in (X_i)_\eta^+} \in D$. If $\hat{a} \in D$, then $a \in \bigcap_D X_i$ by Definition \ref{definition:reduced meet}. If $\set{i \in I \mid a \in (X_i)_\eta^+} \in D$ then $a \in (\bigcap_D X_i)_\eta^+$ by Lemma \ref{lemma:out_RM}. Consequently, $\bigcup_{i \in I} (X_i \cup (X_i)_\eta^+) \subseteq (\bigcap_D X_i) \cup (\bigcap_D X_i)_\eta^+$. Thus, $\bigcap_D X_i$ is an upper bound of $\set{X_i \mid i \in I}$ w.r.t. $\propto$.
\end{proof}
\begin{definition}\label{definition:closed under reduced meets}
	Given a set $\A$, a set $\Xi \subseteq \pw{\A}$ is said to be closed under reduced meets modulo any ultrafilter if, for any nonempty index set $I$ and ultrafilter $D$ over $I$, $X_i \in \Xi$ for all $i \in I$ implies $\bigcap_D X_i \in \Xi$.
\end{definition}
\begin{definition}
	Let $\Omega$ be a class of AAFs and $\varepsilon$ an extension-based semantics. The semantics $\varepsilon$ is said to be closed under reduced meets modulo any ultrafilter w.r.t. $\Omega$ if $\varepsilon(F)$ is closed under reduced meets modulo any ultrafilter for any $F \in \Omega$.
\end{definition}
By Lemma \ref{lemma:RM_ub}, we have the following metatheorem about the universal definability of range related semantics. Here $\varepsilon_{rr \sigma}^\eta$ is the $\eta$-range related semantics induced by a given extension-based semantics $\varepsilon_{\sigma}$, which is defined similarly to the ones in Definition \ref{definition:range related extensions}.
\begin{theorem} \label{theorem:close under meet and universal definability}
	Let $F$ be a finitary AAF and $\varepsilon_{\sigma}$ an extension-based semantics. If $\varepsilon_{\sigma}(F)$ is closed under reduced meets modulo any ultrafilter then $|\varepsilon_{\sigma}(F)| \geq 1$ iff $|\varepsilon_{rr \sigma}^\eta(F)| \geq 1$.	Hence, $\varepsilon_{rr \sigma}^\eta$ agrees with $\varepsilon_{\sigma}$ on the universal definability w.r.t. finitary AAFs whenever $\varepsilon_{\sigma}$ is closed under reduced meets modulo any ultrafilter.
\end{theorem}
\begin{proof}
	We only show the nontrivial direction.
	Let $\set{X_i \cup (X_i)_\eta^+ \mid i \in I}$ be a chain in the poset 
	\begin{eqnarray*}
		P \triangleq \tuple{\set{X \cup X_\eta^+ \mid X \in \varepsilon_{\sigma}(F)}, \subseteq}.
	\end{eqnarray*}
	Thus, $\set{X_i \mid i \in I}$ is a $\propto$-chain.
	By Lemma \ref{lemma:RM_ub}, there exists an ultrafilter $D$ over $I$ such that $\bigcap_D X_i$ is an upper bound of $\set{X_i \mid i \in I}$ w.r.t. $\propto$.
	Moreover, $\bigcap_D X_i \in \varepsilon_{\sigma}(F)$ by the assumption that $\varepsilon_{\sigma}(F)$ is closed under reduced meets modulo any ultrafilter.
	Thus, $\bigcap_D X_i \cup (\bigcap_D X_i)_\eta^+$ is an upper bound of the chain $\set{X_i \cup (X_i)_\eta^+ \mid i \in I}$ in $P$.
	Consequently, by Zorn's lemma, there exist maximal elements in $P$, and hence $|\varepsilon_{rr \sigma}^\eta(F)| \geq 1$.
\end{proof}
By the metatheorem above, to establish the universal definability of a range related semantics $\varepsilon_{rr \sigma}^\eta$ w.r.t. a class $\Omega$ of finitary AAFs, it is enough to prove that, for any $F \in \Omega$, $\varepsilon_{\sigma}(F)$ is closed under reduced meets modulo any ultrafilter. In the following, we will show that a variety of fundamental semantics in argumentation have such closeness w.r.t. finitary AAFs. To this end, the next result is needed, which asserts that both the neutrality function $N_\ell$ and the defense function $D_n^m$ are distributive over reduced meets modulo any ultrafilter.
\begin{lemma}[Distributive law] \label{lemma:RMMU_Distributive}
	Let $F = \tuple{\A, \rightarrow}$ be an AAF, $I \neq \emptyset$ and $D$ an ultrafilter over $I$, and let $E_i \subseteq \A$ for each $i \in I$.
	\begin{itemize}
		\item[a.] $\bigcap_D N_\ell(E_i) \subseteq N_\ell(\bigcap_D E_i)$.
		\item[b.] $N_\ell(\bigcap_D E_i) \subseteq \bigcap_D N_\ell(E_i)$ whenever $F$ is finitary.
		\item[c.] $D_n^m(\bigcap_D E_i) = \bigcap_D D_n^m(E_i)$ whenever $F$ is finitary.
	\end{itemize}
\end{lemma}
\begin{proof}
	\par($a$) Let $a \notin N_\ell(\bigcap_D E_i)$. Then $X_0 \subseteq \bigcap_D E_i$ for some $X_0 \in [a^-]^\ell$. Hence $J \triangleq \set{ i \in I \mid X_0 \subseteq E_i} \in D$. For each $j \in J$, $a \notin N_\ell(E_j)$ due to $X_0 \in [a^-]^\ell$ and $X_0 \subseteq E_j$. Thus, $J \subseteq \set{i \in I \mid a \notin N_\ell(E_i)} \in D$. Consequently, $a \notin \bigcap_D N_\ell(E_i)$.
	\par($b$) Let $a \notin \bigcap_D N_\ell(E_i)$. Then $\hat{a} = \set{i \in I \mid a \in N_\ell(E_i)} \notin D$. Since $D$ is an ultrafilter, we get $\set{i \in I \mid a \notin N_\ell(E_i)} = I - \hat{a} \in D$. For each $X \in [a^-]^\ell$, put
	\begin{eqnarray*}
		J_X \triangleq \set{ i \in I \mid X \subseteq E_i}.
	\end{eqnarray*}
	For each $j \in I - \hat{a}$, $j \in J_X$ for some $X \in [a^-]^\ell$ because of $a \notin N_\ell(E_i)$. Thus
	\begin{eqnarray*}
		I - \hat{a} \subseteq \bigcup_{X \in [a^-]^\ell} J_X,
	\end{eqnarray*}
	and hence $\bigcup_{X \in [a^-]^\ell} J_X \in D$ due to $I - \hat{a} \in D$. Further, since $F$ is finitary, the set $[a^-]^\ell$ is finite. Thus, $J_{X_0} \in D$ for some $X_0 \in [a^-]^\ell$. Then, for each $d \in X_0$, we have $J_{X_0} \subseteq \set{ i \in I \mid d \in E_i}$, and hence $d \in \bigcap_D E_i$. Thus, $X_0 \subseteq \bigcap_D E_i$. Then it follows from $X_0 \in [a^-]^\ell$ that $a \notin N_\ell(\bigcap_D E_i)$.
	\par ($c$) By the clauses ($a$) and ($b$), we have $D_n^m(\bigcap_D E_i) = N_m(N_n(\bigcap_D E_i)) = N_m(\bigcap_D N_n(E_i)) = \bigcap_D N_m(N_n(E_i)) = \bigcap_D D_n^m(E_i)$.
\end{proof}
The clause ($a$) in Lemma \ref{lemma:RMMU_Distributive} may be viewed as a generalization of the directed De Morgan law (see, Lemma \ref{lemma:properties_cf_def}($b$)) due to the assertion below, which asserts that any union of directed sets is a reduced meet modulo an ultrafilter.
\begin{observation} \label{observation:RMMU_continuity}
	Let $A$ and $I$ be any nonempty sets. If $\set{E_i \subseteq A \mid i \in I}$ is a directed set then $\bigcap_D E_i = \bigcup_{i \in I} E_i$ for some ultrafilter $D$ over $I$.
\end{observation}
\begin{proof}
	Let $S \triangleq \set{\hat{a} \mid a \in \bigcup_{i \in I} E_i}$. Assume that $a_1, \cdots, a_n \in \bigcup_{i \in I} E_i$. Then, for each $k$ with $1 \leq k \leq n$, $a_k \in E_{i_k}$ for some $i_k \in I$. Since $\set{E_i \mid i \in I}$ is directed, there exists $i_0 \in I$  such that $E_{i_k} \subseteq E_{i_0}$ for each $1 \leq k \leq n$. Hence, $i_0 \in \hat{a}_1 \cap \cdots \cap \hat{a}_n$. Thus, $S$ has the finite intersection property. By Theorem \ref{theorem:ultrafilter theorem}, $S \subseteq D$ for some ultrafilter $D$ over $I$. Clearly, $\bigcup_{i \in I} E_i \subseteq \bigcap_D E_i$. Moreover, $\bigcap_D E_i \subseteq \bigcup_{i \in I} E_i$ because of $\hat{b} = \emptyset \notin D$ for each $b \notin \bigcup_{i \in I} E_i$.
\end{proof}
Let $\set{E_i \subseteq \A \mid i \in I}$ be directed. By Observation \ref{observation:RMMU_continuity}, $\bigcup_{i \in I} E_i = \bigcap_D E_i$ for some ultrafilter $D$ over $I$. Then, we have $N_\ell(\bigcup_{i \in I} E_i) = N_\ell(\bigcap_D E_i) \supseteq \bigcap_D N_\ell(E_i)$ by the clause ($a$) in Lemma \ref{lemma:RMMU_Distributive}. Furthermore, $\bigcap_{i \in I} N_\ell(E_i) \subseteq \bigcap_D N_\ell(E_i)$ due to $I \in D$ and, by Lemma \ref{lemma:Grossi_simple}($a$), $N_\ell(\bigcup_{i \in I} E_i) \subseteq \bigcap_{i \in I} N_\ell(E_i)$. Thus $N_\ell(\bigcup_{i \in I} E_i) = N_\ell(\bigcap_D E_i) \supseteq \bigcap_D N_\ell(E_i) \supseteq \bigcap_{i \in I} N_\ell(E_i) \supseteq N_\ell(\bigcup_{i \in I} E_i)$. Consequently, the directed De Morgan law is implied by Lemma \ref{lemma:RMMU_Distributive}($a$).
\par It is well known that the function $D_n^m$ is continuous whenever the AAF $F$ is finitary. The clause ($c$) in the lemma above may be viewed as a generalization of the continuity of $D_n^m$.
For any finitary AAF $F = \tuple{\A, \rightarrow}$ and directed subset $\set{E_i \mid i \in I}$ of $\pw{\A}$, by Observation \ref{observation:RMMU_continuity}, there exists an ultrafilter $D$ over $I$ such that $\bigcap_D E_i = \bigcup_{i \in I} E_i$. Moreover, since $D_n^m$ is monotonic, $D_n^m(E_i) \subseteq D_n^m(\bigcap_D E_i)$ for each $i \in I$. Thus $\bigcup_{i \in I} D_n^m(E_i) \subseteq D_n^m(\bigcap_D E_i)$. On the other hand, it is obvious that $\bigcap_D D_n^m(E_i) \subseteq \bigcup_{i \in I} D_n^m(E_i)$. Then, by the clause ($c$) in Lemma \ref{lemma:RMMU_Distributive}, we have
\begin{eqnarray*}
	D_n^m(\bigcup_{i \in I} E_i) = D_n^m(\bigcap_D E_i) = \bigcap_D D_n^m(E_i)  \subseteq \bigcup_{i \in I} D_n^m(E_i) \subseteq D_n^m(\bigcap_D E_i).
\end{eqnarray*}
Hence $D_n^m(\bigcup_{i \in I} E_i) = D_n^m(\bigcap_D E_i) = \bigcap_D D_n^m(E_i)  = \bigcup_{i \in I} D_n^m(E_i)$. Consequently, for finitary AAFs, the continuity of $D_n^m$ is an instantiation of the distributivity of $D_n^m$ over the operator $\bigcap_D$.
\begin{lemma}\label{lemma:RMMU_subset}
	Let $A$ and $I$ be any nonempty sets and let $X_i, Y_i \subseteq A$ for each $i \in I$. For any ultrafilter $D$ over $I$, if $\set{i \in I \mid X_i \subseteq Y_i} \in D$ then $\bigcap_D X_i \subseteq \bigcap_D Y_i$.
\end{lemma}
\begin{proof}
	Let $a \in \bigcap_D X_i$. Then $\set{i \in I \mid a \in X_i} \in D$. Thus 
	\begin{eqnarray*}
		\set{i \in I \mid X_i \subseteq Y_i \text{~and~} a \in X_i} \in D
	\end{eqnarray*}
	due to $\set{i \in I \mid X_i \subseteq Y_i} \in D$. Since $\set{i \in I \mid X_i \subseteq Y_i \text{~and~} a \in X_i} \subseteq \set{i \in I \mid a \in Y_i}$, we have $\set{i \in I \mid a \in Y_i} \in D$, and hence $a \in \bigcap_D Y_i$.
\end{proof}
\begin{lemma} \label{lemma:RMMU_ultrafilter}
	Let $F = \tuple{\A, \rightarrow}$ be a finitary AAF and $E_i \subseteq \A$ for each $i \in I \neq \emptyset$. If $D$ is an ultrafilter over $I$ then 
	\begin{itemize}
		\item[a.] $\set{i \in I \mid E_i \subseteq N_\ell(E_i)} \in D$ implies $\bigcap_D E_i \subseteq N_\ell(\bigcap_D E_i)$\footnote{This clause doesn't depend on the assumption that $F$ is finitary.}.
		\item[b.] $\set{i \in I \mid N_\ell(E_i) \subseteq E_i} \in D$ implies $N_\ell(\bigcap_D E_i) \subseteq \bigcap_D E_i$.
		\item[c.] $\set{i \in I \mid E_i \subseteq D_n^m(E_i)} \in D$ implies $\bigcap_D E_i \subseteq D_n^m(\bigcap_D E_i)$.
		\item[d.] $\set{i \in I \mid D_n^m(E_i) \subseteq E_i} \in D$ implies $D_n^m(\bigcap_D E_i) \subseteq \bigcap_D E_i$.
	\end{itemize}
	Hence, for $\varepsilon \in \set{\textit{Def}^{mn}, \varepsilon_{cf}^{\ell}, \varepsilon_{ad}^{\ell mn}, \varepsilon_{co}^{\ell mn}, \varepsilon_{stb}^{\ell mn}}$, $\set{i \in I \mid E_i \in 
	\varepsilon(F)} \in D$ implies $\bigcap_D E_i \in \varepsilon(F)$.
\end{lemma}
\begin{proof}
	We only show the clause ($a$), other clauses may be proved similarly. Since $\set{i \in I \mid E_i \subseteq N_\ell(E_i)} \in D$, by Lemma \ref{lemma:RMMU_subset}, we have $\bigcap_D E_i \subseteq \bigcap_D N_\ell(E_i)$. Moreover, by the clause ($a$) in Lemma \ref{lemma:RMMU_Distributive}, $\bigcap_D N_\ell(E_i) \subseteq N_\ell(\bigcap_D E_i)$. Hence $\bigcap_D E_i \subseteq N_\ell(\bigcap_D E_i)$.	 
\end{proof}
Intuitively, these assertions in the above lemma say that, for any family $\set{E_i}_{i \in I}$ of subsets of $\A$, if the majority of elements in this family are prefixed (postfixed) points of the function $f$ then so is its reduced meets modulo any ultrafilter, and hence for any finitary AAF $F$, the set of all prefixed (postfixed) points of $f$ is closed under the operator $\bigcap_D$, where $f$ is either $N_\ell$ or $D_n^m$. The principal significance of this lemma is that it allows us to assert an interesting property of some extension-based semantics as below. A further result about this property will be given in Section \ref{Sec: Extension-based abstract semantics}, which will bring another metatheorem concerning the universal definability of range related semantics.
\begin{theorem} \label{theorem:close under meet_fundamental semantics}
	The semantics $\textit{Def}^{mn}$, $\varepsilon_{cf}^{\ell}$, $\varepsilon_{ad}^{\ell mn}$, $\varepsilon_{co}^{\ell mn}$, $\varepsilon_{gr}^{\ell mn}$ and $\varepsilon_{stb}^{\ell mn}$ are closed under reduced meets modulo any ultrafilter w.r.t. the class of all finitary AAFs\footnote{For $\varepsilon_{cf}^{\ell}$ and $\varepsilon_{gr}^{\ell mn}$, it doesn't depend on the assumption that $F$ is finitary.}.
\end{theorem}
\begin{proof}
	 For these semantics except $\varepsilon_{gr}^{\ell mn}$, it follows immediately from Lemma \ref{lemma:RMMU_ultrafilter} based on the fact that $I \in D$ for each ultrafilter $D$ over $I$. Next, we consider $\varepsilon_{gr}^{\ell mn}$.
	 Let $F = \tuple{\A, \rightarrow}$ be any AAF.
	 Clearly, $|\varepsilon_{gr}^{\ell mn}(F)| \leq 1$.
	 If $\varepsilon_{gr}^{\ell mn}(F) = \emptyset$ then it holds trivially.
	 In the situation that $\varepsilon_{gr}^{\ell mn}(F) = \set{E}$ for some $E \subseteq \A$, let $I \neq \emptyset$ and $E_i \in \varepsilon_{gr}^{\ell mn}(F) (i \in I)$.
	 Then $E_i = E$ for all $i \in I$.
	 For any ultrafilter $D$ over $I$, it is easy to see that, for each $a \in \A$, if $a \in E$ then $\hat{a} = I$ else $\hat{a} = \emptyset$, which implies $\bigcap_D E_i = E$, as desired.
\end{proof}
\begin{remark}\label{remark:RMMU_finite AAFs}
	For any AAF $F$ and extension-based semantics $\varepsilon$, $\varepsilon(F)$ is always closed under reduced meets modulo any ultrafilter whenever $|\varepsilon(F)| < \omega$. In particular, $\varepsilon(F)$ is closed under the operator $\bigcap_D$ for any finite AAF $F$.
	If $\varepsilon(F) = \emptyset$ then it holds trivially. Next assume $0 < |\varepsilon(F)| < \omega$. Let $I \neq \emptyset$ and $D$ be any ultrafilter over $I$. Assume $E_i \in \varepsilon(F)$ for each $i \in I$. For each $E \in \varepsilon(F)$, put 
	\begin{eqnarray*}
		I_E \triangleq \set{i \in I \mid E_i = E}.
	\end{eqnarray*}
	Then $\bigcup_{E \in \varepsilon(F)} I_E = I$. Since $|\varepsilon(F)| < \omega$ and $I \in D$, $I_{E_0} \in D$ for some $E_0 \in \varepsilon(F)$. Hence, $E_0 \subseteq \bigcap_D E_i$. Let $b \in \bigcap_D E_i$. Then $\set{i \mid b \in E_i} \in D$. Thus, $I_{E_0} \cap \set{i \mid b \in E_i} \in D$. Hence $b \in E_0$ due to $I_{E_0} \cap \set{i \mid b \in E_i} \neq \emptyset$. Thus, $\bigcap_D E_i = E_0 \in \varepsilon(F)$.	
\end{remark}
Now we arrive at one of the main results in this section.
\begin{theorem} \label{theorem:existence_stg_ss_rra_rrs}
	Let $F$ be any finitary AAF. Then
	\begin{itemize}
		\item[a.] $|\varepsilon_{stg}^{\ell \eta}(F)| \geq 1$ and $|\varepsilon_{rra}^{\ell mn \eta}(F)| \geq 1$.
		\item[b.] $|\varepsilon_{ss}^{\ell mn \eta}(F)| \geq 1$ whenever $\varepsilon_{co}^{\ell mn}(F) \neq \emptyset$.
		\item[c.] $|\varepsilon_{rrs}^{\ell mn \eta}(F)| \geq 1$ whenever $\varepsilon_{stb}^{\ell mn}(F) \neq \emptyset$.
	\end{itemize}
\end{theorem}
\begin{proof}
	Immediately follows from Theorems \ref{theorem:close under meet and universal definability} and \ref{theorem:close under meet_fundamental semantics}.
\end{proof}
Here, the clause ($a$) holds unconditionally because of $|\varepsilon_{cf}^{\ell}(F)| \geq 1$ and $|\varepsilon_{ad}^{\ell mn}(F)| \geq 1$ due to $\emptyset \in \varepsilon_{cf}^{\ell}(F) \cap \varepsilon_{ad}^{\ell mn}(F)$, whereas the assumptions in the clauses ($b$) and ($c$) are necessary because that neither $\varepsilon_{co}^{\ell mn}$ nor $\varepsilon_{stb}^{\ell mn}$ is always universally defined, for instance, $\varepsilon_{co}^{\ell mn}(F) = \varepsilon_{stb}^{\ell mn}(F) = \emptyset$ for $F$, $\ell$, $m$ and $n$ considered in Example \ref{Ex:counterexampleforgr}. However, by Proposition \ref{proposition:structure_gr}, $\varepsilon_{co}^{\ell mn}$ is universally defined w.r.t. all AAFs whenever $\ell \geq m$ and $n \geq m$. Thus, we have the following corollary.
\begin{corollary} \label{corollary:existence_ss}
	For any finitary AAF $F$, $|\varepsilon_{ss}^{\ell mn \eta}(F)| \geq 1$ whenever $\ell \geq m$ and $n \geq m$.
\end{corollary}
Since the function $\lambda_X.X_\eta^+$ is monotone w.r.t. $\subseteq$, each range related semantics $\varepsilon_{rr \sigma}^\eta$ is up-closed, that is, $E_1 \in \varepsilon_{rr \sigma}^\eta(F)$ and $E_1 \subseteq E_2 \in \varepsilon_{\sigma}(F)$ implies $E_2 \in \varepsilon_{rr \sigma}^\eta(F)$. Moreover, by Corollary \ref{corollary:structure_co} and the clauses ($b$) and ($c$) in Corollary \ref{corollary:structure_cf_ad} , we have
\begin{corollary} \label{corollary:structure_stg_ss}
	Let $F = \tuple{\A, \rightarrow}$ be a finitary AAF.
	\begin{itemize}
			\item[a.] Both $\tuple{\varepsilon_{stg}^{\ell \eta}(F), \subseteq}$ and $\tuple{\varepsilon_{rra}^{\ell mn \eta}(F), \subseteq}$ are dcpos.
			\item[b.] $\tuple{\varepsilon_{ss}^{\ell mn \eta}(F), \subseteq}$ is a dcpo whenever $\ell \geq m$ and $n \geq m$.
	\end{itemize}
	Hence, all these posets have the Lindenbaum property.
\end{corollary}
For any AAF $F$ with $n \geq \ell \geq m$, by Proposition \ref{proposition:equ_ss and rra}($b$), $\varepsilon_{ss}^{\ell mn \eta}(F) \subseteq \varepsilon_{rra}^{\ell mn \eta}(F)$. Moreover, by Corollary \ref{corollary:properties_defense}($b$), it is easy to verify that $D_{\substack{m \\ n}}^{\lambda_F}(E) \in \varepsilon_{ss}^{\ell mn \eta}(F)$ whenever $E \in \varepsilon_{rra}^{\ell mn \eta}(F)$. More precisely, we have
\begin{corollary}
	Let $F = \tuple{\A, \rightarrow}$ be a finitary AAF and $n \geq \ell \geq m$. Then $D_{\substack{m \\ n}}^{\omega} \dashv \textit{incl}$ is a Galois adjunction between the posets $\tuple{\varepsilon_{rra}^{\ell mn \eta}(F), \subseteq}$ and $\tuple{\varepsilon_{ss}^{\ell mn \eta}(F), \subseteq}$.
\end{corollary}
\section{More on reduced meets modulo ultrafilters} \label{Sec: More on RM}
Although the operator $\bigcap_D$ is introduced as the tool to establish the universal definability of range related semantics, it also contributes to explore other issues of AAFs. This section will explore it further.
\subsection{Some metatheorems}
The closeness under the operator $\bigcap_D$ implies a number of interesting properties. This subsection intends to provide some examples to illustrate this.
\begin{definition}
	Let $F = \tuple{\A, \rightarrow}$ be an AAF and $X \subseteq \A$. Given an extension-based semantics $\varepsilon$, $X$ is said to be $\varepsilon$-extensible if $X \subseteq E$ for some $E \in \varepsilon(F)$, and $X$ is said to be finitely $\varepsilon$-extensible if every finite subset of $X$ is $\varepsilon$-extensible.
\end{definition}
As usual in logic and topology, in exploring infinite AAFs, the compactness is an important desired property because it provides the possibility of studying infinite objects based on finite ones. Here we intend to show that the property $\varepsilon$-extensibility is compact whenever $\varepsilon$ is closed under reduced meets modulo any ultrafilter, that is, whether a given set can be extended to an $\varepsilon$-extension may be boiled down to the $\varepsilon$-extensibility of its finite subsets.
\begin{theorem} \label{theorem:epsilon_extensible}
	Let $F = \tuple{\A, \rightarrow}$ be any AAF and $\varepsilon$ an extension-based semantics. If $\varepsilon(F)$ is closed under reduced meets modulo any ultrafilter, then the property $\varepsilon$-extensibility is compact, that is, for each $X \subseteq \A$, $X$ is $\varepsilon$-extensible iff $X$ is finitely $\varepsilon$-extensible.
\end{theorem}
\begin{proof}
	It is trivial that the left implies the right.
	Next we consider another direction.
	By the assumption, for each finite subset $Y$ of $X$ (i.e., $Y \in \wp_f(X)$), we may choose arbitrarily and fix a set $E_Y \in \varepsilon(F)$ such that $Y \subseteq E_Y$.
	For each $x \in X$, put $\hat{x} \triangleq \set{Y \in \wp_f(X) \mid x \in E_Y}$.
	For any finite subset $Y = \set{x_1, \cdots, x_n}$ of $X$, it is easy to see that $\hat{x}_1 \cap \cdots \cap \hat{x}_n \neq \emptyset$ due to $x_i \in Y \subseteq E_Y$ for each $1 \leq i \leq n$.
	Hence, the set $\set{\hat{x} \mid x \in X}$ has the finite intersection property.
	Thus, by Theorem \ref{theorem:ultrafilter theorem}, there exists an ultrafilter $D$ over $\wp_f(X)$ such that $\set{\hat{x} \mid x \in X} \subseteq D$.
	Since $\varepsilon(F)$ is closed under reduced meets modulo any ultrafilter, we get $\bigcap_D E_Y \in \varepsilon(F)$.
	Moreover, by Definition \ref{definition:reduced meet}, $X \subseteq \bigcap_D E_Y$ due to $\set{\hat{x} \mid x \in X} \subseteq D$, as desired.
\end{proof}
\begin{corollary} \label{corollary:RMMU_extensible}
	Let $F = \tuple{\A, \rightarrow}$ be any AAF and $\varepsilon$ an extension-based semantics, and let $X_i \subseteq \A$ for each $i \in I \neq \emptyset$. If $\varepsilon(F)$ is closed under reduced meets modulo any ultrafilter, then, for each ultrafilter $D$ over $I$, $\set{i \in I \mid X_i \text{~is~} \varepsilon \text{-extensible}} \in D$ implies $\bigcap_D X_i$ is $\varepsilon$-extensible, hence the set $\set{X \subseteq \A \mid X \text{~is~} \varepsilon \text{-extensible}}$ is closed under the operator $\bigcap_D$.
\end{corollary}
\begin{proof}
	Let $Y$ be any finite subset of $\bigcap_D X_i$. Then $\set{i \in I \mid Y \subseteq X_i} \in D$, and hence $\set{i \in I \mid Y \subseteq X_i \text{~and~} X_i \text{~is~} \varepsilon \text{-extensible}} \in D$. Thus, $Y \subseteq X_{i_0}$ and $X_{i_0}$ is $\varepsilon$-extensible for some $i_0 \in I$. Hence $Y$ is $\varepsilon$-extensible, and so is $\bigcap_D X_i$ by Theorem \ref{theorem:epsilon_extensible}.
\end{proof}
\begin{definition} \label{definition:infer}
	Let $F = \tuple{\A, \rightarrow}$ be any AAF and $\varepsilon$ an extension-based semantics. For any $X \subseteq \A$ and $a \in \A$, $X$ $\varepsilon$-infers $a$, in symbols $X \nc{\varepsilon}{F} a$ (or simply $X \nc{\varepsilon}{} a$ if no confusion will arise), if, for any $E \in \varepsilon(F)$, $X \subseteq E$ implies $a \in E$. Otherwise, we say that $X$ doesn't $\varepsilon$-infer $a$, denoted by $X \nNc{\varepsilon}{F} a$.
\end{definition}
\begin{theorem} \label{theorem:epsilon_inference}
	Let $F = \tuple{\A, \rightarrow}$ be any AAF and $\varepsilon$ an extension-based semantics. If $\varepsilon(F)$ is closed under reduced meets modulo any ultrafilter then $\varepsilon$-inference is co-compact, that is, for any $X \subseteq \A$ and $a \in \A$, $X \nNc{\varepsilon}{} a$ iff $Y \nNc{\varepsilon}{} a$ for each finite subset $Y$ of $X$ (equivalently, $X \nc{\varepsilon}{} a$ iff $X_0 \nc{\varepsilon}{} a$ for some finite subset $X_0$ of $X$).
\end{theorem}
\begin{proof}
	We only show that the right implies the left, the other direction is trivial. 
	Assume $Y \nNc{\varepsilon}{} a$ for each finite subset $Y$ of $X$.
	Then, for each $Y \in \wp_f(X)$, we may choose and fix an extension $E_Y \in \varepsilon(F)$ such that $Y \subseteq E_Y$ and $a \notin E_Y$.
	Similar to Theorem \ref{theorem:epsilon_extensible}, there exists an ultrafilter $D$ over $\wp_f(X)$ such that $\set{\hat{x} \mid x \in X} \subseteq D$, where $\hat{x} \triangleq \set{Y \in \wp_f(X) \mid x \in E_Y}$.
	Hence $X \subseteq \bigcap_D E_Y \in \varepsilon(F)$.
	Moreover, due to $a \notin E_Y$ for each $Y \in \wp_f(X)$, $\hat{a} = \emptyset$ and hence $\hat{a} \notin D$.
	Thus, $a \notin \bigcap_D E_Y$.
	Consequently, $X \nNc{\varepsilon}{} a$, as desired.
\end{proof}
Similar to Corollary \ref{corollary:RMMU_extensible}, we have the following result. Its proof is left to the reader.
\begin{corollary}
	Let $F = \tuple{\A, \rightarrow}$ be any AAF and $\varepsilon$ an extension-based semantics. If $\varepsilon(F)$ is closed under reduced meets modulo any ultrafilter, $a \in \A$ and $X_i \subseteq \A$ for each $i \in I \neq \emptyset$ then, for any ultrafilter $D$ over $I$, $\set{i \in I \mid X_i \nNc{\varepsilon}{} a} \in D$ implies $\bigcap_D X_i \nNc{\varepsilon}{} a$, and hence the set $\set{X \subseteq \A \mid X \nNc{\varepsilon}{} a}$ is closed under the operator $\bigcap_D$.
\end{corollary}
As shown in Theorem \ref{theorem:close under meet_fundamental semantics}, all semantics in Definition \ref{definition:graded_extensions} but $\varepsilon_{na}^{\ell}$ and $\varepsilon_{pr}^{\ell mn}$ are closed under reduced meets modulo any ultrafilter w.r.t. finitary AAFs. Consequently, Theorems \ref{theorem:epsilon_extensible} and \ref{theorem:epsilon_inference} can be applied to these concrete semantics. Formally, we have
\begin{corollary} \label{corollary:epsilon_extensible}
	For any finitary AAF and $\varepsilon \in \set{\textit{Def}^{mn}, \varepsilon_{cf}^{\ell}, \varepsilon_{ad}^{\ell mn}, \varepsilon_{co}^{\ell mn}, \varepsilon_{gr}^{\ell mn}, \varepsilon_{stb}^{\ell mn}}$, $\varepsilon$-extensibility is compact and $\varepsilon$-inference is co-compact\footnote{For $\varepsilon_{cf}^{\ell}$ and $\varepsilon_{gr}^{\ell mn}$, it doesn't depend on the assumption that $F$ is finitary.}.
\end{corollary}
In Corollaries \ref{corollary:structure_cf_ad} and \ref{corollary:structure_co}, we have shown that $\textit{Def}^{mn}$, $\varepsilon_{cf}^{\ell}$, $\varepsilon_{ad}^{\ell mn}$ and $\varepsilon_{co}^{\ell mn}$ are dcpos with $\textit{sup}S = \bigcup S$ for any directed subset $S$. We may lift these results to a metatheoretical level as follows.
\begin{theorem}[Dcpo]\label{theorem:dcpo}
	Let $F = \tuple{\A, \rightarrow}$ be any AAF and $\varepsilon$ an extension-based semantics. If $\varepsilon(F)$ is closed under reduced meets modulo any ultrafilter then $\tuple{\varepsilon(F), \subseteq}$ is a dcpo with $\textit{sup}S = \bigcup S$ for any directed subset $S$ of $\varepsilon(F)$ whenever $\varepsilon(F) \neq \emptyset$.
\end{theorem}
\begin{proof}
	Let $S = \set{E_i \mid i \in I}$ be any directed subset of $\varepsilon(F)$. By Observation \ref{observation:RMMU_continuity}, $\bigcup S = \bigcap_D E_i$ for some ultrafilter $D$ over $I$. By the assumption, we have $\bigcap_D E_i \in \varepsilon(F)$. Then it is easy to see that $\textit{sup}S = \bigcup S$.
\end{proof}
For each extension $E$ in $\varepsilon_{cf}^{\ell}(F)$ (or, $\varepsilon_{ad}^{\ell mn}(F)$, $\varepsilon_{co}^{\ell mn}(F)$), we have shown that  $E$ can be extended to a maximal one in $\varepsilon_{cf}^{\ell}(F)$ ($\varepsilon_{ad}^{\ell mn}(F)$, $\varepsilon_{co}^{\ell mn}(F)$, resp.), see Corollaries \ref{corollary:Lindenbaum} and \ref{corollary:existence_pr}. For $\varepsilon_{stb}^{\ell mn}(F)$, it also holds trivially because that $\varepsilon_{stb}^{\ell mn}(F)$ is an antichain (i.e., for any $E_1$, $E_2$ in $\varepsilon_{stb}^{\ell mn}(F)$, $E_1 \subseteq E_2$ implies $E_1 = E_2$). These results may be lifted to a metatheoretical level as follows:
\begin{theorem}[Lindenbaum property] \label{theorem:Lindenbaum_RM}
	Let $F = \tuple{\A, \rightarrow}$ be any AAF and $\varepsilon_\sigma$ an extension-based semantics. If $\varepsilon_\sigma(F)$ is closed under reduced meets modulo any ultrafilter, then
	\begin{itemize}
		\item[a.] The set $\varepsilon_\sigma(F)$ has maximal elements w.r.t. $\subseteq$ whenever $\varepsilon_\sigma(F) \neq \emptyset$.
		\item[b.] Each element in $\varepsilon_\sigma(F)$ can be extended to a maximal element in $\varepsilon_\sigma(F)$.
		\item[c.] The clauses ($a$) and ($b$) also hold for the range related semantics $\varepsilon_{rr \sigma}^\eta$ associated with $\varepsilon_\sigma$.
	\end{itemize}
\end{theorem}
\begin{proof}
	($a$) By Theorem \ref{theorem:dcpo}, the poset $\tuple{\varepsilon_\sigma(F), \subseteq}$ is a dcpo. Hence  each chain has an upper bound. Then, by Zorn's lemma, $\varepsilon_\sigma(F)$ has maximal elements w.r.t. $\subseteq$.
	\par ($b$) Similar to ($a$) but consider the poset $\tuple{\set{X \in \varepsilon_\sigma(F) \mid E \subseteq X}, \subseteq}$ instead of $\tuple{\varepsilon_\sigma(F), \subseteq}$ for each $E \in \varepsilon_\sigma(F)$.
	\par ($c$)
	Let $\set{E_i \mid i \in I}$ be a chain in $\tuple{\varepsilon_{rr \sigma}^\eta(F), \subseteq}$. Clearly, it has an upper bound $E$ in the dcpo $\tuple{\varepsilon_\sigma(F), \subseteq}$. Since $\varepsilon_{rr \sigma}^\eta(F)$ is up-closed, $E \in \varepsilon_{rr \sigma}^\eta(F)$, and hence, it is also an upper bound of $\set{E_i \mid i \in I}$ in $\tuple{\varepsilon_{rr \sigma}^\eta(F), \subseteq}$. Further, we get the conclusion by the same reasoning as ($a$) and ($b$).
\end{proof}
\subsection{Properties of the operator $\bigcap_D$ related to $\varepsilon_{\textit{max}, \sigma}$}
Given an extension-based semantics $\varepsilon_{\sigma}$, we can define the derived semantics $\varepsilon_{\textit{max}, \sigma}$ by
\begin{eqnarray*}
	\varepsilon_{\textit{max}, \sigma}(F) \triangleq \set{E \mid E \text{~is maximal in~} \varepsilon_{\sigma}(F)} \text{~for each AAF~} F.
\end{eqnarray*}
For example, the semantics $\varepsilon_{na}^\ell$, $\varepsilon_{pr, \textit{Dung}}^{\ell mn}$, $\varepsilon_{pr}^{\ell mn}$, $\varepsilon_{id}^{\ell mn}$ and $\varepsilon_{eg}^{\ell mn \eta}$ are all defined in this pattern, where $\varepsilon_{id}^{\ell mn}$ and $\varepsilon_{eg}^{\ell mn \eta}$ will be introduced in the next section. Clearly, by Theorem \ref{theorem:Lindenbaum_RM}, for any class $\Omega$ of AAFs, the semantics $\varepsilon_{\textit{max}, \sigma}$ agrees with $\varepsilon_{\sigma}$ on the universal definability w.r.t. $\Omega$ whenever $\varepsilon_{\sigma}$ is closed under reduced meets modulo any ultrafilter w.r.t. $\Omega$. At this point, a natural problem arises, that is, whether $\varepsilon_{\textit{max}, \sigma}(F)$ is closed under reduced meets modulo any ultrafilter. We provide the example below to illustrate that the derived semantics $\varepsilon_{\textit{max}, \sigma}$ isn't closed under reduced meets modulo any ultrafilter in general.
\begin{example} \label{Ex:counterexample for the derived semantics}
	Consider the AAF $F \triangleq \tuple{\omega \times \omega, \rightarrow}$ and $\ell = m = n = 1$, where $\rightarrow$ is defined as $\tuple{k, j} \rightarrow \tuple{k', j'}$ iff $k \neq k'$. Then it is easy to see that 
	\begin{eqnarray*}
		\varepsilon_{na}^\ell(F) = \varepsilon_{pr, \textit{Dung}}^{\ell mn}(F) = \varepsilon_{pr}^{\ell mn}(F) = \varepsilon_{stb}^{\ell mn}(F) = \set{\set{\tuple{k, j} \mid j < \omega} \mid k < \omega}.
	\end{eqnarray*}
	For each $i < \omega$, put $E_i \triangleq \set{\tuple{i, j} \mid j < \omega}$. Consider the well-known Fr$\acute{e}$chet filter $\mathcal{F} = \set{X \in \pw{\omega} \mid \omega - X \text{~ is finite}}$. It is easy to see that $\mathcal{F}$ has the finite intersection property. Let $D$ be any ultrafilter over $\omega$ such that $\mathcal{F} \subseteq D$. Then, for any $X \in \wp_f(\omega)$, we have $X \notin D$ due to $\omega - X \in \mathcal{F} \subseteq D$. Since $\set{ i \mid \tuple{k, j} \in E_i} = \set{k} \notin D$ for any $\tuple{k, j} \in \omega \times \omega$, we have $\bigcap_D E_i = \emptyset$. Thus, none of $\varepsilon_{na}^\ell(F)$, $\varepsilon_{pr, \textit{Dung}}^{\ell mn}(F)$, $\varepsilon_{pr}^{\ell mn}(F)$ and $\varepsilon_{stb}^{\ell mn}(F)$ is closed under the operator $\bigcap_D$. Notice that, since $F$ isn't finitary, it doesn't contradict Theorem \ref{theorem:close under meet_fundamental semantics} that $\varepsilon_{stb}^{\ell mn}(F)$ is not closed under $\bigcap_D$.
\end{example}
In the following, we intend to consider the related properties of the operator $\bigcap_D$ w.r.t. the semantics $\varepsilon_{na}^{\ell}$, $\varepsilon_{pr, \textit{Dung}}^{\ell mn}$ and $\varepsilon_{pr}^{\ell mn}$ in turn. To this end, we introduce the notion below.
\begin{definition}
	An AAF $F = \tuple{\A, \rightarrow}$ is said to be co-finitary if the set $\set{b \in \A \mid a \rightarrow b}$ is finite for each $a \in \A$.
\end{definition}
\begin{observation}\label{observation:ultrafilter}
	Let $A$ and $I (\neq \emptyset)$ be any set and $a \in A$. If $E_i \subseteq A$ for each $i \in I$ then $\bigcap_D (E_i \cup \set{a}) = \bigcap_D E_i \cup \set{a}$ for any ultrafilter $D$ over $I$.
\end{observation}
\begin{proof}
	It immediately follows from the facts that $\set{i \in I \mid a \in E_i \cup \set{a}} = I$ and $\set{i \in I \mid b \in E_i \cup \set{a}} = \set{i \in I \mid b \in E_i}$ for any $b \neq a$.
\end{proof}
\begin{lemma}\label{lemma:cf-extensible}
	Let $F = \tuple{\A, \rightarrow}$ be any finitary and co-finitary AAF, $I \neq \emptyset$ and $D$ an ultrafilter over $I$, and let $E_i \subseteq \A$ for each $i \in I$ such that $\set{i \in I \mid E_i \in \varepsilon_{cf}^\ell(F)} \in D$. Then, for each $a \in \A$, $\bigcap_D (E_i \cup \set{a}) \in \varepsilon_{cf}^\ell(F)$ iff $\set{i \in I \mid E_i \cup \set{a} \in \varepsilon_{cf}^\ell(F)} \in D$.
\end{lemma}
\begin{proof}
	It immediately follows from Lemma \ref{lemma:RMMU_ultrafilter} that the right implies the left. Next we consider another direction.
	Assume that $\set{i \in I \mid E_i \cup \set{a} \in \varepsilon_{cf}^\ell(F)} \notin D$. Then
	\begin{eqnarray*}
		I_0 \triangleq \set{i \in I \mid E_i \cup \set{a} \notin \varepsilon_{cf}^\ell(F)} \cap \set{i \in I \mid E_i \in \varepsilon_{cf}^\ell(F)} \in D.
	\end{eqnarray*}
	For each $i \in I_0$, we have either $a \notin N_\ell(E_i \cup \set{a})$ or $\exists b \in E_i (b \notin N_\ell(E_i \cup \set{a}))$. Put
	\begin{eqnarray*}
		I_1 \triangleq \set{i \in I_0 \mid a \notin N_\ell(E_i \cup \set{a})} \text{~and~} I_2 \triangleq \set{i \in I_0 \mid \exists b \in E_i (b \notin N_\ell(E_i \cup \set{a}))}.
	\end{eqnarray*}
	Then $I_1 \in D$ or $I_2 \in D$ due to $I_1 \cup I_2 = I_0 \in D$.\\
	\\
	Case 1 : $I_1 \in D$.\\
	First, we consider the case that $a \rightarrow a$ in $F$. If $\ell = 1$ then, by Observation \ref{observation:ultrafilter}, $\bigcap_D (E_i \cup \set{a}) \notin \varepsilon_{cf}^\ell(F)$ immediately follows. Next we consider the case that $\ell > 1$. For each $i \in I_1$, since $a \notin N_\ell(E_i \cup \set{a})$, there exists a set $X \in [a^-]^{\ell - 1}$ such that $X \subseteq E_i$ and $a \notin X$. Thus
	\begin{eqnarray*}
		I_1 \subseteq \bigcup_{a \notin X \in [a^-]^{\ell - 1}} \set{i \mid X \subseteq E_i}.
	\end{eqnarray*}
	Since $I_1 \in D$ and $F$ is finitary, $\set{i \mid a \notin X_0 \subseteq E_i} \in D$ for some $X_0 \in [a^-]^{\ell - 1}$. Hence, $X_0 \subseteq \bigcap_D E_i$ and $X_0 \cup \set{a} \in [a^-]^{\ell}$. Thus, $\bigcap_D (E_i \cup \set{a}) = \bigcap_D E_i \cup \set{a} \notin \varepsilon_{cf}^\ell(F)$. We may deal with another case that $a \nrightarrow a$ in $F$ similarly but considering $X_0 \in [a^-]^{\ell}$ instead of $X_0 \in [a^-]^{\ell - 1}$.\\
	\\
	Case 2 : $I_2 \in D$.\\
	Let $i \in I_2$. Then there exists an argument $b \in E_i$ such that $b \notin N_\ell(E_i \cup \set{a})$. Clearly, $b \in N_\ell(E_i)$ due to $E_i \in \varepsilon_{cf}^\ell(F)$. Hence, $a \notin E_i$ and $a \rightarrow b$ in $F$. For each $b \in \A$ with $a \rightarrow b$, set
	\begin{eqnarray*}
		I_b \triangleq \set{i \mid b \notin N_\ell(E_i \cup \set{a}), a \notin E_i \text{~and~} b \in E_i}.
	\end{eqnarray*}
	Then $I_2 \subseteq \bigcup \set{I_b \mid a \rightarrow b}$. Since $F$ is co-finitary, there exists an argument $b_0 \in \A$ such that $a \rightarrow b_0$ and $I_{b_0} \in D$. Clearly, $b_0 \in \bigcap_D E_i$ because of $I_{b_0} \subseteq \hat{b}_0$. If $\ell = 1$ then $\bigcap_D (E_i \cup \set{a}) \notin \varepsilon_{cf}^\ell(F)$ due to $\set{a, b_0} \subseteq \bigcap_D E_i \cup \set{a} = \bigcap_D (E_i \cup \set{a})$. Next we consider the case that $\ell > 1$. For each $i \in I_{b_0}$, there exists a set $X \in [b_0^-]^{\ell - 1}$ such that $X \subseteq E_i$ due to $b_0 \notin N_\ell(E_i \cup \set{a})$ and $a \rightarrow b_0$. Thus
	\begin{eqnarray*}
		I_{b_0} \subseteq \bigcup_{X \in [b_0^-]^{\ell - 1}} \set{i \mid X \subseteq E_i, a \notin E_i \text{~and~} b_0 \in E_i}.
	\end{eqnarray*}
	Since $F$ is finitary and $I_{b_0} \in D$, there exists a set $X_0 \in [b_0^-]^{\ell - 1}$ such that 
	\begin{eqnarray*}
		\set{i \mid X_0 \subseteq E_i, a \notin E_i \text{~and~} b_0 \in E_i} \in D.
	\end{eqnarray*}
	Thus $X_0 \cup \set{b_0} \subseteq \bigcap_D E_i \subseteq \bigcap_D (E_i \cup \set{a})$. Hence $\bigcap_D (E_i \cup \set{a}) \notin \varepsilon_{cf}^\ell(F)$ due to $\set{a} \cup X_0 \in [b_0^-]^{\ell}$.
\end{proof}
\begin{proposition}\label{proposition:RMMU_na}
	Let $F = \tuple{\A, \rightarrow}$ be any finitary and co-finitary AAF, $I \neq \emptyset$ and $D$ an ultrafilter over $I$. Let $E_i \subseteq \A$ for each $i \in I$, if $\set{i \in I \mid E_i \in \varepsilon_{na}^\ell(F)} \in D$ then $\bigcap_D E_i \in \varepsilon_{na}^\ell(F)$. Hence $\varepsilon_{na}^\ell$ is closed under reduced meets modulo any ultrafilter w.r.t. finitary and co-finitary AAFs.
\end{proposition}
\begin{proof}
	Assume $\bigcap_D E_i \notin \varepsilon_{na}^\ell(F)$. By Lemma \ref{lemma:RMMU_ultrafilter}, $\bigcap_D E_i \in \varepsilon_{cf}^\ell(F)$ due to
	\begin{eqnarray*}
		\set{i \in I \mid E_i \in \varepsilon_{na}^\ell(F)} \subseteq \set{i \in I \mid E_i \in \varepsilon_{cf}^\ell(F)} \in D,
	\end{eqnarray*}
	and hence $\bigcap_D E_i \subset E$ for some $E \in \varepsilon_{cf}^\ell(F)$. Thus, there exists an argument, say $a$, such that $a \in E - \bigcap_D E_i$. Since $a \notin \bigcap_D E_i$, $\hat{a} \notin D$ and hence $\set{i \in I \mid a \notin E_i} \in D$. Thus $J \triangleq \set{i \in I \mid E_i \in \varepsilon_{na}^\ell(F)} \cap \set{i \in I \mid a \notin E_i} \in D$. For each $i \in J$, we have $E_i \cup \set{a} \notin \varepsilon_{cf}^\ell(F)$ due to $E_i \in \varepsilon_{na}^\ell(F)$. Thus, $\set{i \in I \mid E_i \cup \set{a} \notin \varepsilon_{cf}^\ell(F)} \in D$ because of $J \subseteq \set{i \in I \mid E_i \cup \set{a} \notin \varepsilon_{cf}^\ell(F)}$. Since 
	\begin{eqnarray*}
		\set{i \in I \mid E_i \in \varepsilon_{na}^\ell(F)} \subseteq \set{i \in I \mid E_i \in \varepsilon_{cf}^\ell(F)} \in D,
	\end{eqnarray*}
	by Lemmas \ref{lemma:cf-extensible} and \ref{lemma:properties_cf_def}($d$), a contradiction arises because of $\bigcap_D(E_i \cup \set{a}) = \bigcap_D E_i \cup \set{a} \subseteq E \in \varepsilon_{cf}^\ell(F)$. 
\end{proof}
In the following, for finitary and co-finitary AAFs, we will provide sufficient and necessary conditions so that $\bigcap_D E_i \in \varepsilon_{pr, \textit{Dung}}^{\ell mn}(F)$ (or, $\varepsilon_{pr}^{\ell mn}(F)$) whenever each $E_i \in \varepsilon_{pr, \textit{Dung}}^{\ell mn}(F)$ ($\varepsilon_{pr}^{\ell mn}(F)$, resp.).
\begin{proposition} \label{proposition:RMMU_pr_Dung}
	Let $F = \tuple{\A, \rightarrow}$ be any finitary and co-finitary AAF, $I \neq \emptyset$ and $D$ an ultrafilter over $I$, and let $E_i \subseteq \A$ for each $i \in I$. If $\set{i \in I \mid E_i \in \varepsilon_{pr, \textit{Dung}}^{\ell mn}(F)} \in D$ (in particular, $E_i \in \varepsilon_{pr, \textit{Dung}}^{\ell mn}(F)$ for each $i \in I$), then the following are equivalent:
	\begin{itemize}
		\item[a.]  $\bigcap_D E_i \in \varepsilon_{pr, \textit{Dung}}^{\ell mn}(F)$.
		\item[b.] For each $a \in \A$, $a \in D_n^m(\bigcap_D (E_i \cup \set{a}))$ whenever $\bigcap_D (E_i \cup \set{a})$ is $\varepsilon_{ad}^{\ell mn}$-extensible.
	\end{itemize}
\end{proposition}
\begin{proof}
	The assertion ($a \Rightarrow b$) is trivial. Next we show ($b \Rightarrow a$). Assume that $\bigcap_D E_i \notin \varepsilon_{pr, \textit{Dung}}^{\ell mn}(F)$. By Lemma \ref{lemma:RMMU_ultrafilter}, $\bigcap_D E_i \in \varepsilon_{ad}^{\ell mn}(F)$ because of 
	\begin{eqnarray*}
		\set{i \in I \mid E_i \in \varepsilon_{pr, \textit{Dung}}^{\ell mn}(F)} \subseteq \set{i \in I \mid E_i \in \varepsilon_{ad}^{\ell mn}(F)} \in D.
	\end{eqnarray*}
	Thus, there exists a set $E \in \varepsilon_{ad}^{\ell mn}(F)$ such that $\bigcap_D E_i \subset E$. Let $a \in E - \bigcap_D E_i$. Then, $\set{i \in I \mid a \in E_i} \notin D$, and hence $\set{i \in I \mid a \notin E_i} \in D$. Thus 
	\begin{eqnarray*}
		I_a \triangleq \set{i \in I \mid a \notin E_i} \cap \set{i \in I \mid E_i \in \varepsilon_{pr, \textit{Dung}}^{\ell mn}(F)} \in D.
	\end{eqnarray*}
	For each $i \in I_a$, since $E_i \in \varepsilon_{pr, \textit{Dung}}^{\ell mn}(F)$ and $a \notin E_i$, we have $E_i \cup \set{a} \notin \varepsilon_{ad}^{\ell mn}(F)$, and hence either $E_i \cup \set{a} \notin \varepsilon_{cf}^{\ell}(F)$ or $E_i \cup \set{a} \notin \textit{Def}^{mn}(F)$. Put 
	\begin{eqnarray*}
			I_1 \triangleq \set{i \in I \mid E_i \cup \set{a} \notin \varepsilon_{cf}^{\ell}(F)} \text{~and~} 
			I_2 \triangleq \set{i \in I \mid E_i \cup \set{a} \notin \textit{Def}^{mn}(F)}.
	\end{eqnarray*}
	Then $I_a \subseteq I_1 \cup I_2$. Thus $I_1 \in D$ or $I_2 \in D$. Since $\bigcap_D(E_i \cup \set{a}) = \bigcap_D E_i \cup \set{a} \subseteq E \in \varepsilon_{cf}^{\ell}(F)$, by Lemma \ref{lemma:cf-extensible}, $I_1 \notin D$. Hence $I_2 \in D$ and 
	\begin{eqnarray*}
		I_3 \triangleq I_2 \cap \set{i \in I \mid E_i \in \varepsilon_{pr, \textit{Dung}}^{\ell mn}(F)} \in D.
	\end{eqnarray*}
	For each $i \in I_3$, $a \notin D_n^m(E_i \cup \set{a})$ because of $E_i \subseteq D_n^m(E_i) \subseteq D_n^m(E_i \cup \set{a})$. Then, $I_3 \subseteq \set{i \in I \mid a \notin D_n^m(E_i \cup \set{a})} \in D$, that is, $a \notin \bigcap_D D_n^m(E_i \cup \set{a})$. Thus, by Lemma \ref{lemma:RMMU_Distributive}($c$), $a \notin D_n^m(\bigcap_D (E_i \cup \set{a}))$. Further, by the condition ($b$), $\bigcap_D (E_i \cup \set{a})$ isn't $\varepsilon_{ad}^{\ell mn}$-extensible, which contradicts $\bigcap_D (E_i \cup \set{a}) = \bigcap_D E_i \cup \set{a} \subseteq E \in \varepsilon_{ad}^{\ell mn}(F)$, as desired.
\end{proof}
\begin{lemma} \label{lemma:def-extensible}
	Let $F = \tuple{\A, \rightarrow}$ be any finitary and co-finitary AAF, $I \neq \emptyset$ and $D$ an ultrafilter over $I$, and let $E_i \subseteq \A$ for each $i \in I$ such that $\set{i \in I \mid D_n^m(E_i) \subseteq E_i} \in D$. Then, for each $a \in \A$,
	\begin{eqnarray*}
		D_n^m(\bigcap_D (E_i \cup \set{a})) \subseteq \bigcap_D (E_i \cup \set{a}) \text{~iff~} \set{i \in I \mid D_n^m(E_i \cup \set{a}) \subseteq E_i \cup \set{a}} \in D.
	\end{eqnarray*}
\end{lemma}
\begin{proof}
	It immediately follows from Lemma \ref{lemma:RMMU_ultrafilter}($d$) that the right implies the left. Next we deal with another direction.
	Assume that $\set{i \in I \mid D_n^m(E_i \cup \set{a}) \subseteq E_i \cup \set{a}} \notin D$. Then
	\begin{eqnarray*}
		I_0 \triangleq \set{i \in I \mid D_n^m(E_i \cup \set{a}) \nsubseteq E_i \cup \set{a}} \cap \set{i \in I \mid D_n^m(E_i) \subseteq E_i} \in D.
	\end{eqnarray*}
	For each $i \in I_0$, we may choose arbitrarily and fix an argument $b_i$ such that $b_i \in D_n^m(E_i \cup \set{a})$ but $b_i \notin E_i \cup \set{a}$. For each $i \in I_0$, the set $w(b_i)$ is defined as
	\begin{eqnarray*}
			w(b_i) \triangleq \set{\tuple{c, Y} \bigg|
		 		\begin{array}{*{20}{c}}
					{ a \rightarrow c \text{~and~} Y \subseteq E_i \text{~such that}} \\ 
					{ Y \cup \set{a} \in [c^-]^n \text{and~} c \in X \text{~for some~} X \in [b_i^-]^m}
				\end{array}
			}.
	\end{eqnarray*}
	Moreover, the set $S$ is defined as
	\begin{eqnarray*}
				S \triangleq \set{\tuple{c, Y} \mid a \rightarrow c \text{~and~} Y \subseteq \A \text{~such that~} Y \cup \set{a} \in [c^-]^n}.
	\end{eqnarray*}
	\\ \textbf{Claim 1~~} For each $i \in I_0$, $w(b_i) \neq \emptyset$.\\
	Since $b_i \in D_n^m(E_i \cup \set{a})$, $X \nsubseteq N_n(E_i \cup \set{a})$ for each $X \in [b_i^-]^m$. On the other hand, it follows from $b_i \notin E_i$ and $D_n^m(E_i) \subseteq E_i$ that $X_0 \subseteq N_n(E_i)$ for some $X_0 \in [b_i^-]^m$. Then it isn't difficult to see that $w(b_i) \neq \emptyset$ due to $X_0 \subseteq N_n(E_i)$ and $X_0 \nsubseteq N_n(E_i \cup \set{a})$.\\
	\\
	Clearly, by Claim 1, $w(b_i) \in \pw{S} - \set{\emptyset}$ for each $i \in I_0$ and
	\begin{eqnarray*}
			I_0 = \bigcup_{\emptyset \neq Z \in \pw{S}} \set{i \in I_0 \mid w(b_i) = Z}.
	\end{eqnarray*}
	Since $F$ is finitary and co-finitary, $|\pw{S}| < \omega$ due to $|S| < \omega$. Further, it follows from $I_0 \in D$ that there exists a nonempty set $Z_0 \in \pw{S}$ such that
	\begin{eqnarray*}
				I_1\triangleq \set{i \in I_0 \mid w(b_i) = Z_0} \in D.
	\end{eqnarray*}
	\\ \textbf{Claim 2~~} For each $i \in I_1$, $X \nsubseteq N_n(\bigcap_D E_i \cup \set{a})$ for each $X \in [b_i^-]^m$, and hence $b_i \in D_n^m(\bigcap_D E_i \cup \set{a})$.\\
	Let $X \in [b_i^-]^m$ with $i \in I_1$. Then $X \nsubseteq N_n(E_i \cup \set{a})$ because of $X \in [b_i^-]^m$ and $b_i \in D_n^m(E_i \cup \set{a})$. Put
	\begin{eqnarray*}
			I_2 \triangleq \set{i \in I_1 \mid X \subseteq N_n(E_i)} \text{~and~} I_3 \triangleq \set{i \in I_1 \mid X \nsubseteq N_n(E_i)}.
	\end{eqnarray*}
	Since $I_1 = I_2 \cup I_3 \in D$, we have either $I_2 \in D$ or $I_3 \in D$.\\
	\\
	Case 1 : $I_2 \in D$.\\
	Let $i \in I_2$. Since $X \subseteq N_n(E_i)$ and $X \nsubseteq N_n(E_i \cup \set{a})$, there exist $c \in X$ and $Y \subseteq E_i$ such that $a \rightarrow c$ and $Y \cup \set{a} \in [c^-]^n$, that is, $\tuple{c, Y} \in w(b_i) = Z_0$. Since $I_2 \subseteq I_1 = \set{i \in I_0 \mid w(b_i) = Z_0}$, by the definition of $w(b_i)$, we have $Y \subseteq E_k$ for each $k \in I_2$. Further, $Y \subseteq \bigcap_D E_i$ due to $I_2 \in D$. Thus $Y \cup \set{a} \subseteq \bigcap_D E_i \cup \set{a}$, and hence $X \nsubseteq N_n(\bigcap_D E_i \cup \set{a})$ because of $c \in X$ and $Y \cup \set{a} \in [c^-]^n$.\\
	\\
	Case 2 : $I_3 \in D$.\\
	For each $i \in I_3$, since $X \nsubseteq N_n(E_i)$, there exists an argument $c \in X$ such that $c \notin N_n(E_i)$ and hence $Y \subseteq E_i$ for some $Y \in [c^-]^n$. Then
	\begin{eqnarray*}
			I_3 = \bigcup_{V \in H} \set{i \in I_3 \mid V \subseteq E_i} \text{~with~} H \triangleq \bigcup_{d \in X} [d^-]^n.
	\end{eqnarray*}
	Since $F$ is finitary, the set $H$ is finite due to $|X| = m$. Thus $\set{i \in I_3 \mid Y_0 \subseteq E_i} \in D$ for some $Y_0 \in H$. Hence, $Y_0 \subseteq \bigcap_D E_i$. Then $d_0 \notin N_n(\bigcap_D E_i)$ due to $Y_0 \in [d_0^-]^n$ for some $d_0 \in X$. Thus $X \nsubseteq N_n(\bigcap_D E_i)$, and hence $X \nsubseteq N_n(\bigcap_D E_i \cup \set{a})$ because of $N_n(\bigcap_D E_i \cup \set{a}) \subseteq N_n(\bigcap_D E_i)$.\\
	\\
	So far, we have shown that $b_i \in D_n^m(\bigcap_D E_i \cup \set{a}) = D_n^m(\bigcap_D (E_i \cup \set{a}))$ for each $i \in I_1(\neq \emptyset)$. To complete the proof, in the following, we will show that there exists $i \in I_1$ such that $b_i \notin \bigcap_D (E_i \cup \set{a})$. Since $b_i \neq a$ for each $i \in I_1$, by Observation \ref{observation:ultrafilter}, it is enough to prove that $b_i \notin \bigcap_D E_i$ for some $i \in I_1$. Put
	\begin{eqnarray*}
			W \triangleq \set{b_i \mid i \in I_1}.
	\end{eqnarray*}
	For each $b \in W$, by the definition of $I_1$, $w(b) = Z_0 \neq \emptyset$, and hence there exists an argument $c$ such that $a \rightarrow c$ and $c \rightarrow b$. Thus
	\begin{eqnarray*}
			W \subseteq W' \triangleq \set{x \mid a \rightarrow c \text{~and~} c \rightarrow x \text{~for some~} c}.
	\end{eqnarray*}
	Since $F$ is co-finitary, $W'$ is finite and hence so is $W$. For each $b \in W$, set
	\begin{eqnarray*}
			I_b \triangleq \set{i \in I_1 \mid b \notin E_i}.
	\end{eqnarray*}
	Since $b_i \notin E_i$ for each $i \in I_1$, $I_1 = \bigcup_{b \in W} I_b$. Since $|W| < \omega$ and $I_1 \in D$, there exists an argument $b_0 \in W$ such that $I_{b_0} \in D$. Thus, $\set{i \in I_1 \mid b_0 \notin E_i} \in D$. Since $\set{i \in I_1 \mid b_0 \notin E_i} \subseteq \set{i \in I \mid b_0 \notin E_i}$, we have $\set{i \in I \mid b_0 \in E_i} \notin D$, that is, $b_0 \notin \bigcap_D E_i$, as desired.
\end{proof}
\begin{proposition}\label{proposition:RMMU_pr}
	Let $F = \tuple{\A, \rightarrow}$ be any finitary and co-finitary AAF, $I \neq \emptyset$ and $D$ an ultrafilter over $I$. Let $E_i \subseteq \A$ for each $i \in I$, if $\set{i \in I \mid E_i \in \varepsilon_{pr}^{\ell mn}(F)} \in D$ (in particular, $E_i \in \varepsilon_{pr}^{\ell mn}(F)$ for each $i \in I$) then the following are equivalent:
	\begin{itemize}
		\item[a.]  $\bigcap_D E_i \in \varepsilon_{pr}^{\ell mn}(F)$.
		\item[b.] For each $a \in \A$, $a \in D_n^m(\bigcap_D (E_i \cup \set{a})) \subseteq \bigcap_D (E_i \cup \set{a})$ whenever $\bigcap_D (E_i \cup \set{a})$ is $\varepsilon_{co}^{\ell mn}$-extensible.
	\end{itemize}
\end{proposition}
\begin{proof}
	The assertion ($a \Rightarrow b$) is trivial. Next we show ($b \Rightarrow a$).
	Assume that $\bigcap_D E_i \notin \varepsilon_{pr}^{\ell mn}(F)$. By Lemma \ref{lemma:RMMU_ultrafilter}, $\bigcap_D E_i \in \varepsilon_{co}^{\ell mn}(F)$ because of $\set{i \in I \mid E_i \in \varepsilon_{pr}^{\ell mn}(F)} \subseteq \set{i \in I \mid E_i \in \varepsilon_{co}^{\ell mn}(F)} \in D$. Hence, $\bigcap_D E_i \subset E$ for some $E \in \varepsilon_{co}^{\ell mn}(F)$. Let $a \in E - \bigcap_D E_i$. Then $\set{i \in I \mid a \notin E_i} \in D$. Thus 
	\begin{eqnarray*}
		I_a \triangleq \set{i \in I \mid a \notin E_i} \cap \set{i \in I \mid E_i \in \varepsilon_{pr}^{\ell mn}(F)} \in D.
	\end{eqnarray*}
	For each $i \in I_a$, $E_i \cup \set{a} \notin \varepsilon_{co}^{\ell mn}(F)$ due to $E_i \in \varepsilon_{pr}^{\ell mn}(F)$ and $a \notin E_i$. Thus, for $i \in I_a$, at least one of the following holds:
	\begin{enumerate}
		\item[(1)] $E_i \cup \set{a} \notin \varepsilon_{cf}^{\ell}(F)$,
		\item[(2)] $E_i \cup \set{a} \notin \textit{Def}^{mn}(F)$ and
		\item[(3)] $D_n^m(E_i \cup \set{a}) \nsubseteq E_i \cup \set{a}$.
	\end{enumerate}
	Put $I_k \triangleq \set{i \in I \mid E_i \text{~satisfies the condition~} (k)}$ for $k = 1, 2, 3$.
	Thus, $I_a \subseteq I_1 \cup I_2 \cup I_3$ and hence $I_k \in D$ for some $k \in \set{1, 2, 3}$. Since $\bigcap_D (E_i \cup \set{a}) = \bigcap_D E_i \cup \set{a} \subseteq E \in \varepsilon_{cf}^{\ell}(F)$, by Lemma \ref{lemma:cf-extensible}, $I_1 \notin D$. By the condition ($b$), we have $a \in D_n^m(\bigcap_D (E_i \cup \set{a}))$. Hence, by Lemma \ref{lemma:RMMU_Distributive}($c$), similar to Proposition \ref{proposition:RMMU_pr_Dung}, we also have $I_2 \notin D$. Thus $I_3 \in D$. Then, by Lemma \ref{lemma:def-extensible}, we get $D_n^m(\bigcap_D (E_i \cup \set{a})) \nsubseteq \bigcap_D (E_i \cup \set{a})$ due to $\set{i \in I \mid E_i \in \varepsilon_{pr}^{\ell mn}(F)} \subseteq \set{i \in I \mid D_n^m(E_i) \subseteq E_i} \in D$, which contradicts the condition ($b$) and $\bigcap_D E_i \cup \set{a} \subseteq E \in \varepsilon_{co}^{\ell mn}(F)$, as desired.
\end{proof}
We intend to end this subsection by discussing the closeness of abstract semantics $\varepsilon_{\textit{max}, \sigma}$ under the operator $\bigcap_D$. Before giving the result, we introduce the dual notion of $\varepsilon$-extensibility as follows.
\begin{definition} \label{definition:anti-epsilon set}
	Let $F = \tuple{\A, \rightarrow}$ be any AAF and $\varepsilon$ an extension-based semantics. A finite subset $X$ of $\A$ is said to be an anti-$\varepsilon$ set if $X \nsubseteq E$ for each $E \in \varepsilon(F)$. We use the notation Anti-$\varepsilon(F)$ to denote the set of all minimal anti-$\varepsilon$ sets of $F$.
\end{definition}
In the situation that $\varepsilon(F)$ is closed under reduced meets modulo any ultrafilter, by Theorem \ref{theorem:epsilon_extensible}, it is easy to see that, for any $X \subseteq \A$, X is $\varepsilon$-extensible iff $\wp_f(X) \cap \text{Anti-}\varepsilon(F) = \emptyset$. 
\begin{definition} \label{definition:antiset finitary}
	Let $F = \tuple{\A, \rightarrow}$ be any AAF and $\varepsilon_\sigma$ an extension-based semantics. For each $a \in \A$, the set $\Gamma_{\varepsilon_\sigma}(a)$ is defined as $\Gamma_{\varepsilon_\sigma}(a) \triangleq \set{X \in \text{Anti-}\varepsilon_\sigma(F) \mid a \in X}$. For any nonempty $\Gamma_{\varepsilon_\sigma}(a)$, the equivalent relation $\sim_a$ over $\Gamma_{\varepsilon_\sigma}(a)$ is defined as, for each $X_1, X_2 \in \Gamma_{\varepsilon_\sigma}(a)$, $X_1 \sim_a X_2$ iff for each $E \in \varepsilon_{\textit{max}, \sigma}(F)$, $X_1 - \set{a} \subseteq E \Leftrightarrow X_2 - \set{a} \subseteq E$. We denote the corresponding quotient set by $\Gamma_{\varepsilon_\sigma}(a) / \!\sim_a$.
\end{definition}
\begin{definition}
	Let $F = \tuple{\A, \rightarrow}$ be any AAF and $\varepsilon_\sigma$ an extension-based semantics. $F$ is said to be Anti($\varepsilon_\sigma$)-finitary if for each $a \in \A$, either $\Gamma_{\varepsilon_\sigma}(a) = \emptyset$ or $|\Gamma_{\varepsilon_\sigma}(a) / \!\sim_a| < \omega$.
\end{definition}
\begin{example} \label{Ex:Anti-ad-co}
	Consider the AAF $F = \tuple{\omega, \rightarrow}$ with $\rightarrow \triangleq \set{\tuple{n+1, n} \mid n < \omega}$. Clearly, 
	\begin{eqnarray*}
		\begin{split}
			\varepsilon_{ad}^{111}(F) &= \set{\emptyset} \cup \set{\set{n + 2k \mid k < \omega} \mid n < \omega} \text{~and}\\
			\varepsilon_{co}^{111}(F) &= \set{\set{1, 3, 5, 7, \cdots}, \set{2, 4, 6, 8, \cdots}}.
		\end{split}
	\end{eqnarray*}
	Then it is easy to see that, for each $n < \omega$, although $\Gamma_{\varepsilon_{ad}^{111}}(n)$ and $\Gamma_{\varepsilon_{co}^{111}}(n)$ are infinite, both $\Gamma_{\varepsilon_{ad}^{111}}(n)/ \!\sim_n$ and $\Gamma_{\varepsilon_{co}^{111}}(n)/ \!\sim_n$ are finite. Hence, $F$ is both Anti($\varepsilon_{ad}^{111}$)-finitary and Anti($\varepsilon_{co}^{111}$)-finitary. In fact, it is straightforward to verify that $|\varepsilon_{\textit{max}, \sigma}(F)| < \omega$ implies $F$ is Anti($\varepsilon_\sigma$)-finitary. For the AAF $F$ in Example \ref{Ex:counterexample for the derived semantics}, it is obvious that it is neither Anti($\varepsilon_{ad}^{111}$)-finitary nor Anti($\varepsilon_{co}^{111}$)-finitary.
\end{example}
It isn't difficult to show that, for any finitary and co-finitary AAF $F = \tuple{\A, \rightarrow}$, the set $\set{X \in \text{Anti}(\varepsilon_{cf}^\ell) \mid a \in X}$ is finite for each $a \in \A$, and hence $F$ is Anti($\varepsilon_{cf}^\ell$)-finitary. Its proof may be found in Theorem \ref{theorem:Representation theorem III of l-cf}. But the converse isn't true.
\begin{example} \label{Ex:Anti-cf}
	Consider the AAF $F_1 = \tuple{\omega, \rightarrow}$ with $\rightarrow \triangleq \set{\tuple{0, n} \mid n < \omega}$. Then $F_1$ is Anti($\varepsilon_{cf}^{2}$)-finitary because of Anti-$\varepsilon_{cf}^{2}(F) = \emptyset$, but $F$ is not co-finitary. Similarly, for $F_2 = \tuple{\omega, \rightarrow}$ with $\rightarrow \triangleq \set{\tuple{n, 0} \mid n < \omega}$, $F_2$ is Anti($\varepsilon_{cf}^{2}$)-finitary but not finitary.
\end{example}
\begin{proposition} \label{proposition:RMMU_max}
	Let $F = \tuple{\A, \rightarrow}$ be any AAF and $\varepsilon_\sigma$ an extension-based semantics, and let $F$ be Anti($\varepsilon_\sigma$)-finitary. If $\varepsilon_\sigma(F)$ is closed under reduced meets modulo any ultrafilter then so is $\varepsilon_{\textit{max}, \sigma}(F)$.
\end{proposition}
\begin{proof}
	Let $E_i \in \varepsilon_{\textit{max}, \sigma}(F)$ for each $i \in I (\neq \emptyset)$ and $D$ any ultrafilter over $I$. Assume $\bigcap_D E_i \notin \varepsilon_{\textit{max}, \sigma}(F)$. Since $\varepsilon_{\sigma}(F)$ is closed under the operator $\bigcap_D$, we have $\bigcap_D E_i \in \varepsilon_{\sigma}(F)$. Thus, $\bigcap_D E_i \subset E$ for some $E \in \varepsilon_{\sigma}(F)$. Let $a \in E - \bigcap_D E_i$. Then $I_a \triangleq \set{i \in I \mid a \notin E_i} \in D$. Let $i \in I_a$. Then $E_i \cup \set{a}$ isn't $\varepsilon_\sigma$-extensible because of $a \notin E_i$ and $E_i \in \varepsilon_{\textit{max}, \sigma}(F)$. By Theorem \ref{theorem:epsilon_extensible}, there exists a finite subset $S \subseteq E_i \cup \set{a}$ such that $S$ isn't $\varepsilon_\sigma$-extensible. W.l.o.g., we may assume that $S \in \text{Anti-}\varepsilon_\sigma(F)$ due to the finiteness of $S$. Since both $\set{a}$ and $E_i$ are $\varepsilon_\sigma$-extensible, we have $a \in S$ and $S \cap E_i \neq \emptyset$. Hence
	\begin{eqnarray*}
		I_a \subseteq \bigcup_{X \in \Gamma_{\varepsilon_\sigma}(a)} \set{i \in I \mid X \subseteq E_i \cup \set{a}} \in D.
	\end{eqnarray*}
	It is easy to see that, for any $X_1, X_2 \in \Gamma_{\varepsilon_\sigma}(a)$, 
	\begin{eqnarray*}
		X_1 \sim_a X_2 \text{~implies~} \set{i \in I \mid X_1 \subseteq E_i \cup \set{a}} = \set{i \in I \mid X_2 \subseteq E_i \cup \set{a}}.
	\end{eqnarray*}
	Further, since $F$ is Anti($\varepsilon_\sigma$)-finitary, there exists $X_0 \in \Gamma_{\varepsilon_\sigma}(a)$ such that 
	\begin{eqnarray*}
		\set{i \in I \mid X_0 \subseteq E_i \cup \set{a}} \in D.
	\end{eqnarray*}
	Thus, $X_0 \subseteq \bigcap_D (E_i \cup \set{a})$, which contradicts that $X_0$ isn't $\varepsilon_\sigma$-extensible and $\bigcap_D (E_i \cup \set{a}) = \bigcap_D E_i \cup \set{a} \subseteq E \in \varepsilon_\sigma(F)$, as desired.
\end{proof}
We may strengthen the conclusion of Proposition \ref{proposition:RMMU_max} above as 
\begin{eqnarray*}
	\set{i \in I \mid E_i \in \varepsilon_{\textit{max}, \sigma}(F)} \in D \text{~implies~} \bigcap_D E_i \in \varepsilon_{\textit{max}, \sigma}(F),
\end{eqnarray*}
which may be proved by modifying the proof of Proposition \ref{proposition:RMMU_max} slightly. Its detail is left to the reader.
\begin{remark}
	Since any finitary and co-finitary AAF must be Anti($\varepsilon_{cf}^{\ell}$)-finitary but the converse is not true (see Example \ref{Ex:Anti-cf}), Proposition \ref{proposition:RMMU_max} contains a generalization of Proposition \ref{proposition:RMMU_na}. However, the advantage of Proposition \ref{proposition:RMMU_na} lies in: it only depends on structural properties of AAFs. Unlike the conditions of finitary and co-finitary, the condition of Anti($\varepsilon$)-finitary refers to both AAFs and the semantics.
\end{remark}
Although $\varepsilon_{\textit{max}, \sigma}(F)$ isn't closed under under the operator $\bigcap_D$ in general, by Theorems \ref{theorem:epsilon_extensible} and \ref{theorem:Lindenbaum_RM}, we still can obtain some interesting properties of $\varepsilon_{\textit{max}, \sigma}(F)$ whenever $\varepsilon_{\sigma}(F)$ is closed under reduced meets modulo any ultrafilter.
\begin{corollary}
	Let $F$ be any AAF and $\varepsilon_{\sigma}$ an extension-based semantics. If $\varepsilon_{\sigma}(F)$ is closed under reduced meets modulo any ultrafilter then
	\begin{itemize}
			\item[a.] $\varepsilon_{\textit{max}, \sigma}$-extensibility is compact.
			\item[b.] $|\varepsilon_{rr\textit{max}(\sigma)}^\eta(F)| \geq 1$ whenever $F$ is finitary and $|\varepsilon_{\sigma}(F)| \geq 1$, where $\varepsilon_{rr\textit{max}(\sigma)}^\eta$ is the range related semantics induced by $\varepsilon_{\textit{max}, \sigma}$.
	\end{itemize}
\end{corollary}
\begin{proof}
	($a$) Assume that $X (\subseteq \A)$ is finitely $\varepsilon_{\textit{max}, \sigma}$-extensible. Since $\varepsilon_{\textit{max}, \sigma}(F) \subseteq \varepsilon_{\sigma}(F)$, $X$ is finitely $\varepsilon_{\sigma}$-extensible. Then, by Theorem \ref{theorem:epsilon_extensible}, $X \subseteq E$ for some $E \in \varepsilon_{\sigma}(F)$. Moreover, by Theorem \ref{theorem:Lindenbaum_RM}($b$), $E \subseteq E^*$ for some $E^* \in \varepsilon_{\textit{max}, \sigma}(F)$. Thus, we have $X \subseteq E \subseteq E^*$, and hence $X$ is $\varepsilon_{\textit{max}, \sigma}$-extensible.
	\par ($b$) Since $|\varepsilon_{\sigma}(F)| \geq 1$, by Theorem \ref{theorem:close under meet and universal definability}, $|\varepsilon_{rr\sigma}^\eta(F)| \geq 1$. Let $E \in \varepsilon_{rr\sigma}^\eta(F) \subseteq \varepsilon_{\sigma}(F)$. By Theorem \ref{theorem:Lindenbaum_RM}($b$), $E \subseteq E'$ for some $E' \in \varepsilon_{\textit{max}, \sigma}(F)$. Thus, $E' \in \varepsilon_{rr\sigma}^\eta(F)$ due to $E \cup E_\eta^+ \subseteq E' \cup E_\eta'^+$. Hence, $E' \in \varepsilon_{rr\textit{max}(\sigma)}^\eta(F)$, as desired.
\end{proof}
\subsection{Some applications}
This subsection intends to provide some applications of Theorem \ref{theorem:epsilon_extensible} (the compactness of the extensibility) and Theorem \ref{theorem:Lindenbaum_RM} (the Lindenbaum property). First, we consider the equivalence of two AAFs $\tuple{\A, \rightarrow_i}(i = 1, 2)$ w.r.t. a given semantics. The next result characterizes equivalent AAFs up to $\varepsilon_{\textit{max}, \sigma}$ (or, $\varepsilon_{\textit{max}, \sigma}$-inference, $\approx$) in terms of anti-$\varepsilon_{\sigma}$ sets. Here the binary relation $\approx$ is defined below.
\begin{definition}
	Let $F_i = \tuple{\A, \rightarrow_i}(i = 1, 2)$ be two AAFs and $\varepsilon$ an extension-based semantics.
	\begin{itemize}
		\item $\varepsilon(F_1) \sqsubseteq \varepsilon(F_2)$ iff $\forall E \in \varepsilon(F_1) \exists E' \in \varepsilon(F_2) (E \subseteq E')$.
		\item $\varepsilon(F_1) \approx \varepsilon(F_2)$ iff $\varepsilon(F_1) \sqsubseteq \varepsilon(F_2)$ and $\varepsilon(F_2) \sqsubseteq \varepsilon(F_1)$.
	\end{itemize}
\end{definition}
\begin{theorem} \label{theorem:anti-epsilon set}
	Let $F_i = \tuple{\A, \rightarrow_i}(i = 1, 2)$ be two AAFs and $\varepsilon_{\sigma}$ an extension-based semantics. If $\varepsilon_{\sigma}(F_i)(i = 1, 2)$ is closed under reduced meets modulo any ultrafilter then the following are equivalent:
	\begin{itemize}
		\item[a.] $\text{Anti-}\varepsilon_{\sigma}(F_1) = \text{Anti-}\varepsilon_{\sigma}(F_2)$.
		\item[b.] $\varepsilon_{\sigma}(F_1) \approx \varepsilon_{\sigma}(F_2)$.
		\item[c.] $\varepsilon_{\textit{max}, \sigma}(F_1) = \varepsilon_{\textit{max}, \sigma}(F_2)$.
	\end{itemize}
	Moreover, under the additional assumption $\text{Anti-}\varepsilon_{\sigma}(F_i) \neq \emptyset$ (or, $A \notin \varepsilon_{\sigma}(F_i)$) for $i = 1, 2$, each of the clauses ($a$), ($b$)  and ($c$) is also equivalent to
	\begin{itemize}
		\item[d.] $\nc{\varepsilon_{\textit{max}, \sigma}}{F_1} = \nc{\varepsilon_{\textit{max}, \sigma}}{F_2}$.
	\end{itemize}
\end{theorem}
\begin{proof}
	($a \Rightarrow b$)
	Let $E \subseteq \A$ be a $\varepsilon_{\sigma}$-extension of $F_1$. Clearly, $\pw{E} \cap \text{Anti-}\varepsilon_{\sigma}(F_2) = \emptyset$ due to $E \in \varepsilon_{\sigma}(F_1)$ and $\text{Anti-}\varepsilon_{\sigma}(F_1) = \text{Anti-}\varepsilon_{\sigma}(F_2)$. Further, since $\varepsilon_{\sigma}(F_2)$ is closed under reduced meets modulo any ultrafilter, by Theorem \ref{theorem:epsilon_extensible} and $\pw{E} \cap \text{Anti-}\varepsilon_{\sigma}(F_2) = \emptyset$, there exists $E_1 \in \varepsilon_{\sigma}(F_2)$ such that $E \subseteq E_1$. Thus $\varepsilon_{\sigma}(F_1) \sqsubseteq \varepsilon_{\sigma}(F_2)$. Similarly, we also have $\varepsilon_{\sigma}(F_2) \sqsubseteq \varepsilon_{\sigma}(F_1)$.
	\par ($b \Rightarrow c$)
	Assume $E \subseteq \A$ is a maximal $\varepsilon_{\sigma}$-extension of $F_1$. It is enough to show that $E$ is also a maximal $\varepsilon_{\sigma}$-extension of $F_2$. By the clause ($b$), $E \subseteq E_1$ for some $E_1 \in \varepsilon_{\sigma}(F_2)$. Thus, by Theorem \ref{theorem:Lindenbaum_RM}, there exists a maximal $\varepsilon_{\sigma}$-extension $E_2$ of $F_2$ such that $E \subseteq E_1 \subseteq E_2$. To complete the proof, it is enough to show $E = E_2$. Since $E_2 \in \varepsilon_{\sigma}(F_2)$, applying the clause ($b$) again, there exists an $\varepsilon_{\sigma}$-extension $E_3$ of $F_1$ such that $E \subseteq E_1 \subseteq E_2 \subseteq E_3$. Then, since $E$ is maximal in $\varepsilon_{\sigma}(F_1)$, we have $E = E_1 = E_2 = E_3$, as desired.
	\par ($c \Rightarrow a$) 
	Assume $X \in \text{Anti-}\varepsilon_{\sigma}(F_1)$. Since $\varepsilon_{\sigma}(F_2)$ is closed under reduced meets modulo any ultrafilter and $\varepsilon_{\textit{max}, \sigma}(F_1) = \varepsilon_{\textit{max}, \sigma}(F_2)$, by Theorem \ref{theorem:Lindenbaum_RM}, it is easy to verify that $X$ is an anti-$\varepsilon_{\sigma}$ set of $F_2$. To complete the proof, it suffices to show that $X$ is a minimal anti-$\varepsilon_{\sigma}$ set of $F_2$. Assume that $X_1$ is an anti-$\varepsilon_{\sigma}$ set of $F_2$ such that $X_1 \subseteq X$. Clearly, by Theorem \ref{theorem:Lindenbaum_RM} and $\varepsilon_{\textit{max}, \sigma}(F_1) = \varepsilon_{\textit{max}, \sigma}(F_2)$ again, $X_1$ is also an anti-$\varepsilon_{\sigma}$ set of $F_1$. Then $X_1 = X$ immediately follows from $X \in \text{Anti-}\varepsilon_{\sigma}(F_1)$, as desired.
	\par($c \Leftrightarrow d$)
	It is trivial that the clause ($c$) implies ($d$). Next we prove another direction. Let $E \in \varepsilon_{\textit{max}, \sigma}(F_1)$. By the clause ($d$) and Definition \ref{definition:infer}, we have
	\begin{eqnarray*}
		E \longequal{\text{Def.} \ref{definition:infer}} \set{a \in \A \mid E \nc{\varepsilon_{\textit{max}, \sigma}}{F_1} a} \overset{(d)} = \set{a \in \A \mid E \nc{\varepsilon_{\textit{max}, \sigma}}{F_2} a}.
	\end{eqnarray*}
	It follows from $\text{Anti-}\varepsilon_{\sigma}(F_1) \neq \emptyset$ and $E \in \varepsilon_{\sigma}(F_1)$ that $E \neq \A$. Further, due to 
	\begin{eqnarray*}
		E = \set{a \in \A \mid E \nc{\varepsilon_{\textit{max}, \sigma}}{F_2} a} = \bigcap \set{X \in \varepsilon_{\textit{max}, \sigma}(F_2) \mid E \subseteq X},
	\end{eqnarray*}
	we get $\set{X \in \varepsilon_{\textit{max}, \sigma}(F_2) \mid E \subseteq X} \neq \emptyset$ (otherwise, it follows that $E = \A$, and a contradiction arises).
	Thus, there exits $E_1 \in \varepsilon_{\textit{max}, \sigma}(F_2)$ such that $E \subseteq E_1$.
	Similarly, $E_1 \subseteq E_2$ for some $E_2 \in \varepsilon_{\textit{max}, \sigma}(F_1)$.
	Hence, $E = E_1 = E_2$ due to $E \in \varepsilon_{\textit{max}, \sigma}(F_1)$.
	Consequently, $E \in \varepsilon_{\textit{max}, \sigma}(F_2)$, as desired.
\end{proof}
Clearly, in the above theorem, the assertion ($c \Leftrightarrow d$) doesn't depend on the assumption that $\varepsilon_{\sigma}(F_i)$ is closed under reduced meets modulo any ultrafilter. But other equivalence needs it.
\begin{corollary}\label{corollary:anti-epsilon set}
	Let $F_i = \tuple{\A, \rightarrow_i}(i = 1, 2)$ be two finitary AAFs.
	\begin{itemize}
		\item[a.] $\varepsilon_{na}^\ell(F_1) = \varepsilon_{na}^\ell(F_2)$ iff $\varepsilon_{cf}^\ell(F_1) = \varepsilon_{cf}^\ell(F_2)$ iff $\text{Anti-}\varepsilon_{cf}^\ell(F_1) = \text{Anti-}\varepsilon_{cf}^\ell(F_2)$\footnote{This clause doesn't depend on the assumption that $F_i (i = 1, 2)$ is finitary.}.
		\item[b.] $\varepsilon_{pr, \textit{Dung}}^{\ell mn}(F_1) = \varepsilon_{pr, \textit{Dung}}^{\ell mn}(F_2)$ iff $\text{Anti-}\varepsilon_{ad}^{\ell mn}(F_1) = \text{Anti-}\varepsilon_{ad}^{\ell mn}(F_2)$.
		\item[c.] $\varepsilon_{pr}^{\ell mn}(F_1) = \varepsilon_{pr}^{\ell mn}(F_2)$ iff $\text{Anti-}\varepsilon_{co}^{\ell mn}(F_1) = \text{Anti-}\varepsilon_{co}^{\ell mn}(F_2)$.
	\end{itemize}
	Moreover, in the situation that $\ell \geq m$, $n \geq m$ and $F_i$ is well-founded $(i = 1, 2)$, we also have
	\begin{itemize}
		\item[d.] $\text{Anti-}\varepsilon_{co}^{\ell mn}(F_1) = \text{Anti-}\varepsilon_{co}^{\ell mn}(F_2)$ iff $\varepsilon_{ss}^{\ell mn \eta}(F_1) = \varepsilon_{ss}^{\ell mn \eta}(F_2)$.
		\item[e.] $\text{Anti-}\varepsilon_{ad}^{\ell mn}(F_1) = \text{Anti-}\varepsilon_{ad}^{\ell mn}(F_2)$ iff $\varepsilon_{rra}^{\ell mn \eta}(F_1) = \varepsilon_{rra}^{\ell mn \eta}(F_2)$ whenever $\eta \geq \ell$.
	\end{itemize}
\end{corollary}
\begin{proof}
	The clauses ($a$), ($b$) and ($c$) follow from Lemma \ref{lemma:properties_cf_def}($d$) and Theorems \ref{theorem:close under meet_fundamental semantics} and \ref{theorem:anti-epsilon set}. The clauses ($d$) and ($e$) come from Proposition \ref{proposition:range_relative}, Corollary \ref{corollary:well-founded_co_empty} and Theorem \ref{theorem:anti-epsilon set}.
\end{proof}
A consequence of the result above is that, for any extension-based semantics $\varepsilon_{\textit{max}, \sigma}$ induced by the maximality (e.g., $\varepsilon_{na}^\ell$, $\varepsilon_{pr, \textit{Dung}}^{\ell mn}$ and $\varepsilon_{pr}^{\ell mn}$), an operator on finitary AAFs is $\varepsilon_{\textit{max}, \sigma}$-safe whenever it doesn't change the minimal anti sets of its base semantics $\varepsilon_{\sigma}$ (correspondingly, $\varepsilon_{cf}^\ell$, $\varepsilon_{ad}^{\ell mn}$ and $\varepsilon_{co}^{\ell mn}$). Here, the notion of a safe operator is defined formally below.
\begin{definition} \label{definition:safe-operator}
	Given an extension-based semantics $\varepsilon$ and a class $\Omega$ of AAF, an operator $\mathcal{O}_\varepsilon$ on AAFs  is $\varepsilon$-safe w.r.t $\Omega$ if $\varepsilon(F) = \varepsilon(\mathcal{O}_\varepsilon(F))$ for any AAF $F \in \Omega$.
\end{definition}
Next we provide a safe operator related to the conflict-free semantics and leave ones related to other semantics to the future work.
\begin{definition} \label{definition:o-operator}
	The operator $\mathcal{O}_{cf}^\ell$ on AAFs is defined as, for any AAF $F = \tuple{A, \rightarrow}$, $\mathcal{O}_{cf}^\ell(F) \triangleq \tuple{A, \rightarrow'}$, where, for any $a, b \in \A$, 
	\begin{eqnarray*}
		a \rightarrow' b \text{~iff~} a \rightarrow b \text{~and~} \exists X \in \text{Anti-}\varepsilon_{cf}^{\ell}(F)(a, b \in X).
	\end{eqnarray*}
\end{definition}
\begin{example} \label{Ex: safe operator}
	Let $\ell = 3$. For the AAF $F$ in Figure \ref{figure:counterexample}, $\mathcal{O}_{cf}^\ell(F)$ is given graphically below.
	\begin{figure}[H]
		\centering
		\vspace{-0.2cm}
		\includegraphics[width=0.5\linewidth]{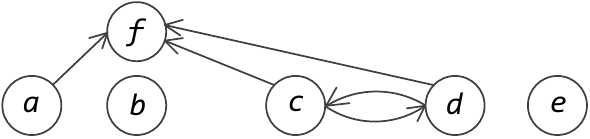}
		\label{fig:Anti-cf}
	\end{figure}
\end{example}
\begin{lemma} \label{lemma:safe operator_cf}
	For any AAF $F$, $\text{Anti-}\varepsilon_{cf}^{\ell}(F) = \text{Anti-}\varepsilon_{cf}^{\ell}(\mathcal{O}_{cf}^\ell(F))$.
\end{lemma}
\begin{proof}
	``$\subseteq$'' 
	Let $X \in \text{Anti-}\varepsilon_{cf}^{\ell}(F)$. Then, by Definition \ref{definition:anti-epsilon set}, $X$ is an anti-$\varepsilon_{cf}^{\ell}$ set of $F$, and hence $X \notin \varepsilon_{cf}^{\ell}(F)$. Thus, there exist $a, b_1, \cdots, b_\ell \in X$ such that $b_i \rightarrow a$ for each $1 \leq i \leq \ell$. Further, due to $X \in \text{Anti-}\varepsilon_{cf}^{\ell}(F)$, by Definition \ref{definition:o-operator}, $b_i \rightarrow' a$ for each $1 \leq i \leq \ell$. Then $X \notin \varepsilon_{cf}^{\ell}(\mathcal{O}_{cf}^\ell(F))$. So, by Lemma \ref{lemma:properties_cf_def}($d$), $X$ is also an anti-$\varepsilon_{cf}^{\ell}$ set of $\mathcal{O}_{cf}^\ell(F)$. To complete the proof, it suffices to prove that $X$ is a minimal anti-$\varepsilon$ set of $\mathcal{O}_{cf}^\ell(F)$. Suppose that $Y \subseteq X$ is an anti-$\varepsilon_{cf}^{\ell}$ set of $\mathcal{O}_{cf}^\ell(F)$. Then $Y \notin \varepsilon_{cf}^{\ell}(\mathcal{O}_{cf}^\ell(F))$, and hence $Y \notin \varepsilon_{cf}^{\ell}(F)$ due to $\rightarrow' \subseteq \rightarrow$. Thus, by Lemma \ref{lemma:properties_cf_def}($d$) again, $Y$ is an anti-$\varepsilon_{cf}^{\ell}$ set of $F$. So, it follows from $X \in \text{Anti-}\varepsilon_{cf}^{\ell}(F)$ that $X = Y$.
	\par ``$\supseteq$'' 
	Let $X \in \text{Anti-}\varepsilon_{cf}^{\ell}(\mathcal{O}_{cf}^\ell(F))$. Then, by Definition \ref{definition:anti-epsilon set}, $X$ is an anti-$\varepsilon_{cf}^{\ell}$ set of $\mathcal{O}_{cf}^\ell(F)$. Based on $\rightarrow' \subseteq \rightarrow$ and Lemma \ref{lemma:properties_cf_def}($d$), it is easy to see that $X$ is an anti-$\varepsilon_{cf}^{\ell}$ set of $F$. To complete the proof, we intend to show that $X$ is a minimal anti-$\varepsilon$ set of $F$. Set
	\begin{eqnarray*}
		\Omega \triangleq \set{Y \subseteq X \mid Y \text{~is an anti-}\varepsilon_{cf}^{\ell} \text{~set of~} F}.
	\end{eqnarray*}
	Then $\Omega \neq \emptyset$ due to $X \in \Omega$. Moreover, since $\Omega$ is a set of finite sets, there is a minimal set, say $Y_0$, in $\Omega$. Thus $Y_0 \in \text{Anti-}\varepsilon_{cf}^{\ell}(F)$. So, by the $\subseteq$-direction above, $Y_0 \in \text{Anti-}\varepsilon_{cf}^{\ell}(\mathcal{O}_{cf}^\ell(F))$. Further, it follows from $X \in \text{Anti-}\varepsilon_{cf}^{\ell}(\mathcal{O}_{cf}^\ell(F))$ and $Y_0 \subseteq X$ that $Y_0 = X$. Therefore, $X \in \text{Anti-}\varepsilon_{cf}^{\ell}(F)$, as desired.
\end{proof}
\begin{theorem} \label{theorem:safe operators}
	For all AAFs, the operator $\mathcal{O}_{cf}^\ell$ is both $\varepsilon_{cf}^{\ell}$-safe and $\varepsilon_{na}^{\ell}$-safe.
\end{theorem}
\begin{proof}
	Immediately follows from Corollary \ref{corollary:anti-epsilon set} and Lemma \ref{lemma:safe operator_cf}.
\end{proof}
We intend to end this subsection by considering an inverse problem related to the conflict-free semantics. Given an extension-based semantics $\varepsilon$, the inverse problem related with $\varepsilon$ is, to find an AAF $F = \tuple{\A, \rightarrow}$ such that $\Omega = \varepsilon(F)$ for a given set $\Omega \subseteq \pw{\A}$. Next, we will provide sufficient and necessary conditions on $\Omega$ so that the equation $\Omega = \varepsilon_{cf}^\ell(F)$ has solutions.
\begin{definition}
	Let $\A \neq \emptyset$ and $\Omega \subseteq \pw{\A}$. $\Gamma_\Omega$ is the set of all minimal finite subsets $Y$ of $\A$ such that $Y \nsubseteq X$ for each $X \in \Omega$, that is,
	\begin{eqnarray*}
		\Gamma_\Omega \triangleq \textit{min}\set{Y \in \wp_f(\A) \mid \neg \exists X \in \Omega(Y \subseteq X)}.
	\end{eqnarray*}
\end{definition}
We first deal with the non-graded case, i.e., $\ell = 1$.
\begin{theorem}[Representation Theorem of $\varepsilon_{cf}^1$] \label{theorem:Representation theorem of 1-cf}
	Let $\A \neq \emptyset$ and $\Omega \subseteq \pw{\A}$. There exists an AAF $F = \tuple{\A, \rightarrow}$ such that $\varepsilon_{cf}^1(F) = \Omega$ iff $\Omega$ satisfies the following conditions:
	\begin{itemize}
		\item[a.] $\Omega \neq \emptyset$.
		\item[b.] $\Omega$ is down-closed.
		\item[c.] $\Omega$ is closed under reduced meets modulo any ultrafilter.
		\item[d.] $0 < |Y| \leq 2$ for each $Y \in \Gamma_\Omega$.
	\end{itemize}
\end{theorem}
\begin{proof}
	($\Rightarrow$) It is obvious that $\Omega \neq \emptyset$ due to $\emptyset \in \varepsilon_{cf}^1(F)$. The clauses ($b$) and ($c$) immediately follow from Lemma \ref{lemma:properties_cf_def}($d$) and Theorem \ref{theorem:close under meet_fundamental semantics}, respectively. Next we prove the clause ($d$). Let $Y \in \Gamma_\Omega$. Then $Y \notin \Omega = \varepsilon_{cf}^1(F)$, and hence $Y \nsubseteq N_1(Y)$. Thus, $a \rightarrow b$ in $F$ for some $a, b \in Y$. Since $\varepsilon_{cf}^1(F)$ = $\Omega$, $\set{a, b} \nsubseteq X$ for each $X \in \Omega$. Then, due to the minimality of $Y$, we get $Y = \set{a, b}$, and hence $0 < |Y| \leq 2$. Clearly, the size of $Y$ depends on whether it holds that $a = b$.
	\par
	($\Leftarrow$) Put $F_\Omega \triangleq \tuple{\A, \rightarrow}$ with
	\begin{eqnarray*}
		\rightarrow \triangleq \set{\tuple{a, b} \mid a, b \in X \text{~for some~} X \in \Gamma_\Omega \text{~and~} a \neq b} \cup \set{\tuple{a, a} \mid a \in \A \text{~and~} \set{a} \in \Gamma_\Omega}.
	\end{eqnarray*}
	Let $E \subseteq \A$ and $E \notin \Omega$. Since $\Omega$ is closed under reduced meets modulo any ultrafilter, similar to Theorem \ref{theorem:epsilon_extensible}, it is easy to see that there exists a finite subset of $E$, say $E_0$, such that $E_0 \notin \Omega$. Since $\Omega$ is down-closed, there is no $X \in \Omega$ such that $E_0 \subseteq X$. Then the set $\set{Y \subseteq E_0 \mid \neg \exists X \in \Omega (Y \subseteq X)}$ is nonempty. Due to the finiteness of $E_0$, this set has a minimal element, say $E_1$. Clearly, $E_1 \in \Gamma_\Omega$. Since $\Omega \neq \emptyset$, we have $\emptyset \notin \Gamma_\Omega$. Thus, $E_1 \neq \emptyset$. Then, by the definition of $\rightarrow$, $a \rightarrow b$ for some $a, b \in E_1$, and hence $E \notin \varepsilon_{cf}^1(F_\Omega)$ because of $E_1 \subseteq E$. Consequently, $\varepsilon_{cf}^1(F_\Omega) \subseteq \Omega$.
	Next, we show $\Omega \subseteq \varepsilon_{cf}^1(F_\Omega)$. Let $E \subseteq \A$ and $E \notin \varepsilon_{cf}^1(F_\Omega)$. Then, there exist $a, b \in E$ such that $a \rightarrow b$. If $a \neq b$ then, by the definition of $\rightarrow$ and the condition ($d$), we get $\set{a, b} \in \Gamma_\Omega$, and hence $E \notin \Omega$ due to $\set{a, b} \subseteq E$. If $a = b$ then $\set{a} \in \Gamma_\Omega$ by the definition of $\rightarrow$, and hence $E \notin \Omega$ due to $\set{a} \subseteq E$ and $\set{a} \in \Gamma_\Omega$. Consequently, $\Omega \subseteq \varepsilon_{cf}^1(F_\Omega)$, as desired.
\end{proof}
A natural conjecture is that a representation theorem of $\varepsilon_{cf}^\ell$ for each $\ell > 0$ may be obtained from the above theorem by replacing the condition $0 < |Y| \leq 2$ by $\ell - 1 < |Y| \leq \ell + 1$. Unfortunately, it seems to be not true. In the graded situation, we need more conditions to refine the construction in the proof of Theorem \ref{theorem:Representation theorem of 1-cf}. To this end, we introduce the notion below.
\begin{definition} \label{definition:choice function}
	Let $\A \neq \emptyset$ and $\Omega \subseteq \pw{\A}$.
	\begin{itemize}
		\item[a.] A function $\ch$ from $\Gamma_\Omega$ to $\A$ is said to be a $\Gamma_\Omega$-choice function if $\ch(Y) \in Y$ for each $Y \in \Gamma_\Omega$. The function $\ch$ is said to be well-founded if there is no $Y_i \in \Gamma_\Omega (i < \omega)$ such that $\ch(Y_{i + 1}) \in Y_i$ and $\ch(Y_{i}) \neq \ch(Y_{i + 1})$ for each $i < \omega$. The function $\ch$ is said to be finitary, if for each $a \in \A$, there are only finitely many $Y \in \Gamma_\Omega$ such that $\ch(Y) = a$.
		\item[b.] The set $\Omega$ is said to be well-organized if either $\Gamma_\Omega = \emptyset$ or there exists a $\Gamma_\Omega$-choice function $\ch$ such that, for each $Y \in \Gamma_\Omega$ and $Z \subseteq \bigcup \set{X \in \Gamma_\Omega \mid \ch(Y) = \ch(X)}$, if $\ch(Y) \in Z$ and $|Z| \geq |Y|$ then $Z \notin \Omega$.
		\item[c.] The set $\Omega$ is said to be well-organized with a well-founded (or, finitary) $\Gamma_\Omega$-choice if there exists a well-founded (finitary, resp.) $\Gamma_\Omega$-choice function $\ch$ satisfying the requirement in ($b$) whenever $\Gamma_\Omega \neq \emptyset$.
	\end{itemize}
\end{definition}
\begin{theorem}[Representation Theorem (I) of $\varepsilon_{cf}^\ell$] \label{theorem:Representation theorem I of l-cf}
	Let $\A \neq \emptyset$ and $\Omega \subseteq \pw{\A}$. There exists an AAF $F = \tuple{\A, \rightarrow}$ such that $\varepsilon_{cf}^\ell(F) = \Omega$ iff $\Omega$ satisfies the following conditions:
	\begin{itemize}
		\item[a.] $\Omega \neq \emptyset$.
		\item[b.] $\Omega$ is down-closed.
		\item[c.] $\Omega$ is closed under reduced meets modulo any ultrafilter.
		\item[d.] $\ell - 1 < |Y| \leq \ell + 1$ for each $Y \in \Gamma_\Omega$.
		\item[e.] $\Omega$ is well-organized.
	\end{itemize}
\end{theorem}
\begin{proof}
	($\Rightarrow$)  Similar to Theorem \ref{theorem:Representation theorem of 1-cf}, the clauses ($a$)-($d$) hold. In the following, we deal with ($e$). Assume that $\Gamma_\Omega \neq \emptyset$. A function $\ch \colon \Gamma_\Omega \to \A$ is defined as follows. For each $Y \in \Gamma_\Omega$, since $Y \notin \Omega = \varepsilon_{cf}^\ell(F)$, there exists an argument, say $a_Y$, such that $a_Y \in Y$ and $b_i \rightarrow a_Y$ for some $b_i (1 \leq i \leq \ell)$ in $Y$. We may choose arbitrarily and fix such an argument $a_Y$, and put $\ch(Y) \triangleq a_Y$. Then, by the minimality of $Y$, it isn't difficult to see that $b \rightarrow \ch(Y)$ in $F$ for each $b \in Y$ with $b \neq \ch(Y)$, moreover $\ch(Y) \rightarrow \ch(Y)$ in $F$ whenever $|Y| = \ell$. Based on this fact, it is straightforward to verify that, for each $Y \in \Gamma_\Omega$ and $Z \subseteq \bigcup \set{X \in \Gamma_\Omega \mid \ch(Y) = \ch(X)}$ with $\ch(Y) \in Z$, $|Z| \geq |Y|$ implies $Z \notin \varepsilon_{cf}^\ell(F)$, as desired.
	\par
	($\Leftarrow$) We consider only the nontrivial case that $\Gamma_\Omega \neq \emptyset$. By the condition ($e$), there exists a $\Gamma_\Omega$-choice function $\ch$ satisfying the requirement in Definition \ref{definition:choice function}($b$). For each $Y \in \Gamma_\Omega$, by the clause ($d$), either $|Y| = \ell$ or $|Y| = \ell + 1$. Put $F_\Omega \triangleq \tuple{\A, \rightarrow}$ with $\rightarrow \triangleq \bigcup_{Y \in \Gamma_\Omega} \rightarrow_{Y, \ch}$, where
	\begin{eqnarray*}
		\rightarrow_{Y, \ch} \triangleq 
		\begin{cases}
			\set{\tuple{b, \ch(Y)} \mid b \in Y \text{~and~} b \neq \ch(Y)} &\text{\textit{if}~~$|Y| = \ell + 1$}  \\ 
			\set{\tuple{b, \ch(Y)} \mid b \in Y} &\text{\textit{if}~~$|Y| = \ell$}
		\end{cases}
	\end{eqnarray*}
	In the following, we intend to show that $\Omega = \varepsilon_{cf}^\ell(F_\Omega)$.
	Let $E \subseteq \A$ and $E \notin \Omega$. By the condition ($c$), similar to Theorem \ref{theorem:epsilon_extensible}, there exists $Y \in \Gamma_\Omega$ such that $Y \subseteq E$. By the definition of $\rightarrow_{Y, \ch}$, it is easy to see that there exist distinct arguments $b_1, \cdots, b_\ell \in Y$ such that $b_i \rightarrow_{Y, \ch} \ch(Y)$. Hence, $E \notin \varepsilon_{cf}^\ell(F_\Omega)$ due to $\rightarrow_{Y, \ch} \subseteq \rightarrow$ and $Y \subseteq E$. Thus, $\varepsilon_{cf}^\ell(F_\Omega) \subseteq \Omega$.
	Let $E \subseteq \A$ and $E \notin \varepsilon_{cf}^\ell(F_\Omega)$. Thus, there exist distinct arguments $b_1, \cdots, b_\ell \in E$ such that $b_i \rightarrow a$ $(1 \leq i \leq \ell)$ for some $a \in E$. Then, by the definition of attack relation $\rightarrow$, there exists $Y_i (1 \leq i \leq \ell) \in \Gamma_\Omega$ such that, for each $1 \leq i \leq \ell$, $b_i \in Y_i$ and $\ch(Y_i) = a$. If $a \rightarrow a$ then $a = \ch(Y_a)$ for some $Y_a$ with $|Y_a| = \ell$, further, by the condition ($e$), $\set{b_1, \cdots, b_\ell, a} \notin \Omega$ due to $\set{b_1, \cdots, b_\ell, a} \subseteq \bigcup \set{X \in \Gamma_\Omega \mid \ch(X) = \ch(Y_a) = a}$ and $|\set{b_1, \cdots, b_\ell, a}| \geq \ell = |Y_a|$. If $a \nrightarrow a$ in $F_\Omega$ then $a \neq b_i$ for each $1 \leq i \leq \ell$, and hence, applying the condition ($e$) again, we also have $\set{b_1, \cdots, b_\ell, a} \notin \Omega$ due to $|\set{b_1, \cdots, b_\ell, a}| = \ell + 1 \geq |Y_i|$ for each $1 \leq i \leq \ell$. Anyway, we have $\set{b_1, \cdots, b_\ell, a} \notin \Omega$ in each case. Then, by the condition ($b$), we have $E \notin \Omega$ due to $\set{b_1, \cdots, b_\ell, a} \subseteq E$. Consequently, $\Omega \subseteq \varepsilon_{cf}^\ell(F_\Omega)$, as desired.
\end{proof}
\begin{remark}
	In the non-graded case, since a coarse construction is enough, the proof of Theorem \ref{theorem:Representation theorem of 1-cf} doesn't depend on the condition of well-organization. However, in this situation, the conditions ($b$) and ($d$) in Theorem \ref{theorem:Representation theorem of 1-cf} indeed imply that $\Omega$ is well-organized, in fact, each $\Gamma_\Omega$-choice function $\ch$ realizes the requirement in the clause ($b$) of Definition \ref{definition:choice function} whenever $\Gamma_\Omega \neq \emptyset$. For each $Y \in \Gamma_\Omega$, by ($d$), either $|Y| = 1$ or $|Y| = 2$. If $|Y| = 1$, then, for each $X \subseteq \A$ such that $|X| \geq |Y|$ and $\ch(Y) \in X$, we have $X \supseteq Y$, and hence $X \notin \Omega$ due to $Y \notin \Omega$ and ($b$). If $|Y| = 2$, then, for each $X \subseteq \bigcup \set{Z \in \Gamma_\Omega \mid \ch(Y) = \ch(Z)}$ with $|X| \geq |Y|$ and $\ch(Y) \in X$, we have $X \supseteq Y_1$ for some $Y_1 \in \Gamma_\Omega$, thus we also have $X \notin \Omega$. By the way, it is not difficult to see that Theorems \ref{theorem:Representation theorem of 1-cf} and \ref{theorem:Representation theorem I of l-cf} still hold if replacing the condition ($c$) by that $\Omega$ is closed under directed unions. The detail is left to the reader.
\end{remark}
Through strengthening the conditions ($d$) and ($e$) in the result above, we may obtain the representation theorem of $\varepsilon_{cf}^\ell$ w.r.t. well-founded AAFs. 
\begin{theorem}[Representation Theorem (II) of $\varepsilon_{cf}^\ell$] \label{theorem:Representation theorem II of l-cf}
	Let $\A \neq \emptyset$ and $\Omega \subseteq \pw{\A}$. There exists a well-founded AAF $F = \tuple{\A, \rightarrow}$ such that $\varepsilon_{cf}^\ell(F) = \Omega$ iff $\Omega$ satisfies the conditions ($a$)-($c$) in Theorem \ref{theorem:Representation theorem I of l-cf} and
	\begin{itemize}
		\item[d'.] $|Y| = \ell + 1$ for each $Y \in \Gamma_\Omega$.
		\item[e'.] $\Omega$ is well-organized with a well-founded $\Gamma_\Omega$-choice.
	\end{itemize}
\end{theorem}
\begin{proof}
	($\Rightarrow$) Clearly, $\Gamma_\Omega = \text{Anti-}\varepsilon_{cf}^{\ell}(F)$. Since $F$ is well-founded, there is no attack relation like $a \rightarrow a$ in $F$. Thus, it is not difficult to see that $|Y| = \ell + 1$ for each $Y \in \text{Anti-}\varepsilon_{cf}^{\ell}(F)$ (i.e., $\Gamma_\Omega$). Next, we show that the $\Gamma_\Omega$-choice function $\ch$ given in the proof ($\Rightarrow$) of Theorem \ref{theorem:Representation theorem I of l-cf} is well-founded. Assume there exist $Y_i \in \text{Anti-}\varepsilon_{cf}^{\ell}(F)$ ($i < \omega$) such that $\ch(Y_{i + 1}) \in Y_i$ and $\ch(Y_{i}) \neq \ch(Y_{i + 1})$ for each $i < \omega$. Then $|Y_i| = \ell + 1$ for each $i < \omega$. Since $b \rightarrow \ch(Y_i)$ in $F$ for each $b (\neq \ch(Y_i)) \in Y_i$ and $i < \omega$, there is an infinite attack sequence $\ch(Y_0) \leftarrow \ch(Y_1) \leftarrow \cdots \leftarrow \ch(Y_n) \leftarrow \ch(Y_{n + 1}) \leftarrow \cdots$ in $F$, which contradicts the well-foundedness of $F$.
	\par
	($\Leftarrow$) It is enough to show that the AAF $F_\Omega$ constructed in the proof ($\Leftarrow$) of Theorem \ref{theorem:Representation theorem I of l-cf} is well-founded. On the contrary, assume that there exists an infinite attack sequence $a_0 \leftarrow a_1 \leftarrow \cdots \leftarrow a_n \leftarrow a_{n + 1} \leftarrow \cdots$ ($n < \omega$) in $F_\Omega$. Then, by the construction of $F_\Omega$, there exists $Y_i \in \Gamma_\Omega$ ($i < \omega$) such that, for each $i < \omega$, $a_i = \ch(Y_i)$. Moreover, by the definition of $\rightarrow_{Y, \ch}$, it is easy to see that $a_{i + 1} \in Y_i$ and $a_i \neq a_{i + 1}$ for each $i < \omega$ due to the condition ($d'$), which contradicts the condition ($e'$), as desired.
\end{proof}
Similarly, if we strengthen the condition ($e$) in Theorem \ref{theorem:Representation theorem I of l-cf} as ``$\Omega$ is well-organized with a finitary $\Gamma_\Omega$-choice", we will get a representation theorem of $\varepsilon_{cf}^\ell$ w.r.t. finitary AAFs. The detail is left to the reader.
\begin{theorem}[Representation Theorem (III) of $\varepsilon_{cf}^\ell$] \label{theorem:Representation theorem III of l-cf}
	Let $\A \neq \emptyset$ and $\Omega \subseteq \pw{\A}$. There exists a finitary and co-finitary AAF $F = \tuple{\A, \rightarrow}$ such that $\varepsilon_{cf}^\ell(F) = \Omega$ iff $\Omega$ satisfies the conditions ($a$)-($e$) in Theorem \ref{theorem:Representation theorem I of l-cf} and
	\begin{itemize}
		\item[f.] $|\set{X \in \Gamma_\Omega \mid a \in X}| < \omega$ for each $a \in \A$.
	\end{itemize}
\end{theorem}
\begin{proof}
We prove it by continuing the proof of Theorem \ref{theorem:Representation theorem I of l-cf}.
	\par ($\Rightarrow$) It is enough to verify the condition ($f$). Let $a \in \A$ and $X \in \Gamma_\Omega$ (= Anti-$\varepsilon_{cf}^\ell(F)$) with $a \in X$. Then $X \notin \varepsilon_{cf}^\ell(F)$. Hence, there exists $b \in X$ such that $Y \subseteq X$ for some $Y \in [b^-]^\ell$. Due to the minimality of $X$, we have $Y \cup \set{b} = X$. Thus, either $a = b$ or $a \rightarrow b$. Hence
	\begin{eqnarray*}
		X \in S \triangleq \set{Y \cup \set{b} \mid (a = b \text{~or~} a \rightarrow b) \text{~and~} Y \in [b^-]^\ell}.
	\end{eqnarray*}
	Since $F$ is both finitary and co-finitary, the set $S$ is finite. Hence the condition ($f$) holds.
	\par ($\Leftarrow$) It suffices to show the AAF $F_\Omega$ constructed in the proof of Theorem \ref{theorem:Representation theorem I of l-cf} is finitary and co-finitary. By the construction of $F_\Omega$, for any $a, b \in \A$, $a \rightarrow b$ in $F_\Omega$ only if $\set{a, b} \subseteq Y$ for some $Y \in \Gamma_\Omega$. Since each set in $\Gamma_\Omega$ is finite (due to the condition ($d$)), by the condition ($f$), it is not difficult to see that, for each $a \in \A$, the set $\set{b \in \A \mid a \rightarrow b \text{~or~} b \rightarrow a \text{~in~} F_\Omega}$ is finite. Hence, $F_\Omega$ is both finitary and co-finitary.
\end{proof}
By the proof of Theorem \ref{theorem:Representation theorem I of l-cf}, it is easy to see that the operator $\mathcal{C}_{cf}^\ell$ defined below is $\varepsilon_{cf}^\ell$-safe, for each AAF $F = \tuple{\A, \rightarrow}$, $\mathcal{C}_{cf}^\ell(F) \triangleq \tuple{\A, \rightarrow'}$ with
\begin{eqnarray*}
	\rightarrow' \triangleq
	\begin{cases}
		\emptyset &\text{\textit{if}~~$\text{Anti-}\varepsilon_{cf}^{\ell}(F) = \emptyset$}  \\ 
		\bigcup_{Y \in \text{Anti-}\varepsilon_{cf}^{\ell}(F)} \rightarrow_{Y, \ch} &\text{\textit{otherwise}}
	\end{cases}
\end{eqnarray*}
where $\ch$ is an arbitrary but fixed choice function of $\text{Anti-}\varepsilon_{cf}^{\ell}(F)$ satisfying the requirement in Definition \ref{definition:choice function}($b$). Notice that, by Theorem \ref{theorem:Representation theorem I of l-cf}, such a choice function always exists for each $F$ with $\text{Anti-}\varepsilon_{cf}^{\ell}(F) \neq \emptyset$. For example, consider the AAF in Example \ref{Ex: safe operator} (i.e., AAF in Figure \ref{figure:counterexample}) with $\ell = 3$ and the choice function $\ch$ defined as $\ch(\set{a, c, f, d}) = f$, then $\mathcal{C}_{cf}^\ell(F)$ is given graphically below.
\begin{figure}[H]
	\centering
	\vspace{-0.2cm}
	\includegraphics[width=0.5\linewidth]{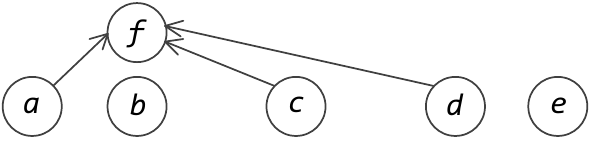}
	\label{figure:Choice function}
\end{figure}
By Theorem \ref{theorem:Representation theorem I of l-cf}, for a given set $\Omega$, there exists at most finitely many $\ell$ so that $\Omega = \varepsilon_{cf}^\ell(F)$ for some $F$ in the nontrivial case. Formally, we have
\begin{corollary}
	Let $\A \neq \emptyset$ and $\Omega \subseteq \pw{\A}$, and let $k \triangleq \textit{min}\set{|Y| \mid Y \in \Gamma_\Omega}$ whenever $\Gamma_\Omega \neq \emptyset$, and $\rho(\Omega) \triangleq \set{\ell \mid \Omega = \varepsilon_{cf}^\ell(F) \text{~for some AAF~} F = \tuple{\A, \rightarrow}}$. If $\Gamma_\Omega = \emptyset$ (equivalently, $A \in \Omega$) then $\rho(\Omega) = \omega - \set{0}$, else
	\begin{itemize}
			\item[a.] $\rho(\Omega) = \set{k, k - 1}$ if $\rho(\Omega) \neq \emptyset$ and $|Y| = k > 1$ for each $Y \in \Gamma_\Omega$,
			\item[b.] $\rho(\Omega) = \set{1}$ if $\rho(\Omega) \neq \emptyset$ and $|Y| = 1$ for each $Y \in \Gamma_\Omega$,
			\item[c.] $\rho(\Omega) = \set{k}$ if $\rho(\Omega) \neq \emptyset$ and $\set{||Y_1| - |Y_2|| \mid Y_1, Y_2 \in \Gamma_\Omega} = \set{0, 1}$ and
			\item[d.] $\rho(\Omega) = \emptyset$ if $||Y_1| - |Y_2|| \geq 2$ for some $Y_1, Y_2 \in \Gamma_\Omega$.
		\end{itemize}
\end{corollary}
In the next section, more applications of the operator reduced meet modulo an ultrafilter will be given. All these indicate that this operator is useful for studying theoretically on AAF.
\section{Interval semantics} \label{Sec: Graded parametrized semantics}
So far, in addition to a variety of fundamental semantics, e.g., $\varepsilon_{cf}^{\ell}$, $\varepsilon_{ad}^{\ell mn}$, $\varepsilon_{co}^{\ell mn}$ and $\varepsilon_{stb}^{\ell mn}$, we have considered their derived ones that are defined in terms of the maximality or minimality, e.g., $\varepsilon_{na}^{\ell}$, $\varepsilon_{pr,\textit{Dung}}^{\ell mn}$, $\varepsilon_{pr}^{\ell mn}$, $\varepsilon_{gr}^{\ell mn}$, $\varepsilon_{stg}^{\ell}$ and $\varepsilon_{ss}^{\ell mn \eta}$, etc. There exists another kind of derived semantics called parameterized ones \cite{Dunne13ParametricSemantics}, whose definitions are based on intervals. This section will preliminarily discuss interval semantics, a generalization of parameterized ones.
\begin{definition}[Interval semantics] \label{definition: interval semantics}
	Let $\varepsilon_{\sigma_l}$, $\varepsilon_{\sigma}$ and $\varepsilon_{\sigma_r}$ be three extension-based semantics. The $(\sigma_l, \sigma, \sigma_r)$-interval semantics, denoted by $\varepsilon_{\sigma_l, \sigma, \sigma_r}$, is defined as, for any AAF $F$,	
	\begin{eqnarray*}
		\varepsilon_{\sigma_l,\sigma, \sigma_r}(F) \triangleq \set{ E \in \varepsilon_{\sigma}(F) \mid \bigcap \varepsilon_{\sigma_l}(F) \subseteq E \subseteq \bigcap \varepsilon_{\sigma_r}(F)}.
    \end{eqnarray*}
	We use the notation $\varepsilon_{\textit{max}, (\sigma_l,\sigma,\sigma_r)}$ to denote the semantics induced by the maximality based on $\varepsilon_{\sigma_l,\sigma, \sigma_r}$, that is,
	\begin{eqnarray*}
		\varepsilon_{\textit{max}, (\sigma_l,\sigma,\sigma_r)}(F) \triangleq \set{ E \mid E \text{~is maximal in~} \varepsilon_{\sigma_l, \sigma, \sigma_r}(F)} \text{~for each AAF~} F.
	\end{eqnarray*}
\end{definition}
First, we intend to explore the properties of abstract interval semantics $\varepsilon_{\sigma_l,\sigma, \sigma_r}$. As more applications of the operator $\bigcap_D$, we will focus on $\varepsilon_{\sigma_l,\sigma, \sigma_r}$ such that $\varepsilon_{\sigma}$ is closed under this operator. A simple but crucial observation is that $\varepsilon_{\sigma_l,\sigma, \sigma_r}$ inherits the closeness of $\varepsilon_{\sigma}$ under the operator $\bigcap_D$. Formally, we have
\begin{theorem} \label{theorem:RMMU_interval}
	Let $\varepsilon_{\sigma_l}$, $\varepsilon_{\sigma}$ and $\varepsilon_{\sigma_r}$ be three extension-based semantics. For any AAF $F$, if $\varepsilon_{\sigma}(F)$ is closed under reduced meets modulo any ultrafilter, then so is $\varepsilon_{\sigma_l,\sigma, \sigma_r}(F)$. Hence, $\varepsilon_{\sigma_l,\sigma, \sigma_r}(F)$ is closed under reduced meets modulo any ultrafilter for each $\varepsilon_{\sigma} \in \set{\varepsilon_{cf}^{\ell}, \varepsilon_{ad}^{\ell mn}, \varepsilon_{co}^{\ell mn}, \varepsilon_{stb}^{\ell mn}}$ whenever $F$ is finitary\footnote{For $\varepsilon_{cf}^{\ell}$, it doesn't depend on the assumption that $F$ is finitary.}.
\end{theorem}
\begin{proof}
	Let $I \neq \emptyset$ and $E_i \in \varepsilon_{\sigma_l,\sigma, \sigma_r}(F)$ for each $i \in I$.
	Assume that $D$ is an ultrafilter over $I$.
	Since $E_i \in \varepsilon_{\sigma_l,\sigma, \sigma_r}(F) \subseteq \varepsilon_{\sigma}(F)$ for each $i \in I$, we have $\bigcap_D E_i \in \varepsilon_{\sigma}(F)$ by the assumption that  $\varepsilon_{\sigma}(F)$ is closed under $\bigcap_D$.
	Let $b \notin \bigcap \varepsilon_{\sigma_r}(F)$.
	Then, for each $i \in I$, $b \notin E_i$ due to $E_i \subseteq \bigcap \varepsilon_{\sigma_r}(F)$.
	Thus $\hat{b} \triangleq \set{i \in I \mid b \in E_i} = \emptyset$, and hence $b \notin \bigcap_D E_i$.
	So, $\bigcap_D E_i \subseteq \bigcap \varepsilon_{\sigma_r}(F)$.
	On the other hand, since $\bigcap \varepsilon_{\sigma_l}(F) \subseteq E_i$ for each $i \in I$, we have $\hat{a} = I$ for each $a \in \bigcap \varepsilon_{\sigma_l}(F)$, and hence $\bigcap \varepsilon_{\sigma_l}(F) \subseteq \bigcap_D E_i$.
	Consequently, $\bigcap_D E_i \in \varepsilon_{\sigma_l,\sigma, \sigma_r}(F)$.
\end{proof}
By the result above and the ones in the precedent section, the following is to be expected.
\begin{corollary} \label{corollary:RM of interval semantics}
	Let $F = \tuple{\A, \rightarrow}$ be an AAF, and let $\varepsilon_{\sigma_l}$, $\varepsilon_{\sigma}$ and $\varepsilon_{\sigma_r}$ be three extension-based semantics. If $\varepsilon_{\sigma}(F)$ is closed under reduced meets modulo any ultrafilter then
	\begin{itemize}
		\item[a.] $|\varepsilon_{\sigma_l,\sigma, \sigma_r}(F)| \geq 1$ iff $|\varepsilon_{\textit{max}, (\sigma_l, \sigma,\sigma_r)}(F)| \geq 1$.
		\item[b.] $|\varepsilon_{\sigma_l,\sigma, \sigma_r}(F)| \geq 1$ iff $|\varepsilon _{rr({\sigma _l},{\sigma},{\sigma_r})}^\eta(F)| \geq 1$ whenever $F$ is finitary, where $\varepsilon _{rr({\sigma _l},{\sigma},{\sigma_r})}^\eta$ is the range related semantics induced by $\varepsilon_{\sigma_l,\sigma, \sigma_r}$.
		\item[c.] Each $E$ in $\varepsilon_{\sigma_l,\sigma, \sigma_r}(F)$ can be extended to a maximal one in $\varepsilon_{\sigma_l,\sigma, \sigma_r}(F)$.
		\item[d.] For each $X \subseteq \A$, $X$ is $\varepsilon_{\sigma_l,\sigma, \sigma_r}$-extensible iff $X$ is finitely $\varepsilon_{\sigma_l,\sigma, \sigma_r}$-extensible.
		\item[e.] For any $X \subseteq \A$ and $a \in \A$, $X \nc{\varepsilon_{\sigma_l,\sigma, \sigma_r}}{} a$ iff $X_0 \nc{\varepsilon_{\sigma_l,\sigma, \sigma_r}}{} a$ for some finite subset $X_0$ of $X$.
	\end{itemize}
	In particular, for all finitary AAFs, the clauses ($a$)-($e$) above hold for any $\varepsilon_{\sigma_l,\sigma, \sigma_r}$ with $\varepsilon_{\sigma} \in \set{\varepsilon_{cf}^{\ell}, \varepsilon_{ad}^{\ell mn}, \varepsilon_{co}^{\ell mn}, \varepsilon_{stb}^{\ell mn}}$.
\end{corollary}
\begin{proof}
	Immediately follows from Theorem \ref{theorem:RMMU_interval} and Theorems \ref{theorem:close under meet and universal definability}, \ref{theorem:close under meet_fundamental semantics}, \ref{theorem:epsilon_inference} and \ref{theorem:Lindenbaum_RM}.
\end{proof}
In a word, all clauses in Corollary \ref{corollary:RM of interval semantics} assert that $\varepsilon_{\sigma_l,\sigma, \sigma_r}$ adopts the related properties of $\varepsilon_{\sigma}$ whenever the latter is closed under reduced meets modulo any ultrafilter. Moreover, $\varepsilon_{\sigma_l,\sigma, \sigma_r}$ also inherits certain structural nature of $\varepsilon_{\sigma}$ as stated below.
\begin{theorem} \label{theorem:structure of interval semantics}
	Let $F = \tuple{\A, \rightarrow}$ be an AAF, $\varepsilon_{\sigma_l}$, $\varepsilon_{\sigma}$ and $\varepsilon_{\sigma_r}$ three extension-based semantics such that $\varepsilon_{\sigma_l,\sigma, \sigma_r}(F) \neq \emptyset$ and let $\varepsilon_{\sigma}(F)$ be closed under reduced meets modulo any ultrafilter. Then $\tuple{\varepsilon_{\sigma_l,\sigma, \sigma_r}(F), \subseteq}$ is a sub-dcpo of $\tuple{\varepsilon_{\sigma}(F), \subseteq}$.
\end{theorem}
\begin{proof}
	Since $\varepsilon_{\sigma_l,\sigma, \sigma_r}(F) \neq \emptyset$, so is $\varepsilon_{\sigma}(F)$. Then, by Theorem \ref{theorem:dcpo}, both $\tuple{\varepsilon_{\sigma_l,\sigma, \sigma_r}(F), \subseteq}$ and $\tuple{\varepsilon_{\sigma}(F), \subseteq}$ are dcpos. Moreover, by Theorem \ref{theorem:dcpo} again, any directed subset $S$ of $\varepsilon_{\sigma_l,\sigma, \sigma_r}(F)$ has the same supremum in these two dcpos.
\end{proof}
As mentioned at the beginning of this section, interval semantics is a  generalization of parameterized semantics. The latter is defined as follows.
\begin{definition}[Parametrized semantics] \label{definition:parametrized set}
	Let $\varepsilon_{\sigma}$ and $\varepsilon_{\sigma_r}$ be two extension-based semantics. The $(\sigma, \sigma_r)$-parametrized semantics, denoted by $\varepsilon_{\sigma,\sigma_r}$, is defined as $\varepsilon_{\sigma,\sigma_r} \triangleq \varepsilon_{\sigma_l,\sigma,\sigma_r}$ with $\varepsilon_{\sigma_l} \triangleq \lambda_F.\set{\emptyset}$, the constant function mapping each AAF $F$ to $\set{\emptyset}$.
	That is, for any AAF $F$,
	\begin{eqnarray*}
		\varepsilon_{\sigma,\sigma_r}(F) \triangleq \set{ E \in \varepsilon_{\sigma}(F) \mid E \subseteq \bigcap \varepsilon_{\sigma_r}(F)}.
	\end{eqnarray*}
\end{definition}
Similarly, we use the notation $\varepsilon_{\textit{max}, (\sigma, \sigma_r)}$ to denote the semantics induced by the maximality based on $\varepsilon_{\sigma, \sigma_r}$. Clearly, parameterized semantics has the properties asserted by Corollary \ref{corollary:RM of interval semantics} and Theorems \ref{theorem:RMMU_interval} and \ref{theorem:structure of interval semantics}. Moreover, for parameterized semantics $\varepsilon_{\sigma,\sigma_r}$, Theorem \ref{theorem:structure of interval semantics} may be strengthened as below.
\begin{theorem} \label{theorem:structure_para}
	Let $F = \tuple{\A, \rightarrow}$ be an AAF, $\varepsilon_{\sigma}$ and $\varepsilon_{\sigma_r}$ two extension-based semantics such that $\varepsilon_{\sigma,\sigma_r}(F) \neq \emptyset$ and let $\varepsilon_{\sigma}(F)$ be closed under reduced meets modulo any ultrafilter.
	\begin{itemize}
		\item[a.] If $\tuple{\varepsilon_{\sigma}(F), \subseteq}$ is a (algebraic) cpo (dcpo) then so is $\tuple{\varepsilon_{\sigma,\sigma_r}(F), \subseteq}$.
		\item[b.] If $\tuple{\varepsilon_{\sigma}(F), \subseteq}$ is a (algebraic) complete semilattice then so is $\tuple{\varepsilon_{\sigma,\sigma_r}(F), \subseteq}$.
	\end{itemize}	 
\end{theorem}
\begin{proof}
	By Theorem \ref{theorem:structure of interval semantics}, $\tuple{\varepsilon_{\sigma,\sigma_r}(F), \subseteq}$ is a sub-dcpo of $\tuple{\varepsilon_{\sigma}(F), \subseteq}$.
	To complete the proof, we need to prove the assertions: (1) The bottom in $\tuple{\varepsilon_{\sigma}(F), \subseteq}$ must be in $\varepsilon_{\sigma,\sigma_r}(F)$ if it exists; (2) $\set{X \in K(\varepsilon_{\sigma}(F)) \mid X \subseteq E} \subseteq \varepsilon_{\sigma,\sigma_r}(F)$ for each $E \in \varepsilon_{\sigma,\sigma_r}(F)$; and (3) For any nonempty subset $S$ of $\varepsilon_{\sigma,\sigma_r}(F)$, the infimum of $S$ in $\tuple{\varepsilon_{\sigma}(F), \subseteq}$ is also in $\varepsilon_{\sigma,\sigma_r}(F)$ whenever it exists. Here $K(\varepsilon_{\sigma}(F))$ is the set of all compact elements in $\tuple{\varepsilon_{\sigma}(F), \subseteq}$. Under the assumption that $\varepsilon_{\sigma,\sigma_r}(F) \neq \emptyset$, all these are straightforward based on the fact that $\varepsilon_{\sigma,\sigma_r}(F)$ is down-closed, that is, $E_1 \subseteq E_2 \in \varepsilon_{\sigma,\sigma_r}(F)$ and $E_1 \in \varepsilon_{\sigma}(F)$ implies $E_1 \in \varepsilon_{\sigma,\sigma_r}(F)$.
\end{proof}
In fact, it is easy to see that any directed subset (or, nonempty subset) of $\varepsilon_{\sigma,\sigma_r}(F)$ has the same supremum (infimum, resp.) in $\tuple{\varepsilon_{\sigma,\sigma_r}(F), \subseteq}$ and $\tuple{\varepsilon_{\sigma}(F), \subseteq}$. Thus, the former is a sub-cpo (or, sub-dcpo, sub-complete semilattice) of the latter.
\par
Next we consider concrete parametrized semantics by instantiating $\varepsilon_\sigma$ and $\varepsilon_{\sigma_r}$ in $\varepsilon_{\sigma,\sigma_r}$. In particular, we will focus on $\varepsilon_{\sigma,\sigma_r}$ with $\varepsilon_{\sigma} = \varepsilon_{ad}^{\ell mn}$, which is the graded variant of the standard parametrized semantics \cite{Dunne13ParametricSemantics}. 
\begin{remark} \label{remark:parametrized semantics}
	 In the situation that $\varepsilon_{\sigma}$ is $\varepsilon_{ad}^{\ell mn}$, we also denote $\varepsilon_{\sigma,\sigma_r}$ (or, $\varepsilon_{\textit{max}, (\sigma, \sigma_r)}$) by $\varepsilon_{ad,\sigma_r}^{\ell mn}$ ($\varepsilon_{\textit{max}, (ad, \sigma_r)}^{\ell mn}$, resp.). A similar notation is adopted when $\varepsilon_{\sigma}$ is other semantics, e.g., $\varepsilon_{cf}^{\ell}$ and $\varepsilon_{co}^{\ell mn}$. If $\varepsilon_{\sigma_r}$ is $\varepsilon_{pr}^{\ell mn}$ (or, $\varepsilon_{ss}^{\ell mn \eta}$) then $\varepsilon_{\textit{max}, (ad, \sigma_r)}^{\ell mn}$ is exactly the graded variant of the ideal semantics (eager semantics, resp.) introduced in \cite{Dung07Ideal} (\cite{Caminada2007Eager}, resp.).
	Thus, in these two situations, we also use the notations $\varepsilon_{id}^{\ell mn}$ and $\varepsilon_{eg}^{\ell mn \eta}$ to denote them respectively.
\end{remark}
\begin{corollary}\label{corollary:structure_para}
	Let $F$ be a finitary AAF and $\varepsilon_{\sigma_r}$ an extension-based semantics\footnote{The clause ($a$) doesn't depend on the assumption that $F$ is finitary.}.
	\begin{itemize}
		\item[a.] $\tuple{\varepsilon_{cf, \sigma_r}^{\ell}(F), \subseteq}$ is a sub-algebraic complete semilattice of $\tuple{\varepsilon_{cf}^{\ell}(F), \subseteq}$.
		\item[b.] $\tuple{\varepsilon_{ad, \sigma_r}^{\ell mn}(F), \subseteq}$ is a sub-complete semilattice of $\tuple{\varepsilon_{ad}^{\ell mn}(F), \subseteq}$.
		\item[c.] $\tuple{\varepsilon_{co, \sigma_r}^{\ell mn}(F), \subseteq}$ is a sub-cpo of $\tuple{\varepsilon_{co}^{\ell mn}(F), \subseteq}$ whenever $\varepsilon_{co, \sigma_r}^{\ell mn}(F) \neq \emptyset$.
		\item[d.] $\tuple{\varepsilon_{co, \sigma_r}^{\ell mn}(F), \subseteq}$ is a sub-complete semilattice of $\tuple{\varepsilon_{co}^{\ell mn}(F), \subseteq}$ whenever $n \geq \ell \geq m$ and $\varepsilon_{co, \sigma_r}^{\ell mn}(F) \neq \emptyset$.
	\end{itemize}
\end{corollary}
\begin{proof}
	Immediately follows from Theorems \ref{theorem:close under meet_fundamental semantics} and \ref{theorem:structure_para}, Corollaries \ref{corollary:structure_cf_ad} and \ref{corollary:structure_co}, and Proposition \ref{proposition:structure_co}.	Notice that both $\varepsilon_{cf, \sigma_r}^{\ell}(F)$ and $\varepsilon_{ad, \sigma_r}^{\ell mn}(F)$ are always nonempty because that there is at least one element $\emptyset$ in them.
\end{proof}
In the nontrivial situation that $\emptyset \neq \varepsilon_{\sigma_r}(F) \subseteq \varepsilon_{cf}^{\ell}(F)$, the clauses ($a$) and ($b$) in Corollary \ref{corollary:structure_para} can be strengthened as below.
\begin{proposition} \label{proposition:structure_para}
	Let $F$ be an AAF and $\varepsilon_{\sigma_r}$ an extension-based semantics such that $\emptyset \neq \varepsilon_{\sigma_r}(F) \subseteq \varepsilon_{cf}^{\ell}(F)$. 
	\begin{itemize}
		\item[a.] $\tuple{\varepsilon_{cf, \sigma_r}^{\ell}(F), \subseteq}$ is an algebraic complete lattice with $\textit{sup}\D = \bigcup \D$ for any $\D \subseteq \varepsilon_{cf,\sigma_r}^{\ell}(F)$ and the top $\bigcap \varepsilon_{\sigma_r}(F)$. Hence $\varepsilon_{\textit{max}, (cf, \sigma_r)}^{\ell}(F) = \set{\bigcap \varepsilon_{\sigma_r}(F)}$.
		\item[b.] $\tuple{\varepsilon_{ad, \sigma_r}^{\ell mn}(F), \subseteq}$ is a complete lattice with $\textit{sup}\D = \bigcup \D$ for any $\D \subseteq \varepsilon_{ad,\sigma_r}^{\ell mn}(F)$. Hence $\varepsilon_{\textit{max}, (ad, \sigma_r)}^{\ell mn}(F) = \set{\bigcup \varepsilon_{ad, \sigma_r}^{\ell mn}(F)}$.
	\end{itemize}
\end{proposition}
\begin{proof}
	($a$) Immediately follows from Corollary \ref{corollary:structure_cf_ad}($d$) and Lemma \ref{lemma:properties_cf_def}($d$).
	\par ($b$) Clearly, $\varepsilon_{ad, \sigma_r}^{\ell mn}(F)$ is nonempty due to $\emptyset \in \varepsilon_{ad,\sigma_r}^{\ell mn}(F)$.
	Let $\D \subseteq \varepsilon_{ad,\sigma_r}^{\ell mn}(F)$. 
	Then 
	\begin{eqnarray}
		\bigcup \D \subseteq \bigcap \varepsilon_{\sigma_r}(F). \label{formulas:para}
	\end{eqnarray}
	Since $\emptyset \neq \varepsilon_{\sigma_r}(F) \subseteq \varepsilon_{cf}^{\ell}(F)$, by Lemma \ref{lemma:properties_cf_def}($d$), we have $\bigcup \D \in \varepsilon_{cf}^{\ell}(F)$.
	Moreover, by Lemma \ref{lemma:properties_cf_def}($a$), $\bigcup \D \in \textit{Def}^{mn}(F)$ due to $\D \subseteq \varepsilon_{ad}^{\ell mn}(F)$.
	Thus, $\bigcup \D \in \varepsilon_{ad}^{\ell mn}(F)$, and hence $\bigcup \D \in \varepsilon_{ad,\sigma_r}^{\ell mn}(F)$ due to (\ref{formulas:para}).
	Consequently, $\textit{sup}\D = \bigcup \D$.
\end{proof}
We end this section by considering graded variants of ideal semantics and eager semantics, which are derived from parameterized ones in terms of the maximality, see, Remark \ref{remark:parametrized semantics}. As shown in \cite{Baroni11Introduction}, ideal semantics may be regarded as a parametric model of acceptance for non-graded argumentation, which is inspired by an analogous notion of ideal skeptical semantics for extended logic programs \cite{Alferes1993Ideal}. Eager semantics, proposed in \cite{Caminada2007Eager}, is a more credulous semantics. The parametric properties of these semantics are given in \cite{Dunne13ParametricSemantics} to support skeptical reasoning within abstract argumentation settings. In the situation that $\emptyset \neq \varepsilon_{\sigma_r}(F) \subseteq \varepsilon_{ad}^{\ell mn}(F)$, by Proposition \ref{proposition:structure_para} and Corollary \ref{corollary:structure_cf_ad}($c$), it is obvious that the unique extension in $\varepsilon_{\textit{max}, (ad, \sigma_r)}^{\ell mn}(F)$ is the infimum of the set $\varepsilon_{\sigma_r}(F)$ in the complete semilattice $\tuple{\varepsilon_{ad}^{\ell mn}(F), \subseteq}$. In particular, we have
\begin{corollary} \label{corollary:structure_id_eg}
	Let $F$ be an AAF.
	\begin{itemize}
		\item[a.] $\varepsilon_{id}^{\ell mn}(F) = \set{\textit{inf}(\varepsilon_{pr}^{\ell mn}(F))}$ whenever $\ell \geq m$, $n \geq m$ and $F$ is finitary.
		\item[b.] $\varepsilon_{id}^{\ell mn}(F) = \set{\textit{inf}(\varepsilon_{pr}^{\ell mn}(F))}$ whenever $n \geq \ell \geq m$. 
		\item[c.] $\varepsilon_{eg}^{\ell mn \eta}(F) = \set{\textit{inf}(\varepsilon_{ss}^{\ell mn \eta}(F))}$ whenever $\ell \geq m$, $n \geq m$ and $F$ is finitary. 
	\end{itemize}
	where all infimums are considered in the complete semilattice $\tuple{\varepsilon_{ad}^{\ell mn}(F), \subseteq}$.
\end{corollary}
\begin{proof}
	Immediately follows from Proposition \ref{proposition:structure_para} and Corollaries \ref{corollary:structure_cf_ad}($c$), \ref{corollary:existence_pr} and \ref{corollary:existence_ss}.
\end{proof}
If $\eta \geq \ell$, by Proposition \ref{proposition:range_relative} and Definition \ref{definition:graded_extensions}, we have $\varepsilon_{ss}^{\ell mn \eta}(F) \subseteq \varepsilon_{pr}^{\ell mn}(F) \subseteq \varepsilon_{co}^{\ell mn}(F)$ for any AAF $F$, hence their infimums in the complete semilattice $\tuple{\varepsilon_{ad}^{\ell mn}(F), \subseteq}$ are comparable as below whenever these sets are nonempty:
\begin{eqnarray}
	\textit{inf}(\varepsilon_{co}^{\ell mn}(F)) \subseteq \textit{inf}(\varepsilon_{pr}^{\ell mn}(F)) \subseteq \textit{inf}(\varepsilon_{ss}^{\ell mn \eta}(F)). \label{formulas:credulity}
\end{eqnarray}
For any finitary AAF $F$, in the situation that $\eta \geq \ell \geq m$ and $n \geq m$, by Corollary \ref{corollary:structure_id_eg} and Proposition \ref{proposition:structure_gr}, these infimums are exactly the unique extension in $\varepsilon_{gr}^{\ell mn}(F)$, $\varepsilon_{id}^{\ell mn}(F)$ and $\varepsilon_{eg}^{\ell mn \eta}(F)$, respectively. Thus, the credulity degrees of the semantics $\varepsilon_{gr}^{\ell mn}$, $\varepsilon_{id}^{\ell mn}$ and $\varepsilon_{eg}^{\ell mn \eta}$ are comparable as displayed in (\ref{formulas:credulity}) in this situation, which is a graded version of \cite[Theorem 5]{Caminada2007Eager}. 
\par Under the assumption of well-foundedness, the unique extension in $\varepsilon_{id}^{\ell mn}(F)$ (or, $\varepsilon_{eg}^{\ell mn \eta}(F)$) can be explicitly constructed in terms of iterations of $D_n^m$.
\begin{corollary} \label{corollary:well-founded_id_nonempty}
	Let $F = \tuple{\A, \rightarrow}$ be an AAF, $n \geq \ell \geq m$ and $E \in \varepsilon_{ad}^{\ell mn}(F)$. If $\rightarrow^+$ is well-founded on $\A - D_{\substack{m \\ n}}^{\lambda_F}(E)$ then $\varepsilon_{id}^{\ell mn}(F) = \set{D_{\substack{m \\ n}}^{\lambda_F}(E)}$, moreover, $\varepsilon_{eg}^{\ell mn \eta}(F) = \set{D_{\substack{m \\ n}}^{\lambda_F}(E)}$ whenever $\eta \geq \ell$.
\end{corollary}
\begin{proof}
	Immediately follows from Remark \ref{remark:parametrized semantics}, Corollaries \ref{corollary:well-founded_co_nonempty} and \ref{corollary:well-founded_ss_rrs_nonempty}.
\end{proof}
\begin{corollary} \label{corollary:well-founded_id_empty}
	Let $F = \tuple{\A, \rightarrow}$ be an AAF, $\ell \geq m$ and $n \geq m$. If $\rightarrow^+$ is well-founded on $\A - D_{\substack{m \\ n}}^{\lambda_F}(\emptyset)$ then $\varepsilon_{id}^{\ell mn}(F) = \varepsilon_{eg}^{\ell mn \eta}(F) = \set{D_{\substack{m \\ n}}^{\lambda_F}(\emptyset)}$.
\end{corollary}
\begin{proof}
	Immediately follows from Remark \ref{remark:parametrized semantics}, Corollaries \ref{corollary:well-founded_co_empty} and \ref{corollary:well-founded_ss_rrs_empty}.
\end{proof}
\section{Extension-based abstract semantics}\label{Sec: Extension-based abstract semantics}
The preceding sections have considered a variety of extension-based concrete semantics. This section aims to explore abstract semantics defined by a first order language. The reader may find that the operator $\bigcap_D$ plays a central role in the proofs of Corollary \ref{corollary:RM of interval semantics} and Theorems \ref{theorem:close under meet and universal definability}, \ref{theorem:existence_stg_ss_rra_rrs}, \ref{theorem:epsilon_extensible}, \ref{theorem:epsilon_inference}, \ref{theorem:Lindenbaum_RM}, \ref{theorem:anti-epsilon set} and \ref{theorem:structure_para}. This depends on the fact (or, assumption) that the graded semantics involved in these results are closed under the operator $\bigcap_D$. Clearly, we can apply metatheorems (e.g., Corollary \ref{corollary:RM of interval semantics} and Theorems \ref{theorem:close under meet and universal definability}, \ref{theorem:epsilon_extensible}, \ref{theorem:epsilon_inference}, \ref{theorem:Lindenbaum_RM}, \ref{theorem:anti-epsilon set} and \ref{theorem:structure_para}) on a given concrete semantics $\varepsilon$ only when we know that $\varepsilon$ is closed under $\bigcap_D$. However, so far, we only show that a family of concrete semantics have the $\bigcap_D$-closeness (Theorem \ref{theorem:close under meet_fundamental semantics} and Proposition \ref{proposition:RMMU_na}) and interval semantics inherit such closeness (Theorem \ref{theorem:RMMU_interval}). Thus, at this point, a problem arises naturally, that is, what kind of extension-based semantics is closed under reduced meets modulo any ultrafilter. This section intends to solve this problem partially. To this end, we must provide a method to describe semantics abstractly. This section will focus on the semantics defined by a first order language and consider this problem from a model-theoretical viewpoint.
\par 
Compared to studying on finite AAFs, we often need more powerful mathematical tools to explore infinite AAFs. This section will adopt a technique so that we can take advantage of fruitful results in model theory. Through formalizing the semantics of a given formal system in FoL, one often can deal with some issues of this system by using results, ideas and proof methods in FoL. Such technique is widely applied in, e.g., modal logic \cite{Blackburn03Book}, argumentation \cite{Weydert11SsForInf} and nonmonotonic logic \cite{Zhaohui06TCS, Zhaohui07JSL, Zhaohui07NML}, etc. This section will apply this technique to deal with the problem mentioned above. However, unlike \cite{Weydert11SsForInf,Blackburn03Book,Zhaohui06TCS, Zhaohui07JSL, Zhaohui07NML}, what we care about is not a certain extension-based concrete semantics but a family of semantics. We intend to provide a segment of a first order language and show that each extension-based semantics defined by this segment is closed under reduced meets modulo any ultrafilter w.r.t finitary AAFs. The reader is assumed to be familiar with some elementary knowledge in model theory (e.g., \cite{Chang1990Book}).
\subsection{A first order language $\textit{FS}(\Omega)$}
Given an AAF $F = \tuple{\A, \rightarrow}$, it is obvious that $F$ itself may be viewed as a model of the first order language $\Lan = \set{Att}$, where $Att$ is a binary relation symbol. In the model $F$, $Att$ is interpreted by the attack relation $\rightarrow$ in $F$, that is, for any $a, b \in \A$, $F \vDash Att(x, y)[a, b]$ iff $a \rightarrow b$ in $F$. Here, as usual, $F \vDash Att(x, y)[a, b]$ means that the formula $Att(x, y)$ is satisfied by the sequence $a$, $b$ in $F$. To consider the problem mentioned above, the language $\Lan$ is expanded to $\Lan' = \Lan \cup \set{P}$, where $P$ is an unary relation symbol. A $F$-based model for $\Lan'$ is a triple $\mathfrak{A} = \tuple{\A, \rightarrow, E}$ with $E \subseteq \A$, and $P$ is interpreted by $E$, that is, for any $a \in \A$, $\mathfrak{A} \vDash P(x)[a]$ iff $a \in E$. In the following, to simplify notations, we often adopt the notation $\rho$ to denote an assignment abstractly. Given an assignment $\rho$, the value assigned to a variable $x$ is denoted by $\rho(x)$, and $\rho \set{b/x}$ is used to denote the assignment which agrees with $\rho$ except that $b$, instead of $\rho(x)$, is assigned to $x$. The notation $\mathfrak{A} \vDash \phi [\rho]$ means that the formula $\phi$ is satisfied by the assignment $\rho$ in the model $\mathfrak{A}$, which is defined by induction on formulas, see \cite{Chang1990Book} for more details.
\begin{definition} \label{definition:universal assignment}
	Let $F = \tuple{\A, \rightarrow}$ be an AAF, $E \subseteq \A$ and $\alpha \triangleq \exists {x}\phi$ a formula of $\Lan'$. An assignment $\rho$ is said to be universal w.r.t. the model $\mathfrak{A} \triangleq \tuple{\A, \rightarrow, E}$ and the formula $\alpha$ if
	\begin{eqnarray*}
		\mathfrak{A} \vDash \exists {x} \phi[\rho] \text{~implies~} \mathfrak{A} \vDash \forall {x} \phi[\rho].
	\end{eqnarray*}
	Otherwise, $\rho$ is said to be non-universal w.r.t. $\mathfrak{A}$ and $\alpha$, and the set of all such assignments is denoted by $\textit{Nua}(\mathfrak{A}, \alpha)$ or $\textit{Nua}(F, E, \alpha)$.
\end{definition}
Roughly speaking, for any assignment $\rho$ that is universal w.r.t. the model $\mathfrak{A}$ and the formula $\exists {x}\phi$, the value assigned to the variable $x$ has no effect on the satisfiability of $\phi$ (notice, not $\exists {x}\phi$) under the assignment $\rho$ in $\mathfrak{A}$, in other words, $\mathfrak{A} \vDash \phi[\rho \set{b/x}]$ for some $b \in \A$ implies $\mathfrak{A} \vDash \phi[\rho \set{a/x}]$ for each $a \in \A$.
\begin{example}
	(1) Consider the formula $\alpha \triangleq \exists {x_2} (P(x_1) \vee \neg P(x_1) \vee Att(x_2, x_1))$.
	For any AAF $F = \tuple{\A, \rightarrow}$, $E \subseteq \A$ and assignment $\rho$, due to 
	\begin{eqnarray*}
		\tuple{\A, \rightarrow, E} \vDash P(x_1) \vee \neg P(x_1) [\rho],
	\end{eqnarray*}
	we always have $\tuple{\A, \rightarrow, E} \vDash \forall {x_2} (P(x_1) \vee \neg P(x_1) \vee Att(x_2, x_1)) [\rho]$, and hence $\textit{Nua}(F, E, \alpha) = \emptyset$.
	\par (2) Consider the left AAF $F$ in Figure \ref{figure:counterexampleforgr}. Let $\alpha \triangleq \exists {x_2} (P(x_1) \vee Att(x_2, x_3))$ and $E = \set{a_1, b_1}$. 
	\par For each assignment $\rho$ with $\rho(x_1) \in E$, due to $\tuple{\A, \rightarrow, E} \vDash P(x_1) [\rho]$, we have $\tuple{\A, \rightarrow, E} \vDash \forall {x_2} (P(x_1) \vee Att(x_2, x_3))[\rho]$. Thus, all assignments $\rho$ with $\rho(x_1) \in E$ are universal w.r.t. the model $\tuple{\A, \rightarrow, E}$ and the formula $\alpha$.
	\par On the other hand, for each assignment $\rho$ with $\rho(x_1) = c_1$ and $\rho(x_3) \neq d_1$, since $u^- \neq \A$ for each argument $u$ in $F$ and $\tuple{\A, \rightarrow, E} \vDash \neg P(x_1) [\rho]$, we have 
	\begin{eqnarray*}
		\tuple{\A, \rightarrow, E} \vDash \neg \forall {x_2} (P(x_1) \vee Att(x_2, x_3))[\rho],
	\end{eqnarray*}
	however $\tuple{\A, \rightarrow, E} \vDash \exists {x_2} (P(x_1) \vee Att(x_2, x_3))[\rho]$ due to $u^- \neq \emptyset$ for each argument $u~(\neq d_1)$ in $F$. Thus, all assignments $\rho$ with $\rho(x_1) = c_1$ and $\rho(x_3) \neq d_1$ are non-universal w.r.t. the model $\tuple{\A, \rightarrow, E}$ and the formula $\alpha$.
\end{example}
In dealing with infinite AAFs, it is often assumed that the AAF under consideration is finitary, see, e.g. \cite{Dung95Acceptability, Grossi19Graded, Weydert11SsForInf, Baumann15Infinite, Baroni13Infinite}. 
The next notion reflects this assumption in the model theoretical sense.
\begin{definition} \label{definition:omega_finitary}
	Given a class $\Omega$ of AAFs, a formula $\alpha \triangleq \exists {x} \phi$ of $\Lan'$ is said to be $\Omega$-finitary if, for any AAF $F = \tuple{\A, \rightarrow} \in \Omega$ and assignment $\rho$,
	\begin{eqnarray*}
		|\bigcup_{E \subseteq \A} \set{b \in \A \mid \rho \in \textit{Nua}(F, E, \alpha) \text{~and~} \tuple{\A, \rightarrow, E} \vDash \phi[\rho\set{b/x}]}| < \omega.
	\end{eqnarray*}
\end{definition}
Thus, for any $F \in \Omega$, assignment $\rho$ and $\Omega$-finitary formula $\exists{x} \phi$, there are only finitely many arguments $b$ such that $\phi$ is realized nontrivially in $F$-based models under the assignment $\rho \set{b/x}$. Here, the statement ``$\phi$ is realized nontrivially under the assignment $\rho \set{b/x}$'' means that it depends essentially on the value $b$ that $\tuple{\A, \rightarrow, E} \vDash \phi[\rho \set{b/x}]$, which is captured by the condition $\rho \in \textit{Nua}(F, E, \exists x \phi)$.
\begin{example}
	For any finitary AAF $F = \tuple{\A, \rightarrow}$, since $|\set{b \in \A \mid b \rightarrow a}| < \omega$ for each $a \in \A$, we have
	\begin{eqnarray*}
		| \set{b\in \A \mid \tuple{\A, \rightarrow} \vDash Att(x, y)[\rho \set{b/x}]} | < \omega \text{~for any assignment~} \rho.
	\end{eqnarray*}
	Based on this, it is easy to see that the formula $\exists x Att(x, y)$ is $\Omega$-finitary for the class $\Omega$ of all finitary AAFs. More nontrivial $\Omega$-finitary formulas will be given in the proof of Proposition \ref{proposition:FoL_fundamental semantics}.
\end{example}
This section aims to characterize extension-based semantics that is closed under the operator $\bigcap_D$ in terms of model theoretical notions. To this end,  a segment of $\Lan'$ is introduced below.
\begin{definition} \label{definition:FS_formulas}
	Given a class $\Omega$ of AAFs, the $\Omega$-finitary segment of $\Lan'$, denoted by $\textit{FS}(\Omega)$, is the least set of $\Lan'$-formulas satisfying the following closure conditions:	
	\begin{itemize}
		\item[$FS_1$.] $P(x)$, $Att(x, y)$, $x \equiv y \in \textit{FS}(\Omega)$ for any variables $x$, $y$. That is, $\textit{FS}(\Omega)$ contains all atomic formulas in $\Lan'$.
		\item[$FS_2$.] If $\alpha$, $\beta \in \textit{FS}(\Omega)$ in which neither $\forall$ nor $\exists$ occurs, then ($\alpha \wedge \beta$), ($\alpha \vee \beta$), ($\neg \alpha) \in \textit{FS}(\Omega)$. That is, $\textit{FS}(\Omega)$ contains all boolean combinations of atomic formulas in $\Lan'$.
		\item[$FS_3$.] If $\alpha(x_1 \dots x_n) \in \textit{FS}(\Omega)$ and $\exists {x_i} \alpha(x_1 \dots x_n)$ ($1 \leq i \leq n$) is $\Omega$-finitary (see, Definition \ref{definition:omega_finitary}) then  $\exists {x_i} \alpha(x_1 \dots x_n) \in \textit{FS}(\Omega)$.
		\item[$FS_4$.] If $\alpha(x_1 \dots x_n) \in \textit{FS}(\Omega)$ and $1 \leq i \leq n$ then $\forall {x_i} \alpha(x_1 \dots x_n) \in \textit{FS}(\Omega)$.
	\end{itemize}	
\end{definition}
It is not difficult to see that each formula in $\textit{FS}(\Omega)$ is in prenex form. Here, as usual, the notation $\alpha(x_1 \dots x_n)$ means that the free variables of $\alpha$ are among $x_1 \dots x_n$. A formula $\alpha$ is said to be a \textit{sentence} (i.e., closed formula) if all variables occurring in $\alpha$ are bounded by quantifiers, that is, $\alpha$ doesn't contain any free variable. The parentheses in ($\alpha \wedge \beta$), ($\alpha \vee \beta$) and ($\neg \alpha$) will be omitted in the remainder of this paper when no confusion can arise. In addition to $\wedge$, $\vee$ and $\neg$, the boolean operator of implication is often used, which is defined as $\alpha \rightarrow \beta \triangleq \neg \alpha \vee \beta$ as usual. At the beginning of this section, we mention that this section will focus on the semantics defined by a first order language. Now we define this notion formally.
\begin{definition} \label{definition:fol_ext}
	Given a class $\Omega$ of AAFs, an extension-based semantics $\varepsilon$ is said to be FoL-definable w.r.t. $\Omega$ if there exists a set $\Sigma_{\varepsilon}$ of sentences in $\Lan'$ such that, for any AAF $F = \tuple{\A, \rightarrow} \in \Omega$ and $E \subseteq \A$,
	\begin{eqnarray*}
		\tuple{\A, \rightarrow, E} \vDash \Sigma_{\varepsilon} \text{~iff~} E \in \varepsilon(F).
	\end{eqnarray*}
	In particular, $\varepsilon$ is said to be $\textit{FS}(\Omega)$-definable if $\Sigma_{\varepsilon} \subseteq \textit{FS}(\Omega)$.
	Here $\tuple{\A, \rightarrow, E} \vDash \Sigma_{\varepsilon}$ means that each sentence in $\Sigma_{\varepsilon}$ is true in the model $\tuple{\A, \rightarrow, E}$.	
\end{definition}
The semantics, playing fundamental roles in AAFs, $\textit{Def}^{mn}$, $\varepsilon_{cf}^\ell$, $\varepsilon_{ad}^{\ell mn}$, $\varepsilon_{co}^{\ell mn}$ and $\varepsilon_{stb}^{\ell mn}$ are $\textit{FS}(\Omega)$-definable w.r.t. the class $\Omega$ of all finitary AAFs.
By Definition \ref{definition:graded_extensions}, it is immediate to see that these semantics are FoL-definable in $\Lan'$ w.r.t. all AAFs because their definitions may be formalized in $\Lan'$. For instance, the sentences $\alpha_{mn}^1$, $\alpha_{mn}^2$, $\alpha_{\ell}^3$ and $\alpha_{\ell}^4$ given in the proof of Proposition \ref{proposition:FoL_fundamental semantics} may be used to do this. 
A further observation is that $\alpha_{mn}^1$, $\alpha_{mn}^2$, $\alpha_{\ell}^3$ and $\alpha_{\ell}^4$ are essentially in $\textit{FS}(\Omega)$ up to logical equivalence, which is demonstrated in the next proposition.
\begin{proposition}\label{proposition:FoL_fundamental semantics}
	The semantics $\textit{Def}^{mn}$, $\varepsilon_{cf}^{\ell}$, $\varepsilon_{ad}^{\ell mn}$, $\varepsilon_{co}^{\ell mn}$ and $\varepsilon_{stb}^{\ell mn}$ are $\textit{FS}(\Omega)$-definable w.r.t. the class $\Omega$ of all finitary AAFs.
\end{proposition}
\begin{proof}
	For any positive number $n$, put
	\begin{eqnarray*}
		Cf(x_1, \cdots, x_n, x) \triangleq (\bigwedge_{1 \leq i \neq j \leq n} \neg(x_i \equiv x_j)) \wedge (\bigwedge_{1 \leq i \leq n} Att(x_i, x)).
	\end{eqnarray*}
	The proof falls naturally into five parts. The procedure is to find $\textit{FS}(\Omega)$-sentences to formalize these semantics.
	\par (I) We intend to show that $mn$-self-defended sets may be defined by $\textit{FS}(\Omega)$-sentences. To this end, for any positive number $m$ and $n$, let
	\begin{eqnarray*}
		\begin{split}
			\alpha_{mn}^1 \triangleq 
				&~ \forall x(P(x) \rightarrow \neg \exists {x_1 \cdots x_m}(Cf(x_1, \cdots, x_m, x) \wedge \\
				& \bigwedge_{1 \leq i \leq m} \neg \exists {y_i^1 \cdots y_i^n}(Cf(y_i^1, \cdots, y_i^n, x_i) \wedge \bigwedge_{1 \leq k \leq n}P(y_i^k)))).
		\end{split}
	\end{eqnarray*}
	It is easy to check that, for any AAF $F = \tuple{\A, \rightarrow}$ and $E \subseteq \A$, we have
	\begin{eqnarray*}
		\begin{split}
				&\tuple{\A, \rightarrow, E} \vDash \alpha_{mn}^1\\
			\Leftrightarrow &~ \text{For each~} a \in E, \tuple{\A, \rightarrow, E} \vDash \neg \exists {x_1 \cdots x_m}(Cf(x_1, \cdots, x_m, x) \wedge \\
			& \bigwedge_{1 \leq i \leq m} \neg \exists {y_i^1 \cdots y_i^n}(Cf(y_i^1, \cdots, y_i^n, x_i) \wedge \bigwedge_{1 \leq k \leq n}P(y_i^k)))[a]. \\
			\Leftrightarrow &~ \text{For each $a \in E$, it holds in $F$ that~} \nexistsn{b}{m}(~b \rightarrow a \text{ and } \nexistsn{c}{n}(c \rightarrow b \text{ and } c \in E)).\\
			\Leftrightarrow &~ \text{For each $a \in E$, it holds in $F$ that $a \in D_n^m(E)$.}\\
			\Leftrightarrow &~ \text{It holds in $F$ that~} E \subseteq D_n^m(E).\\
			\Leftrightarrow &~ E \in \textit{Def}^{mn}(F).
		\end{split}
	\end{eqnarray*}
	Thus, $mn$-self-defended sets are FoL-definable.
	Moreover, it isn't difficult to check that $\alpha_{mn}^1$ is equivalent to
	\begin{eqnarray*}
	\begin{split}
		\beta_{mn}^1 \triangleq 
		&~ \forall {x~x_1 \cdots x_m}\exists {y_1^1 \cdots y_1^n \cdots y_m^1 \cdots y_m^n}(\neg P(x) \vee \neg Cf(x_1, \cdots, x_m, x) \vee\\
		& \bigvee_{1 \leq i \leq m}(Cf(y_i^1, \cdots, y_i^n, x_i) \wedge \bigwedge_{1 \leq k \leq n}P(y_i^k))).
	\end{split}
	\end{eqnarray*}
	Thus, $mn$-self-defended sets are also defined by $\beta_{mn}^1$.	
	Next we verify that $\beta_{mn}^1 \in \textit{FS}(\Omega)$.
	By Definition \ref{definition:FS_formulas}, it suffices to show that, for any $u$, $v$ with $1 \leq u \leq m$ and $1 \leq v \leq n$, the formula $\beta_u^v$ is $\Omega$-finitary, where
	\begin{eqnarray*}
		\beta_u^v \triangleq
		\begin{cases}
			\exists {y_u^v \cdots y_u^n y_{u + 1}^1 \cdots y_{u + 1}^n \cdots y_m^1 \cdots y_m^n}\alpha &{\text{if~} u < m \text{~and~} v < n}  \\ 
			\exists {y_u^n y_{u + 1}^1 \cdots y_{u + 1}^n \cdots y_m^1 \cdots y_m^n}\alpha &{\text{if~} u < m \text{~and~} v = n}  \\ 
			\exists {y_m^v \cdots y_m^n}\alpha &{\text{if~} u = m \text{~and~} v < n}  \\ 
			\exists {y_m^n}\alpha &{\text{if~} u = m \text{~and~} v = n}
		\end{cases}
	\end{eqnarray*}
	and $\alpha$ is the matrix of the prenex formula $\beta_{mn}^1$, that is
	\begin{eqnarray*}
		\alpha \triangleq \neg P(x) \vee \neg Cf(x_1, \cdots, x_m, x) \vee \bigvee_{1 \leq i \leq m}(Cf(y_i^1, \cdots, y_i^n, x_i) \wedge \bigwedge_{1 \leq k \leq n}P(y_i^k)).
	\end{eqnarray*}
	We give the proof only for the case $u < m$ and $v < n$, the same reasoning may be applied to other cases.
	To simplify notations, put
	\begin{eqnarray*}
		\exists {\widetilde{y}} \triangleq \exists {y_u^{v + 1} \cdots y_u^n y_{u + 1}^1 \cdots y_{u + 1}^n \cdots y_m^1 \cdots y_m^n}.
	\end{eqnarray*}
	Thus, $\beta_u^v = \exists {y_u^v} \exists {\widetilde{y}}\alpha$.
	Clearly, the formula $\alpha$ is equivalent to $\gamma_1 \vee \gamma_2$ with
	\begin{eqnarray*}
		\begin{split}
			\gamma_1 &\triangleq \neg P(x) \vee \neg Cf(x_1, \cdots, x_m, x) \vee \bigvee_{1 \leq i \ne u \leq m}(Cf(y_i^1, \cdots, y_i^n, x_i) \wedge \bigwedge_{1 \leq k \leq n}P(y_i^k)),\\
			\gamma_2 &\triangleq Cf(y_u^1, \cdots, y_u^n, x_u) \wedge \bigwedge_{1 \leq k \leq n}P(y_u^k).
		\end{split}
	\end{eqnarray*}
	Moreover, it is easy to see that the variable $y_u^v$ doesn't occur in $\gamma_1$.
	In order to show that $\beta_u^v$ is $\Omega$-finitary, we prove the following claim.\\
	\\ \textbf{Claim 1~~} Let $F = \tuple{\A, \rightarrow}$ be an AAF. For each $E \subseteq \A$ and $\rho \in \textit{Nua}(F, E, \exists {y_u^v} \exists {\widetilde{y}} \alpha)$, 
	\begin{eqnarray*}
		\set{b \in \A \mid \tuple{\A, \rightarrow, E} \vDash \exists {\widetilde{y}}\alpha[\rho \set{b/y_u^v}]} \subseteq \set{b \in \A \mid b \rightarrow \rho(x_u) \text{~in~} F}.
	\end{eqnarray*}
	Assume that $E \subseteq \A$ and $\mathfrak{A} \triangleq \tuple{\A, \rightarrow, E}$. 
	We first prove the following assertion.\\
	\\
	\textbf{Claim 1.1~~} $\mathfrak{A} \vDash \neg \exists {\widetilde{y}}\gamma_1[\rho]$ for each assignment $\rho \in \textit{Nua}(\mathfrak{A}, \exists {y_u^v} \exists {\widetilde{y}}\alpha)$.\\
	\\
	Otherwise, since the variable $y_u^v$ doesn't occur in $\gamma_1$, it follows from $\mathfrak{A} \vDash \exists {\widetilde{y}}\gamma_1[\rho]$ that $\mathfrak{A} \vDash \exists {\widetilde{y}}\gamma_1[\rho \set{c/y_u^v}] \text{~for each~} c \in \A$.
	Thus, $\mathfrak{A} \vDash \forall y_u^v \exists {\widetilde{y}}\gamma_1[\rho]$.
	Further, since $\alpha$ is equivalent to $\gamma_1 \vee \gamma_2$, it follows that $\mathfrak{A} \vDash \forall y_u^v \exists {\widetilde{y}}\alpha[\rho]$,
	which contradicts $\rho \in \textit{Nua}(\mathfrak{A}, \exists {y_u^v} \exists {\widetilde{y}}\alpha)$, as desired.\\
	\\
	Next we show Claim 1 itself. 
	Let $\rho \in \textit{Nua}(\mathfrak{A}, \exists {y_u^v} \exists {\widetilde{y}}\alpha)$.
	By Claim 1.1, $\mathfrak{A} \vDash \neg \exists {\widetilde{y}}\gamma_1[\rho]$.
	Further, since $y_u^v$ doesn't occur in $\gamma_1$, we get 
	\begin{eqnarray}
		\mathfrak{A} \vDash \neg \exists {\widetilde{y}}\gamma_1[\rho \set{b/y_u^v}] \text{~for any~} b \in \A \label{formula:fol}
	\end{eqnarray}
	Thus,
	\begin{eqnarray*}
		\begin{split}
			&\qquad \quad \set{b \in \A \mid \mathfrak{A} \vDash \exists {\widetilde{y}} \alpha[\rho \set{b/y_u^v}]}\\
			&= \qquad \set{b \in \A \mid \mathfrak{A} \vDash \exists {\widetilde{y}}(\gamma_1 \vee \gamma_2)[\rho 
			      \set{b/y_u^v}]}\\
			&= \qquad \set{b \in \A \mid \mathfrak{A} \vDash \exists {\widetilde{y}}\gamma_1 [\rho \set{b/y_u^v}]
					\text{~or~} \mathfrak{A} \vDash \exists {\widetilde{y}}\gamma_2[\rho \set{b/y_u^v}]}\\
			& \longequal{\textit{by}~(\ref{formula:fol})} \set{b \in \A \mid \mathfrak{A} \vDash \exists    
			                                                        {\widetilde{y}}\gamma_2[\rho \set{b/y_u^v}]}\\ 
			&= \qquad \set{b \in \A \bigg| \mathfrak{A} \vDash \exists {\widetilde{y}}(Cf(y_u^1, \cdots, y_u^n, x_u)  
			       \wedge \bigwedge_{1 \leq k \leq n}P(y_u^k))[\rho \set{b/y_u^v}]}\\
			&\overset{*} \subseteq \qquad \set{b \in \A \mid \mathfrak{A} \vDash Att(y_u^v, x_u)[\rho \set{b/y_u^v}]}\\
			&= \qquad \set{b \in \A \mid b \rightarrow \rho(x_u) \text{~in~} F}.
		\end{split}
	\end{eqnarray*}
	\\
	Here, the inequality ($*$) comes from the definition of $Cf(y_u^1, \cdots, y_u^n, x_u)$. Now we return to (I) itself. Let $F = \tuple{\A, \rightarrow}$ be a finitary AAF and $\rho$ an assignment.
	 By Claim 1, we have
	\begin{eqnarray*}
		\begin{split}
			&\bigcup_{E \subseteq \A} \set{b \in \A \mid \rho \in \textit{Nua}(F, E, \exists {y_u^v} \exists {\widetilde{y}} \alpha) \text{~and~} \tuple{\A, \rightarrow, E} \vDash \exists {\widetilde{y}}\alpha[\rho \set{b/y_u^v}]} \\
			\subseteq &~ \set{b \in \A \mid b \rightarrow \rho(x_u) \text{~in~} F}.
		\end{split}
	\end{eqnarray*}	
	Further, since $F$ is finitary, $|\set{b \in \A \mid b \rightarrow \rho(x_u) \text{~in~} F}| < \omega$. Hence
	\begin{eqnarray*}
		|\bigcup_{E \subseteq \A} \set{b \in \A \mid \rho \in \textit{Nua}(F, E, \exists {y_u^v} \exists {\widetilde{y}} \alpha) \text{~and~} \tuple{\A, \rightarrow, E} \vDash \exists {\widetilde{y}}\alpha[\rho \set{b/y_u^v}]}| < \omega.
	\end{eqnarray*}
	Thus, $\exists {y_u^v} \exists {\widetilde{y}} \alpha$ (i.e., $\beta_u^v$) is $\Omega$-finitary, as desired.
	\par (II) This part aims to give $\textit{FS}(\Omega)$-sentences to capture pre-fixed points of the function $D_n^m$ (that is, sets $X$ such that $D_n^m(X) \subseteq X$).
	For any positive number $m$ and $n$, put
	\begin{eqnarray*}
		\begin{split}
			&\alpha_{mn}^2 \triangleq \forall x(\neg \exists {x_1 \cdots x_m}\gamma_3 \rightarrow P(x)) \text{, where}\\
			&\gamma_3 \triangleq Cf(x_1, \cdots, x_m, x) \wedge \! \bigwedge_{1 \leq i \leq m} \! (\neg \exists {y_i^1 \cdots y_i^n}(Cf(y_i^1, \cdots, y_i^n, x_i) \wedge \! \bigwedge_{1 \leq k \leq n} \!\! P(y_i^k))).
		\end{split}
	\end{eqnarray*}
	Clearly, for any AAF $F = \tuple{\A, \rightarrow}$ and $E \subseteq \A$, we have
	\begin{eqnarray*}
		\begin{split}
			&\tuple{\A, \rightarrow, E} \vDash \alpha_{mn}^2\\
			\Leftrightarrow &~ \text{For any~} a \in \A, \tuple{\A, \rightarrow, E} \vDash \neg \exists {x_1 \cdots x_m}\gamma_3(x)[a] \text{~implies~} \tuple{\A, \rightarrow, E} \vDash P(x)[a]. \\
			\Leftrightarrow &~ \text{For any~} a \in \A, \nexistsn{b}{m}(~b \rightarrow a \text{ and } \nexistsn{c}{n}(c \rightarrow b \text{ and } c \in E)) \text{~implies~} a \in E.\\
			\Leftrightarrow &~ \text{For any~} a \in \A, a \in D_n^m(E) \text{~implies~} a \in E.\\
			\Leftrightarrow &~ D_n^m(E) \subseteq E.
		\end{split}
	\end{eqnarray*}
	Moreover, it is easy to see that $\alpha_{mn}^2$ is equivalent to 
	\begin{eqnarray*}
		\begin{split}
			&\beta_{mn}^2 \triangleq \forall x \exists {x_1 \cdots x_m} \forall {y_1^1 \cdots y_1^n \cdots y_m^1 \cdots y_m^n}\gamma_4 \text{~with~}\\
			&\gamma_4 \triangleq P(x) \vee (Cf(x_1, \cdots, x_m, x) \wedge \! \bigwedge_{1 \leq i \leq m} \! (\neg Cf(y_i^1, \cdots, y_i^n, x_i) \vee \! \bigvee_{1 \leq k \leq n}\!\!\!\! \neg P(y_i^k))).
		\end{split}
	\end{eqnarray*}
	Next we intend to prove $\beta_{mn}^2 \in \textit{FS}(\Omega)$. 
	By Definition \ref{definition:FS_formulas}, $\gamma_4$, $\forall {y_1^1 \cdots y_m^n}\gamma_4 \in \textit{FS}(\Omega)$, moreover, $\beta_{mn}^2 \in \textit{FS}(\Omega)$ whenever $\exists {x_1 \cdots x_m} \forall {y_1^1 \cdots y_m^n}\gamma_4 \in \textit{FS}(\Omega)$. To complete the proof, it suffices to demonstrate that, for each $u$ with $1 \leq u \leq m$, the formula
	$\theta_u \triangleq \exists {x_u \cdots x_m} \forall {y_1^1 \cdots y_m^n}\gamma_4$ is $\Omega$-finitary.
	Assume that $F = \tuple{\A, \rightarrow}$ is an AAF and $E \subseteq \A$.
	Let $\mathfrak{A} = \tuple{\A, \rightarrow, E}$ and $\rho \in \textit{Nua}(\mathfrak{A}, \theta_u)$. Then, we have $\mathfrak{A} \vDash \neg P(x)[\rho]$, otherwise, since $x$ and $x_u$ are distinct variables, it immediately follows from $\mathfrak{A} \vDash P(x)[\rho]$ that
	\begin{eqnarray*}
		\mathfrak{A} \vDash \forall x_u \exists {x_{u + 1} \cdots x_m} \forall {y_1^1 \cdots y_m^n}\gamma_4[\rho],
	\end{eqnarray*}
	which contradicts that $\rho \in \textit{Nua}(\mathfrak{A}, \theta_u)$.
	Thus, $\mathfrak{A} \vDash \neg P(x)[\rho \set{b/x_u}]$ for each $b \in \A$.
	Then
	\begin{eqnarray*}
		\begin{split}
			&~ \set{b \in \A \mid \mathfrak{A} \vDash \exists {x_{u + 1} \cdots x_m} \forall {y_1^1 \cdots y_m^n}\gamma_4[\rho \set{b/x_u}]} \\
			= &~ \set{b \in \A \bigg|
		 		\begin{array}{*{20}{c}}
					{ \mathfrak{A} \vDash \exists {x_{u + 1} \cdots x_m} \forall {y_1^1 \cdots y_m^n}(Cf(x_1, \cdots, x_m, x) \wedge} \\ 
					{ \bigwedge_{1 \leq i \leq m} (\neg Cf(y_i^1, \cdots, y_i^n, x_i) \vee \bigvee_{1 \leq k \leq n} \neg P(y_i^k)))[\rho \set{b/x_u}]}
				\end{array}
				}\\
			\subseteq &~ \set{b \in \A \mid \mathfrak{A} \vDash Att(x_u, x)[\rho \set{b/x_u}]}\\
			= &~ \set{b \in \A \mid b \rightarrow \rho(x) \text{~in~} F}.
		\end{split}
	\end{eqnarray*}
	Further, we may conclude that $\theta_u$ is $\Omega$-finitary by the same reasoning applied in (I).
	\par (III) Consider the formula $\alpha_{\ell}^3$ defined as
	\begin{eqnarray*}
		\alpha_{\ell}^3 \triangleq \forall x(P(x) \rightarrow \neg \exists {x_1 \cdots x_\ell}(\bigwedge_{1 \leq i \leq \ell}P(x_i) \wedge Cf(x_1, \cdots, x_\ell, x))).
	\end{eqnarray*}
	It is easy to see that $\varepsilon_{cf}^\ell$ is defined by $\alpha_{\ell}^3$, moreover, $\alpha_{\ell}^3$ is equivalent to
	\begin{eqnarray*}
		\beta_{\ell}^3 \triangleq \forall {x~x_1 \cdots x_\ell}(\neg P(x) \vee \neg(\bigwedge_{1 \leq i \leq \ell}P(x_i) \wedge Cf(x_1, \cdots, x_\ell, x))).
	\end{eqnarray*}
	Since $\neg P(x) \vee \neg(\bigwedge_{1 \leq i \leq \ell}P(x_i) \wedge Cf(x_1, \cdots, x_\ell, x))$ is in $\textit{FS}(\Omega)$, by Definition \ref{definition:FS_formulas}, so is $\beta_{\ell}^3$.
	Thus, $\varepsilon_{cf}^\ell$ is $\textit{FS}(\Omega)$-definable.
	\par (IV) This part intends to provide $\textit{FS}(\Omega)$-sentences to define the semantics $\varepsilon^\ell$, where
	\begin{eqnarray*}
		\varepsilon^\ell(F) \triangleq \set{E \subseteq \A \mid N_\ell(E) \subseteq E} \text{~for each AAF~} F = \tuple{\A, \rightarrow}.
	\end{eqnarray*}
	For this purpose, for any positive number $\ell$, set
	\begin{eqnarray*}
		\alpha_{\ell}^4 \triangleq \forall x(\neg \exists {x_1 \cdots x_\ell}(\bigwedge_{1 \leq i \leq \ell}P(x_i) \wedge Cf(x_1, \cdots, x_\ell, x)) \rightarrow P(x)).
	\end{eqnarray*}
	Suppose that $F = \tuple{\A, \rightarrow}$ is an AAF and $E \subseteq \A$.
	Let $\mathfrak{A} = \tuple{\A, \rightarrow, E}$.
	\begin{eqnarray*}
		\begin{split}
			&~ \mathfrak{A} \vDash \alpha_{\ell}^4\\
			\Leftrightarrow &~ \text{For each~} a \in \A, \mathfrak{A} \vDash \neg \exists {x_1 \cdots x_\ell}(\bigwedge_{1 \leq i \leq \ell}P(x_i) \wedge Cf(x_1, \cdots, x_\ell, x))[a] \\
			&~ \text{~implies~} \mathfrak{A} \vDash P(x)[a]. \\
			\Leftrightarrow &~ \text{For each~} a \in \A, a \in N_\ell(E) \text{~implies~} a \in E. \\
			\Leftrightarrow &~ N_\ell(E) \subseteq E\\
			\Leftrightarrow &~ E \in \varepsilon^\ell(F).
		\end{split}
	\end{eqnarray*}
	Thus, $\varepsilon^\ell$ may be defined by $\alpha_{\ell}^4$.
	Moreover, it is obvious that $\alpha_{\ell}^4$ is equivalent to
	\begin{eqnarray*}
		\beta_{\ell}^4 \triangleq \forall {x} \exists{x_1 \cdots x_\ell}((\bigwedge_{1 \leq i \leq \ell}P(x_i) \wedge Cf(x_1, \cdots, x_\ell, x)) \vee P(x)).
	\end{eqnarray*}
	In order to show $\beta_{\ell}^4 \in \textit{FS}(\Omega)$, it suffices to prove that $\beta_u \in \textit{FS}(\Omega)$ for each $u$ with $1 \leq u \leq \ell$, where
	\begin{eqnarray*}
		\beta_u \triangleq \exists {x_u, x_{u + 1} \cdots x_\ell} ((\bigwedge_{1 \leq i \leq \ell}P(x_i) \wedge Cf(x_1, \cdots, x_\ell, x)) \vee P(x)).
	\end{eqnarray*}
	This may be demonstrated similarly to (I) based on the fact that $\mathfrak{A} \vDash \neg P(x)[\rho]$ for any model $\mathfrak{A}$ based on AAFs and $\rho \in \textit{Nua}(\mathfrak{A}, \beta_u)$.
	The detailed verification is left to the reader. 
	\par (V) By (I), (II), (III) and (IV), it is not difficult to see that $\textit{Def}^{mn}$, $\varepsilon_{cf}^{\ell}$, $\varepsilon_{ad}^{\ell mn}$, $\varepsilon_{co}^{\ell mn}$ and $\varepsilon_{stb}^{\ell mn}$ can be $\textit{FS}(\Omega)$-defined by $\Sigma_{mn\textit{Def}} \triangleq \set{\beta_{mn}^1}$, $\Sigma_{cf}^\ell \triangleq \set{\beta_\ell^3}$, $\Sigma_{ad}^{\ell mn} \triangleq \set{\beta_\ell^3, \beta_{mn}^1}$, $\Sigma_{co}^{\ell mn} \triangleq \set{\beta_\ell^3, \beta_{mn}^1, \beta_{mn}^2}$, and $\Sigma_{stb}^{\ell mn} \triangleq \set{\beta_n^3, \beta_m^3, \beta_n^4, \beta_m^4, \beta_\ell^3}$, respectively.
\end{proof}
The result above indicates that the segment $\textit{FS}(\Omega)$ of $\Lan'$ associated with finitary AAFs is nontrivial, which is enough to formalize some important semantics of AAF, including $\varepsilon_{cf}^{\ell}$, $\varepsilon_{ad}^{\ell mn}$, $\varepsilon_{co}^{\ell mn}$ and $\varepsilon_{stb}^{\ell mn}$. 
\subsection{Ultraproduct of models and reduced meet modulo an ultrafilter}
This subsection will demonstrate that any extension-based semantics defined by $\textit{FS}(\Omega)$-sentences must be closed under reduced meets modulo any ultrafilter. In the following, some concepts and results in model theory will be referred to, which are recalled below. Here, these concepts and results are given in a limited form in which neither function nor constant symbol is involved. For a fuller treatment, we refer the reader to \cite{Chang1990Book}.
\par Given nonempty sets $I$ and $\A$, the set of all functions from $I$ to $\A$ is denoted by $A^I$. For any ultrafilter $D$ over $I$, the binary relation $\equiv_D$ on $A^I$ is defined as, for any $f, g \in A^I$,
\begin{eqnarray*}
	f \equiv_D g \text{~iff~} \set{i \in I \mid f(i) = g(i)} \in D.
\end{eqnarray*}
It is well-known that the relation $\equiv_D$ is an equivalent relation on $A^I$ \cite{Chang1990Book}. In the following, for any $f \in A^I$, the equivalence class $\set{g \mid g \in A^I \text{~and~} f \equiv_D g}$ is denoted by $f_D$. Following the standard definition (see \cite{Chang1990Book}), an ultraproduct of a family of $F$-based models is defined below.
\begin{definition}[Ultraproduct of $F$-based Models] \label{definition:ultraproduct}
	Given an AAF $F = \tuple{\A, \rightarrow}$, let $\mathfrak{A}_i = \tuple{\A, \rightarrow, E_i} (i \in I \neq \emptyset)$ be $F$-based models of the language $\Lan'$ and $D$ an ultrafilter over $I$.
	The ultraproduct of these $\mathfrak{A}_i$'s modulo $D$, denoted by $\prod_D \mathfrak{A}_i$, is defined as $\prod_D \mathfrak{A}_i \triangleq \tuple{\prod_D \A, \rightarrow_D, E_D}$, where
	\begin{itemize}
		\item[$D_1$.] $\prod_D \A \triangleq \set{f_D \mid f \in \A^I}$.
		\item[$D_2$.] For any $f_D$, $g_D \in \prod_D \A$, $f_D \rightarrow_D g_D$ iff $\set{i \in I \mid f(i) \to g(i) \text{~in~} F} \in D$.
		\item[$D_3$.] For any $f_D \in \prod_D \A$, $f_D \in E_D$ iff $\set{i \in I \mid f(i) \in E_i} \in D$.
	\end{itemize}
\end{definition}
As usual, the model $\prod_D \mathfrak{A}_i$ is well-defined. In other words, the definitions of relations $\rightarrow_D$ and $E_D$ don't depend on the representatives, see, Proposition 4.1.7 in \cite{Chang1990Book}. The next well-known result is the fundamental theorem of ultraproducts (see, e.g., \cite{Chang1990Book}). Here we only list the clauses that we need.
\begin{theorem}[The Fundamental Theorem of Ultraproducts]\label{theorem:ft_ultraproducts}
	Let $\prod_D \mathfrak{A}_i$ be the ultraproduct of $\mathfrak{A}_i$'s modulo an ultrafilter $D$ defined in Definition \ref{definition:ultraproduct}. Then
	\begin{itemize}
		\item[a.] For any formula $\phi(x_1 \dots x_n)$ of $\Lan'$ and elements $f_D^1, \dots, f_D^n \in \prod_D \A$, we have\\
		$\prod_D \mathfrak{A}_i \vDash \phi[f_D^1 \dots f_D^n] \text{~iff~} \set{i \in I \mid \mathfrak{A}_i \vDash \phi[f^1(i) \dots f^n(i)]} \in D$.
		\item[b.] For any sentence $\phi$ of $\Lan'$, $\prod_D \mathfrak{A}_i \vDash \phi \text{~iff~} \set{i \in I \mid \mathfrak{A}_i \vDash \phi} \in D$.
	\end{itemize}
\end{theorem}
The next lemma is a crucial step in proving the main result of this section, which provides a connection between the ultraproduct of a family of models $\mathfrak{A}_i = \tuple{\A, \rightarrow, E_i}$ $(i \in I)$ modulo an ultrafilter $D$ and the model $\mathfrak{A} \triangleq \tuple{\A, \rightarrow, \bigcap_D E_i}$, where $\bigcap_D E_i$ is the reduced meets of $E_i$'s modulo $D$, see Definition \ref{definition:reduced meet}.
\begin{lemma} \label{lemma:RM_ultraproduct models}
	Let $\Omega$ be a class of AAFs and $F = \tuple{\A, \rightarrow} \in \Omega$, and let $\mathfrak{A}_i = \tuple{\A, \rightarrow, E_i}$ with $E_i \subseteq \A (i \in I \neq \emptyset)$ be $F$-based models and $D$ an ultrafilter over $I$. Then, for any formula $\phi(x_1 \dots x_n) \in \textit{FS}(\Omega)$, $a_1, \dots, a_n \in \A$ and assignment $\rho$ such that $\rho(x_i) = a_i$ for $1 \leq i \leq n$,
	\begin{eqnarray*}
		\prod_D \mathfrak{A}_i \vDash \phi[d_D^{a_1} \cdots d_D^{a_n}] \text{~implies~} \mathfrak{A} \vDash \phi[\rho].
	\end{eqnarray*}
	In particular, for any sentence $\alpha \in \textit{FS}(\Omega)$,
	\begin{eqnarray*}
		\prod_D \mathfrak{A}_i \vDash \alpha \text{~implies~} \mathfrak{A} \vDash \alpha.
	\end{eqnarray*}
	Here $\mathfrak{A} \triangleq \tuple{\A, \rightarrow, \bigcap_D E_i}$, and, for any $a \in \A$, $d^a \triangleq \lambda_{i \in I}.a$, the constant function mapping each $i \in I$ to $a$.
\end{lemma}
\begin{proof}
	Since each formula in $\textit{FS}(\Omega)$ is in prenex form, it suffices to show the following two claims.\\
	\\ 
	\textbf{Claim 1~~} For any formula $\phi(x_1 \dots x_n)$ of $\Lan'$ containing no quantifier,
	\begin{eqnarray*}
		\prod_D \mathfrak{A}_i \vDash \phi[d_D^{a_1} \cdots d_D^{a_n}] \text{~iff~} \mathfrak{A} \vDash \phi[\rho].
	\end{eqnarray*}
	We proceed by induction on the formulas.
	\par (a) $\phi = x_{i_1} \equiv x_{i_2}$ with $1 \leq i_1, i_2 \leq n$.
	\begin{eqnarray*}
	\mathfrak{A} \vDash x_{i_1} \equiv x_{i_2}[\rho] &\Leftrightarrow& a_{i_1} = a_{i_2}\\
	&\Leftrightarrow& d^{a_{i_1}} = d^{a_{i_2}}\\
	&\overset{*_1} \Longleftrightarrow & d_D^{a_{i_1}} = d_D^{a_{i_2}}\\
	&\Leftrightarrow& \prod_D \mathfrak{A}_i \vDash x_{i_1} \equiv x_{i_2}[d_D^{a_1} \cdots d_D^{a_n}].
	\end{eqnarray*}
	$*_1$: Clearly, the set $J \triangleq \set{i \in I \mid d^{a_{i_1}}(i) = d^{a_{i_2}}(i)}$ is either $I$ or $\emptyset$, which depends on whether it holds that $a_{i_1} = a_{i_2}$.
	If $d_D^{a_{i_1}} = d_D^{a_{i_2}}$ then $d^{a_{i_1}} \equiv_D d^{a_{i_2}}$, i.e., $J \in D$, and hence $J = I$, that is, $d^{a_{i_1}} = d^{a_{i_2}}$. Another direction is trivial.
	\par (b) $\phi = Att(x_{i_1}, x_{i_2})$ with $1 \leq i_1, i_2 \leq n$.
	\begin{eqnarray*}
	\mathfrak{A} \vDash Att(x_{i_1}, x_{i_2})[\rho] &\Leftrightarrow& a_{i_1} \rightarrow a_{i_2} \text{~in~} F\\
	&\Leftrightarrow& \set{i \in I \mid d^{a_{i_1}}(i) \rightarrow d^{a_{i_2}}(i) \text{~in~} F} = I\\
	&\overset{*_2} \Longleftrightarrow & \set{i \in I \mid d^{a_{i_1}}(i) \rightarrow d^{a_{i_2}}(i) \text{~in~} F} \in D\\
	&\Leftrightarrow& d_D^{a_{i_1}} \rightarrow_D d_D^{a_{i_2}}\\
	&\Leftrightarrow& \prod_D \mathfrak{A}_i \vDash Att(x_{i_1}, x_{i_2})[d_D^{a_1} \cdots d_D^{a_n}].
	\end{eqnarray*}
	Similar to $*_1$, the step $*_2$ holds because $\set{i \in I \mid d^{a_{i_1}}(i) \rightarrow d^{a_{i_2}}(i) \text{~in~} F}$ is either $I$ or $\emptyset$.
	\par (c) $\phi = P(x_{j})$ with $1 \leq j \leq n$.
	\begin{eqnarray*}
	\mathfrak{A} \vDash P(x_{j})[\rho] &\Leftrightarrow& a_{j} \in \bigcap_D E_i\\
	&\Leftrightarrow& \set{i \in I \mid a_{j} \in E_i} \in D\\
	&\Leftrightarrow& \set{i \in I \mid d^{a_{j}}(i) \in E_i} \in D\\
	&\Leftrightarrow& d_D^{a_{j}} \in E_D\\
	&\Leftrightarrow& \prod_D \mathfrak{A}_i \vDash P(x_{j})[d_D^{a_1} \cdots d_D^{a_n}].
	\end{eqnarray*}
	\par (d) $\phi$ is of the form $\neg \alpha$, $\alpha_1 \vee \alpha_2$ or $\alpha_1 \wedge \alpha_2$.\\
	In this situation, the proof is straightforward by using IH.\\
	\\ \textbf{Claim 2~~} For any prenex formula $\phi = Q_1 x_{i_1} \cdots Q_m x_{i_m} \beta(x_1 \dots x_n) \in \textit{FS}(\Omega)$ with $1 \leq i_1, \cdots, i_m \leq n$, 
	\begin{eqnarray*}
		\prod_D \mathfrak{A}_i \vDash \phi[d_D^{a_1} \cdots d_D^{a_n}] \text{~implies~} \mathfrak{A} \vDash \phi[\rho],
	\end{eqnarray*}
	where $Q_i (1 \leq i \leq m)$ is either $\forall$ or $\exists$, and $\beta(x_1 \dots x_n)$ is the matrix of $\phi$.\\
	\\
	We proceed by induction on $m$. If $m = 0$ then $\phi$ contains no quantifier, and hence the conclusion holds due to Claim 1.
	Next we deal with the inductive step by distinguishing two cases based on $Q_1$. Put $\alpha(x_1 \dots x_n) \triangleq Q_2 x_{i_2} \cdots Q_m x_{i_m} \beta(x_1 \dots x_n)$. It is easy to see that $\alpha(x_1 \dots x_n) \in \textit{FS}(\Omega)$.\\
	\\Case 1 : $Q_1$ is $\forall$.\\
	Then $\phi = \forall {x_j} \alpha(x_1 \dots x_n)$ for some $1 \leq j \leq n$. Let $\prod_D \mathfrak{A}_i \vDash \phi[d_D^{a_1} \cdots d_D^{a_n}]$.  Hence, for each $a \in \A$, we have
	\begin{eqnarray*}
		\prod_D \mathfrak{A}_i \vDash \alpha[d_D^{a_1} \cdots d_D^{a_{j - 1}}~d_D^{a}~d_D^{a_{j + 1}} \cdots d_D^{a_n}].
	\end{eqnarray*}
	Thus, by IH, $\mathfrak{A} \vDash \alpha[\rho \set{a/x_j}]$ for each $a \in \A$.
	Consequently, we get $\mathfrak{A} \vDash \forall {x_j} \alpha[\rho]$.\\
	\\Case 2 : $Q_1$ is $\exists$.\\
	Then $\phi = \exists {x_j} \alpha(x_1 \dots x_n)$ for some $1 \leq j \leq n$. Assume $\prod_D \mathfrak{A}_i \vDash \phi[d_D^{a_1} \cdots d_D^{a_n}]$. Put
	\begin{eqnarray*}
		I_1 \triangleq \set{i \in I \mid \rho \in \textit{Nua}(\mathfrak{A}_i, \phi)} \text{~and~} I_2 \triangleq \set{i \in I \mid \rho \notin \textit{Nua}(\mathfrak{A}_i, \phi)}.
	\end{eqnarray*}
	Thus, $I_1 \cup I_2 = I$, and hence, either $I_1 \in D$ or $I_2 \in D$. In the following, we consider two cases.\\
	\\Case 2.1 : $I_1 \in D$.\\
	Since $\prod_D \mathfrak{A}_i \vDash \phi[d_D^{a_1} \cdots d_D^{a_n}]$, there exists $f_D \in \prod_D \A$ such that
	\begin{eqnarray*}
		\prod_D \mathfrak{A}_i \vDash \alpha[d_D^{a_1} \cdots d_D^{a_{j - 1}}~f_D~d_D^{a_{j + 1}} \cdots d_D^{a_n}].
	\end{eqnarray*}
	By Theorem \ref{theorem:ft_ultraproducts}, we get
	\begin{eqnarray*}
			\set{i \in I \mid \mathfrak{A}_i \vDash \alpha[d^{a_1}(i) \cdots d^{a_{j - 1}}(i)~f(i)~d^{a_{j + 1}}(i) \cdots d^{a_n}(i)]} \in D.
	\end{eqnarray*}
	Since $\rho(x_k) = a_k$ and $d^{a_k}(i) = a_k$ for each $i \in I$ and $1 \leq k \leq n$, by the assertion above, we have
	\begin{eqnarray}
		I_3 \triangleq \set{i \in I \mid \mathfrak{A}_i \vDash \alpha[\rho \set{f(i)/x_j}]} \in D. \label{formulas: I3}
	\end{eqnarray}
	Then $I_4 \triangleq I_1 \cap I_3 \in D$.
	Put 
	\begin{eqnarray*}
		S \triangleq \bigcup_{i \in I_4} \set{b \in \A \mid \mathfrak{A}_i \vDash \alpha[\rho \set{b/x_j}]}.
	\end{eqnarray*}
	Due to $\phi \in \textit{FS}(\Omega)$, by Definition \ref{definition:FS_formulas}, the formula $\exists {x_j} \alpha$ is $\Omega$-finitary.
	Moreover, since $I_4 \subseteq I_1$, we have $\rho \in \textit{Nua}(\mathfrak{A}_i, \exists {x_j} \alpha)$ for each $i \in I_4$. Thus
	\begin{eqnarray*}
		S &=& \bigcup_{i \in I_4} \set{b \in \A \mid \tuple{\A, \rightarrow, E_i} \vDash \alpha[\rho \set{b/x_j}] \text{~and~} \rho \in \textit{Nua}(\mathfrak{A}_i, \exists {x_j} \alpha)} \\
		&\subseteq& \bigcup_{E \subseteq \A} \set{b \in \A \mid \tuple{\A, \rightarrow, E} \vDash \alpha[\rho \set{b/x_j}] \text{~and~} \rho \in \textit{Nua}(F, E, \exists {x_j} \alpha)}.
	\end{eqnarray*}
	Further, since $\exists {x_j} \alpha$ is $\Omega$-finitary, we have
	\begin{eqnarray*}
		|S| \leq |\bigcup_{E \subseteq \A} \set{b \in \A \mid \tuple{\A, \rightarrow, E} \vDash \alpha[\rho \set{b/x_j}] \text{~and~} \rho \in \textit{Nua}(F, E, \exists {x_j} \alpha)}| < \omega.
	\end{eqnarray*}
	For each $b \in S$, set
	\begin{eqnarray*}
		I_b \triangleq \set{i \in I_4 \mid \mathfrak{A}_i \vDash \alpha[\rho \set{b/x_j}]}.
	\end{eqnarray*}
	Clearly, for each $i \in I_4 \subseteq I_3$, $f(i) \in S$ and $i \in I_{f(i)}$ due to (\ref{formulas: I3}).
	Thus $I_4 \subseteq \bigcup_{b \in S} I_b$, and hence $\bigcup_{b \in S} I_b \in D$.
	Further, due to $|S| < \omega$, we have $I_{b_0} \in D$ for some $b_0 \in S$.
	So, it follows from $I_{b_0} \subseteq I_5 \triangleq \set{i \in I \mid \mathfrak{A}_i \vDash \alpha[\rho \set{b_0/x_j}]}$
	that $I_5 \in D$.
	Then, by Theorem \ref{theorem:ft_ultraproducts}, we obtain
	$\prod_D \mathfrak{A}_i \vDash \alpha[d_D^{a_1} \cdots d_D^{a_{j - 1}}~d_D^{b_0}~d_D^{a_{j + 1}} \cdots d_D^{a_n}]$.
	Thus, by IH, it follows that $\mathfrak{A} \vDash \alpha[\rho \set{b_0/x_j}]$. 
	Hence, $\mathfrak{A} \vDash \exists {x_j} \alpha[\rho]$.\\
	\\Case 2.2 : $I_2 \in D$.\\
	Since $\prod_D \mathfrak{A}_i \vDash \exists {x_j} \alpha[d_D^{a_1} \cdots d_D^{a_n}]$, by Theorem \ref{theorem:ft_ultraproducts} and $\rho(x_i) = a_i$ for $1 \leq i \leq n$, $I_6 \triangleq \set{i \in I \mid \mathfrak{A}_i \vDash \exists {x_j} \alpha[\rho]} \in D$.
	Then, due to $I_2 \in D$, we get
	\begin{eqnarray*}
		I_2 \cap I_6 = \set{i \in I \mid \mathfrak{A}_i \vDash \exists {x_j} \alpha[\rho] \text{~and~} \rho \notin \textit{Nua}(\mathfrak{A}_i, \exists {x_j} \alpha)} \in D.
	\end{eqnarray*}
	Further, by Definition \ref{definition:universal assignment}, it follows that $\mathfrak{A}_i \vDash \forall x_j \alpha[\rho]$ for each $i \in I_2 \cap I_6$.
	Thus, $I_2 \cap I_6 \subseteq \set{i \in I \mid \mathfrak{A}_i \vDash \forall {x_j} \alpha[\rho]}$, and hence $\set{i \in I \mid \mathfrak{A}_i \vDash \forall {x_j} \alpha[\rho]} \in D$.
	Then, since $\rho(x_k) = a_k = d^{a_k}(i)$ for $1 \leq k \leq n$ and $i \in I$, we have
	\begin{eqnarray*}
		\set{i \in I \mid \mathfrak{A}_i \vDash \forall {x_j} \alpha[d^{a_1}(i) \cdots d^{a_n}(i)]} \in D.
	\end{eqnarray*}
	Thus, by Theorem \ref{theorem:ft_ultraproducts}, 
	\begin{eqnarray*}
		\prod_D \mathfrak{A}_i \vDash \forall {x_j} \alpha[d_D^{a_1} \cdots d_D^{a_n}].
	\end{eqnarray*}
	Moreover, $\forall {x_j} \alpha(x_1 \cdots x_n) \in \textit{FS}(\Omega)$ due to $\alpha(x_1 \cdots x_n) \in \textit{FS}(\Omega)$. Then, by Case 1, we get $\mathfrak{A} \vDash \forall {x_j} \alpha[\rho]$.
	Thus, $\mathfrak{A} \vDash \exists {x_j} \alpha[\rho]$, as desired.	
\end{proof}
By Theorem \ref{theorem:ft_ultraproducts} and Lemma \ref{lemma:RM_ultraproduct models}, the following result follows immediately, which may be regarded as a generalization of Lemma \ref{lemma:RMMU_ultrafilter} in a model-theoretical style.
\begin{corollary} \label{corollary:RM_D_N}
	Let $\Omega$ be a class of AAFs and $F = \tuple{\A, \rightarrow} \in \Omega$, and let $\mathfrak{A}_i = \tuple{\A, \rightarrow, E_i} (i \in I \neq \emptyset)$ be $F$-based models and $D$ an ultrafilter over $I$. Then, for any sentence $\phi \in \textit{FS}(\Omega)$, $\set{i \in I \mid \mathfrak{A}_i \vDash \phi} \in D$ implies $\tuple{\A, \rightarrow, \bigcap_D E_i} \vDash \phi$.
\end{corollary}
We are now in a position to give the main result in this section.
\begin{theorem} \label{theorem:RM_FoL}
	Let $\Omega$ be a class of AAFs and $\varepsilon$ an extension-based semantics. If $\varepsilon$ is $\textit{FS}(\Omega)$-definable w.r.t. $\Omega$, then $\varepsilon$ is closed under reduced meets modulo any ultrafilter w.r.t. $\Omega$.
\end{theorem}
\begin{proof}
	Let $F = \tuple{\A, \rightarrow} \in \Omega$ and $E_i \in \varepsilon(F) (i \in I \neq \emptyset)$, and let $D$ be an ultrafilter over $I$. 
	Assume that $\varepsilon$ is defined by a set $\Sigma_\varepsilon$ of sentences in $\textit{FS}(\Omega)$.
	Put 
	\begin{eqnarray*}
		\mathfrak{A}_i \triangleq \tuple{\A, \rightarrow, E_i} \text{~for~} i \in I.
	\end{eqnarray*}
	Then, for each $i \in I$, due to $E_i \in \varepsilon(F)$, we have $\mathfrak{A}_i \vDash \Sigma_\varepsilon$.
	Thus, by Theorem \ref{theorem:ft_ultraproducts}, we get $\prod_D \mathfrak{A}_i \vDash \Sigma_\varepsilon$.
	Since $\Sigma_\varepsilon \subseteq \textit{FS}(\Omega)$, by Lemma \ref{lemma:RM_ultraproduct models}, we have $\tuple{\A, \rightarrow, \bigcap_D E_i} \vDash \Sigma_\varepsilon$.
	Hence, $\bigcap_D E_i \in \varepsilon(F)$ due to Definition \ref{definition:fol_ext}.
\end{proof}
Clearly, Proposition \ref{proposition:FoL_fundamental semantics} and Theorem \ref{theorem:RM_FoL} provide an alternative proof of Theorem \ref{theorem:close under meet_fundamental semantics} except the trivial case that $\varepsilon = \varepsilon_{gr}^{\ell mn}$. It should be pointed out that there exist extension-based semantics that are closed under the operator $\bigcap_D$ but aren't $\textit{FS}(\Omega)$-definable. For example, the semantics $\varepsilon_{gr}^{\ell mn}$ is closed under this operator trivially (Theorem \ref{theorem:close under meet_fundamental semantics}), however, it seems to be not $\textit{FS}(\Omega)$-definable because it is defined in terms of the minimum, whose formalization involves applying quantifiers on extensions, however, the language $\textit{FS}(\Omega)$ is first-order.
\begin{corollary} \label{corollary:epsilon_extensible2}
	Let $\Omega$ be a class of AAFs and $\varepsilon_\sigma$ an extension-based semantics. If $\varepsilon_\sigma$ is $\textit{FS}(\Omega)$-definable w.r.t. $\Omega$, then, for each $F = \tuple{\A, \rightarrow} \in \Omega$, $a \in \A$ and $X \subseteq \A$,
	\begin{itemize}
		\item[a.] $X$ is $\varepsilon_\sigma$-extensible iff $X$ is finitely $\varepsilon_\sigma$-extensible.
		\item[b.] $X \nc{\varepsilon_\sigma}{} a$ iff $X_0 \nc{\varepsilon_\sigma}{} a$ for some finite subset $X_0$ of $X$.
		\item[c.] There exist maximal extensions in $\varepsilon_\sigma(F)$ whenever $\varepsilon_\sigma(F) \neq \emptyset$. Hence, the semantics $\varepsilon_\sigma$ agrees with $\varepsilon_{\textit{max}, \sigma}$ on the universal definability w.r.t. $\Omega$.
		\item[d.] Each element in $\varepsilon_\sigma(F)$ can be extended to a maximal one.
		\item[e.] $\tuple{\varepsilon_\sigma(F), \subseteq}$ is a dcpo whenever $\varepsilon_\sigma(F) \neq 0$.
	\end{itemize}
\end{corollary}
\begin{proof}
	Follows from Theorems \ref{theorem:epsilon_extensible}, \ref{theorem:epsilon_inference}, \ref{theorem:dcpo}, \ref{theorem:Lindenbaum_RM} and \ref{theorem:RM_FoL}.
\end{proof}
For the class $\Omega$ of all finitary AAFs and any extension-based semantics $\varepsilon_{\sigma}$ defined by $\textit{FS}(\Omega)$-sentences, it is straightforward to establish the universal definability of the range related semantics $\varepsilon_{rr\sigma}^\eta$ associated with $\varepsilon_{\sigma}$. Formally, we have the following metatheorem.
\begin{theorem} \label{theorem:FoL and universal definability}
	Let $\Omega$ be a class of finitary AAFs. For any extension-based semantics $\varepsilon_{\sigma}$, if $\varepsilon_{\sigma}$ is $\textit{FS}(\Omega)$-definable w.r.t. $\Omega$ then, for any $F \in \Omega$,
	\begin{eqnarray*}
		|\varepsilon_{\sigma}(F)| \geq 1 \text{~implies~} |\varepsilon_{rr \sigma}^\eta(F)| \geq 1,
	\end{eqnarray*}
	and hence $\varepsilon_{\sigma}$ is universally defined w.r.t. $\Omega$ iff so is $\varepsilon_{rr \sigma}^\eta$.
\end{theorem}
\begin{proof}
	Immediately follows from Theorems \ref{theorem:close under meet and universal definability} and \ref{theorem:RM_FoL}.
\end{proof}
Clearly, Corollary \ref{corollary:RM of interval semantics}, Proposition \ref{proposition:RMMU_max} and Theorems \ref{theorem:anti-epsilon set}, \ref{theorem:RMMU_interval}, \ref{theorem:structure of interval semantics} and \ref{theorem:structure_para} also have corresponding ones respectively in the situation that the related semantics is $\textit{FS}(\Omega)$-defined. The detail is left to the reader.
\section{Summarization, discussion and future work}\label{Sec:Conclusion}
\begin{figure}[t]
	\centering
	\includegraphics[width=1 \linewidth]{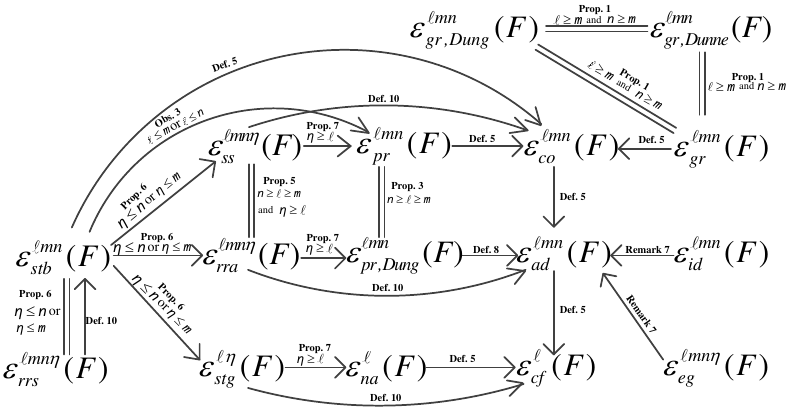}
	\caption{Relations between graded semantics, where two nodes are connected by $\rightarrow$ if the left one is a subset of the right one.}
	\label{figure:relations between semantics}
\end{figure}
This paper explores the extension-based concrete and abstract semantics of AAFs. The relationships among graded concrete semantics considered in \cite{Grossi19Graded} and this paper are summarized in Figure \ref{figure:relations between semantics}. Results obtained in this paper concern different aspects of various graded semantics, including universal definability, structural properties and explicit constructions of some special extensions, etc, which are summarized in Tables \ref{table:conclusion2}, \ref{table:conclusion3} and \ref{table:conclusion4}, respectively. Much of these results depend on the fundamental lemma obtained in this paper (Lemma \ref{lemma:fundamental lemma}), which provides a new sufficient condition for preserving $\ell$-conflict-freeness during iterations of $D_n^m$ starting at $\ell mn$-admissible sets and induces a Galois adjunction between admissible sets and complete extensions (Corollary \ref{corollary:properties_defense}). In addition, we present the operator reduced meet modulo an ultrafilter $\bigcap_D$, which is a useful tool in exploring theoretical issues of AAFs. We provide the distributive laws of the functions $N_\ell$ and $D_n^m$ over the operator $\bigcap_D$. Based on this, it has been shown that a variety of semantics are closed under the operator $\bigcap_D$ (Theorem \ref{theorem:close under meet_fundamental semantics}), moreover, the interval semantics and parameterized semantics inherit this closeness (Theorem \ref{theorem:RMMU_interval}), and the derived semantics $\varepsilon_{\textit{max}, \sigma}$ and $\varepsilon_{rr \sigma}^\eta$ agree with a given semantics $\varepsilon_{\sigma}$ on the universal definability w.r.t. a class $\Omega$ of finitary AAFs whenever $\varepsilon_{\sigma}$ is closed under $\bigcap_D$ w.r.t. $\Omega$ (Theorems \ref{theorem:close under meet and universal definability} and \ref{theorem:Lindenbaum_RM}). The properties of $\bigcap_D$ related to the semantics $\varepsilon_{na}^\ell$, $\varepsilon_{pr, \textit{Dung}}^{\ell mn}$ and $\varepsilon_{pr}^{\ell mn}$ are explored (Propositions \ref{proposition:RMMU_na}, \ref{proposition:RMMU_pr_Dung} and \ref{proposition:RMMU_pr}). Taking advantage of this operator, we provide a simple and uniform method to establish the universal definability of a family of range related semantics (Theorems \ref{theorem:existence_stg_ss_rra_rrs} and \ref{theorem:FoL and universal definability}), and obtain a number of metatheorems about $\varepsilon$-extensible (Theorem \ref{theorem:epsilon_extensible} and Corollaries \ref{corollary:RM of interval semantics} and \ref{corollary:epsilon_extensible2}), $\varepsilon$-inference (Theorem \ref{theorem:epsilon_inference} and Corollaries \ref{corollary:RM of interval semantics} and \ref{corollary:epsilon_extensible2}) and Lindenbaum property (Theorem \ref{theorem:Lindenbaum_RM} and Corollaries \ref{corollary:RM of interval semantics} and \ref{corollary:epsilon_extensible2}), etc. Moreover, making use of model-theoretical tools, we characterize a family of extension-based semantics that are closed under the operator $\bigcap_D$ in terms of $\textit{FS}(\Omega)$-definability (Theorem \ref{theorem:RM_FoL}). 
\par Some assumptions about AAFs and parameters are involved in this paper. In the following, we discuss them in turn.
\par 
We begin with summarizing the effect of two basic assumptions: one is $\ell \geq m$ and $n \geq m$, another is $n \geq \ell \geq m$. These two assumptions may ensure that $\ell$-conflict-freeness is preserved during iterations of the defense function $D_n^m$ starting at $\ell mn$-admissible sets. The important point to note here is that: \textit{the former works for iterations starting at any $E \in \varepsilon_{ad}^{\ell mn}(F)$ with $E \subseteq D_{\substack{m \\ n}}^{\lambda_F}(\emptyset)$ (in particular, $\emptyset$), while the latter is effective for ones starting at any $\ell mn$-admissible set.} By the way, these two assumptions have nothing to do with preserving $mn$-self-defense. In fact, $mn$-self-defense is always preserved during iterations of $D_n^m$. This is a consequence of the standard result in partial order theory which is the generalized version of Observation \ref{observation:D_ordinal} and asserts that, for any ordinal $\xi_1$ and $\xi_2$ with $\xi_1 \leq \xi_2$, $f^{\xi_1}(x) \leq f^{\xi_2}(x)$ for any monotone function $f$ on a poset and post-fixpoint $x$ of $f$.
\par 
Next, we discuss the assumption involved in Theorems \ref{theorem:close under meet and universal definability}, \ref{theorem:close under meet_fundamental semantics}, \ref{theorem:existence_stg_ss_rra_rrs} and \ref{theorem:FoL and universal definability}: AAFs are assumed to be finitary. This assumption is often adopted in exploring infinite AAFs, e.g., \cite{Dung95Acceptability,Grossi19Graded,Weydert11SsForInf,Baumann15Infinite,Baroni13Infinite}, which brings some nice properties, for instance, the function $D_n^m$ is continuous for finitary AAFs \cite{Grossi19Graded}, both the functions $N_\ell$ and $D_n^m$ are distributive over the operator $\bigcap_D$ (Lemma \ref{lemma:RMMU_Distributive}), and a variety of fundamental semantics are closed under reduced meets modulo any ultrafilter (Theorem \ref{theorem:close under meet_fundamental semantics}). It should be pointed out that some results depend on this assumption essentially, whereas others assume it only for simplification. For instance, as mentioned by Grossi and Modgil, without this assumption, constructions of some special extensions based on the iteration of $D_n^m$ can also go through by admitting transfinite iterations \cite{Grossi19Graded}. However, Theorems \ref{theorem:close under meet and universal definability}, \ref{theorem:existence_stg_ss_rra_rrs} and \ref{theorem:FoL and universal definability} in this paper depend on this assumption essentially because, in the situation $\ell = m = n = \eta =  1$, Baumann has provided an example to reveal that neither $\varepsilon_{stg}^{\ell \eta}$ nor $\varepsilon_{ss}^{\ell mn \eta}$ is universally defined for the class of all infinite AAFs \cite{Baumann15Infinite}. For Theorem \ref{theorem:close under meet_fundamental semantics}, the semantics $\varepsilon_{cf}^{\ell}$ and $\varepsilon_{gr}^{\ell mn}$ don't depend on this assumption, and an example has been provided (see Example \ref{Ex:counterexample for the derived semantics}) to illustrate that $\varepsilon_{stb}^{\ell mn}(F)$ is not always closed under the operator $\bigcap_D$ if $F$ isn't finitary, however, for $\textit{Def}^{mn}$, $\varepsilon_{ad}^{\ell mn}$ and $\varepsilon_{co}^{\ell mn}$, we do not know whether this assumption is necessary. For the universal definability of the semantics $\varepsilon_{pr}^{\ell mn}$ in the situation that $\ell \geq m$ and $n \geq m$, there is an analogous open problem, see Remark \ref{remark:pr}.
\par In the following, the special properties caused by (local) well-foundedness are summarized. Logically, they all are consequences of Lemma \ref{lemma:transitive closure}.
\par
First, as mentioned above, the assumptions ``$\ell \geq m$ and $n \geq m$'' and ``$n \geq \ell \geq m$'' are used in different situations to preserve $\ell$-conflict-freeness during iterations of $D_n^m$. However, for any AAF $F$ with a well-founded attack relation, such difference disappears.
In this situation, Corollary \ref{corollary:well-founded_co_empty}($c$) implies that the condition $\ell \geq m$ and $n \geq m$ is sufficient for preserving $\ell$-conflict-freeness during iterations of $D_n^m$ starting at $\ell mn$-admissible sets. In fact, for any well-founded AAF $F$, $\ell$-conflict-freeness is always preserved during iterations of $D_n^m$ starting at $\ell mn$-admissible sets for any $\ell$, $m$ and $n$ such that $\varepsilon_{co}^{\ell mn}(F) \neq \emptyset$ (see Corollary \ref{corollary:well-founded_AAF_stb}($c$)). By the way, by Corollary \ref{corollary:well-founded_co_empty}($b$),  both (I) and (II) in Section \ref{Sec:Fundamental Lemma and its application} hold for any finitary AAF $F$ with a well-founded attack relation, of course, in this situation, the assumption $\A \subseteq \Sigma_E$ in (II) is redundant. 
\begin{figure}[t]
	\centering
	\subfigtopskip=2pt
	\subfigbottomskip=2pt
	\subfigcapskip=-5pt
	{
		\includegraphics[scale=1]{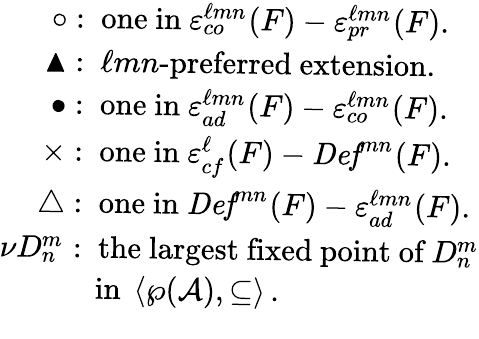}
	}
	\quad
	\subfigure[$n \geq \ell \geq m$]{
		\includegraphics[scale=0.8]{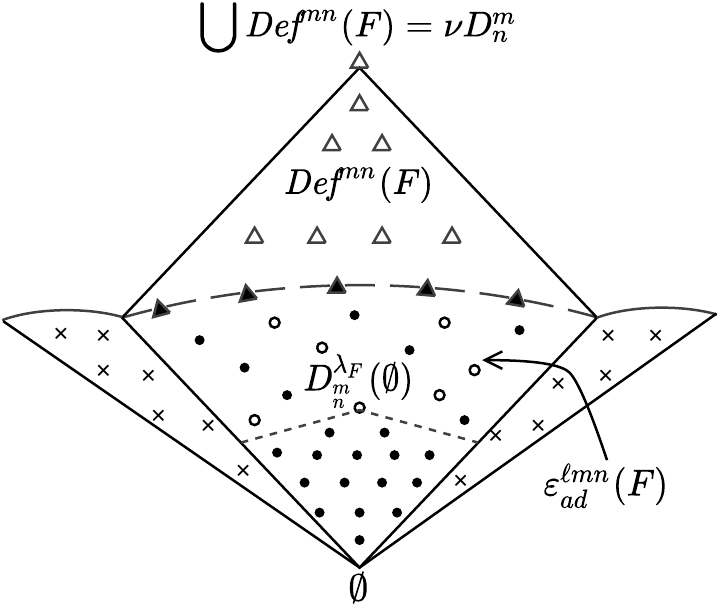}
	}
	\\
	\subfigure[$n \geq \ell \geq m$ and $E \nsubseteq D_{\mathop{}_{n}^{m}}^{\lambda_F}(\emptyset)$]{
		\includegraphics[scale=0.58]{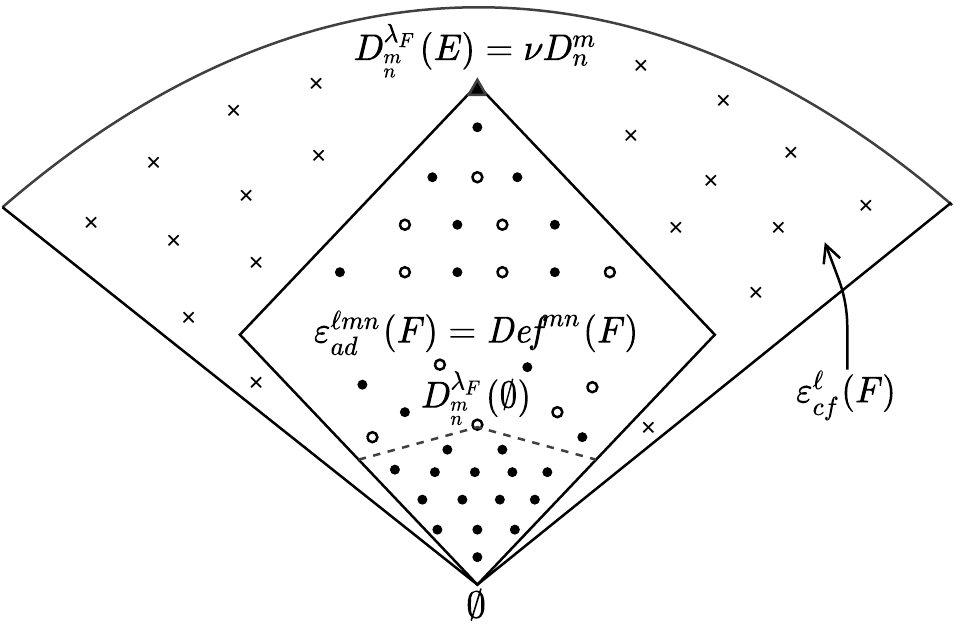}
	}
	\quad
	\subfigure[$\ell \geq m$, $n \geq m$ and $E \subseteq D_{\mathop{}_{n}^{m}}^{\lambda_F}(\emptyset)$]{
		\includegraphics[scale=0.58]{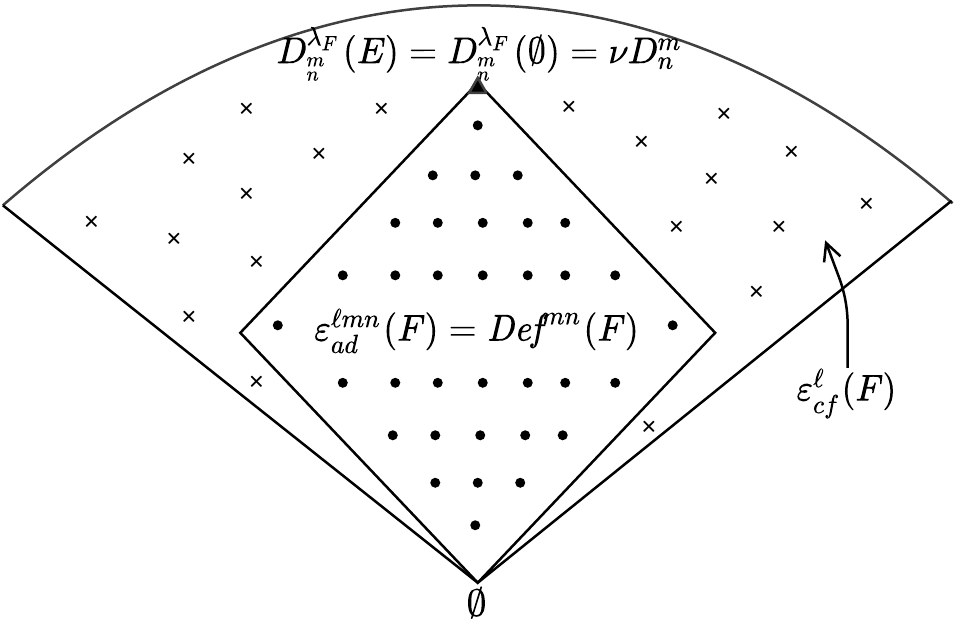}
	}
	\caption{Distributions of graded extensions in different situations.}
	\centering
	\label{figure:distributions}
\end{figure}
\par
Second, the distributions of $\ell$-conflict-free sets, $mn$-defended sets, $\ell mn$-admissible sets, $\ell mn$-complete extensions and $\ell mn$-preferred extensions are also related with well-foundedness, which are displayed graphically in Figure \ref{figure:distributions}. Figure \ref{figure:distributions}(a) displays their distributions in the general case, while Figure \ref{figure:distributions}(b) and (c) displays their distributions in the situation that $\rightarrow^+$ is well-founded on $\A - D_{\substack{m \\ n}}^{\lambda_F}(E)$ for some $E \in \varepsilon_{ad}^{\ell mn}(F)$. Corresponding assumptions on the admissible set $E$ and the parameters $\ell$, $m$ and $n$ are indicated in the identifiers. By the way, the equations $\varepsilon_{ad}^{\ell mn}(F) = \textit{Def}^{mn}(F)$ and $\bigcup \textit{Def}^{mn}(F) = \nu D_n^m$ in Figure \ref{figure:distributions} come from Corollaries \ref{corollary:equ_ad_def} and \ref{corollary:structure_cf_ad}($a$), respectively, and the distributions of $\bigtriangleup$ and $\times$ in Figure \ref{figure:distributions}(a) depend on the property of non-interpolant (see Corollary \ref{corollary:non-interpolant}).
\par 
Third, for any AAF $F$ such that $\rightarrow^+$is well-founded on $\A - D_{\substack{m \\ n}}^{\lambda_F}(\emptyset)$, since $D_n^m$ has the unique fixed point $D_{\substack{m \\ n}}^{\lambda_F}(\emptyset)$ (see, Corollary \ref{corollary:well-founded_co_empty}($a$)), its spectrum of extension-based semantics may collapse, that is, $\varepsilon_{co}^{\ell mn}(F)$, $\varepsilon_{gr}^{\ell mn}(F)$, $\varepsilon_{pr}^{\ell mn}(F)$, $\varepsilon_{stb}^{\ell mn}(F)$, $\varepsilon_{rra}^{\ell mn \eta}(F)$, $\varepsilon_{ss}^{\ell mn \eta}(F)$, $\varepsilon_{rrs}^{\ell mn \eta}(F)$, $\varepsilon_{id}^{\ell mn}(F)$ and $\varepsilon_{eg}^{\ell mn \eta}(F)$ are equivalent because that they all have the unique element $D_{\substack{m \\ n}}^{\lambda_F}(\emptyset)$ whenever the parameters $\ell$, $m$, $n$ and $\eta$ satisfy certain constraints, see Corollaries \ref{corollary:well-founded_co_empty}, \ref{corollary:well-founded_AAF_stb}, \ref{corollary:well-founded_ss_rrs_empty} and \ref{corollary:well-founded_id_empty}. 
\par 
Anyway, in addition to the assumptions on AAFs (e.g., finitary, co-finitary and well-founded), a number of results obtained in this paper depend on the assumptions about the parameters $\ell$, $m$, $n$ and $\eta$. It is a very marked difference compared to ones for standard semantics whose parameters are fixed as $\ell = m = n = \eta = 1$. Clearly, all results obtained in this paper also hold in the standard situation through instantiating all parameters as the constant 1.
\par 
We intend to end the paper by discussing future work simply. This paper has explored the structural properties of various extension-based semantics in detail. As summarized in Table \ref{table:conclusion3}, they form different ordered structures w.r.t. the inclusion relation such as complete lattice, (algebraic) complete semilattice, cpo and dcpo, etc. It is well known that there exist natural topological structures on these ordered structures, e.g., the Alexandor topology and the Scott topology, etc. One of our future work will study (infinite) AAFs in the view of topology, which may bring new viewpoints on semantics of AAFs, for instance, the topological relationships among various semantics, limits of sequences of extensions, evaluating the acceptability of arguments in a topological way and capturing extensions of infinite AAFs as limits of ones associated with finite AAFs, etc. In addition, based on the connection between reduced meet modulo an ultrafilter and ultraproducts in model theory (see Lemma \ref{lemma:RM_ultraproduct models}), we intend to take advantage of model theory to consider the issue that how to characterize equivalent AAFs w.r.t. a given semantics, and explore safe operators (see Definition \ref{definition:safe-operator}) on attack relations, which may be applied to simplify AAFs so that reducing the computational complexity of some computation issues. As applications of the operator reduced meet modulo an ultrafilter, some preliminary results on this issue have been given in this paper (see Theorems \ref{theorem:anti-epsilon set} and \ref{theorem:safe operators} and Corollary \ref{corollary:anti-epsilon set}). Finally, this paper provides only representation theorems for $\varepsilon_{cf}^\ell$ (see Theorems \ref{theorem:Representation theorem I of l-cf}, \ref{theorem:Representation theorem II of l-cf} and \ref{theorem:Representation theorem III of l-cf}), for other graded semantics, this issue needs to be explored further.
\newpage
\renewcommand{\arraystretch}{1.6}
\begin{center}
	\fontsize{8.0}{8.1}\selectfont 
	\setlength{\tabcolsep}{1pt}{
		\begin{longtable}{c | c | c | c | c}
			\caption{Definability for graded semantics of abstract argumentation. The symbol $\exists$ (or, $\exists!$) indicate that the semantics are universally (uniquely, resp.) defined.}
			\label{table:conclusion2} \\
			\hline
				Semantics & AAFs & Definability & Conditions & Sources \\
			\hline
			\endfirsthead
			\caption{Continue} \\
				\hline
					Semantics & AAFs & Definability & Conditions & Sources \\
				\hline                                                          
			\endhead
			\hline
				$\varepsilon_{cf}^\ell(F)$ &arbitrary &$\exists$  & & Trivial\\
			\hline
				$\varepsilon_{na}^\ell(F)$ &arbitrary &$\exists$  & & Corollary \ref{corollary:Lindenbaum}\\
			\hline
				$\varepsilon_{ad}^{\ell mn}(F)$  &arbitrary &$\exists$  &  & Trivial\\
			\hline
				\multirow{2}{*}{$\varepsilon_{co}^{\ell mn}(F)$}       
					& arbitrary &$\exists$  &$\ell \geq m$ and $n \geq m$  & Corollary \ref{corollary:existence_gr_co}\\
				\cline{2-5}
					& well-founded &$\exists!$ &$\ell \geq m$ and $n \geq m$  & Corollary \ref{corollary:well-founded_AAF_stb}       \\	
			\hline
				\multirow{2}{*}{$\varepsilon_{gr}^{\ell mn}(F)$}   
					& arbitrary &$\exists!$  &$\ell \geq m$ and $n \geq m$  & Corollary \ref{corollary:existence_gr_co}\\
				\cline{2-5}
					& arbitrary &$\exists!$  &$\varepsilon_{co}^{\ell mn}(F) \neq \emptyset$  & Observation \ref{observation:properties_gr}\\
			\hline
				\multirow{3}{*}{$\varepsilon_{pr}^{\ell mn}(F)$}  
					& finitary &$\exists$ &$\ell \geq m$ and $n \geq m$  & Corollary \ref{corollary:existence_pr}\\
				\cline{2-5}
					& arbitrary &$\exists$ &$n \geq \ell \geq m$ &Corollary \ref{corollary:existence_pr}       \\
				\cline{2-5}
					& well-founded &$\exists!$ &$\ell \geq m$ and $n \geq m$ & Corollary \ref{corollary:well-founded_AAF_stb} \\	
			\hline
				$\varepsilon_{stb}^{\ell mn}(F)$ & well-founded  &$\exists!$  &$\ell = m = n$  & Corollary \ref{corollary:well-founded_AAF_stb}\\
			\hline
				\multirow{2}{*}{$\varepsilon_{rr \sigma}^\eta(F)$}  
					& finitary  & $\exists$ &$|\varepsilon_{\sigma}(F)| \geq 1$ and $\varepsilon_{\sigma}(F)$ is closed under $\bigcap_D$  & Theorem \ref{theorem:close under meet and universal definability}\\
				\cline{2-5}
					& finitary&$\exists$ &$|\varepsilon_{\sigma}(F)| \geq 1$ and $\varepsilon_{\sigma}$ is $\textit{FS}(\Omega)$-definable & Theorem \ref{theorem:FoL and universal definability}\\
			\hline
				$\varepsilon_{stg}^{\ell \eta}(F)$  & finitary  &$\exists$  & & Theorem \ref{theorem:existence_stg_ss_rra_rrs}\\
			\hline
				\multirow{2}{*}{$\varepsilon_{rra}^{\ell mn \eta}(F)$} 
					& finitary  &$\exists$  &  & Theorem \ref{theorem:existence_stg_ss_rra_rrs}\\
				\cline{2-5}
					& well-founded  &$\exists!$  & $\eta \geq \ell \geq m$ and $n \geq m$  & Corollary \ref{corollary:well-founded_ss_rrs_empty}\\
			\hline
				\multirow{3}{*}{$\varepsilon_{ss}^{\ell mn \eta}(F)$} 
					& finitary& $\exists$ &$\varepsilon_{co}^{\ell mn}(F) \neq \emptyset$& Theorem \ref{theorem:existence_stg_ss_rra_rrs}\\
				\cline{2-5}
					& finitary  &  $\exists$  &$\ell \geq m$ and $n \geq m$  & Corollary \ref{corollary:existence_ss}\\
				\cline{2-5}
					& well-founded  &$\exists!$   & $\ell \geq m$ and $n \geq m$  & Corollary \ref{corollary:well-founded_ss_rrs_empty}\\
			\hline
				\multirow{2}{*}{$\varepsilon_{rrs}^{\ell mn \eta}(F)$}  
					& finitary  &$\exists$  &$\varepsilon_{stb}^{\ell mn}(F) \neq \emptyset$  &Theorem \ref{theorem:existence_stg_ss_rra_rrs}\\
				\cline{2-5}
					& well-founded  &$\exists!$  & $\ell = m = n$ & Corollary \ref{corollary:well-founded_ss_rrs_nonempty}\\
			\hline
				\multirow{2}{*}{$\varepsilon _{rr({\sigma _l},{\sigma },{\sigma _r})}^\eta(F)$} & \multirow{2}{*}{finitary}  &\multirow{2}{*}{$\exists$}  &$|\varepsilon_{\sigma_l, \sigma, \sigma_r}(F)| \geq 1$ and  & \multirow{2}{*}{Corollary \ref{corollary:RM of interval semantics}}\\
				&&&$\varepsilon_{\sigma}(F)$ is closed under $\bigcap_D$&\\
			\hline
				$\varepsilon_{cf, \sigma_r}^{\ell}(F)$  &arbitrary &$\exists$  &  & Trivial\\
			\hline
				$\varepsilon_{ad, \sigma_r}^{\ell mn}(F)$  &arbitrary &$\exists$  &  & Trivial\\
			\hline
				\multirow{2}{*}{$\varepsilon_{\textit{max}, (\sigma_l, \sigma, \sigma_r)}(F)$} & \multirow{2}{*}{arbitrary} &\multirow{2}{*}{$\exists$}  &$|\varepsilon_{\sigma_l, \sigma, \sigma_r}(F)| \geq 1$ and  & \multirow{2}{*}{Corollary \ref{corollary:RM of interval semantics}}\\
				&&&$\varepsilon_{\sigma}(F)$ is closed under $\bigcap_D$&\\
			\hline
				\multirow{2}{*}{$\varepsilon_{\textit{max}, (cf, \sigma_r)}^{\ell}(F)$} 
					& arbitrary &$\exists$  &   & Corollary \ref{corollary:RM of interval semantics}\\
				\cline{2-5}
					& arbitrary  &$\exists$!  & $\emptyset \neq \varepsilon_{\sigma_r}(F) \subseteq \varepsilon_{cf}^{\ell}(F)$   & Proposition \ref{proposition:structure_para}\\
			\hline
				\multirow{2}{*}{$\varepsilon_{\textit{max}, (ad, \sigma_r)}^{\ell mn}(F)$} 
					& finitary &$\exists$  &   & Corollary \ref{corollary:RM of interval semantics}\\
				\cline{2-5}
					& arbitrary  &$\exists$!  & $\emptyset \neq \varepsilon_{\sigma_r}(F) \subseteq \varepsilon_{cf}^{\ell}(F)$   & Proposition \ref{proposition:structure_para}\\
			\hline
				$\varepsilon_{\textit{max}, (co, \sigma_r)}^{\ell mn}(F)$ & finitary &$\exists$  & $\varepsilon_{co, \sigma_r}^{\ell mn}(F) \neq \emptyset$  & Corollary \ref{corollary:RM of interval semantics}\\
			\hline
				\multirow{3}{*}{$\varepsilon_{id}^{\ell mn}(F)$} 
					& finitary  &$\exists!$ &$\ell \geq m$ and $n \geq m$  & Corollary \ref{corollary:structure_id_eg}\\
				\cline{2-5}
					& arbitrary &$\exists!$ &$n \geq \ell \geq m$ & Corollary \ref{corollary:structure_id_eg}\\
				\cline{2-5}
					& well-founded & $\exists!$ & $\ell \geq m$ and $n \geq m$& Corollary \ref{corollary:well-founded_id_empty}\\
			\hline
				\multirow{2}{*}{$\varepsilon_{eg}^{\ell mn \eta}(F)$}   
					&finitary  &$\exists!$  &$\ell \geq m$ and $n \geq m$  &Corollary \ref{corollary:structure_id_eg}\\
				\cline{2-5}
					& well-founded & $\exists!$ & $\ell \geq m$ and $n \geq m$& Corollary \ref{corollary:well-founded_id_empty}\\
			\hline
		\end{longtable}%
	}
\end{center}%
\renewcommand{\arraystretch}{1.6}
\begin{center}
	\fontsize{7.8}{7.8}\selectfont 
	\setlength{\tabcolsep}{1mm}{
		\begin{longtable}{c | c | c | c}
			\caption{Structural properties (w.r.t $\subseteq$).}
			\label{table:conclusion3} \\
			\hline
				Semantics & Structure properties& Conditions & Sources \\
			\hline
			\endfirsthead
			\caption{Continue} \\
				\hline
				Semantics & Structure properties & Conditions & Sources \\
				\hline                                                          
			\endhead
				\multirow{3}{*}{$\textit{Def}^{mn}(F)$}  
					&complete lattice  &  &Corollary \ref{corollary:structure_cf_ad} \\
				\cline{2-4}
					& closed under union &  & Lemma \ref{lemma:properties_cf_def} \\
				\cline{2-4}
					&closed under $\bigcap_D$ &$F$ is finitary&Theorem \ref{theorem:close under meet_fundamental semantics} \\
			\hline
				\multirow{4}{*}{$\varepsilon_{cf}^\ell(F)$} 
					& algebraic complete semilattice  &  & Corollary \ref{corollary:structure_cf_ad}\\
				\cline{2-4}
					& closed under directed union &     &Lemma \ref{lemma:properties_cf_def}\\
				\cline{2-4}
					& down-closed &     & Lemma \ref{lemma:properties_cf_def}\\
				\cline{2-4}
					& closed under $\bigcap_D$  &  &Theorem \ref{theorem:close under meet_fundamental semantics}\\
			\hline
				$\varepsilon_{na}^\ell(F)$ &closed under $\bigcap_D$ &$F$ is both finitary and co-finitary& Proposition \ref{proposition:RMMU_na}\\
			\hline
				\multirow{7}{*}{$\varepsilon_{ad}^{\ell mn}(F)$}  
					&complete semilattice & &Corollary \ref{corollary:structure_cf_ad}       \\
					\cline{2-4}
					&\multirow{5}{*}{complete lattice} 
						&$\ell \geq m$, $n \geq m$ and $F$ is well-founded & Corollary \ref{corollary:well-founded_co_empty}\\
					\cline{3-4}
						&&$n \geq \ell \geq m$ and $\exists E \in \varepsilon_{ad}^{\ell mn}(F)$ s.t.          &\multirow{2}{*}{Corollary \ref{corollary:well-founded_co_nonempty}} \\
						&&$\rightarrow^+$ is well-founded on $\A - D_{\substack{m \\ n}}^{\lambda_F}(E)$&\\
					\cline{3-4}
						&&$\ell \geq m$, $n \geq m$ and &\multirow{2}{*}{Corollary \ref{corollary:well-founded_co_empty}}\\
						&&$\rightarrow^+$ is well-founded on $\A - D_{\substack{m \\ n}}^{\lambda_F}(\emptyset)$&\\
				\cline{2-4}
					&closed under $\bigcap_D$   &$F$ is finitary   &Theorem \ref{theorem:close under meet_fundamental semantics}\\
			\hline
				\multirow{6}{*}{$\varepsilon_{co}^{\ell mn}(F)$} 
					&\multirow{2}{*}{cpo} 
						&$\ell \geq m$, $n \geq m$ and $F$ is finitary  & Corollary \ref{corollary:structure_co}\\
					\cline{3-4}
						&  &$\varepsilon_{co}^{\ell mn}(F) \neq \emptyset$ and $F$ is finitary  & Corollary \ref{corollary:structure_co}\\
				\cline{2-4}
					&complete semilattice   &$n \geq \ell \geq m$   &Proposition \ref{proposition:structure_co} \\
				\cline{2-4}
					& \multirow{2}{*}{complete lattice}  &$n \geq \ell \geq m$ and $\exists E \in \varepsilon_{ad}^{\ell mn}(F)$ s.t.  &\multirow{2}{*}{Corollary \ref{corollary:well-founded_co_nonempty}} \\
					& &$\rightarrow^+$ is well-founded on $\A - D_{\substack{m \\ n}}^{\lambda_F}(E)$&\\
				\cline{2-4}
					&closed under $\bigcap_D$  &$F$ is finitary  &Theorem \ref{theorem:close under meet_fundamental semantics}\\
			\hline
				\multirow{2}{*}{$\varepsilon(F)$}  &closed under $\bigcap_D$& $|\varepsilon(F)| < \omega$ & Remark \ref{remark:RMMU_finite AAFs}\\
				\cline{2-4}	
				& dcpo &$\varepsilon(F) \neq \emptyset$ and $\varepsilon(F)$  is closed under $\bigcap_D$ & Theorem \ref{theorem:dcpo}\\			
			\hline
				$\varepsilon_{gr}^{\ell mn}(F)$ & closed under $\bigcap_D$  &  &Theorem \ref{theorem:close under meet_fundamental semantics}\\
			\hline
				$\varepsilon_{stb}^{\ell mn}(F)$ & closed under $\bigcap_D$ &$F$ is finitary &Theorem \ref{theorem:close under meet_fundamental semantics}\\
			\hline
				$\varepsilon_{stg}^{\ell \eta}(F)$  & dcpo &$F$ is finitary & Corollary \ref{corollary:structure_stg_ss}\\
			\hline
				$\varepsilon_{rra}^{\ell mn \eta}(F)$ & dcpo  &$F$ is finitary & Corollary \ref{corollary:structure_stg_ss}\\ 					     
			\hline
				$\varepsilon_{ss}^{\ell mn \eta}(F)$ & dcpo  &$\ell \geq m$, $n \geq m$ and $F$ is finitary  & Corollary \ref{corollary:structure_stg_ss}\\		     				
			\hline
				\multirow{6}{*}{$\varepsilon_{\sigma, \sigma_r}(F)$}  
					& \multirow{2}{*}{(algebraic) cpo (dcpo)}&$\varepsilon_{\sigma, \sigma_r}(F) \neq \emptyset$, $\varepsilon_{\sigma}(F)$ is closed under $\bigcap_D$  &\multirow{2}{*}{Theorem \ref{theorem:structure_para}}\\ 
					&& and $\varepsilon_{\sigma}(F)$ is a (algebraic) cpo (dcpo)&\\
				\cline{2-4}
					& \multirow{3}{*}{(algebraic) complete semilattice}&$\varepsilon_{\sigma,\sigma_r}(F) \neq \emptyset$,  &\multirow{3}{*}{Theorem \ref{theorem:structure_para}}\\ 
					&& $\varepsilon_{\sigma}(F)$ is closed under $\bigcap_D$ and &\\ 
					&& $\varepsilon_{\sigma}(F)$ is a (algebraic) complete semilattice&\\
				\cline{2-4}	
					&closed under $\bigcap_D$ &$\varepsilon_{\sigma}(F)$ is closed under $\bigcap_D$  &Theorem \ref{theorem:RMMU_interval}\\					
			\hline
				\multirow{3}{*}{$\varepsilon_{cf,\sigma_r}^{\ell}(F)$}  
					&algebraic complete lattice  &$\emptyset \neq \varepsilon_{\sigma_r}(F) \subseteq \varepsilon_{cf}^{\ell}(F)$  &Proposition \ref{proposition:structure_para}\\  
				\cline{2-4}
					&sub-algebraic complete  &  &\multirow{2}{*}{Corollary \ref{corollary:structure_para}}\\ 
					&semilattice of $\varepsilon_{cf}^\ell(F)$&&\\
			\hline
				\multirow{3}{*}{$\varepsilon_{ad,\sigma_r}^{\ell mn}(F)$}  
					&complete lattice  &$\emptyset \neq \varepsilon_{\sigma_r}(F) \subseteq \varepsilon_{cf}^{\ell}(F)$  &Proposition \ref{proposition:structure_para}\\
				\cline{2-4}
					&sub-complete semilattice  &  \multirow{2}{*}{$F$ is finitary} &\multirow{2}{*}{Corollary \ref{corollary:structure_para}}\\
					&of $\varepsilon_{ad}^{\ell mn}(F)$&&\\
			\hline
				\multirow{3}{*}{$\varepsilon_{co,\sigma_r}^{\ell mn}(F)$}  
					&sub-cpo of $\varepsilon_{co}^{\ell mn}(F)$ & $F$ is finitary and $\varepsilon_{co, \sigma_r}^{\ell mn}(F) \neq \emptyset$ &Corollary \ref{corollary:structure_para}\\
				\cline{2-4}
					&sub-complete semilattice  & $F$ is finitary, $n \geq \ell \geq m$ and  &\multirow{2}{*}{Corollary \ref{corollary:structure_para}}\\
					&of $\varepsilon_{co}^{\ell mn}(F)$& $\varepsilon_{co, \sigma_r}^{\ell mn}(F) \neq \emptyset$&\\
			\hline
				\multirow{3}{*}{$\varepsilon_{\sigma_l, \sigma, \sigma_r}(F)$}  
					& \multirow{2}{*}{sub-dcpo of $\varepsilon_{\sigma}(F)$}&$\varepsilon_{\sigma_l, \sigma, \sigma_r}(F) \neq \emptyset$ and   &\multirow{2}{*}{Theorem \ref{theorem:structure of interval semantics}}\\ 
					&& $\varepsilon_{\sigma}(F)$ is closed under $\bigcap_D$&\\
				\cline{2-4}	
					&closed under $\bigcap_D$ &$\varepsilon_{\sigma}(F)$ is closed under $\bigcap_D$  &Theorem \ref{theorem:RMMU_interval}\\
			\hline		
		\end{longtable}%
	}
\end{center}
\renewcommand{\arraystretch}{1.6}
\begin{center}
	\vspace{-1cm}
	\fontsize{6.7}{6.7}\selectfont 
	\setlength{\tabcolsep}{1mm}{
		\begin{longtable}{c | c | c | c | c}
			\caption{Constructions for a series of special extensions, where $S^\downarrow \triangleq \set{X \in \varepsilon_{ad}^{\ell mn}(F) \mid X \subseteq S}$, $E \in \varepsilon_{ad}^{\ell mn}(F)$, and $\lambda_F \triangleq \omega$ whenever $F$ is finitary.}
			\label{table:conclusion4} \\
			\hline
			Constructions & Conditions & AAFs & Sources &  Properties\\
			\hline
			\endfirsthead
			\caption{Continue} \\
			\hline
			Constructions & Conditions & AAFs & Sources &  Properties\\
			\hline                                                          
			\endhead
				&  & \multirow{5}{*}{arbitrary}& \multirow{2}{*}{Corollary \ref{corollary:well-founded_co_empty}}& the unique $\ell mn$-co(pr) extension\\
				& $\ell \geq m$, $n \geq m$ and&&& and the largest $\ell mn$-ad set\\
				\cline{4-5}
				& $\rightarrow^+$ is well-founded && Corollary \ref{corollary:well-founded_ss_rrs_empty}& the unique $\ell mn$-$\eta$-ss extension\\
				\cline{4-5}
				& on $\A - D_{\substack{m \\ n}}^{\lambda_F}(\emptyset)$& & \multirow{2}{*}{Corollary \ref{corollary:well-founded_id_empty}}& the unique $\ell mn$-id extension and\\
				& & & & the unique $\ell mn$-$\eta$-eg extension \\
				\cline{2-5}
				& \multirow{6}{*}{$\ell \geq m$ and $n \geq m$}& \multirow{5}{*}{well-founded} & \multirow{2}{*}{Corollary \ref{corollary:well-founded_AAF_stb}}& the unique $\ell mn$-co(pr) extension\\
				&&&& and the largest $\ell mn$-ad set\\
				\cline{4-5}
				&&& Corollary \ref{corollary:well-founded_ss_rrs_empty}& the unique $\ell mn$-$\eta$-ss extension\\
				\cline{4-5}
				$D_{\substack{m \\ n}}^{\lambda_F}(E)$ &&& \multirow{2}{*}{Corollary \ref{corollary:well-founded_id_empty}}& the unique $\ell mn$-id extension and\\
				with & & & & the unique $\ell mn$-$\eta$-eg extension \\
				\cline{3-5}
				$E \subseteq D_{\substack{m \\ n}}^{\lambda_F}(\emptyset)$& & arbitrary &  Corollary \ref{corollary:existence_gr_co} & \\
				\cline{2-4}
				&$\varepsilon_{co}^{\ell mn}(F) \neq \emptyset$ & arbitrary& \multirow{2}{*}{Observation \ref{observation:properties_gr}} & the least $\ell mn$-co extension\\
				\cline{2-3}
				&$D_{\substack{m \\ n}}^{\lambda_F}(\emptyset) \in \varepsilon_{cf}^{\ell}(F)$ & arbitrary&  & (containing $E$)\\
				\cline{2-5}
				& \multirow{2}{*}{$\ell = m = n$} & \multirow{2}{*}{well-founded} &  Corollary \ref{corollary:well-founded_AAF_stb} & the unique $\ell mn$-stb extension\\	
				\cline{4-5}
				&&& Corollary \ref{corollary:well-founded_ss_rrs_nonempty}& the unique $\ell mn$-$\eta$-rrs extension\\
				\cline{2-5}
				& $\eta \geq \ell \geq m$, $n \geq m$ &\multirow{3}{*}{arbitrary} & \multirow{4}{*}{Corollary \ref{corollary:well-founded_ss_rrs_empty}} & \multirow{4}{*}{the unique $\ell mn$-$\eta$-rra set}\\
				& and $\rightarrow^+$ is well-founded & &  &\\
				&on $\A - D_{\substack{m \\ n}}^{\lambda_F}(\emptyset)$&&&\\
				\cline{2-3}
				& $\eta \geq \ell \geq m$ and $n \geq m$& well-founded & &\\
			\hline
				\multirow{14}{*}{$D_{\substack{m \\ n}}^{\lambda_F}(E)$}
				& $n \geq \ell \geq m$& arbitrary& Corollary \ref{corollary:properties_defense} & the least $\ell mn$-co ext. containing $E$\\
				\cline{2-5}
				& $\ell \geq m$ and &\multirow{2}{*}{arbitrary} & \multirow{2}{*}{Proposition \ref{proposition:structure_stb}} & \multirow{2}{*}{the least $\ell mn$-stb ext. containing $E$}\\
				& $D_{\substack{m \\ n}}^{\lambda_F}(E) = N_n(D_{\substack{m \\ n}}^{\lambda_F}(E))$&  & & \\
				\cline{2-5}
				& $\eta \geq \ell$, $n \geq \ell \geq m$ and &\multirow{3}{*}{arbitrary} & \multirow{2}{*}{Corollary \ref{corollary:well-founded_ss_rrs_nonempty}} & the unique $\ell mn$-$\eta$-rra extension and\\	
				& $\rightarrow^+$ is well-founded &  & & the unique $\ell mn$-$\eta$-ss extension\\
				\cline{4-5}
				&on $\A - D_{\substack{m \\ n}}^{\lambda_F}(E)$&&Corollary \ref{corollary:well-founded_id_nonempty}&the unique $\ell mn$-$\eta$-eg extension\\
				\cline{2-5}
				&  &\multirow{5}{*}{arbitrary} & \multirow{2}{*}{Corollary \ref{corollary:well-founded_co_nonempty}} & the largest $\ell mn$-co ext. and the\\
				& $n \geq \ell \geq m$ and&&& unique $\ell mn$-co(pr) ext. containing $E$\\
				\cline{4-5}
				& $\rightarrow^+$ is well-founded & & \multirow{2}{*}{Corollary \ref{corollary:well-founded_ss_rrs_nonempty}} & the unique $\ell mn$-$\eta$-ss ext. \\
				&on $\A - D_{\substack{m \\ n}}^{\lambda_F}(E)$&&&containing $E$\\
				\cline{4-5}
				& & & Corollary \ref{corollary:well-founded_id_nonempty} & the unique $\ell mn$-id extension\\
				\cline{2-5}
				& $\ell = m = n$ and &\multirow{2}{*}{arbitrary} & Corollary \ref{corollary:well-founded_stb} & the unique $\ell mn$-stb extension \\
				\cline{4-5}
				& $\rightarrow^+$ is well-founded & & \multirow{2}{*}{Corollary \ref{corollary:well-founded_ss_rrs_nonempty}}  & \multirow{2}{*}{the unique $\ell mn$-$\eta$-rrs extension}\\
				&on $\A - D_{\substack{m \\ n}}^{\lambda_F}(E)$&&&\\
			\hline		
				\multirow{2}{*}{$\bigcap \varepsilon_{\sigma_r}(F)$} &\multirow{2}{*}{$\emptyset \neq \varepsilon_{\sigma_r}(F) \subseteq \varepsilon_{cf}^{\ell}(F)$}  &\multirow{2}{*}{arbitrary}   &\multirow{4}{*}{Proposition \ref{proposition:structure_para}}& the unique maximal \\
				&&&&$(cf, \sigma_r)$-parametrized extension\\
			\cline{1-3} \cline{5-5}
				\multirow{2}{*}{$\bigcup \varepsilon_{ad, \sigma_r}^{\ell mn}(F)$} &\multirow{2}{*}{$\emptyset \neq \varepsilon_{\sigma_r}(F) \subseteq \varepsilon_{cf}^{\ell}(F)$}  &\multirow{2}{*}{arbitrary}   & & the unique maximal \\
				&&&&$(ad, \sigma_r)$-parametrized extension\\
			\hline			
				\multirow{2}{*}{$\bigcup (\bigcap \varepsilon_{pr}^{\ell mn}(F))^\downarrow$} 
				&$\ell \geq m$ and $n \geq m$ &finitary   &\multirow{2}{*}{Corollary \ref{corollary:structure_cf_ad}}& \multirow{2}{*}{the unique $\ell mn$-id extension} \\
				\cline{2-3}
				&$n \geq \ell \geq m$&arbitrary&&\\
			\cline{1-3} \cline{5-5}	
				\multirow{2}{*}{$\bigcup (\bigcap \varepsilon_{ss}^{\ell mn \eta}(F))^\downarrow$}
				&\multirow{2}{*}{$\ell \geq m$ and $n \geq m$} &\multirow{2}{*}{finitary}  &Corollary \ref{corollary:structure_id_eg}& \multirow{2}{*}{the unique $\ell mn$-$\eta$-eg extension} \\
				&&&&\\
			\hline	
		\end{longtable}%
	}
\end{center}
\section{Acknowledgments}
The work was supported by the National Natural Science Foundation of China (Grant No. 61602249).

\section*{References}

\end{document}